\pdfoutput=1
\documentclass[letterpaper,twocolumn,10pt]{article}
\usepackage{usenix2019_v3}

\usepackage{amsmath,amssymb,amsthm}
\usepackage{booktabs,tabularx,array,colortbl}
\usepackage{enumitem}
\usepackage{alltt}
\usepackage{upquote}
\usepackage{refcount}
\usepackage{comment}
\usepackage{tikz}
\usepackage[numbers]{natbib}
\usetikzlibrary{arrows.meta,positioning,fit,calc,shapes.geometric,shapes.misc,decorations.pathreplacing,backgrounds,patterns}
\usepackage{xurl}

\newif\ifdonotdistribute \donotdistributefalse
\ifdonotdistribute
  \usepackage{draftwatermark}
  \DraftwatermarkOptions{text={Draft --- Do Not Distribute}, angle=0, fontsize=9pt,
    color=black!45, pos={0.5\paperwidth, 0.4in}, anchor=ct}
\fi

\newcommand{\AuthorBlock}{%
{\rm Kefan Liu$^{1,2}$ \quad Fengning Ou$^{3}$ \quad Yelin Luo$^{2,4}$ \quad Jingdi Lei$^{5}$}\\
$^{1}$Institute of Computing Technology, CAS \quad $^{2}$University of Chinese Academy of Sciences\\
$^{3}$Nanjing University \quad $^{4}$Institute of Automation, CAS \quad $^{5}$Nanyang Technological University\\
{\normalsize\texttt{Kefan.Liu@outlook.com \quad reverymoon@gmail.com}}\\
{\normalsize\texttt{Yelin\_Luo@outlook.com \quad jingdi001@e.ntu.edu.sg}}}

\definecolor{pcol}{HTML}{0072B2}
\definecolor{ocol}{HTML}{E69F00}
\definecolor{ostroke}{HTML}{B37200}
\colorlet{pfill}{pcol!10}
\colorlet{ofill}{ocol!16}
\colorlet{pdark}{pcol!80!black}
\colorlet{odark}{ocol!55!black}
\colorlet{inkmuted}{black!55}
\colorlet{rulegray}{black!18}

\colorlet{rowrulegray}{black!30}
\newcommand{\rowrule}{\arrayrulecolor{rowrulegray}\specialrule{0.3pt}{2pt}{2pt}\arrayrulecolor{black}}
\newcommand{\grouprule}{\midrule[\heavyrulewidth]}

\theoremstyle{definition}
\newtheorem{theorem}{Theorem}[section]
\newtheorem{proposition}[theorem]{Proposition}
\newtheorem{lemma}[theorem]{Lemma}
\newtheorem{corollary}[theorem]{Corollary}
\newtheorem{definition}[theorem]{Definition}
\newtheorem{construction}[theorem]{Construction}

\newcommand{\envendsymbol}{\ensuremath{\lozenge}}
\newcommand{\envendhook}[1]{%
  \AddToHook{env/#1/begin}{\pushQED{\qed}\renewcommand{\qedsymbol}{\envendsymbol}}%
  \AddToHook{env/#1/end}{\popQED}}
\envendhook{theorem}\envendhook{proposition}\envendhook{lemma}\envendhook{corollary}
\envendhook{definition}\envendhook{construction}\envendhook{example}

\newcommand{\namedpara}[1]{\par\medskip\noindent\textbf{#1}\ }
\newcounter{insight}
\newcommand{\insight}{\refstepcounter{insight}\namedpara{Insight \theinsight.}}

\newcommand{\sectionepigraph}[2]{\begin{flushright}\small\parbox{0.84\columnwidth}{\itshape #1\\[2pt]
  \hspace*{\fill}{\upshape --- #2}}\end{flushright}}

\makeatletter
\renewcommand\paragraph{\@startsection{paragraph}{4}{\z@}%
  {\medskipamount}{-1em}{\normalfont\normalsize\bfseries}}
\renewcommand\subparagraph{\@startsection{subparagraph}{5}{\z@}%
  {\medskipamount}{-1em}{\normalfont\normalsize\itshape}}
\makeatother

\newcommand{\mach}[1]{\mbox{#1}}
\AtBeginDocument{\pdfstringdefDisableCommands{\def\mach#1{#1}}}

\setlist{itemsep=1pt,topsep=3pt}

\newenvironment{pseudo}{\par\small\begin{alltt}}{\end{alltt}\par}
\newcommand{\tc}[3][l]{\makebox[#2\fontcharwd\font`x][#1]{#3}}

\hypersetup{linkcolor=pdark, citecolor=odark, urlcolor=pdark}

\begin{document}

\date{}

\title{\Large \bf We Query, Therefore We Compute:\\
  On Oracle Computation beyond the Machine, with an Application to Agents}

\author{\AuthorBlock}

\maketitle

\begin{abstract}
Agentic systems use large language models (LLMs) to carry out concrete tasks. Prior work often borrows abstractions such as
scheduling, caching or isolation piecemeal from operating systems, so the mechanisms it builds share little common ground, and
the shared view of the two forms of agentic system, Workflows and Agents, is limited. We construct an abstract machine that provides both.

We treat the LLM as an Oracle and extend a two-stack pushdown automaton with one instruction, which hands the Oracle a whole stack
as its query and appends the answer to that same stack. The machine thus performs two computations, the Oracle's and a
Turing-complete one that we call the Priestess. A stack that the program only appends to grows autoregressively, as an agent's
context does. Two symmetry breakings, S in storage and T in transitions, make a Priestess program the operating system of the
programs the Oracle runs, and produce the Agent and the Workflow as the two placements of a task's program.

For internally autoregressive Oracles, the two computations synchronize at the end of every answer under certain conditions, and
through that synchronization we model caching and analyse scheduling. No guarantee that holds for every
Oracle can fix which content crosses between the two computations, but such a guarantee does fix the boundary itself.

The construction V fits the machine to a von Neumann computer. To show that it is realizable, we propose \textbf{ArchNights}, an
extended \mbox{RISC-V} ISA and a Linux-style operating system implementing the machine by design. ArchNights-SE runs on gem5 as a
computer system, becomes an agentic system when it runs an LLM as the Oracle, and will be open source.

Agentic systems can then be designed as computer systems are. With a foundation built and a unified view, future work can share
invariants and bounds, each with its conditions.
\end{abstract}

{\renewcommand{\abstractname}{Extended Abstract}
\begin{abstract}
Agentic systems use large language models (LLMs) to carry out concrete tasks. Prior work often borrows abstractions such as
scheduling, caching or isolation piecemeal from operating systems. An analogy reaches only as far as what it borrows, so the
mechanisms share little common ground on which to state or compare them. Moreover, agentic systems appear in two forms, Workflows
and Agents, and the shared view that subsumes them is limited. We thus construct an abstract machine that provides this common
ground and a unified view.

We treat the LLM as an Oracle and extend a two-stack pushdown automaton with an instruction that queries the Oracle with a whole
stack and appends the answer to that same stack. The machine thus performs two computations: the Oracle's, and a Turing-complete
one that we call the Priestess. For the Oracle's computation, a stack that the program only appends to grows autoregressively, as an
agent's context does. We then colour the two stacks, type the machine with call-by-push-value, give it two modes and break its symmetry in
storage and in transitions. The storage (S) symmetry breaking makes the Oracle execute a program only in the Oracle colour,
while content stays data in the Priestess colour. Under T, the Oracle mode can switch to the
Priestess mode only at a fixed entry. Together they make the \mach{O-2-PDA-TS}, on which a Priestess program is the operating
system of the programs the Oracle runs. On this machine, when an LLM is the Oracle, a task whose program runs on the
Oracle is an Agent. When
the task's program is placed differently, we find the Agent and Workflow forms on this machine. At this point the machine is defined
and the operating system is derived, but what invariants hold between the Oracle and the Priestess remains open.

We then study Oracles that are autoregressive inside, and link their one-step transitions to the machine's history through the
encoder and decoder of the Oracle. The two computations synchronize at the end of every answer that yields, and
through that synchronization we model caching and analyse scheduling behaviour, once for both forms. We further analyse the
boundary between the Oracle and the Priestess and find that no guarantee that holds for every Oracle can fix which content
crosses, but that such a guarantee does fix the boundary itself.

To build a practical machine, we first ask which Oracle can be part of it, and then analyse how to fit the Oracle into the
Priestess. The answers to these questions are the construction V, which fits the \mach{O-2-PDA-TS} to a von Neumann computer.
To show that the machine is realizable, we propose \textbf{ArchNights}: an extended \mbox{RISC-V} ISA and a Linux-style operating
system. Accordingly, when it runs an LLM as the Oracle, ArchNights becomes an agentic system. ArchNights implements the machine by design, and
we implement ArchNights-SE, in gem5's syscall-emulation mode, as a computer system. ArchNights-SE will be open source. We further
discuss ArchNights-FS, which is designed to run on real hardware.

On this machine, agentic systems can be designed as computer systems are. With a foundation built and a unified view, future work
can share invariants and bounds, each under the conditions it states.
\end{abstract}}

\section{Introduction}
\label{sec:intro}

Agents predate language models~\cite{shoham1993agent,rao1995bdi}, but only large language models~\cite{brown2020language} gave
them a practical executor for open-ended tasks, and the abilities of these models as agents have grown quickly
since then~\cite{kwa2025measuring}. Agents are now used in real work, such as writing
software~\cite{robbes2026agentic}. How well, and at
what cost, a model carries out a task also depends on the harness that runs it, the program that gives the model its context and
tools, and acts on its answers~\cite{yang2024sweagent}. With its scaffolding, a model has come to be viewed as a
computer~\cite{millidge2023scaffolded}, and the harness has become a system in its own right.

We follow the terminology of Building Effective Agents~\cite{anthropic2024building}, which distinguishes Workflows from Agents by
how control is organized: in a \textbf{Workflow}, model and tool calls follow predefined code paths, and in an \textbf{Agent}, the
model dynamically determines its actions and tool use. Workflow engines such as n8n, Dify and Coze execute models or other computations
at nodes of an orchestration, while systems such as Claude Code, Codex, Pi Agent and DeepSeek Harness repeatedly invoke a model,
dispatch its tool calls and return their results as the task
proceeds~\cite{n8n2019workflow,dify2023agentic,coze2025studio,claudecode,codex2025lightweight,pi2025agent,deepseek2026harness}. In
practice, the two forms are increasingly composed. A Workflow node may run an Agent, while an Agent may use task modes that impose
workflow-like structure, such as drafting a plan before execution. We call a system that runs Agents an \textbf{Agent harness},
and one that runs Workflows a \textbf{Workflow engine}; Appendix~\ref{app:glossary} defines these and the other terms used in the
paper.

\subsection{Two Clouds over Agentic Systems}
\label{sec:intro-clouds}

As agentic systems multiply, their study is becoming systematic.
Where is the security boundary of a model? What is the cache of an agent? What separates a person's output from a
model's?
Such questions arise whenever a concrete problem is analysed, and prior work has answered many of them, each from its own
perspective.
Take the cache as an example.
For those who build agents, the cache is the context a program keeps and reuses; for those who serve models, it is the model's
attention state, its KV cache~\cite{kwon2023efficient,zheng2024sglang,gim2024prompt}.
Both are right and have solved real-world problems, yet few works examine their foundations or the relation between the two
views.
Work on agentic systems abounds (\S\ref{sec:related}), new solutions appear, and harnesses are revised from day to day.
Research on agentic systems is growing, yet is at present obscured by two clouds.

\paragraph{The first cloud: foundations.} On foundations, the research goes in two directions. One is formal: calculi and
languages that give agent programs a semantics and prove what they cannot do, such as $\lambda_A$~\cite{liu2026lambda},
ZipperGen~\cite{bollig2026provable}, LLMbda~\cite{garby2026llmbda} and the Calculus of Intelligence~\cite{yuan2026calculus}, as
well as rule languages enforced at run time, such as
AgentSpec~\cite{wang2026agentspec}. The other is
structural: designs and systems built by analogy with operating systems, such as AIOS~\cite{mei2025aios}, the AgentOS
papers~\cite{li2026architecting,liu2026agentos} and the AOS papers~\cite{sharma2026aagent,sharma2026bagent}. Quine takes one
abstraction further, realizing agents as POSIX processes that the host's kernel isolates~\cite{ke2026quine}; its operating system
is the host's and stands between processes.
Both directions have produced real results: a lint for
agent configurations~\cite{liu2026lambda}, a verified interpreter that is itself a harness~\cite{garby2026llmbda},
and runtimes in use.
Yet the formal works define languages, and none states what makes one part of an agentic system the operating
system of another, or what that part guarantees.
An analogy's theory reaches only as far as the analogy, and a new question needs a new analogy before it can even be asked,
so the mechanisms built on analogies share little common ground on which they can be
stated and checked.
The two views of the cache are one example: each is stated in its own side's terms, with no ground common to both.

\paragraph{The second cloud: unification.} Work on Workflows and Agents usually starts from one of the two. FMOS states the
problem and treats workflows and agentic loops as two execution forms over shared primitives, with guarantees under a
virtualization condition assumed at its interface~\cite{bhattacharya2026position}.
Harness as a Language derives both from one primitive of a language, and defines no operating system~\cite{li2026harness}.
A concrete mechanism is still built in one form and studied apart from its counterpart in the other.
The cache is one example. One line of work decides what to do with the state of an
agent's task while its tool call runs~\cite{abhyankar2024infercept,li2026continuum}; another decides which prefixes of a
workflow's steps to keep or prefetch, by how soon each will run~\cite{pan2025kvflow}.
On one machine both target the same interval, between two calls that
share a prefix, in the Agent form and in the Workflow form (Appendix~\ref{app:further-sched}). ``This duplication of work is
tiresome. Furthermore, it makes the [systems] involved seem inherently arbitrary. We would prefer to study a single, canonical
[system]''~\cite[\S1.2.1]{levy2001call}, where ``systems'' replaces Levy's ``languages'', one that subsumes both forms. \S\ref{sec:os} shows that one does.

\paragraph{Beyond the two clouds.} The third question, what separates a person's output from a model's, asks for what neither
cloud names: invariant properties or bounds that hold
no matter who or what answers the queries, a language model or a person. On the machine a person and a language model are alike as
the source of answers, provided a semantics can be defined for parsing the person's words. What separates their outputs is how each
is called (\S\ref{sec:query}). None of the three questions is asked for its own sake. Their answers let future work design agentic
systems as computer systems are designed, with each design choice argued from the theories or the experience of computer
systems. \S\ref{sec:archnights} designs ArchNights in this way.

\subsection{One Instruction, Two Symmetry Breakings}
\label{sec:intro-idea}

Our solution rests on two methods.
\textbf{Relativization} separates the language model from the rest of the
computation as a black-box Oracle and makes their interaction a machine.
\textbf{Symmetry breaking} makes one part of that machine
the operating system of the other. Two further steps extend the machine to the questions of the field and to hardware: opening the
Oracle by one step, and realizing the machine on a von Neumann computer.
Figure~\ref{fig:readings} follows one agent session through
the construction, and Figure~\ref{fig:construction} shows the whole of it.

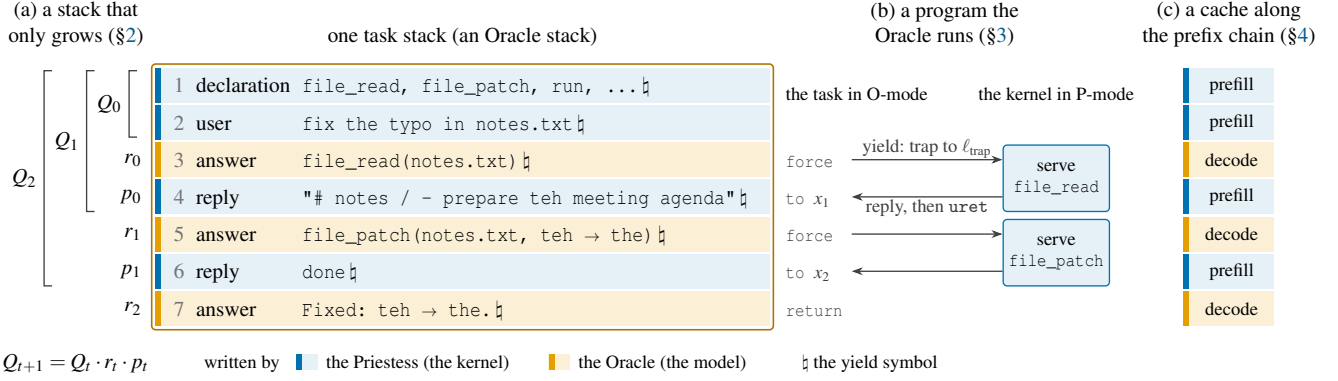
\begin{figure*}[t]
\centering
\begin{tikzpicture}[x=1pt, y=1pt, font=\footnotesize,
  >={Stealth[length=4.2pt, width=3pt]},
  title/.style={align=center, anchor=south, inner sep=0pt, font=\footnotesize},
  brk/.style={line width=0.5pt, black!75},
  flow/.style={->, line width=0.5pt, black!75},
  note/.style={font=\scriptsize, inner sep=1pt},
  cbpv/.style={font=\scriptsize, anchor=base west, inner sep=0pt, text=black!75},
  ker/.style={draw=pcol, fill=pfill, line width=0.5pt, rounded corners=1.5pt, align=center,
    font=\scriptsize, inner sep=1.5pt}]
  \def\W{231}%
  \def\rec#1#2#3#4{%
    \fill[#2fill] (0,{-14*(#1-1)}) rectangle (\W,{-14*(#1-1)-13});
    \fill[#2col] (0,{-14*(#1-1)}) rectangle (2,{-14*(#1-1)-13});
    \node[anchor=base east, inner sep=0pt, text=inkmuted] at (11,{-14*(#1-1)-9.2}) {#1};
    \node[anchor=base west, inner sep=0pt] at (15,{-14*(#1-1)-9.2}) {#3};
    \node[anchor=base west, inner sep=0pt, font=\footnotesize\ttfamily] at (55,{-14*(#1-1)-9.2})
      {#4\,\normalfont$\natural$};}
  \rec{1}{p}{declaration}{file\_read, file\_patch, run, \dots}
  \rec{2}{p}{user}{fix the typo in notes.txt}
  \rec{3}{o}{answer}{file\_read(notes.txt)}
  \rec{4}{p}{reply}{"\# notes / - prepare teh meeting agenda"}
  \rec{5}{o}{answer}{file\_patch(notes.txt, teh \(\scriptstyle\to\) the)}
  \rec{6}{p}{reply}{done}
  \rec{7}{o}{answer}{Fixed: teh \(\scriptstyle\to\) the.}
  \draw[ostroke, line width=0.6pt, rounded corners=1.5pt] (-1.5,1.5) rectangle (\W+1.5,-98+1-1.5);
  \node[title] at (\W/2,8) {one task stack (an Oracle stack)};
  \def\qbr#1#2#3{%
    \draw[brk] (#1+3,-1) -- (#1,-1) -- (#1,{-14*#2+2}) -- (#1+3,{-14*#2+2});
    \node[anchor=east, inner sep=0pt] at (#1-2.5,{(-14*#2+1)/2}) {#3};}
  \qbr{-10}{2}{$Q_0$}
  \qbr{-26}{4}{$Q_1$}
  \qbr{-42}{6}{$Q_2$}
  \foreach \i/\l in {3/r_0, 4/p_0, 5/r_1, 6/p_1, 7/r_2}
    \node[anchor=east, inner sep=0pt] at (-5,{-14*(\i-1)-6.5}) {$\l$};
  \node[anchor=base, inner sep=0pt] at (-30,-113) {$Q_{t+1} = Q_t \cdot r_t \cdot p_t$};
  \node[title] at (-30,8) {(a) a stack that\\only grows (\S\ref{sec:machine})};
  \foreach \i/\c in {3/force, 4/{to $x_1$}, 5/force, 6/{to $x_2$}, 7/return}
    \node[cbpv] at (\W+6,{-14*(\i-1)-9.2}) {\texttt{\c}};
  \node[ker, minimum width=40pt, minimum height=25pt] (k1) at (\W+108,-41.5) {serve\\\texttt{file\_read}};
  \node[ker, minimum width=40pt, minimum height=25pt] (k2) at (\W+108,-69.5) {serve\\\texttt{file\_patch}};
  \draw[flow] (\W+31,-34.5) -- node[note, above, inner sep=0.6pt] {yield: trap to $\ell_{\mathrm{trap}}$} (k1.west |- 0,-34.5);
  \draw[flow] (k1.west |- 0,-48.5) -- node[note, below, inner sep=0.6pt] {reply, then $\mathtt{uret}$} (\W+31,-48.5);
  \draw[flow] (\W+31,-62.5) -- (k2.west |- 0,-62.5);
  \draw[flow] (k2.west |- 0,-76.5) -- (\W+31,-76.5);
  \node[note, anchor=base west, inner sep=0pt] at (\W+6,-12) {the task in O-mode};
  \node[note, anchor=base, inner sep=0pt] at (\W+108,-12) {the kernel in P-mode};
  \node[title] at (\W+66,8) {(b) a program the\\Oracle runs (\S\ref{sec:os})};
  \foreach \i/\w/\c in {1/p/prefill, 2/p/prefill, 3/o/decode, 4/p/prefill, 5/o/decode, 6/p/prefill, 7/o/decode} {
    \fill[\w fill] (\W+156,{-14*(\i-1)}) rectangle (\W+190,{-14*(\i-1)-13});
    \fill[\w col] (\W+156,{-14*(\i-1)}) rectangle (\W+158,{-14*(\i-1)-13});
    \node[note] at (\W+174,{-14*(\i-1)-6.5}) {\c};
  }
  \node[title] at (\W+173,8) {(c) a cache along\\the prefix chain (\S\ref{sec:eclipse})};
  \begin{scope}[shift={(18,-113)}, every node/.style={anchor=base west, inner sep=0pt, font=\scriptsize}]
    \node at (0,0) {written by};
    \fill[pfill] (35,-0.5) rectangle (43,5.5); \fill[pcol] (35,-0.5) rectangle (37,5.5);
    \node at (46,0) {the Priestess (the kernel)};
    \fill[ofill] (130,-0.5) rectangle (138,5.5); \fill[ocol] (130,-0.5) rectangle (132,5.5);
    \node at (141,0) {the Oracle (the model)};
    \node at (225,0) {$\natural$\ the yield symbol};
  \end{scope}
\end{tikzpicture}
\caption{One history, three readings. A session in which an agent fixes a typo, as the records of one task stack. The
Priestess (the kernel) writes the blue records and the Oracle (the model) writes the amber ones. Each record ends with
the yield symbol $\natural$. All of them lie on one Oracle stack and so have colour $O$
(Definition~\ref{def:colouring}), and blue and amber show only who wrote them. (a) As a stack (\S\ref{sec:machine}):
each query is the whole history so far, and the next query extends it by the Oracle's answer $r_t$ and the kernel's
reply $p_t$, so the history grows autoregressively (Proposition~\ref{prop:ar-form}). (b) As a program (\S\ref{sec:os}):
an answer that ends with $\natural$ traps to the one service entry $\ell_{\mathrm{trap}}$, and the kernel serves the
request in P-mode and appends its reply before it re-enters the task with $\mathtt{uret}$. In call-by-push-value,
the history is $\mathtt{force}\ \mathit{file\_read}\ \mathtt{to}\ x_1.\ \mathtt{force}\ \mathit{file\_patch}\
\mathtt{to}\ x_2.\ \mathtt{return}\ \mathit{result}$. (c) As a cache (\S\ref{sec:eclipse}): a private cache that keeps
every state computes each token of the history once, by prefill where the Priestess wrote and by decode where the Oracle
wrote. This needs the hypotheses of Proposition~\ref{prop:prefix-chain}, among them that the tokenizer encodes the text between
yield symbols piece by piece, that the model is called one whole answer at a time and writes each answer back as it generated it,
and that the kernel only appends records ending with $\natural$.}
\label{fig:readings}
\end{figure*}
\begin{figure*}[t]
\centering
\def\sub{\scriptsize\color{inkmuted}}%
\begin{tikzpicture}[
  font=\footnotesize,
  m/.style={draw=black!80, fill=white, line width=0.5pt, rounded corners=1.5pt, minimum width=21mm,
    minimum height=6mm, inner sep=2pt, align=center},
  lab/.style={font=\scriptsize, inner sep=1.5pt},
  note/.style={font=\scriptsize, align=left, inner sep=0pt},
  frame/.style={draw=black!30, line width=0.5pt, rounded corners=3pt},
  row/.style={font=\scriptsize\bfseries, anchor=west, inner sep=0pt},
  col/.style={font=\scriptsize, anchor=base, inner sep=0pt},
  bx/.style={line width=0.5pt, rounded corners=1pt, align=center, font=\scriptsize, inner sep=1.5pt,
    minimum height=11mm, text width=40mm},
  bO/.style={bx, draw=ostroke, fill=ofill},
  bP/.style={bx, draw=pcol, fill=pfill},
  bG/.style={bx, draw=black!45, fill=black!5},
  design/.style={dash pattern=on 2.5pt off 1.5pt},
  flow/.style={->, line width=0.5pt, black!75},
  >={Stealth[length=4.2pt, width=3pt]}]
  \node[m] (pda)  at (3.1, 0.3)   {PDA};
  \node[m] (tpda) at (3.1,-1.1)   {\mach{2-PDA}\\[-1pt]{\sub Priestess computation}};
  \node[m] (o)    at (3.1,-2.3)   {\mach{O-2-PDA}};
  \node[m] (s)    at (1.5,-3.5)   {\mach{O-2-PDA-S}\\[-1pt]{\sub storage only}};
  \node[m] (t)    at (4.7,-3.5)   {\mach{O-2-PDA-T}\\[-1pt]{\sub transitions only}};
  \node[m, minimum width=27mm] (ts)  at (3.1,-4.8) {\mach{O-2-PDA-TS}\\[-1pt]{\sub Priestess above Oracle}};
  \node[m, minimum width=27mm] (tsv) at (3.1,-6.0) {\mach{O-2-PDA-TSV}};
  \draw[flow] (pda) -- node[lab, right, pos=0.3] {a second stack (\S\ref{sec:machine-pda})} (tpda);
  \draw[flow] (tpda) -- node[lab, right] (oilab) {one Oracle instruction (\S\ref{sec:machine-oracle})} (o);
  \draw[flow] (o.south) -- node[lab, above left, pos=0.5, inner sep=1pt] {S (\S\ref{sec:os-s})} (s.north);
  \draw[flow] (o.south) -- node[lab, above right, pos=0.5, inner sep=1pt] {T (\S\ref{sec:os-t})} (t.north);
  \draw[flow] (s.south) -- node[lab, above right, pos=0.6, inner sep=1pt] {T} (ts.north);
  \draw[flow] (t.south) -- node[lab, above left, pos=0.6, inner sep=1pt] {S} (ts.north);
  \draw[flow] (ts) -- node[lab, right] {V (\S\ref{sec:query})} (tsv);
  \node[lab, anchor=north] (oslab) at ([yshift=-1.2mm]tsv.south) {an operating system (Definition~\ref{def:os})};
  \begin{scope}[on background layer]
    \node[frame, draw=pcol, fill=pfill, fit=(ts)(tsv)(oslab), inner xsep=2.5mm, inner ysep=1.5mm] (osf) {};
  \end{scope}
  \node[frame, fit=(tpda)(s)(t)(osf)(oilab), inner sep=2.5mm] (tc) {};
  \node[lab, anchor=north west] at (tc.north west) {Turing complete};
  \node[note, anchor=north west] (st) at ([yshift=-2mm]tc.south west) {%
    \textbf{S}, on storage: where a program the Oracle runs sits\\[1pt]
    \textbf{T}, on transitions: who decides the next transition};
  \coordinate (bl) at (7.3,0 |- st.south);
  \coordinate (tr) at (17.7,1.0);
  \draw[frame] ([yshift=-3mm]bl) rectangle (tr);
  \node[m, anchor=west] (an) at (7.5,0.5) {ArchNights};
  \draw[flow] (tsv.east) -- (7.0,0 |- tsv) -- (7.0,0.5) -- (an.west);
  \node[col] at (11.1,-0.4) {\textbf{O-mode}: what the Oracle runs};
  \node[col] at (15.45,-0.4) {\textbf{P-mode}: what the Priestess runs};
  \foreach \y/\r in {-1.3/Programs, -2.75/Library, -4.2/Kernel, -5.65/ISA}
    \node[row] at (7.5,\y) {\r};
  \node[bO] at (11.1,-1.3) {init, the shell and the core tools,\\each a text that the Oracle runs\\as a program};
  \node[bP] at (15.45,-1.3) {the shell's loop that reads the user,\\init's loop that restarts the shell,\\and the check on host commands};
  \node[bO] at (11.1,-2.75) {the tool calls (operators) the Oracle\\may make, listed at the start of\\every task's history};
  \node[bP] at (15.45,-2.75) {the code that runs each tool call, in\\a kernel module; one table generates\\the list and the argument checks};
  \node[bO] at (11.1,-4.2) {each task's history, the context\\that the Oracle reads, kept in a\\memory segment of its own};
  \node[bP] at (15.45,-4.2) {PRTS: the scheduler, fork and exec,\\a virtual file system, checkpoints,\\and drivers that call LLMs or people};
  \node[bO] at (11.1,-5.65) {the Oracle mode, a privilege level\\below user mode with no code of its\\own, run by microcode};
  \node[bP] at (15.45,-5.65) {64-bit RISC-V, plus instructions that\\ask what the Oracle can do, call it,\\and enter the Oracle mode};
  \draw[black!30, line width=0.4pt] (7.45,-6.55) -- (17.55,-6.55);
  \node[row] at (7.5,-7.4) {gem5};
  \node[bG] at (11.1,-7.4) {ArchNights-SE: PRTS on gem5's\\syscall-emulation (SE) mode};
  \node[bG, design] at (15.45,-7.4) {ArchNights-FS, a design, on gem5's\\full-system (FS) mode: ZOOT in user\\mode, Linux in supervisor mode};
\end{tikzpicture}
\caption{The construction. Left: the abstract machines from the PDA to the \mach{O-2-PDA-TSV}. S and T are the two
symmetry breakings, which can be taken in either order and together make a Priestess program the operating system of the tasks
the Oracle runs. V realizes the \mach{O-2-PDA-TS} on a von Neumann computer. Right: ArchNights, the prototype that implements the \mach{O-2-PDA-TSV} on gem5, with what runs in
O-mode on the Oracle side (amber) and what runs in P-mode on the Priestess side (blue). The bottom row, not split by side,
holds the two realizations. ArchNights-SE runs the kernel PRTS on gem5's SE mode, and ArchNights-FS (dashed) is a design
with the kernel ZOOT.}
\label{fig:construction}
\end{figure*}
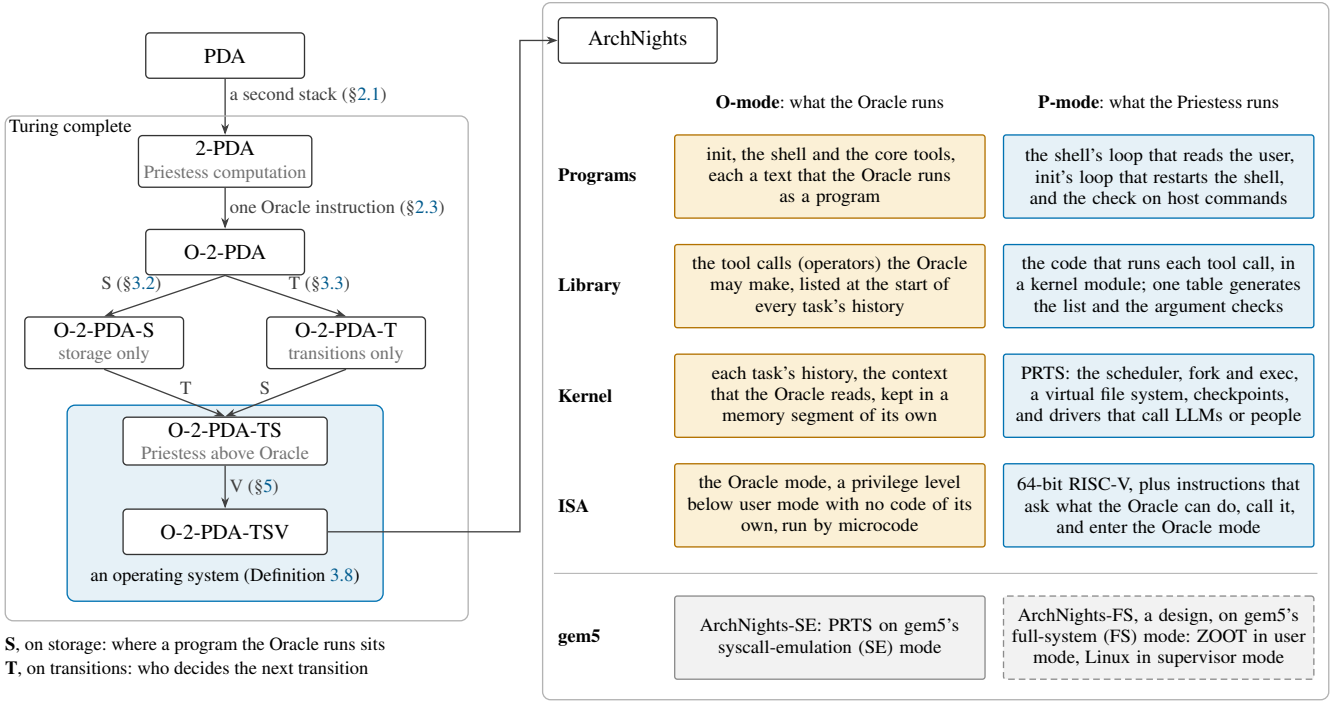

\paragraph{One instruction.} By one instruction we connect an oracle of computability theory to a Turing-complete machine, a
two-stack pushdown automaton written as a program.
The instruction $\mathtt{oracle}\ s$ hands the Oracle the whole content of the stack $s$
and appends the answer to that same stack. The result is the Oracle \mach{2-PDA},
\mach{O-2-PDA} (\S\ref{sec:machine}).
Its computation divides in two: the computation inside the Oracle, and the Turing-complete
rest, which queries the Oracle; we call it the \textbf{Priestess}.
We colour one stack O, after the Oracle, and the other P, after the Priestess.
Two stacks make the machine Turing complete and give each colour a stack of its own.
Nothing is assumed of the Oracle, and yet the machine carries a language model's inference: while the program only appends to a stack,
the history that the Oracle sees on it grows autoregressively from call to call (query, answer, tool result, query), as a model's
context does (Figure~\ref{fig:readings}(a)).
We call this growth the autoregressive form.
It belongs to the interaction, not to the model.

\paragraph{Two breakings.} The machine is symmetric, though: its colours are labels, and nothing in it distinguishes an operating system
from the other programs. We type it with call-by-push-value~\cite{levy2003call}, in which an entry on a stack is either data
or a suspended program, much as a page's permissions say whether it may be executed. We then break the symmetry twice
(\S\ref{sec:os}). The storage breaking S decides where the Oracle may find a program. It may find one on an \textbf{Oracle stack},
a stack coloured O. On a \textbf{Priestess stack}, a stack coloured P, the Oracle finds no program, and whatever it answers there is
data. The transition breaking T decides who chooses the next transition. It adds
an \textbf{O-mode}, in which the machine queries one Oracle stack again and again and appends each answer, with no Priestess
instruction between the queries, until an answer ends with a fixed \textbf{yield symbol}. Control then traps to one fixed label in the
Priestess's code. Outside O-mode the machine is in the privileged \textbf{P-mode}, where the Priestess's code chooses the next
transition. The
machine with both breakings is the \mach{O-2-PDA-TS}. For readers who build agents, O-mode is the model generating, a tool call is
a yield, and the dispatcher that runs the tool and returns its result is the Priestess's code in P-mode. For readers from
operating systems, O-mode stands to P-mode as user mode to the kernel, and a yield trap is a system call. Table~\ref{tab:terms-quick}
sets these terms out by side, and \S\ref{sec:os} defines them formally; \S\ref{sec:eclipse-safety} says where the correspondence stops.
Appendix~\ref{app:further-blackbox} separates the two breakings into a general machine and derives the safety boundary and the
trusted computing base on it, and Appendix~\ref{app:further-levels} asks what the several levels of privilege of an ordinary computer amount to.

\paragraph{An operating system, derived.} Under this choice of S and T the colours are ordered, in storage and in mode, with P
above O. A Priestess program therefore confines the tasks that the Oracle runs: it decides what they see, when they run and what
answers they receive, and they reach it only at one fixed entry. Such a Priestess program is an operating system, and its contract
is an invariant of the machine (Theorem~\ref{thm:contract}; Figure~\ref{fig:readings}(b)). Under the contract, a task that the
Oracle runs reads and writes only its own stack, control returns only at the fixed entry, and to the Priestess every answer is
data: none becomes a code address, and none grants itself a right to execute. The contract
is derived from the rules, for every Oracle, and it says what a kernel on this machine may rely on. Both breakings restrict as little
as possible, since a stronger form of either would cost the Workflow form. Neither is a check that a program could pass or fail,
so the contract holds for every program as well (\S\ref{sec:os-t}). With a
budget, control also comes back to the Priestess on a good
machine, one on which no history reached in O-mode is rejected (Figure~\ref{fig:contract}). Each
breaking could be designated the other way, but the two differ in what reversing them would take: reversing S would relabel the two
colours (Appendix~\ref{app:further-blackbox}), while reversing T rewrites the transition rules (Appendix~\ref{app:further-reversed}). An operating system forms only when
the two are opposed, that is, when the side
where a task's program resides is not the side that T privileges. With privilege reversed, the Oracle is the operating system, the
Priestess is confined to its own stacks, and each item of the contract has a counterpart with the roles exchanged, so the
construction reaches the two quadrants in which a person is trusted and the program is confined.

\paragraph{Two placements.} With a language model as the Oracle, the two placements of a task's program are the two forms of
agentic systems. An Agent carries its program in its history, on the Oracle side, and the Priestess is its operating system. A
Workflow's program is Priestess code, on the privileged side, so no operating system mediates it and the model's output is data. A
hybrid has part of its program on each side. ``Workflow'' and ``Agent'' thus have two senses. As \textbf{implementation forms} they differ in whether an
operating system mediates the task, a distinction we derive on the machine. As \textbf{design paradigms}, the sense of Building
Effective Agents above, they differ in whether the model's output or the program's code decides the next step. Either form can
carry either paradigm, with a proper compiler, so the machine is one architecture in which to analyse their combinations. With
privilege reversed and a trusted Oracle, such as a person, the same placements are human--computer interaction and automated
systems, respectively. Their machine confines the program completely but cannot check the judgement of the Oracle it trusts. An
Agent can always run in Workflow form, through an interpreter in the
Priestess that makes the same queries with the same answers. The guarantees of the fixed
entry and of task confinement then become obligations of the interpreter's code. A Workflow runs as an Agent task if a compiler
translates it into content the Oracle runs. An Agent harness on an ordinary computer implements the machine in its host code.

\paragraph{Caching, scheduling and safety.} The contract holds for a black box, so it says nothing about what depends on the
Oracle's interior, such as the cache. \S\ref{sec:eclipse} opens that one black box, for the Oracles that are autoregressive
inside. It takes in the history through an encoder, which for a language model is its tokenizer, and the two views of the cache in
\S\ref{sec:intro-clouds} are then two levels of one hierarchy, from the history down to the model's KV cache. What the hierarchy
recomputes in excess is exactly what re-encoding the history discards. For certain Oracles and
kernels (Figure~\ref{fig:readings}(c)), a private cache that keeps every state
computes each token of a history once while the kernel only appends to it. Compaction and truncation rewrite the history
and break this reuse, at a price stated by the cache model, and LMCache reports that truncation can halve the hit ratio of a prefix
cache~\cite{liu2025lmcache}. For scheduling, the Priestess side (a harness or a workflow's program) knows which histories it will
submit next, and the Oracle side (an inference server) knows what it has cached. Each side's scheduler works from an estimate of the
other. The cache model quantifies the costs of keeping, discarding or swapping a task's state while its tool call runs
~\cite{abhyankar2024infercept}, and the choice among the three needs one more input, which comes from the Priestess side: when the task
returns. Without it, the best that a deterministic Oracle-side choice between keeping and discarding can guarantee is twice the cost
of a choice that knows. Safety, here the safety
of the agentic system rather than of the model, needs no assumption on the Oracle's interior. Another invariant of the machine is its
\textbf{safety boundary}. The Priestess wrote whatever an Oracle call reads, apart from the answers the Oracle appended to that
stack, and the Oracle's answers reach the Priestess's control only at the trap and at three instructions that read content. The
boundary does not make the machine a defence against prompt injection: a guarantee that holds
for every Oracle can fix where content crosses, but not which content crosses. The
boundary says instead where each defence
outside the Oracle has to sit. Appendix~\ref{app:further} examines the caches, schedulers and defences on these models, as case
studies.

Each of these results is a property of the machine of \S\ref{sec:intro-idea}, derived from it. Two questions that opened
\S\ref{sec:intro-clouds} are these results seen from the other side: where the security boundary of a model lies is what
the safety boundary fixes, as the place where content crosses, and what the cache of an agent is is the history end of that
one hierarchy.

\paragraph{On a von Neumann computer.} The construction V of \S\ref{sec:query} makes the \mach{O-2-PDA-TS} compatible with a von Neumann
computer, and the result is the \mach{O-2-PDA-TSV}. The stacks are segments of one memory, with a base, a top and a limit, and the
Oracle operation is one instruction, with no code in O-mode. O-mode feeds the
Oracle its own answers, and this iteration is the Oracle's own computation only when the Oracle is autoregressive inside. V
therefore lets O-mode run only Oracles that are autoregressive inside (\S\ref{sec:query-oracle} states the price of this choice),
while P-mode may call any Oracle. The Oracle learns its
requests from a declaration in its history, and one semantic delegation call, evaluated by the Priestess, can stand for several
system calls. The results above are of the kind that \S\ref{sec:intro-clouds} asks for beyond the two clouds: invariants or bounds
that hold no matter who or what answers the queries, or for every Oracle that meets a stated condition. Where a statement cannot take this
form, as the Lynchpin hypothesis of \S\ref{sec:query-message} cannot, we say so.

\paragraph{Headline properties.} Several of these results are statements that the works of
\S\ref{sec:intro-clouds} do not make, because each needs two ingredients at once: a quantifier over every Oracle, or over every Oracle
that meets a stated condition, and an operating system derived on the same machine as the Oracle's computation.
\begin{itemize}
\item \emph{The contract is an invariant.} Whatever the model answers, a task writes only its own stack and reaches the kernel only
  at one entry. What remains to assume is that a realization follows the runs of the machine
  (\S\ref{sec:related-parallel}).
\item \emph{One statement for both forms.} An Agent task can always run in Workflow form with the same queries and answers,
  and the properties below hold of a task whichever side its program is placed on
  (\S\ref{sec:os-agentic}).
\item \emph{Cache and scheduling.} For an Oracle that is autoregressive inside, an Oracle call is one hierarchy from the history
  down to the model's KV cache, whose excess recomputation is exactly the encoder's loss, and, for
  certain Oracles and kernels, a private cache that keeps every state computes each token of a history once while the kernel only
  appends to it. The Priestess side knows the next keys and the Oracle side knows its cache,
  so each side's scheduler works from an estimate of the other (\S\ref{sec:eclipse-control}).
\item \emph{Safety.} A guarantee for every Oracle fixes where content crosses, and none fixes which
  content crosses, a limit that needs the quantifier over every Oracle to be stated at
  all.
\end{itemize}
We claim only that the works cited do not make these statements, not that no other formalism could.

\subsection{ArchNights}
\label{sec:intro-archnights}

To show that the machine can be realized, we build ArchNights, a prototype that follows the \mach{O-2-PDA-TSV} on gem5, the
architecture simulator among the most used in computer-architecture research~\cite{binkert2011gem5,lowepower2020gem5}. ArchNights' instruction set extends the RV64
ISA~\cite{waterman2019risc,waterman2021risc} with an Oracle mode that lies below user mode and has no addressable code, and with an
Oracle instruction compatible with RoCC, Rocket Chip's interface for custom coprocessors~\cite{asanovic2016rocket}. The
simulated processor runs the Oracle mode by microcode. The prototype's Linux-style kernel, PRTS, takes control back from every task at one
fixed entry, confines ordinary tasks to a virtual file system, and checkpoints and compresses their histories. ArchNights' distribution
has a system library, an init system, a shell and core tools. On this machine every system call costs the Oracle a round trip and
a piece of its history, so the system library moves work out of the Oracle's history into the kernel by \textbf{semantic
delegation} (\S\ref{sec:query-message}): it registers operators as tool calls, each of which packs several system calls into one
and is evaluated in P-mode by either trusted or verified code.

ArchNights-SE, which runs in gem5's syscall-emulation mode, will be open source. We evaluate it to show that the machine can serve
as the basis of an agent harness. With deepseek-v4.1-flash as the Oracle, ArchNights-SE performs comparably on Terminal-Bench
2.1~\cite{merrill2026terminal} to two reference harnesses run with the same model in the same setting:
mini-swe-agent~\cite{minisweagent2025} and Terminus~2, the benchmark's own agent (\S\ref{sec:an-eval}). Given a virtual display and audio output, ArchNights-SE produces a music video of
\emph{world.execute(me);}, working both as an agent harness and as a general-purpose computer. ArchNights-FS is a full-system
design, with a kernel of its own, for a physical processor that implements this instruction set. \S\ref{sec:archnights} argues
ArchNights' design choices one by one against the machine.

\subsection{Contributions}
\label{sec:intro-contrib}

The contribution of this paper is one construction, in three parts.
\begin{enumerate}
\item \textbf{The machine (\S\ref{sec:machine}).} The \mach{O-2-PDA}: a two-stack pushdown automaton, Turing complete, extended
  with one instruction that queries a black-box Oracle with the whole content of a stack and appends the answer to it. While the
  program only appends, the history on a stack grows autoregressively, as an agent's context does
  (Proposition~\ref{prop:ar-form}).
\item \textbf{The operating system derived on it (\S\ref{sec:os}).} Two symmetry breakings---S on storage, T on
  transitions---make a Priestess program the operating system of the programs the Oracle runs. They are the weakest that keep the
  two placements of a task's program apart for every program (Proposition~\ref{prop:forced}). The kernel contract K1--K4 is an
  invariant of the machine, for every Oracle (Theorem~\ref{thm:contract}); with a budget, control also returns
  (Proposition~\ref{prop:k5}). Agent and Workflow follow as the two placements of a task's program (Definition~\ref{def:forms}).
\item \textbf{The realization (\S\S\ref{sec:query}--\ref{sec:archnights}).} The construction V fits the machine to a von Neumann
  computer: stacks become segments of one memory, the Oracle operation is one instruction, and O-mode has no addressable code.
  ArchNights, an extended RISC-V ISA with a Linux-style kernel, runs on gem5 and demonstrates that the machine can be built and can
  serve as the basis of an agent harness (\S\ref{sec:an-eval}).
\end{enumerate}

\paragraph{Claims of this paper.} Every property derived on the machine---the kernel contract, the two placements, the cache model, the safety boundary---is
evidence for the construction: each answers a question that \S\ref{sec:intro-clouds}'s two clouds do not ask. Three assumptions are
used along the way and are named where they are used: factorization at the yield symbol (Definition~\ref{def:factor}),
Frame-faithfulness (\S\ref{sec:eclipse-codec}), and the Lynchpin hypothesis (\S\ref{sec:query-message}). The kernel contract, the
two placements and the safety boundary assume none of the three. A companion prototype demonstrates that the machine is realizable.
What the construction leaves open---complexity, refinement, asynchrony---is listed in \S\ref{sec:future}.

\paragraph{Instances of the theory (\S\ref{sec:eclipse}, Appendix~\ref{app:further}).} The following statements are not
further contributions. They are instances: their number, in the vocabulary of the fields they touch, is how the construction
shows itself to be coherent. Each holds under the hypotheses named.
\begin{center}
\small
\setlength{\tabcolsep}{4pt}
\begin{tabular}{@{}>{\raggedright\arraybackslash}p{0.39\columnwidth}>{\raggedright\arraybackslash}p{0.31\columnwidth}>{\raggedright\arraybackslash}p{0.22\columnwidth}@{}}
\toprule
Statement & Hypotheses & Where \\
\midrule
Kernel contract K1--K4; crossing invariant & none---every Oracle & Theorem~\ref{thm:contract}, Corollary~\ref{cor:crossing} \\
\rowrule
Control returns (K5) & budget; good machine & Proposition~\ref{prop:k5} \\
\rowrule
An Agent can always run in Workflow form & none & Proposition~\ref{prop:interpreter} \\
\rowrule
Frame synchronization & internally autoregressive Oracle; factorization at $\natural$ & Theorem~\ref{thm:run-sync} \\
\rowrule
Excess recomputation is the encoder's loss & + Frame-faithfulness & Proposition~\ref{prop:excess} \\
\rowrule
Prefix chain across tool calls; each token computed once & + kernel's epoch discipline & Proposition~\ref{prop:prefix-chain} \\
\rowrule
Keep-or-discard without the return time costs a factor of two & deterministic policy & Proposition~\ref{prop:return} \\
\rowrule
Safety boundary: where content crosses & none---every Oracle & Theorem~\ref{thm:boundary} \\
\rowrule
No guarantee fixes which content crosses & --- & Proposition~\ref{prop:where-not-which} \\
\bottomrule
\end{tabular}
\end{center}

\paragraph{Non-claims.} Statements about language models rest on the assumption of \S\ref{sec:eclipse} or on reported
observations; the Lynchpin hypothesis (\S\ref{sec:query-message}) is a working hypothesis; the prototype implements the machine
without a refinement proof, and this paper gives it no security guarantee.

\paragraph{Reading the paper.} \S\S\ref{sec:machine}--\ref{sec:eclipse} are the theory, and \S\S\ref{sec:query}
and~\ref{sec:archnights} build the machine and realize it. Readers with other interests can take shorter paths:
\begin{center}
\small
\begin{tabular}{@{}>{\raggedright\arraybackslash}p{0.27\columnwidth}>{\raggedright\arraybackslash}p{0.67\columnwidth}@{}}
\toprule
Interest & Path \\
\midrule
Theory, programming languages & \S\S\ref{sec:machine}--\ref{sec:eclipse} with Appendices~\ref{app:proofs} and~\ref{app:eclipse};
  then Appendices~\ref{app:further-reversed} and~\ref{app:further-blackbox}, and \S\S\ref{sec:related-theory}
  and~\ref{sec:related-pl} \\
\rowrule
Operating systems & skim \S\S\ref{sec:machine-properties} and~\ref{sec:machine-ar}; read \S\ref{sec:os},
  \S\S\ref{sec:eclipse-memory} and~\ref{sec:eclipse-control}, and \S\S\ref{sec:query} and~\ref{sec:archnights} with
  Appendix~\ref{app:archnights}; then Appendices~\ref{app:further-blackbox} and~\ref{app:further-levels} \\
\rowrule
Building agents & \S\S\ref{sec:machine-ar} and~\ref{sec:os-agentic}; the statements of the results of \S\ref{sec:eclipse}; then
  Table~\ref{tab:designs} and Appendix~\ref{app:further} \\
\rowrule
Safety & \S\S\ref{sec:os-kernel}, \ref{sec:eclipse-safety} and~\ref{sec:query-priestess}; then
  Appendices~\ref{app:eclipse-replay}, \ref{app:further-safety} and~\ref{app:further-blackbox} \\
\rowrule
Human--computer interaction, automated systems & \S\S\ref{sec:os-kernel} and~\ref{sec:os-agentic}, and the person as an Oracle in
  \S\ref{sec:query-message}; then Appendix~\ref{app:further-reversed} \\
\bottomrule
\end{tabular}
\end{center}

\begin{table*}[t]
\centering
\small
\setlength{\tabcolsep}{4pt}
\begin{tabular}{@{}>{\raggedright\arraybackslash}p{0.13\textwidth}>{\raggedright\arraybackslash}p{0.41\textwidth}>{\raggedright\arraybackslash}p{0.41\textwidth}@{}}
\toprule
 & \multicolumn{1}{c}{\cellcolor{ofill}Oracle side (amber)} & \multicolumn{1}{c}{\cellcolor{pfill}Priestess side (blue)} \\
\midrule
computation & the \textbf{Oracle}: the source of the answers, treated as a black box; an LLM or a person (\S\ref{sec:machine-oracle}) & the \textbf{Priestess}: the Turing-complete computation that queries the Oracle (\S\ref{sec:machine-oracle}) \\
\rowrule
letters & \textcolor{odark}{O} names the colours, the modes and the order of the sides, and nothing more (\S\ref{sec:machine-split}, \S\ref{sec:os}) & \textcolor{pdark}{P} likewise \\
\rowrule
stack & \textbf{Oracle stack}: the only kind that may hold a program the Oracle runs (S) & \textbf{Priestess stack}: whatever the Oracle produces there is data (S) \\
\rowrule
mode & \textbf{O-mode}: the Oracle answers again and again until an answer yields; no Priestess code runs between the queries (T) & \textbf{P-mode}: the Priestess program chooses the next transition (T) \\
\rowrule
a task's program & \textbf{Agent form}: the program is the task's history, which the Oracle runs & \textbf{Workflow form}: the program is Priestess code; the model's output is data \\
\rowrule
control crossing & an answer ending with $\natural$ \textbf{yields}: it traps to the one fixed label $\ell_{\mathrm{trap}}$, the system call & the \textbf{kernel} enters tasks with $\mathtt{uret}$ and serves their traps; its \textbf{library} declares the services (Defs.~\ref{def:t}, \ref{def:os}) \\
\grouprule
\multicolumn{3}{@{}l}{\textbf{task} and \textbf{history}: the unit of work the system runs, and the content of its task stack (Definition~\ref{def:os}).} \\
\multicolumn{3}{@{}l}{\textbf{Step} and \textbf{Frame}: one internal step of an autoregressive Oracle, and one whole answer (\S\ref{sec:eclipse-codec}).} \\
\bottomrule
\end{tabular}
\caption{The two sides at a glance. Terms that the Introduction uses before their definitions, split by side as the
figures are. This is a map: the full dictionary is Appendix~\ref{app:glossary}, whose Table~\ref{tab:terms} lists the
symbols of \S\ref{sec:os}.}
\label{tab:terms-quick}
\end{table*}

\section{The Machine: \mach{O-2-PDA}}
\label{sec:machine}

\sectionepigraph{``Let us suppose that we are supplied with some unspecified means of solving number-theoretic problems; a kind of
oracle as it were. We shall not go any further into the nature of this oracle apart from saying that it cannot be a
machine.''}{A.\,M.~Turing, 1939~\cite[\S4]{turing1939systems}}

With these words Turing introduced oracle machines. We take his first point: until \S\ref{sec:eclipse} we go no further into the
nature of the Oracle and assume nothing of its interior. Turing's oracle cannot be a machine. Ours may well be one: a language model on
hardware of its own. \S\ref{sec:eclipse} looks inside.

We start with a two-stack pushdown automaton, written as a program, and add one instruction that consults the Oracle: the program
supplies the query as the whole content of one stack; the instruction appends the answer to that same stack. The stack becomes
a record of the interaction. As long as the program only appends to it between queries, each query contains the whole preceding
history, and the history grows autoregressively, as a language model's context does (Proposition~\ref{prop:ar-form}).

\subsection{Two-Stack Pushdown Automata, Written as Programs}
\label{sec:machine-pda}
\label{sec:machine-program}

A pushdown automaton (PDA) is a finite-state control with one stack. In state form it is a tuple $(\Theta, \Sigma_{\mathrm{in}},
\Sigma, \delta, \theta_0, \Theta_F)$ of a finite set of states $\Theta$ with a start state $\theta_0$ and final states $\Theta_F$,
an input alphabet $\Sigma_{\mathrm{in}}$, a stack alphabet $\Sigma$, and a transition relation $\delta$. From the current state,
the next input symbol or none, and the top of the stack, it chooses a next state and a word that replaces the top.
Non-deterministic PDAs recognize exactly the context-free languages~\cite{hopcroft1979introduction}. A two-stack PDA
(\mach{2-PDA}) has the same control and two stacks, and each transition reads both tops and replaces both. The second stack
changes the class of machine: the stacks can hold a Turing machine's tape to the left and to the right of its head, so the two
models compute the same functions~\cite{minsky1967computation,hopcroft1979introduction}. A \mach{2-PDA} is
therefore Turing-complete. We write the
\mach{2-PDA} as a program, in the style of Shepherdson and Sturgis's register machines~\cite{shepherdson1963computability}: the
finite control becomes a numbered list of instructions, and the state becomes a program counter.

\namedpara{Notation.} $\lambda$ is the empty word. A stack's content is a word whose right end is the top. For words $u$ and $v$, $u \cdot v$ is
their concatenation; for a symbol $x$, $w \cdot x$ is $w$ with $x$ pushed. $\mathrm{top}(w \cdot x) = x$ and
$\mathrm{top}(\lambda) = \bot$, where $\bot \notin \Sigma$. $u \sqsubseteq w$ means that $u$ is a prefix of $w$. A \textbf{store}
$\sigma$ maps each stack to its content, so $\sigma(s)$ is the content of stack $s$.

\begin{definition}[program-style \mach{2-PDA}]
\label{def:2pda}
Fix a finite alphabet $\Sigma$ and two stacks, $a$ and $b$. A \textbf{program} is a finite sequence $\Pi = (\iota_1, \dots, \iota_n)$
of instructions:
\begin{center}
\begin{tabular}{@{}l>{\raggedright\arraybackslash}p{0.62\columnwidth}@{}}
\toprule
Instruction & Meaning \\
\midrule
$\mathtt{push}\ s\ x$ & push $x \in \Sigma$ onto stack $s$ \\
$\mathtt{pop}\ s$ & pop the top of $s$ \\
$\mathtt{jeq}\ s\ x\ l$ & jump to label $l$ if $\mathrm{top}(\sigma(s)) = x$, where $x \in \Sigma \cup \{\bot\}$ \\
$\mathtt{halt}$ & stop \\
\bottomrule
\end{tabular}
\end{center}

A \textbf{configuration} is a pair $(\ell, \sigma)$ of a label $\ell \in \{1, \dots, n+1\}$ and a store $\sigma$. The step relation
$\to$ is:
\begin{itemize}
\item if $\iota_\ell = \mathtt{push}\ s\ x$, then $(\ell, \sigma) \to (\ell+1, \sigma[s \mapsto \sigma(s) \cdot x])$;
\item if $\iota_\ell = \mathtt{pop}\ s$ and $\sigma(s) = w \cdot x$, then $(\ell, \sigma) \to (\ell+1, \sigma[s \mapsto w])$;
  if $\sigma(s) = \lambda$, there is no successor and the machine is \textbf{stuck};
\item if $\iota_\ell = \mathtt{jeq}\ s\ x\ l$, then $(\ell, \sigma) \to (l, \sigma)$ when $\mathrm{top}(\sigma(s)) = x$, and
  $(\ell, \sigma) \to (\ell+1, \sigma)$ otherwise;
\item if $\iota_\ell = \mathtt{halt}$ or $\ell = n+1$, the configuration is final.
\end{itemize}

A \textbf{run} is a sequence of configurations linked by $\to$, starting from $(1, \sigma_0)$ for some initial store $\sigma_0$. An
unconditional jump $\mathtt{goto}\ l$ abbreviates, all on one fixed stack $s$, a sequence of $\mathtt{jeq}\ s\ x\ l$ instructions,
one for each $x \in \Sigma \cup \{\bot\}$. Since the sequence tests every possible top, one of these instructions jumps to
$l$. One program runs, one instruction per step.
\end{definition}

What a \mach{2-PDA} program computes is the \textbf{Turing-complete computation}. Once the Oracle is added
in \S\ref{sec:machine-oracle}, it is the \mach{2-PDA} part of the machine.

\subsection{Split Stacks and Their Colours}
\label{sec:machine-split}

A \mach{2-PDA} simulates the components of a computer~\cite{minsky1967computation}. Its two stacks hold a tape on either side of the
head, and its control holds finitely many registers, each with finitely many values, which in the program style it
represents by duplicating code for each value. A symbol at depth $i$ is accessed by moving the $i$ symbols above it to the other stack and back in
reverse order, and the same transfers exchange two symbols. One construction from this simulation is used throughout the paper:
a stack can be split into several stacks.

\begin{lemma}[split]
\label{lem:split}
Let the two stacks $a$ and $b$ hold finitely many \textbf{logical stacks}, each a segment of $a$ or of $b$, separated by marker
symbols that lie outside the alphabet of the logical stacks. Every $\mathtt{push}$, $\mathtt{pop}$ or $\mathtt{jeq}$ on a logical
stack $s$ that has a successor on $s$ is realized by a finite block of \mach{2-PDA} steps such that, at the end of the block, $s$ has changed as the instruction
prescribes, control continues where the instruction prescribes, and every other logical stack holds the same symbols in the same
order as at the start. An instruction with no successor on $s$ leaves the machine stuck, as it would be on a stack of its own.
\end{lemma}

The block moves the segments above $s$ onto the other stack, acts on $s$, and moves them back in reverse order
(Appendix~\ref{app:proofs}). We call the finite sequence of steps that realizes one instruction on a logical stack a
\textbf{block}, and we call the end of a block an \textbf{observable step}. Without splitting, every step is observable. From here
on we describe the machine at its observable steps, with the logical stacks as its stacks. A stack produced by splitting is
\textbf{derived} from the stack it was split from.

\begin{definition}[colouring]
\label{def:colouring}
Stack $a$ is coloured $P$ and stack $b$ is coloured $O$. Every derived stack carries the colour of the stack it was derived from,
so each stack has a colour, $O$ or $P$. At an observable step, a symbol has the colour of the stack it lies on, and a
symbol that is moved inside a block keeps the colour of the logical stack it belongs to.
\end{definition}

The letters denote the Priestess and the Oracle, the two computations \S\ref{sec:os} names. At this point the colours are only
labels. Every instruction may therefore act on every stack, and exchanging the two colours everywhere maps programs to programs
and runs to runs. Write $\varepsilon$ for this exchange. The colours mark where \S\ref{sec:os} breaks this symmetry: each breaking
adds a structure that $\varepsilon$ fails to preserve.

\subsection{The Oracle and Its Instruction}
\label{sec:machine-oracle}

The Oracle answers queries over its own alphabets, and an encoder and a decoder translate between the Oracle's alphabets and the machine's.

\begin{definition}[Oracle, encoder, decoder]
\label{def:oracle}
An \textbf{Oracle} $(A, E, D)$ consists of:
\begin{itemize}
\item Oracle-side alphabets $\Gamma_Q$ and $\Gamma_R$, with sets $X_A \subseteq \Gamma_Q^*$ of queries and $Y_A \subseteq \Gamma_R^*$
  of answers;
\item an answer relation $A \subseteq X_A \times Y_A$ in which every $x \in X_A$ has at least one answer;
\item machine-side query and reply alphabets $\Sigma_Q, \Sigma_R \subseteq \Sigma$;
\item an encoder $E : \Sigma_Q^* \rightharpoonup X_A$, which may be partial, and a decoder $D : Y_A \to \Sigma_R^*$. Together they
  are the Oracle's \textbf{codec}.
\end{itemize}

These components define the machine-side answer map
\[ O(w) = \{\, D(y) : (E(w), y) \in A \,\} \subseteq \Sigma_R^*, \qquad w \in \mathrm{dom}(E), \]
which is non-empty for every $w \in \mathrm{dom}(E)$. A word outside $\mathrm{dom}(E)$ is \textbf{rejected}.
\end{definition}

\namedpara{Convention (non-empty answers).} Every answer is non-empty: $O(w) \subseteq \Sigma_R^+$ for every $w \in
\mathrm{dom}(E)$. An empty answer appends nothing, so a program that queries the same stack again asks the same query. When the
Oracle is iterated on its own history, as it is from \S\ref{sec:os} on, an empty answer can return the machine to the
configuration it left, and the iteration may never end. Appendix~\ref{app:eclipse-scope} states where the convention holds by
construction, where a realization must maintain it, and which results use it.

The answer relation need not be a function: the same query may have more than one possible answer. We reason about
which answers are possible, never about how likely they are, so the results of this section quantify over all possible answers
and hold as well when the answer is chosen adversarially among them.

\begin{definition}[Oracle instruction; \mach{O-2-PDA}]
\label{def:o2pda}
Add to the \mach{2-PDA} with the stacks of Definition~\ref{def:colouring} the instruction $\mathtt{oracle}\ s\ l_e$, where $l_e$ is a
label:
\begin{itemize}
\item if $\sigma(s) \in \mathrm{dom}(E)$, then $(\ell, \sigma) \to (\ell+1, \sigma[s \mapsto \sigma(s) \cdot r])$ for each
  $r \in O(\sigma(s))$;
\item otherwise $(\ell, \sigma) \to (l_e, \sigma)$: the encoder rejected the query.
\end{itemize}

The resulting machine is the \textbf{\mach{O-2-PDA}} (Figure~\ref{fig:o2pda}). The instruction reads the whole content of one
stack and writes the answer next to it, on that same stack. On a tape, it reads a contiguous region and writes new symbols right
after it. An accepted call appends one of the
possible answers to the stack, and each possible answer yields a possible successor configuration.
\end{definition}

\begin{figure}[t]
\centering
\begin{tikzpicture}[
  font=\footnotesize,
  cell/.style={draw, line width=0.5pt, minimum width=7mm, minimum height=3.6mm, inner sep=0pt},
  box/.style={draw=ostroke, fill=white, line width=0.5pt, rounded corners=1pt, minimum width=6mm,
    minimum height=4.6mm, inner sep=1.5pt},
  flow/.style={->, line width=0.5pt, black!75},
  >={Stealth[length=4.2pt, width=3pt]}]
  \node[draw, line width=0.5pt, rounded corners=1.5pt, inner sep=3pt, align=left, anchor=north west]
    (prog) at (0,0) {%
    $\Pi$\\[2pt]
    $\iota_1$\\
    $\vdots$\\
    $\iota_\ell = \mathtt{oracle}\ s\ l_e$\\
    $\vdots$\\
    $\iota_n$};
  \foreach \i in {0,1}
    \node[cell, draw=ostroke] (b\i) at (2.85, -2.05+0.36*\i) {};
  \node[below=2pt of b0, inner sep=1pt] {$b$\,:\,$O$};
  \foreach \i in {0,1,2}
    \node[cell, draw=pcol] (a\i) at (3.75, -2.05+0.36*\i) {};
  \node[cell, draw=pcol, fill=ofill] (a3) at (3.75, -2.05+0.36*3) {$r$};
  \node[below=2pt of a0, inner sep=1pt] {$a$\,:\,$P$};
  \draw[decorate, decoration={brace, amplitude=2.5pt}, line width=0.5pt]
    ([xshift=1.2mm]a2.north east) -- ([xshift=1.2mm]a0.south east) coordinate[midway, xshift=3pt] (sig);
  \node[right=0pt of sig, inner sep=1pt] (siglab) {$\sigma(s)$};
  \node[box] (E) at (6.25, -1.69) {$E$};
  \node[box] (A) at (6.25, -0.91) {$A$};
  \node[box] (D) at (6.25, -0.13) {$D$};
  \begin{scope}[on background layer]
    \node[draw=ostroke, fill=ofill, line width=0.5pt, rounded corners=2pt, fit=(E)(A)(D), inner sep=3.5pt] (orc) {};
  \end{scope}
  \node[above=1pt of orc, inner sep=1pt] {Oracle $(A,E,D)$};
  \draw[flow] (siglab.east) -- node[below, pos=0.45, inner sep=2pt] {read} (E.west |- siglab.east);
  \draw[flow] (E) -- (A);
  \draw[flow] (A) -- (D);
  \draw[flow] (D.west) -- ++(-0.45,0) |- node[above, pos=0.75, inner sep=2pt] {append $r$} (a3.east);
  \draw[flow, dashed, dash pattern=on 2pt off 1.5pt] (orc.south) -- (orc.south |- 0,-3.0) -|
    node[below, pos=0.25, inner sep=2pt] {rejected: continue at $l_e$} (prog.south);
\end{tikzpicture}
\caption{The \mach{O-2-PDA}. The program $\Pi$ runs on two stacks, coloured $P$ (blue outline) and $O$ (amber outline).
The instruction $\mathtt{oracle}\ s\ l_e$ reads the whole content of the stack $s$, here $a$. The encoder $E$ maps it to
a query, and the decoder $D$ maps an answer of $A$ to a word $r \in O(\sigma(s))$, which is appended to $s$. The word $r$ is
filled amber because the Oracle wrote it, and since it lies on $a$, it has colour $P$ in the sense of
Definition~\ref{def:colouring}. When $E$ rejects a query, the machine continues at $l_e$.}
\label{fig:o2pda}
\end{figure}
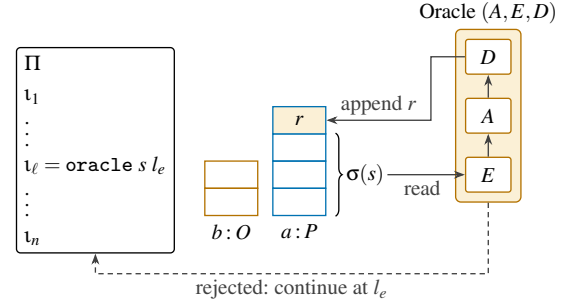

The Oracle instruction is the only access the Turing-complete computation has to the Oracle. The Oracle's internal computation (for a
language model, a forward pass on hardware of its own) belongs to the \mach{O-2-PDA} but not to its \mach{2-PDA} part. An Oracle
is the whole triple $(A, E, D)$: one answer relation with two encoders gives two Oracles, and the machine's behaviour, including
which queries are rejected, may differ between them. A rejection is the encoder's decision, not the answer relation's, and it is
also how a bounded input is modelled: an encoder that accepts only words up to some length rejects the rest. The label $l_e$
matters only when a query is rejected, so \S\ref{sec:intro-idea} writes the instruction as $\mathtt{oracle}\ s$. Because the
\mach{2-PDA} is Turing complete, $l_e$ adds no power when $\mathrm{dom}(E)$ is decidable: the program could test membership
itself. On a machine that is not Turing complete, the jump could add power, since such a machine may be unable to decide
$\mathrm{dom}(E)$. Conditions that link the
Oracle's own process to the history on the machine are the subject of \S\ref{sec:eclipse} and are not assumed here.
Lemma~\ref{lem:split} extends to the Oracle instruction on a logical stack, by a block that gathers the operand alone onto one
physical stack and back again (Appendix~\ref{app:proofs}).

\begin{lemma}[locality]
\label{lem:locality}
At every observable step, the \mach{O-2-PDA} changes at most the stack that its instruction names, and a $\mathtt{jeq}$ changes no
stack.
\end{lemma}

\begin{proof}
The rules of Definitions~\ref{def:2pda} and~\ref{def:o2pda} update $\sigma$ only at the stack that their instruction names, and a
$\mathtt{jeq}$ changes no stack. Where the stacks are logical stacks on $a$ and $b$, Lemma~\ref{lem:split} and its extension to the
Oracle instruction restore every other logical stack at the end of each block. So at every observable step the \mach{O-2-PDA}
changes at most the stack that its instruction names.
\end{proof}

In particular, a step on one stack leaves every stack of the other colour as it was.

\paragraph{A real Oracle.} The definition and the convention require three properties of an Oracle on every query its encoder accepts: that
it answers, that the answer is non-empty, and, since an Oracle instruction is one step, that it answers in finite time. A real
implementation may satisfy none of these requirements. A model behind an API may time out, fail or refuse, and a person may not answer at all. A real Oracle
is \textbf{total} if, on every query its encoder accepts, it returns a non-empty answer in finite time. A total Oracle is an
Oracle in the sense of Definition~\ref{def:oracle}, and \S\S\ref{sec:eclipse} and~\ref{sec:archnights} discuss what a realization
does when the Oracle is not total. A word that the encoder does not accept is rejected before the Oracle sees it, and the program continues at
$l_e$.

\subsection{What the Instruction Adds}
\label{sec:machine-properties}

Oracle pushdown automata have been studied as one-stack devices for reducibilities among context-free
languages~\cite{yamakami2013oracle}, whereas the machine here has two stacks. A program computes a partial function when its
input is written on $a$ at the start, $b$ is empty, and its output is the content of $a$ in a final configuration. Compare
an \mach{O-2-PDA} with an oracle Turing machine~\cite{turing1939systems} that has the same Oracle and may also ask whether a word
lies in $\mathrm{dom}(E)$. With a deterministic Oracle, the two compute exactly the same partial functions, by simulating
Turing machines on two stacks~\cite{minsky1967computation}. One direction keeps each stack on its own track of the Turing
machine, copies the queried track to the query tape and appends the answer to the track; the other simulates the Turing machine on
two stacks and, before a query, gathers the tape onto one stack and moves the query word alone onto the other. So taking a whole
stack as the query adds no power, split stacks add none either, and the Oracle instruction adds no power beyond the Oracle's.
\textbf{What it adds is a form} (\S\ref{sec:machine-ar}): each answer is appended to the query that produced it.

A von Neumann machine stores program and data side by side in one random-access memory~\cite{vonneumann1945first}. As an abstract
machine it is the random-access stored-program machine~\cite{cook1973time}. Its instructions are fetched from memory rather
than given as a fixed list beside it, and its control has a finite set of registers besides the program counter.
\S\ref{sec:query} realizes the machine of \S\ref{sec:os} on such a computer.

\subsection{The Autoregressive Form}
\label{sec:machine-ar}

Let $F$ assign a non-empty set $F(h)$ of words to each word $h$ in a set $\mathrm{dom}(F)$. A sequence of words $h_0, h_1, \dots$ is
an \textbf{autoregressive process driven by $F$} if, for every $t$ such that $h_{t+1}$ is defined, $h_t \in \mathrm{dom}(F)$ and
\[ h_{t+1} = h_t \cdot x_t \cdot p_t \qquad\text{with } x_t \in F(h_t). \]
The words $p_t$ are \textbf{interjections}, not produced by $F$. Without interjections ($p_t = \lambda$ for every $t$),
each new segment is generated from the whole history so far.

\begin{proposition}[autoregressive form]
\label{prop:ar-form}
Consider a run of an \mach{O-2-PDA} and a stack $s$. Suppose that the run executes $\mathtt{oracle}\ s\ l_e$ at observable steps
$n_0 < n_1 < \cdots$, that none of these queries is rejected, and that, between $n_t$ and $n_{t+1}$, every instruction with operand
$s$ is a $\mathtt{push}$ or a $\mathtt{jeq}$. Let $Q_t$ be the content of $s$ just before step $n_t$. Then, for every $t$ such that
$n_{t+1}$ exists,
\[ Q_{t+1} = Q_t \cdot r_t \cdot p_t, \]
where $r_t \in O(Q_t)$ is the answer written at $n_t$, and $p_t \in \Sigma^*$ is the word of symbols pushed onto $s$ between $n_t$ and
$n_{t+1}$, in order.
\end{proposition}

So $Q_0, Q_1, \dots$ is an autoregressive process driven by $O$, with the program's pushes as interjections. In particular
$Q_t \sqsubseteq Q_{t+1}$, and each answer $r_t$ is drawn from $O$ applied to the entire history $Q_t$.

\begin{proof}
Step $n_t$ replaces $Q_t$ by $Q_t \cdot r_t$ (Definition~\ref{def:o2pda}, query accepted). Consider an observable step $n$ with
$n_t < n < n_{t+1}$:
\begin{itemize}
\item if its instruction is a $\mathtt{push}$ on $s$, it appends one symbol to $s$;
\item if it is a $\mathtt{jeq}$, or its instruction has another operand, it leaves $s$ unchanged (Lemma~\ref{lem:locality}).
\end{itemize}

By hypothesis there is no other case, so by induction on
these steps the content of $s$ is always $Q_t \cdot r_t \cdot p$, where $p$ is the word pushed onto $s$ so far. At $n_{t+1}$ this
gives the claim.
\end{proof}

The form belongs to the interaction between the program and the Oracle, not to the Oracle: the program stores the history, each
call appends an answer to it, and the Oracle may produce that answer by any internal process. The form the query history takes is
therefore fixed by the program and the Oracle instruction alone; \S\ref{sec:eclipse} adds conditions on the Oracle's interior
without revising it. In Figure~\ref{fig:readings}(a) the session's queries
are $Q_0 \sqsubseteq Q_1 \sqsubseteq Q_2$, each the whole history so far. If the program answers requests that the Oracle's output
contains, the interjection $p_t$ is that reply, and the history is the familiar context of an agent: history, model output, tool
result, history (\S\ref{sec:os}). A program that pops from $s$, or rewrites it, starts a new history, and the proposition applies
again from the next query. A language model continues the text it is given~\cite{brown2020language}, and a
chat-style use of a language model has this form, since each call queries the history so far, extended by the model's previous output and whatever the program
appends. In such a use, a context bound
is an encoder that rejects longer words.

We take every construction in its simplest form: synchronous, without concurrency, fault tolerance or recovery. Later sections
take up liveness (\S\ref{sec:os}), safety (\S\ref{sec:eclipse}), and asynchrony and finite capacity (\S\S\ref{sec:query}
and~\ref{sec:archnights}). We prove no time or space bound: the cost statements of \S\ref{sec:eclipse} count Steps and
tokens, and an excursion may take as many O-mode transitions as its budget allows.

So the \mach{O-2-PDA} is a Turing-complete machine whose one new instruction appends the Oracle's answer to the stack it queries.
Its stacks split into as many as a computation needs, a call touches only its operand, and a stack accumulates an autoregressive
history while the program only appends to it. Its colours, however, are still labels, and the machine is symmetric between them.
\S\ref{sec:os} breaks that symmetry.

\section{The Operating System: Two Symmetry Breakings}
\label{sec:os}

An operating system must distinguish the programs it governs from the code that governs them. On the machine of \S\ref{sec:machine}
nothing does: every instruction may act on either colour of stack. A type system
from call-by-push-value gives us a language for when the machine treats content as data and when it runs it as a program. Two symmetry
breakings then fix the reading of stored content and the transfer of control. The storage breaking S types the Oracle's operand by its
colour: on a Priestess stack the Oracle only answers, and what it produces there is data, while on an Oracle stack it may run a program.
S restricts the reading, not the runs. The transition breaking T separates task execution from privileged service and decides how control
passes between them. The machine with both, the \mach{O-2-PDA-TS}, has an operating system, and Agent and Workflow emerge as two placements
of a task's program on this machine.

From here on we write \textbf{Oracle stack} and \textbf{Priestess stack} for a stack coloured $O$ and a stack coloured $P$, and
\textbf{O-mode} and \textbf{P-mode} for the two modes T introduces. The \textbf{computation inside the Oracle} is the Oracle's internal process,
encoding, computing and decoding included; the \textbf{Priestess computation} is the Turing-complete part of the \mach{O-2-PDA}.
The names \emph{Oracle} and \emph{Priestess} also denote what belongs to each computation: its stacks, its side of the machine
and its execution. The letters O and P, by contrast, name only the colours, the modes and the order of the two sides.
Appendix~\ref{app:glossary} separates stacks, colours, types, modes and computations.

\subsection{Data and Thunks: a CBPV Fragment}
\label{sec:os-cbpv}

Call-by-push-value (CBPV)~\cite{levy2003call}
distinguishes stored values from computations that run. We use a fragment of it and none of its metatheory
(\S\ref{sec:related-pl}): the fragment gives a vocabulary for when stored content is data and when it runs. It is used here for
semantic analysis, which differs from analysis at the level of the binary, and a rigorous development of the machine in the full
calculus is left for future work. The contract of
\S\ref{sec:os-kernel} is derived from the transition rules, and the types give the contract its meaning.
For now the same rules apply to stacks of both colours, and the restrictions come with S and T.

\begin{definition}[the CBPV fragment]
\label{def:cbpv}
CBPV has two classes of type. \textbf{Value types} $A$ are the types of things that are; \textbf{computation types} $B$ are the types
of things that run.
\[ A ::= \mathsf D \mid U\,B \qquad\qquad B ::= F\,A \]
Here $\mathsf D$ is the type of data, a data symbol or a word of them; $U\,B$ is the type of a \textbf{thunk}, a value that suspends a
computation of type $B$; and $F\,A$ is the type of a computation that produces a value of type $A$. The terms are
\[ \begin{array}{@{}l@{}}
  V ::= x \mid \mathsf d \mid \mathtt{thunk}\,M \\
  M ::= \mathtt{return}\,V \mid M\ \mathtt{to}\ x.\,N \mid \mathtt{force}\,V
\end{array} \]
with the typing rules, where a context $\Gamma$ assigns value types to variables,
\[ \begin{array}{@{}c@{}}
  \dfrac{}{\Gamma, x{:}A \vdash x : A}
  \qquad
  \dfrac{}{\Gamma \vdash \mathsf d : \mathsf D}
  \\[3ex]
  \dfrac{\Gamma \vdash M : B}{\Gamma \vdash \mathtt{thunk}\,M : U\,B}
  \qquad
  \dfrac{\Gamma \vdash V : U\,B}{\Gamma \vdash \mathtt{force}\,V : B}
  \\[3ex]
  \dfrac{\Gamma \vdash V : A}{\Gamma \vdash \mathtt{return}\,V : F\,A}
  \\[3ex]
  \dfrac{\Gamma \vdash M : F\,A \quad \Gamma, x{:}A \vdash N : B}{\Gamma \vdash M\ \mathtt{to}\ x.\,N : B}
\end{array} \]
\end{definition}

So $\mathtt{thunk}\,M$ suspends $M$ as a value, $\mathtt{force}\,V$ runs what the thunk $V$ suspends, $\mathtt{return}\,V$
produces $V$, and $M\ \mathtt{to}\ x.\,N$ runs $M$, binds its value to $x$ and continues with $N$. Only values are bound and
stored, so a computation is stored by thunking it.

\paragraph{The typed \mach{O-2-PDA}.} We attach a value type to each stored symbol to specify how it may be used. The type is
metadata, like a page-table permission, which states whether a page may be executed but does not change the page. The stack therefore
still holds symbols of the alphabet of \S\ref{sec:machine}. An element is $\mathsf d : \mathsf D$, a data symbol, or $\mathsf c :
U\,B$, a stored value that represents a suspended computation. A stack's content as a whole can be read as data, of type $\mathsf
D$, or as a thunk, a program for the Oracle. For a stack $s$ with content $w = \sigma(s)$, the Oracle instruction performs the
computation $\mathcal O(w) : F\,\mathsf D$, whose value for an accepted query is an answer in $O(w)$. CBPV gives the call two
readings of the same stored content.
\begin{itemize}
\item \textbf{Program reading, as a force.} The stored content represents a thunk of $\mathcal O(w)$, of type $U(F\,\mathsf D)$, and
  forcing the thunk runs that computation.
\item \textbf{Data reading, as a return/to.} The stored content is the data value $w : \mathsf D$. The call returns a word
  $r : \mathsf D$, and $\mathcal O(w)\ \mathtt{to}\ r.\,N$ binds it for the Priestess code $N$ that follows.
\end{itemize}
Both readings use the same instruction, which appends the answer to its operand. They differ in whether the stored content is the
program being run or the data supplied to a computation. Colours do not occur in the typing rules. They meet the types only in the
two breakings, which restrict the reading of the Oracle instruction by the colour of its operand. The typed machine thus poses the
question of this section: when is stored content a program for the Oracle, and when is it data?

\subsection{Storage Symmetry Breaking S}
\label{sec:os-s}

\begin{definition}[storage symmetry breaking S]
\label{def:s}
On a Priestess stack, every element has type $\mathsf D$, and the Oracle instruction is typed only as
$\mathcal O(\sigma(s))\ \mathtt{to}\ r.\,N$: a computation that returns a data value, which the code that follows binds. On an Oracle
stack, content may have type $U\,B$, and the Oracle instruction may be typed either way. The \textbf{\mach{O-2-PDA-S}} is the typed
\mach{O-2-PDA} with S.
\end{definition}

In words: whatever the Oracle produces on a Priestess stack is data, and only an Oracle stack can hold a program the Oracle runs.
S restricts only the reading of the Oracle instruction on a Priestess stack, constraining how the instruction's operand is typed
rather than whether the instruction exists; Appendix~\ref{app:further-blackbox} returns to what the other choice would cost. The Priestess's own instructions may still read and
write either colour, including the Oracle stacks on which it supplies task content and service replies. S fixes where the Oracle
may act, and so who is the subject of an Oracle call: on an Oracle stack the Oracle runs a program, and on a Priestess stack it
only answers, with a value that the Priestess program binds and uses. The underlying step is the same in both readings, so erasing the
types preserves the runs of \S\ref{sec:machine}: S by itself changes no run. What S breaks is $\varepsilon$. Exchanging the two
colours maps a system in which the Oracle runs a program to one in which content of the other colour is handed to it as a program,
and the typing admits only the first, so $\varepsilon$ is not an automorphism of typed machines. On the \mach{O-2-PDA-S} this
designation is a relabelling: erasing the types maps the machine whose subject colour is $O$ to the one whose subject colour is
$P$, and this paper fixes it to $O$ without loss. That S restricts no run is deliberate, not incidental. The contract of
\S\ref{sec:os-kernel} holds for every program and every run, an adversarial history among them, so the distinction between data and
program must be one that nothing checks and nothing bypasses: a designation derived from the colours asks nothing of a program, so
there is no well-formedness condition for a program to fail. The symmetry, not the machine, is what S breaks. It fixes how the content of each stack is read,
and once T is added, that reading decides where control may enter: $\mathtt{uret}$ is typed as a force, so it enters only an
Oracle stack (Definition~\ref{def:t}). For the Oracle, S declares the Priestess stacks non-executable.
In an implementation, storage boundaries can be enforced by a \textbf{view}, which limits the storage an execution can name, or a
\textbf{check}, which tests whether a named access is permitted. Page mappings provide views, permission bits and hardware tags
provide checks, and page tables combine the two. A check may also be static: a type system rules out disallowed accesses before
execution, as in Singularity's process isolation~\cite{hunt2007singularity}. \S\ref{sec:eclipse-safety} returns to these
protections and to the type bypasses that occur when the distinction between program and data is violated. On our machine the
Oracle is given only the content of its operand, and the locality lemma (Lemma~\ref{lem:locality}) restricts the effect of each
call. T, next, confines a whole execution to that boundary.

\subsection{Transition Symmetry Breaking T}
\label{sec:os-t}

S fixes \emph{which stack holds a program the Oracle runs}, and T fixes \emph{who decides the next transition}. Privilege enters here, and
it applies only where control changes hands. With S alone, each Oracle instruction makes one call and returns to the Priestess program
at $\ell+1$, or at $l_e$ on rejection: the Oracle determines the appended answer, and the program determines the continuation.
Alongside the mode in which Priestess code runs, T adds a mode in which successive Oracle calls on one task continue without an
intervening Priestess instruction until an answer ends them. The mode identifies which component is executing, so the rules can
allow different operations in the two modes, and an access is legitimate or not according to the mode that performs it. A mode
that is granted some operations and denied others is a privilege level.

\begin{definition}[transition symmetry breaking T]
\label{def:t}
Fix a symbol $\natural \in \Sigma_R$, the \textbf{yield symbol}, and a label $\ell_{\mathrm{trap}}$ of the program, the \textbf{trap
label}. Replace the configurations of the machine by configurations in two modes, P-mode and O-mode, and add one instruction,
$\mathtt{uret}\ u$, for $u$ a stack. The rules:
\begin{itemize}
\item \textbf{P-mode.} A configuration $\langle \mathsf P, \ell, \sigma \rangle$ runs the program. Its instructions are those of the
  machine without $\mathtt{halt}$. The Oracle instruction has a P-mode rule only on Priestess stacks, and continues at $\ell+1$, or at
  $l_e$ on rejection, whatever its answer. $\mathtt{uret}\ u$ enters $\langle \mathsf O, u, \sigma \rangle$.
\item \textbf{O-mode.} A configuration $\langle \mathsf O, u, \sigma \rangle$ runs no Priestess instruction. It calls the Oracle
  on the stack it entered on, $u$. With $w = \sigma(u)$,
  \[ \begin{array}{@{}l@{}}
  \langle \mathsf O, u, \sigma \rangle \to \langle \mathsf P, \ell_{\mathrm{trap}}, \sigma[u \mapsto w \cdot r] \rangle \\
    \qquad \text{for } r \in O(w) \text{ ending with } \natural, \\[0.5ex]
  \langle \mathsf O, u, \sigma \rangle \to \langle \mathsf O, u, \sigma[u \mapsto w \cdot r] \rangle \\
    \qquad \text{for } r \in O(w) \text{ not ending with } \natural.
  \end{array} \]
\end{itemize}

A transition from O-mode to P-mode is a \textbf{trap}. An answer that ends with $\natural$ is a \textbf{yield}, and the trap it
causes is a \textbf{yield trap}. On this machine every trap is one. Runs start in $\langle \mathsf P, 1, \sigma_0 \rangle$. An
\textbf{excursion} is a maximal O-mode segment, together with the $\mathtt{uret}$ that starts it and, if it is finite, the trap
that ends it. The \textbf{\mach{O-2-PDA-T}} is the \mach{O-2-PDA} with T, and the \textbf{\mach{O-2-PDA-TS}} is the
\mach{O-2-PDA-S} with T, on which $\mathtt{uret}\ u$ is typed as $\mathtt{force}\,u$ and so, by S, allowed only on an Oracle
stack.
\end{definition}

In words: the Priestess decides when the Oracle runs a program of its own, and on which stack; the Oracle's answers decide when
that program stops; and when it stops, control goes to one fixed place in the Priestess's program. A P-mode Oracle call, by
contrast, returns to the Priestess program whatever its answer, and that program may inspect the answer and branch with its own
instructions. The Priestess program stops when execution runs past its last label. T designates no colour but a mode: a privilege is
the set of operations a mode is granted, and the privileged mode is the one that runs the program. Reversing that choice is
therefore not a relabelling but a second set of rules, the reversed ordering of Appendix~\ref{app:further-reversed}. Like the
designation of S, T's is fixed for the machine and revised by no run; unlike S's, what T decides is not how content is read but
which of two executions a configuration may take, and so it is visible in runs. S restricts the reading and leaves every run
standing; T restricts the transitions.

In CBPV terms, P-mode on the \mach{O-2-PDA-TS} evaluates Priestess computations, whose Oracle instructions are
$\mathtt{return}$/$\mathtt{to}$ on Priestess stacks; $\mathtt{uret}\ u$ is $\mathtt{force}\,u$ for the thunk on $u$, and only
Oracle stacks hold thunks; and O-mode is the evaluation of that force, which the Oracle carries out. O-mode ends by forcing the
thunk $k = \mathtt{thunk}(\ell_{\mathrm{trap}}) : U(F\,\mathsf D)$, the Priestess code from the trap label. T
itself consults no type, so it applies to the untyped machine as well, and the two breakings can be taken in either order
(Figure~\ref{fig:construction}). Each leaves out what the other supplies: the \mach{O-2-PDA-S} has the types without the modes, so
nothing yet distinguishes the Oracle's execution from the Priestess's, and the \mach{O-2-PDA-T} has the modes without the types, so the
Oracle may be sent to run any stack, Priestess stacks included. Only the two breakings together make the Oracle run nothing but
Oracle stacks, in a mode of its own.

\paragraph{Good machines.} Definition~\ref{def:t} specifies O-mode rules only for queries the encoder accepts, and in O-mode no
program tests acceptance first. A machine, together with its Oracle and codec, is \textbf{good} if every query reachable in O-mode
is accepted, so that each excursion continues or yields. The decoder matters here, since it may write symbols the encoder does not
accept. A machine that is not good has a third outcome, a \textbf{fault}: the encoder rejects the history in O-mode, and
Definition~\ref{def:t} gives no successor. Faults belong to exception handling, outside this section. \S\ref{sec:eclipse} gives a
condition on the codec under which a machine is good. \S\ref{sec:an-obligations} treats a realization's handling of faults as
one more trap to $\ell_{\mathrm{trap}}$, a trap that leaves the store as it was; with this trap the return of control
(Proposition~\ref{prop:k5}) needs no good machine.

\begin{table*}[t]
\centering
\small
\begin{tabular}{@{}>{\raggedright\arraybackslash}p{0.2\textwidth}>{\raggedright\arraybackslash}p{0.38\textwidth}>{\raggedright\arraybackslash}p{0.36\textwidth}@{}}
\toprule
Transfer of control & In CBPV & Who decides what \\
\midrule
entry: P to O, $\mathtt{uret}\ u$ & the Priestess forces the thunk on $u$ & the Priestess decides when to enter, and which stack the Oracle sees \\
\rowrule
trap: O to P, a yield trap to $\ell_{\mathrm{trap}}$ & the Oracle call forces $k$, which reads what the excursion appended since the Priestess last wrote $u$ & the Oracle's answer decides when; where is fixed \\
\rowrule
return: the reply $v$ appended to $u$, then $\mathtt{uret}\ u$ & $k$ returns $v$ to the task, which binds it with $\mathtt{to}$ & the Priestess decides what $v$ is, and when the task receives it \\
\grouprule
a P-mode Oracle call, with no switch & $\mathcal O(\sigma(d))\ \mathtt{to}\ r.\,N$ binds the answer as a value & the Priestess's code: it continues at $\ell+1$, or at $l_e$ on rejection \\
\bottomrule
\end{tabular}
\caption{Where control passes between the Priestess and the Oracle on the \mach{O-2-PDA-TS}, as machine transitions and in CBPV.
Above the thick line, the mode switches: the Priestess enters a task, the task traps, and the Priestess returns to it. A trap and the
return after it are one call of $k$, $\mathtt{force}\ k\ \mathtt{to}\ x$. Below the thick line, a P-mode call passes control to the
Oracle and back without a switch.}
\label{tab:switches}
\end{table*}

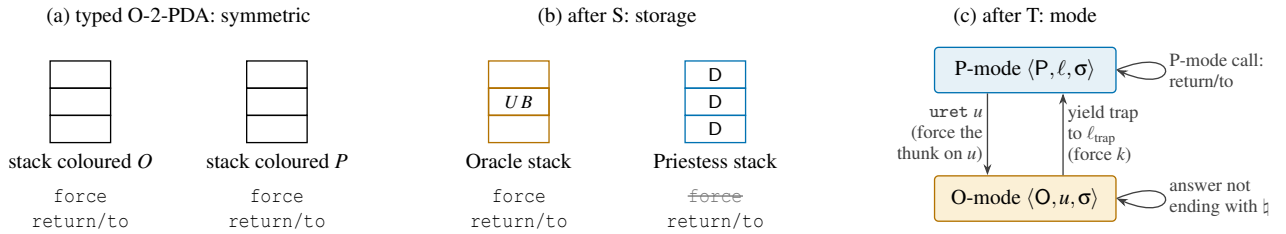
\begin{figure*}[t]
\centering
\begin{tikzpicture}[
  font=\footnotesize,
  cell/.style={draw, line width=0.5pt, minimum width=8mm, minimum height=3.6mm, inner sep=0pt, font=\scriptsize},
  lab/.style={anchor=north, inner sep=0pt},
  rd/.style={anchor=north, align=center, inner sep=0pt, font=\footnotesize\ttfamily},
  title/.style={anchor=base, inner sep=0pt},
  mode/.style={draw, line width=0.5pt, rounded corners=1.5pt, minimum width=24mm, minimum height=6mm, inner sep=2pt},
  note/.style={font=\scriptsize, inner sep=1.5pt},
  flow/.style={->, line width=0.5pt, black!75},
  >={Stealth[length=4.2pt, width=3pt]}]
  \begin{scope}[shift={(0,0)}]
    \foreach \x in {0.8,3.4} \foreach \i in {0,1,2} \node[cell] at (\x, 0.36*\i) {};
    \node[lab] at (0.8,-0.32) {stack coloured $O$};
    \node[lab] at (3.4,-0.32) {stack coloured $P$};
    \node[rd] at (0.8,-0.78) {force\\return/to};
    \node[rd] at (3.4,-0.78) {force\\return/to};
    \node[title] at (2.1,1.42) {(a) typed \mach{O-2-PDA}: symmetric};
  \end{scope}
  \begin{scope}[shift={(5.8,0)}]
    \foreach \i in {0,1,2} \node[cell, draw=ostroke] at (0.8, 0.36*\i) {};
    \node[cell, draw=ostroke] at (0.8, 0.36) {$U\,B$};
    \foreach \i in {0,1,2} \node[cell, draw=pcol] at (3.4, 0.36*\i) {$\mathsf D$};
    \node[lab] at (0.8,-0.32) {Oracle stack};
    \node[lab] at (3.4,-0.32) {Priestess stack};
    \node[rd] at (0.8,-0.78) {force\\return/to};
    \node[rd] at (3.4,-0.78) {\phantom{force}\\return/to};
    \node[rd, text=inkmuted] (pf) at (3.4,-0.78) {force};
    \draw[inkmuted, line width=0.5pt] ([yshift=-0.2pt]pf.west) -- ([yshift=-0.2pt]pf.east);
    \node[title] at (2.1,1.42) {(b) after S: storage};
  \end{scope}
  \begin{scope}[shift={(11.6,0)}]
    \node[mode, draw=pcol, fill=pfill] (P) at (1.7,0.78) {P-mode $\langle \mathsf P, \ell, \sigma\rangle$};
    \node[mode, draw=ostroke, fill=ofill] (O) at (1.7,-0.92) {O-mode $\langle \mathsf O, u, \sigma\rangle$};
    \draw[flow] ([xshift=-5mm]P.south) -- node[note, left, align=right] {$\mathtt{uret}\ u$\\(force the\\thunk on $u$)} ([xshift=-5mm]O.north);
    \draw[flow] ([xshift=5mm]O.north) -- node[note, right, align=left] {yield trap\\to $\ell_{\mathrm{trap}}$\\(force $k$)} ([xshift=5mm]P.south);
    \draw[flow] (P.east) .. controls +(0.85,0.36) and +(0.85,-0.36) .. node[note, right, align=left] {P-mode call:\\return/to} (P.east);
    \draw[flow] (O.east) .. controls +(0.85,0.36) and +(0.85,-0.36) .. node[note, right, align=left] {answer not\\ending with $\natural$} (O.east);
    \node[title] at (1.7,1.42) {(c) after T: mode};
  \end{scope}
\end{tikzpicture}
\caption{The two breakings. (a) The typed \mach{O-2-PDA}: on every stack the Oracle instruction may be read as a force
or as a return/to. (b) After S: on a Priestess stack (blue) every element has type $\mathsf D$ and the Oracle
instruction is only a return/to (force struck out); on an Oracle stack (amber) content may have type $U\,B$, and both
readings remain. (c) After T: $\mathtt{uret}\ u$ enters O-mode on the Oracle stack $u$, O-mode appends answers to $u$
until one ends with $\natural$, and the yield trap continues at $\ell_{\mathrm{trap}}$. P stands above O.}
\label{fig:st}
\end{figure*}

Table~\ref{tab:switches} shows each transfer of control as a machine transition and in CBPV. When forced, $k$ takes no argument. It
reads the request from the stack on which the excursion ran, and produces the reply. The transfers preserve the CBPV distinction:
Priestess code enters a task by forcing its Oracle-stack content through $\mathtt{uret}$, an Oracle call in O-mode forces only
$k$, and a P-mode call binds a data result. These restrictions follow from the transition rules, and therefore hold \emph{for
every program}. Ordered by privilege, the two sides form a two-point lattice with P above O (Figure~\ref{fig:st}), in which the
Priestess governs task entry and service, while O-mode execution is confined to its task. Before the breakings the colours were
unordered labels.

\paragraph{Minimal breakings.} Both breakings restrict as little as possible, and for the same
reason: a stronger form of either would rule out a placement of a task's program. S constrains how the Oracle instruction's operand is
typed; a form that took the instruction away on the side that treats content as data would leave no Workflow form and no hybrid, and would be a
different machine (Appendix~\ref{app:further-blackbox}). T adds a mode beside the one the program runs in and keeps the P-mode
call on Priestess stacks as it was, so a Workflow's calls are made as before. Neither is a check that a program could pass or fail: S designates how stored
content is read, and the side T privileges is derived from the transition rules, not chosen beside them. What that buys is the
quantifier: the contract of \S\ref{sec:os-kernel} holds for every program and every run, an adversarial history among them.
Proposition~\ref{prop:forced} (\S\ref{sec:os-agentic}) makes ``as little as possible'' precise: it states the criterion behind
the two breakings and shows that no other pair of exclusions meets it.

\subsection{The Kernel Contract}
\label{sec:os-kernel}

\textbf{Together, S and T let the operating system defy the program with a power from beyond.} The program runs, but the other
side decides which stack it runs on, when it runs, what replies it receives and where its control returns. The first four
guarantees are safety properties, which hold on every finite prefix of every run. The fifth, that control comes back at all, needs
a budget and a good machine, and follows the first four.

\begin{theorem}[the kernel contract, K1--K4]
\label{thm:contract}
For every Oracle, every program and every run:
\begin{itemize}
\item \textbf{K1 (an excursion writes only its stack).} An excursion changes no stack but the one it entered on, and that one only
  by appending answers.
\item \textbf{K2 (an excursion reads only its stack).} In O-mode the Oracle reads only the stack the excursion entered on: its own
  earlier answers, and what the Priestess wrote there, the initial content included.
\item \textbf{K3 (control returns only at the trap label).} Every transition from P-mode to O-mode is an $\mathtt{uret}$, and every
  transition from O-mode to P-mode is a trap to $\ell_{\mathrm{trap}}$.
\item \textbf{K4 (an answer is data to the Priestess).} A P-mode call changes only its operand, appending the answer, and continues
  at $\ell+1$, or at $l_e$ on rejection, whatever the answer. Every P-mode label in the run is $1$, $\ell_{\mathrm{trap}}$,
  $\ell + 1$, or a label written in $\Pi$, and no rule writes $\Pi$. So no answer, whether a P-mode call returns it or an excursion
  appends it, becomes a code address or instruction text.\qedhere
\end{itemize}
\end{theorem}

\begin{proof}
\leavevmode
\begin{itemize}
\item \textbf{K1.} In an O-mode transition from $\langle \mathsf O, u, \sigma\rangle$, each case of Definition~\ref{def:t} updates
  $\sigma$ only at $u$: it appends the answer to $\sigma(u)$. An excursion is a sequence of such transitions, so the claim follows
  by induction on it.
\item \textbf{K2.} The O-mode rule takes as input $\sigma(u)$ alone. By K1 and the locality lemma, $u$ changes only by O-mode
  appends on itself, or by Priestess instructions with $u$ as operand, the P-mode call among them.
\item \textbf{K3.} Immediate from the rules: only $\mathtt{uret}$ leads from P to O, and every O rule that leads to P goes to
  $\ell_{\mathrm{trap}}$.
\item \textbf{K4.} The first part is derived from the P-mode rule of Definition~\ref{def:t}. For the labels, induct on the run: each
  P-mode rule's successor label is $\ell + 1$ or a label written in the instruction, and after an excursion it is
  $\ell_{\mathrm{trap}}$ by K3. No rule modifies $\Pi$.\qedhere
\end{itemize}
\end{proof}

Each guarantee is derived from the transition rules (K2 also uses the locality lemma), and none consults a type. The contract therefore
holds on the \mach{O-2-PDA-T} already, and S gives it its meaning: what an excursion runs is a program, and what a P-mode call
returns is data. K1 and K2 together are spatial isolation. In Figure~\ref{fig:readings}(b) the answer
\texttt{file\_patch(notes.txt, teh \(\to\) the)} is an O-mode yield. By K3 it reaches the kernel only as a request at
$\ell_{\mathrm{trap}}$, and by K4, whatever it says, no rule makes it a code address. K4 also says what it means to receive an
answer as data: the call returns to the continuation fixed by the instruction. Whatever the Priestess program then does with the
answer (branch on its top symbol, or use it in a later query), it selects among the successors that its labels allow, or leaves
the machine stuck. The answer never grants itself an execution right. The program can still run an answer as part of an
O-mode task, but only explicitly, by copying it to an Oracle stack and entering that stack with $\mathtt{uret}$.

\begin{corollary}[the crossing invariant]
\label{cor:crossing}
Content passes between the two modes in two ways only.
\begin{itemize}
\item \textbf{O to P only by a force.} Content produced in O-mode reaches P-mode only through the force of $k$, the trap at
  $\ell_{\mathrm{trap}}$ that ends the excursion producing it. The trap hands over what the excursion appended.
\item \textbf{P to O only by a return.} Content that the Priestess produces reaches O-mode execution only by a return: the Priestess writes
  the stack, and $\mathtt{uret}$ enters it. A resumption with no write is a return with an empty reply, and the initial content of a
  stack is the first write.\qedhere
\end{itemize}
\end{corollary}

\begin{proof}
O to P: by K1 an excursion changes no stack but its own, and by K3 every return of control is at $\ell_{\mathrm{trap}}$, so if
P-mode reads a symbol that was produced in O-mode, that symbol was produced in an excursion that has already trapped. P to O: the Priestess
changes an Oracle stack only with that stack as operand, by the locality lemma, and enters O-mode only by $\mathtt{uret}$, by K3.
So if the Oracle reads in O-mode a symbol that the Priestess produced, that symbol was written before an $\mathtt{uret}$.
\end{proof}

\paragraph{The budget.} Liveness requires that something good eventually happens~\cite{lamport1977proving,alpern1985defining}. The
standard Oracle of computability is a set, whose membership test is a total function and returns at once. An Oracle that is a partial
function may never return on some inputs, and a real Oracle may fail and never return. On the machine every query has an answer,
so each Oracle transition returns. What need not end is an excursion, which runs as long as the Oracle never answers with
$\natural$. Real systems limit waiting with timeouts such as a watchdog, and for an Oracle the bound can be set on its internal
steps, which for a language model means a limit on output tokens. A budget bounds the excursion.

\begin{definition}[budget]
\label{def:budget}
Change Definition~\ref{def:t} as follows: $\mathtt{uret}\ u\ \beta$ takes a budget $\beta \in \mathbb{N}$ with $\beta \ge 1$, a configuration in O-mode carries the
remaining budget, and with $w = \sigma(u)$ and $r \in O(w)$,
\[ \begin{array}{@{}l@{}}
\langle \mathsf O, u, \beta, \sigma \rangle \to \langle \mathsf P, \ell_{\mathrm{trap}}, \sigma[u \mapsto w \cdot r] \rangle \\
  \qquad \text{if } r \text{ ends with } \natural \text{ (a yield trap)}, \\[0.5ex]
\langle \mathsf O, u, 1, \sigma \rangle \to \langle \mathsf P, \ell_{\mathrm{trap}}, \sigma[u \mapsto w \cdot r] \rangle \\
  \qquad \text{if } r \text{ does not end with } \natural \text{ (a budget trap)}, \\[0.5ex]
\langle \mathsf O, u, \beta, \sigma \rangle \to \langle \mathsf O, u, \beta - 1, \sigma[u \mapsto w \cdot r] \rangle \\
  \qquad \text{if } \beta > 1 \text{ and } r \text{ does not end with } \natural.
\end{array} \]
The kernel distinguishes the two kinds by the history: since answers are non-empty, the task stack ends with $\natural$ after a yield
trap and not after a budget trap.
\end{definition}

\begin{proposition}[K5: control returns]
\label{prop:k5}
On the machine with budgets, Theorem~\ref{thm:contract} and Corollary~\ref{cor:crossing} hold as stated, and on a good machine every
excursion entered with budget $\beta$ traps at $\ell_{\mathrm{trap}}$ within $\beta$ O-mode transitions, for every Oracle.
\end{proposition}

\begin{proof}
The new rules change only the remaining budget and the point at which an excursion traps. Every trap still goes to
$\ell_{\mathrm{trap}}$ and every O-mode transition still appends only to $u$, so the proofs of K1--K4 and of the corollary apply
unchanged. On a good machine every O-mode configuration has a successor, each O-mode transition either traps or lowers the budget
by one, and at budget $1$ every transition traps.
\end{proof}

K5 is the fifth item of the kernel contract. It bounds transitions, not elapsed time, so a watchdog on a single Oracle step is the
realization's concern. A request may now span several excursions ended by the budget. The budget is a slice measured in O-mode
transitions, which \S\ref{sec:eclipse} makes Steps or Frames, and preemption is never finer than one such transition. Without a
budget, control returns only at yield traps, and tasks in O-mode are scheduled cooperatively. Figure~\ref{fig:contract} sums up the
five items and how they depend on one another.

\begin{figure}[t]
\centering
\begin{tikzpicture}[x=1pt, y=1pt, font=\scriptsize,
  >={Stealth[length=3.6pt, width=2.6pt]},
  kO/.style={draw=ostroke, fill=ofill, line width=0.5pt, rounded corners=1.5pt, minimum width=88pt,
    minimum height=21pt, align=center, inner sep=1.5pt},
  kP/.style={draw=pcol, fill=pfill, line width=0.5pt, rounded corners=1.5pt, minimum width=182pt,
    minimum height=21pt, align=center, inner sep=1.5pt},
  cor/.style={draw=black!60, fill=white, line width=0.5pt, rounded corners=1.5pt, minimum width=182pt,
    minimum height=21pt, align=center, inner sep=1.5pt},
  dep/.style={->, line width=0.5pt, black!70},
  row/.style={anchor=east, align=right, inner sep=0pt},
  head/.style={anchor=base, inner sep=0pt, text=inkmuted}]
  \node[cor] (cor) at (141,0) {\textbf{the crossing invariant} (Corollary~\ref{cor:crossing})\\
    O to P only by a force; P to O only by a return};
  \node[head] at (94,-26) {what it touches};
  \node[head] at (188,-26) {where control goes};
  \node[kO] (k1) at (94,-43) {\textbf{K1} writes only its stack,\\and only by appending};
  \node[kO] (k2) at (94,-72) {\textbf{K2} reads only its stack};
  \node[kO] (k3) at (188,-43) {\textbf{K3} entered only by $\mathtt{uret}$,\\returns only at $\ell_{\mathrm{trap}}$};
  \node[kO, dash pattern=on 2pt off 1pt] (k5) at (188,-72) {\textbf{K5} with a budget $\beta$,\\returns within $\beta$ transitions};
  \node[kP] (k4) at (141,-105) {\textbf{K4} a call changes only its operand, by appending the answer,\\
    and continues at $\ell+1$, or at $l_e$, whatever the answer;\\every label is literal, so no answer is a code address};
  \node[row] at (44,-57.5) {an\\excursion\\(O-mode)};
  \node[row] at (44,-105) {P-mode};
  \draw[dep] (k1) -- (k2);
  \draw[dep] (k3.east) -- ++(7,0) |- (k4.east);
  \draw[dep] ([xshift=34pt]k1.north) -- ([xshift=34pt]k1.north |- cor.south);
  \draw[dep] ([xshift=-34pt]k3.north) -- ([xshift=-34pt]k3.north |- cor.south);
  \node[anchor=north, inner sep=0pt, text=inkmuted, align=center] at (141,-124)
    {K1, K3 and K4 follow from the rules of T, and K5 from those of the budget;\\
     K2 and the crossing invariant also use the locality lemma};
\end{tikzpicture}
\caption{The kernel contract, K1--K5 (Theorem~\ref{thm:contract}, Proposition~\ref{prop:k5}). Amber: what holds for an
excursion; blue: what holds in P-mode. The left column shows what a step may touch, and the right column shows where control
goes next. An arrow points from a result to one that uses it: K2 rests on K1, K4 on K3, and the crossing invariant on K1 and
K3. K5 (dashed) needs a budget and a good machine. The other four hold on every run, and on the \mach{O-2-PDA-T} already.}
\label{fig:contract}
\end{figure}

Once a trap has brought control back to P, Priestess code may inspect the task and copy its content from either colour of
stack to the other. The invariant does not forbid this, since it governs only when content becomes available to execution in the other mode.
Whatever program an Oracle stack carries, its traps reach the Priestess at the same trap label, and the Priestess chooses when the
program runs, for how long, and what replies it receives before each resumption. A Priestess program can thus schedule the
programs on Oracle stacks and handle their requests. This is the role of an operating system.

\paragraph{The operating system in CBPV.} With the types, this role is what sets such a program apart from every other. The
machine has two kinds of force, and only two: an $\mathtt{uret}$, which forces the thunk on an Oracle stack, and a trap, which forces
$k$, the code at the trap label. A Priestess program that governs the programs on Oracle stacks is the code at both ends: it issues
every force of a thunk on an Oracle stack, and it is the computation every trap forces. The Oracle calls of any other Priestess
code force nothing (S). \textbf{So an operating system is the Priestess code that forces the thunks on Oracle stacks and is forced
by them; all other Priestess code only returns and binds.}
Table~\ref{tab:switches} is this reading at each transfer of control.

\begin{definition}[operating system]
\label{def:os}
On the \mach{O-2-PDA-TS}, an operating system is a kernel together with a library. The kernel is a Priestess program that
enters Oracle stacks with $\mathtt{uret}$ and serves their traps with its code at the trap label, the \emph{service entry}. The
library declares the kernel's services to the tasks.

A \textbf{task} is a unit of work the system runs. The Oracle stack that the kernel enters on behalf of a task is its \textbf{task stack},
and the content of the task stack is the task's \textbf{history}. The library $\Lambda$ provides tasks with an interface to the kernel's
services. It consists of a request language $L_\Lambda \subseteq \Sigma_R^*\natural$ and a declaration $\delta_\Lambda$ placed on every
task stack. The declaration is the Oracle-side part of the library: it describes the available services. The corresponding
Priestess-side handlers are implemented in the kernel.

At a yield trap, the task's \textbf{request} is the content appended in O-mode since the Priestess last wrote the task stack. The
initial content counts as a write. What the kernel then writes on the task stack before entering it again is its \textbf{reply}. To
use the request, a kernel records the boundary after its last write and reads the content appended beyond it, and a kernel with
several tasks records which task stack it enters.

A yield trap is a \textbf{system call} when its request belongs to $L_\Lambda$. The kernel dispatches such a request to the
corresponding handler and handles other requests as errors.
\end{definition}

Theorem~\ref{thm:contract} with Proposition~\ref{prop:k5} is the \textbf{kernel contract}: whatever the Oracle, every kernel can
rely on these storage and control boundaries and, on a good machine, also on the return of control under a budget. Each kernel
builds its scheduling and services on them. The contract is a property of the machine's rules, not of a kernel: a kernel's
scheduling and services lie outside it and need a verification of their own, as seL4 and CertiKOS give their
kernels~\cite{klein2009sel4,gu2016certikos}. Table~\ref{tab:terms} in Appendix~\ref{app:glossary} collects the terms of this section.

\begin{definition}[program side, privileged side]
\label{def:sides}
Fix the \textbf{architecture}: the data of Definition~\ref{def:t} except the program (the alphabets, the stacks and their colours,
$\natural$ and $\ell_{\mathrm{trap}}$), together with the Oracle. A \textbf{system} on it has a program decomposed into \textbf{kernel
code} $K$ (kernel and library together) and \textbf{task code} $C$; \textbf{task data} $\Delta$, the initial contents of the stacks
of its tasks; and the declaration $\delta_\Lambda$, if any. A \textbf{family} is a set of systems that share the architecture, $K$ and
$\delta_\Lambda$, and differ only in $C$ and $\Delta$.
\begin{itemize}
\item The \textbf{program side} of a task is where its \textbf{program}, its executable part, sits: the \textbf{Oracle side} when the
  program is content of its task stack that the Oracle runs, and the \textbf{Priestess side} when it is task code $C$.
\item The \textbf{privileged side} is the side holding the privileged mode of T. On the \mach{O-2-PDA-TS} it is the Priestess.
\end{itemize}

An operating system \textbf{mediates} a task when the task runs in O-mode and the kernel serves its traps. The Oracle runs
Oracle-stack content only in O-mode (Definition~\ref{def:t}), so a task runs in O-mode exactly when its program is on the Oracle side,
and an operating system mediates a task exactly when the task's program side differs from the privileged side. The separation gives
the two sides different roles: the task supplies the computation, while the privileged side governs its access to services.
\end{definition}

In these terms, the task's program identifies the acting subject, and privilege identifies the trusted party. An operating system
stands between them when they occupy different sides. It defines itself through its opponent: the program acts, and the operating
system determines the conditions under which it may act. S makes the acting subject definable, for it fixes which stacks can hold a
program the Oracle runs; T gives that program a mode in which it runs and fixes the trusted party, the side holding the privileged mode.
The two roles differ in kind. The acting subject is designated: the system decides, task by task, the side whose storage
holds the task's program. The trusted party is not designated beside it but derived from the rules, since T gives the privileged mode
to the side that runs the program, and the two rows of Table~\ref{tab:quadrants} are two machines rather than two settings of one.
Reversing the
privileged side gives the dual: the Oracle is trusted, and the Priestess is confined. With a human as the Oracle, the human is
then the operating system of the computer program and decides the services available at the program's interaction boundaries.
Appendix~\ref{app:further-reversed} gives this dual a machine, on which each item of the contract has a counterpart with the roles
exchanged, and the bottom
row of Table~\ref{tab:quadrants} shows its two placements.

\subsection{Agents and Workflows on One Machine}
\label{sec:os-agentic}

With a language model as the Oracle, the two placements of a task's program are the two forms of agentic system.

\begin{definition}[Agent form, Workflow form, hybrid]
\label{def:forms}
\leavevmode
\begin{itemize}
\item \textbf{Agent form.} The kernel code is fixed, and the tasks of a family differ in the content of their task stacks, where the
  Oracle finds the program it runs. The program side is therefore the Oracle, while the privileged side is the Priestess. Their
  separation makes the kernel and its library an operating system. Oracle output is program by default: once appended to the history,
  it becomes part of what the Oracle runs next. The task meets the operating system through the force/return boundaries of
  Corollary~\ref{cor:crossing}.
\item \textbf{Workflow form.} The task's program is Priestess task code that calls the Oracle in P-mode, and no task stack is
  entered. Its program side and privileged side are both the Priestess, so the task runs directly on the machine, with no operating
  system mediating it. Here the Oracle supplies data to an existing program: its returned value is bound by $\mathtt{to}$ and used by
  the Priestess continuation. The program may inspect the result, branch on it, or explicitly interpret or promote it as program
  content (K4).
\item \textbf{Hybrid.} A task has program parts on both sides: Priestess task code may start Agent tasks, or an Agent may use a
  service implemented by task-specific Priestess code. The two parts retain their respective forms of interaction. The Agent part
  runs as a program served by the operating system. Its yield trap forces a service computation, and the kernel returns a value
  before resuming it with $\mathtt{uret}$. The Workflow part runs in P-mode and uses Oracle answers through the data reading,
  binding a returned value to the continuation with $\mathtt{to}$. Data becomes program there only by the Priestess's explicit
  interpretation or promotion (K4). A hybrid brings these two forms together, with each program part classified by the side on
  which it executes.\qedhere
\end{itemize}
\end{definition}

A shared client library linked into task code belongs to $C$: it executes as part of the task, even when many tasks use it. A
fixed body of code alone therefore does not determine the form. The form is determined by where the task's program executes and
how it meets the other side. This is a machine-level account of the distinction drawn in Building Effective
Agents~\cite{anthropic2024building}: a Workflow follows control paths written in Priestess task code, and an Agent carries its
program in the Oracle-side history and directs its process through successive requests. Table~\ref{tab:quadrants} sets the two
placements beside the reversed assignment of privilege.

\begin{table*}[t]
\centering
\small
\begin{tabular}{@{}>{\raggedright\arraybackslash}p{0.24\textwidth}>{\raggedright\arraybackslash}p{0.355\textwidth}>{\raggedright\arraybackslash}p{0.355\textwidth}@{}}
\toprule
Trusted party $\downarrow$ / acting subject $\rightarrow$ & Oracle & Priestess \\
\midrule
\textbf{Priestess} (this machine; an untrusted Oracle such as a language model) & \emph{Agent}: an Agent operating a computer; the Priestess is the operating system & \emph{Workflow}: non-OS; Oracle output is data \\
\rowrule
\textbf{Oracle} (privilege reversed; a trusted Oracle such as a human) & \emph{human--computer interaction}: a human using a computer, the counterpart of an Agent operating one; non-OS, the program is on the trusted side & \emph{automated systems}: a Priestess program runs by itself, confined, and the human decides at the switches; the human is the operating system \\
\bottomrule
\end{tabular}
\caption{The four combinations of acting subject and trusted party. \emph{Non-OS} marks the two cells where the acting
subject is the trusted party: the task's program is on the privileged side, so no operating system mediates it
(Definition~\ref{def:sides}). The bottom row has privilege reversed, on the machine of Appendix~\ref{app:further-reversed}.}
\label{tab:quadrants}
\end{table*}

\paragraph{Two readings, and form versus paradigm.} Every interaction is one transition read from two sides. From the Oracle side, an
excursion runs the history as a program and calls on the Priestess through a yield trap, and the Priestess supplies the service as
an Oracle would; from the Priestess side, $\mathtt{uret}$ starts a computation whose appended content the program inspects at
$\ell_{\mathrm{trap}}$. The same exchange is a service request in one reading and a returned result in the other, and in CBPV the
same transition is a force in one reading and a return in the other. A P-mode call is likewise a program step that returns a data
value for the Priestess, and a computation over supplied content for the Oracle. In both cases the two readings describe the same
machine transitions. The Agent and Workflow forms are therefore \textbf{implementation forms}: they concern where the task's
program is placed. How that program chooses its next action is its \textbf{design paradigm}. A Workflow may follow a plan the
model generated, and an Agent may carry out a predefined procedure. So
either form can realize either paradigm, given a suitable compiler. \S\ref{sec:eclipse} uses the distinction: because both forms run on this one machine, it states its models once, and
its cache model, its return bound and the safety boundary hold of a task whichever side the task's program lies on.

\paragraph{Simulation across the sides.} The distinction between form and paradigm becomes clearer when a computation is moved
from one side to the other. An Agent computation can be implemented by an \emph{interpreter} in Priestess code that reproduces the
Oracle loop in P-mode, or translated into such code by a \emph{code generator}. Either way, the code preserves each task's queries
and appended answers and directs completed excursions to the service entry. The result has Workflow form, and its P-mode calls preserve the data-result
behaviour of K4. Conversely, a Workflow can run as an Agent task given a \textbf{compiler}, which maps the Workflow with its
inputs to Oracle-side content with the same effect on the Priestess stacks when the kernel serves it. A family of Workflows then runs
as Agent tasks on a fixed kernel, and K1--K4 hold for them; K5 also holds when the kernel enters them with budgets on a good
machine. We run such computations with the same Oracle and compare them by their final contents on the original stacks, taken as sets of
possible results for a nondeterministic Oracle. In CBPV an interpreter replaces one force of a thunk by a loop of
$\mathtt{return}$/$\mathtt{to}$ calls on a Priestess stack, and a compiler goes the other way. The interpreter always exists:

\begin{proposition}[an interpreter for Agent tasks]
\label{prop:interpreter}
For every program $K$ of the \mach{O-2-PDA-TS} there is a program $K^\flat$ of the \mach{O-2-PDA-S} that has no $\mathtt{uret}$ and
runs the same tasks as $K$, making the same Oracle queries with the same answers.
\end{proposition}

The construction (Appendix~\ref{app:proofs}) renames each task stack to a fresh Priestess stack and replaces each $\mathtt{uret}$
by a loop that calls the Oracle there and jumps to $\ell_{\mathrm{trap}}$ when the top is $\natural$. The convention of non-empty
answers makes the test exact, and with budgets the loop maintains a counter. The interpreter makes the role of T explicit. The task
computation can be carried out by ordinary Priestess code, and in the interpreter, that code itself confines each task to its
history and routes completed excursions to the service entry. With T, by contrast, these boundaries belong to the machine's rules
and apply to every kernel. The interpreter does the work of an operating system within a single protection domain, and T is what
gives that work a privileged position. A Workflow can also be installed as a kernel service handler that runs the Workflow in P-mode with
its original queries and answers when an Agent issues the request; the task is then a hybrid.

\paragraph{Where a harness and a scaffold sit.} The interpreter is a program running inside the machine. An ordinary Agent harness
sits at another level: its host code provides the machine on which the Agent runs. In the familiar model--tool loop, model
execution corresponds to O-mode, a tool request triggers a yield trap, and the tool's reply is returned to the task before
execution resumes. The task history carries the Oracle-side program, and the dispatcher and its tool handlers play the kernel and
the library. The harness therefore exhibits the Agent form, although one host program combines the machine, the kernel and part of
the library. ArchNights separates the three. A \textbf{scaffold} is the system library together with the programs a task may run,
that is, the kernel-and-library side of the machine ($\delta_\Lambda$, $L_\Lambda$ and the handlers). It
exposes its operations to the task as thunks: the declaration names them, a request forces one, and the Priestess evaluates it in P-mode.

\paragraph{Placing the Workflows and Agents.} The paradigm is the Workflow paradigm when control follows code paths written in
advance and the Oracle's output at most selects among them, and it is the Agent paradigm when the Oracle's output determines which
steps there are. Prompt chaining, routing, parallelization and evaluator-optimizer~\cite{anthropic2024building} have Priestess
task code that fixes the stages and the control structure, so they have Workflow form. Their paradigm is also the Workflow
paradigm, since routing and the evaluator's verdict select among paths written in advance. In orchestrator-workers, a P-mode
program asks the Oracle for a decomposition and starts workers from it, which gives Workflow form with the Agent
paradigm~\cite{anthropic2024building}. Autonomous agents carry the program in the task history, and model iteration selects
actions and tool requests that the harness dispatches, which gives Agent form with the Agent paradigm. Hybrids, such as agents
calling Workflow tools, Workflow nodes running Agent tasks and harness skills that build task DAGs, have task-specific program
parts on both sides.

\paragraph{Why these two breakings.} This section chooses one criterion and derives the rest from it: both placements of a task's
program exist and stay apart, and what the machine guarantees holds for every program and every Oracle, with no condition asked of a
program. An Oracle execution can be classified by its operand colour, Oracle or Priestess, and by
whether a yielded answer forces the service entry or completes. Two of the four combinations are admitted: an Oracle stack whose
yield forces the service entry, which is O-mode, and a Priestess stack whose answer completes, which is the P-mode call. The first
is the only admitted one in which the machine forces the service entry, and the second is the only admitted one in which the
Priestess uses the Oracle on its own data. \textbf{Each breaking excludes one of the other two: S excludes the one in which
Priestess content would be run, and T excludes the one in which the Oracle's program would complete without reaching the service
entry.} T excludes an Oracle stack whose answer completes, and Proposition~\ref{prop:interpreter} realizes the same behaviour on a
Priestess stack; S excludes a Priestess stack whose answer forces the service entry, since Priestess content is never forced. The
criterion leaves no other choice of the two exclusions.

\begin{proposition}[the breakings are forced]
\label{prop:forced}
Give the \mach{O-2-PDA} the modes of Definition~\ref{def:t}, but admit $\mathtt{uret}$ only on the stacks whose colour lies in a
set $U$, and the P-mode Oracle call only on those whose colour lies in a set $C$, with no condition on the program. A task can run in
Agent form only if $U \ne \emptyset$, and in Workflow form only if $C \ne \emptyset$. For every program, every Oracle and every run,
no P-mode call reads an answer appended in O-mode if and only if $U \cap C = \emptyset$. So exactly two such machines meet the
criterion: the one with $U = \{O\}$ and $C = \{P\}$, which is S together with the P-mode rule of T, and the same machine with the
colours exchanged.
\end{proposition}

The proof is in Appendix~\ref{app:proofs}. A P-mode call that read an excursion's answers would hand the Oracle's program to a
Workflow as its data, and the two placements would no longer stay apart; the condition is the \emph{In} item of the safety
boundary (Theorem~\ref{thm:boundary}) in P-mode. No weaker restriction of this kind meets the criterion, and no stronger one keeps
both placements. A condition on the
program, such as the safety of Definition~\ref{def:safe}, is weaker, but it asks something of every program, which the criterion
rules out. The rest of T follows from the same criterion. Its mode is what the Agent form needs, since the Oracle runs a program by being
queried again and again with no Priestess instruction between the queries; and its one trap label is what K4 needs for every
Oracle, since an exit whose label an answer could select would make that answer a code address. What the section derives is
therefore the breakings from the criterion, the contract from the rules (Theorem~\ref{thm:contract}), and the operating system of
Definition~\ref{def:os} as the code at the two forces that the breakings leave. A harness's dispatcher is a kernel because it is
the code at those two forces, not by definition.

So the contract is established as an invariant of the machine, and Agents and Workflows are the two placements of a task's program
under it. We leave to \S\ref{sec:eclipse} what depends on the Oracle's interior, the cache first of all.

\section{Eclipse: Between Oracle and Priestess}
\label{sec:eclipse}

A solution found for one form of agentic system usually needs to be translated before it applies to the other. \S\ref{sec:os} puts
both forms on one machine, where the Oracle and the Priestess stand in a twin relation, like the two stars of VFTS~352, an
overcontact binary~\cite{almeida2015discovery}. Taken as one whole, the pair lets a problem be modelled once, and the model then applies to either form through the link
between the two sides. So far the Oracle has been a black box, and every result of \S\ref{sec:os} holds for every Oracle. But a
cache of the Oracle's computation is hard to analyse from outside. Turing went no further into the nature of his oracle. We go one
step further: this section assumes that the Oracle is \emph{autoregressive inside}, the paper's one assumption about the Oracle's
interior. The machine's queries already have the autoregressive form as long as the
program only appends to them
(Proposition~\ref{prop:ar-form}). With an autoregressive interior the two processes can be linked through the codec
(\S\ref{sec:eclipse-codec}), and the link underlies a model of the cache (\S\ref{sec:eclipse-memory}) and an analysis of scheduling
(\S\ref{sec:eclipse-control}) that hold for both forms. The assumption is needed for these and for nothing else. The safety
boundary of \S\ref{sec:eclipse-safety} does not use it: like the contract, the boundary is an invariant of the machine. An Oracle
that is not autoregressive inside therefore loses the cache model, and with it the cache affinity the schedulers weigh, but keeps
the contract and the safety boundary. This section builds the model and the analysis. Appendix~\ref{app:eclipse} gives their details, and Appendix~\ref{app:further} applies
them to the encoders, caches, schedulers and defences the field has built.

\subsection{Encoder and Decoder, Step and Frame}
\label{sec:eclipse-codec}

The codec of Definition~\ref{def:oracle} links the two processes. Earlier work often assumes the Oracle's alphabets are the
machine's. We keep them apart, so that the codec is fixed first and the machine-side alphabets are what it accepts and produces.
In this section $w$ and $v$ are machine-side words, and $x$, $y$ and $z$ are Oracle-side ones; $Q_t$ is the history at the $t$-th
call, as in Proposition~\ref{prop:ar-form}; $\mathrm{lcp}(x, x')$ is the longest common prefix of two sequences, and $|x|$ is
the length in symbols. We write $E(w)$ only for $w \in \mathrm{dom}(E)$.

\begin{definition}[internally autoregressive Oracle]
\label{def:int-ar}
An Oracle $(A, E, D)$ (Definition~\ref{def:oracle}) is \textbf{internally autoregressive} if it is given
\begin{itemize}
\item an inclusion $\Gamma_R \subseteq \Gamma_Q$ and a \textbf{terminator} $\tau \in \Gamma_R$, and
\item a \textbf{step relation} $\nu \subseteq \Gamma_Q^* \times \Gamma_R$, its one internal step,
\end{itemize}
such that its answers are what the steps generate,
\[ \begin{array}{@{}l@{}}
A(x) = \{\, y_1 \cdots y_k : y_{i+1} \in \nu(x\,y_1 \cdots y_i) \text{ for } i < k, \\
  \qquad y_k = \tau,\ y_i \ne \tau \text{ for } i < k \,\},
\end{array} \]
and its decoder works symbol by symbol, $D(y_1 \cdots y_k) = D(y_1) \cdots D(y_k)$, with $D(\tau) = \natural$ and $\natural$ in no
other $D(y)$.
\end{definition}

We call an Oracle that is internally autoregressive an $A_{\mathrm{AR}}$ Oracle, after the $A$ of its interior, and we call one
that is not internally autoregressive an $A_{\mathrm{nAR}}$ Oracle. One machine may consult both kinds, and \S\ref{sec:query}
asks which may run in O-mode.
The inclusion is needed because the Oracle reads its own output. Without it, a step could produce an illegal input for the next.
The terminator decodes to the yield symbol, so an answer ends exactly where an excursion yields. Inside an answer the Oracle's
history grows by one symbol per step, but across calls it need not, as shown below. A step relation depends only on the concatenation of
the query and what has been generated, so internal autoregression constrains the answer relation only where one query extends
another by generated symbols, and the answers do not determine which step relation produced them
(Proposition~\ref{prop:steps-invisible}).

\paragraph{Step and Frame.} A \textbf{Step} is one internal step: the Oracle appends one symbol of $\Gamma_R$ to its history, and
the machine receives that symbol's decoding. A \textbf{Frame} is the sequence of Steps from a query to the terminator; it is an
answer of $A$, and its decoding ends with $\natural$. The \textbf{Frame interface} answers $w$ with a decoded Frame,
$O(w) = \{D(y) : y \in A(E(w))\}$, the answer map of Definition~\ref{def:oracle} itself; the \textbf{Step interface} answers $w$
with a decoded Step, $\{D(z) : z \in \nu(E(w))\}$. The Frame interface is an Oracle, and the Step interface is one when no symbol
of $\Gamma_R$ decodes to $\lambda$, as the convention of non-empty answers (\S\ref{sec:machine}) asks. In O-mode the Frame
interface makes one transition per Frame, which traps, while the Step interface makes one per Step and traps at the Step that
decodes to $\natural$.

The two interfaces differ in how often the Oracle encodes. Through the Step interface, two consecutive Steps compute
\[ z_1 \in \nu\bigl(E(Q_t)\bigr), \qquad z_2 \in \nu\bigl(E(Q_t \cdot D(z_1))\bigr). \]
After the first Step the Oracle's internal history is $E(Q_t) \cdot z_1$, but the second encodes the longer machine-side word, and
that encoding need not extend $E(Q_t)$. If $a$, $b$ and $ab$ are all Oracle symbols and the encoder prefers the longest match, the
history $a$ is encoded $[a]$ and the history $ab$ is encoded $[ab]$: appending $b$ does not extend $[a]$ but replaces it, and the
Oracle's internal state, built for $[a]$, has to back up. The Frame interface encodes only at the start of a Frame, and inside a Frame
the Oracle extends its own history.

\begin{definition}[loss and backtrack]
\label{def:loss}
For machine-side words $w$ and $v$, the \textbf{loss} of extending $w$ by $v$ is
\[ \mathrm{loss}_E(w, v) = |E(w)| - \bigl|\mathrm{lcp}\bigl(E(w), E(wv)\bigr)\bigr|, \]
the number of symbols of $E(w)$ that the encoding of $wv$ does not retain. For a set $G$ of pairs $(w, v)$, the \textbf{backtrack}
at $G$ is $L_E^G = \sup_{(w, v) \in G} \mathrm{loss}_E(w, v)$, with $\sup \emptyset = 0$. Two sets are used: the \textbf{Step
pairs}, where $w$ is any word and $v$ is the decoding of one Step, and the \textbf{Frame pairs}, where $w$ is empty or ends with
$\natural$ and $v$ ends with $\natural$.
\end{definition}

These are properties of the codec, so every bound on them holds on every run. A run's pairs lie among these sets: its Step pairs
always, and its Frame pairs when the kernel's writes end with $\natural$. The loss is a Step-level form of
bounded variation~\cite{gastaldi2025foundations} (Appendix~\ref{app:eclipse}).

\paragraph{Separator.} A condition on one symbol removes the backtrack where a Frame ends. The yield symbol $\natural$ is a
\textbf{separator} for $E$ if $E$ first splits its input at every $\natural$, emits the terminator $\tau$ for each $\natural$, and
encodes each piece independently, by an encoder that maps the empty word to the empty sequence and never emits $\tau$. The yield symbol
then marks the boundary of a Frame on both sides. Only one consequence of a separator is used below, and it follows at once from the
definition.

\begin{definition}[factorization at the yield symbol]
\label{def:factor}
$E$ \textbf{factorizes at} $\natural$ if $E(\lambda) = \lambda$, no word without $\natural$ has $\tau$ in its encoding, and for every
word $w\,\natural\,v \in \mathrm{dom}(E)$ the words $w$, $w\,\natural$ and $v$ lie in $\mathrm{dom}(E)$ and
\[ E(w\,\natural\,v) = E(w) \cdot \tau \cdot E(v). \]
\end{definition}

\paragraph{Frame-faithfulness.} Whether the Oracle's own history also follows the machine's is a property of the Oracle with
its codec, not of the encoder alone. Call an Oracle with its codec \textbf{Frame-faithful} if $E(D(y)) = y$ for every Frame $y$ it
can generate: the encoder writes back, symbol for symbol, what the Oracle generated. This follows from neither internal
autoregression nor factorization. For example, a model can be internally
autoregressive and still sample a tokenization that the encoder would not produce; it is then not Frame-faithful.
Appendix~\ref{app:eclipse-scope} collects where Frame-faithfulness holds, how it fails, and which results need it, and
Appendix~\ref{app:eclipse-reasoning} describes in these terms the ways a realization keeps, drops or hides a model's reasoning. The
theorem has a part for each side.

\begin{theorem}[Frame synchronization]
\label{thm:run-sync}
Let $E$ factorize at $\natural$, and let $w$ be empty or end with $\natural$.
\begin{enumerate}
\item \emph{The encoder.} $E(wv) = E(w) \cdot E(v)$ for every $v$ with $wv \in \mathrm{dom}(E)$; in particular
  $E(w) \sqsubseteq E(wv)$.
\item \emph{The Oracle.} If the Oracle answers $w$ with a Frame $y$, and $w\,D(y) \in \mathrm{dom}(E)$, its own history
  $E(w)\,y$ is the encoding $E(w\,D(y))$ of the machine-side history exactly when $E(D(y)) = y$, and otherwise neither is a
  prefix of the other; if $w\,D(y) \notin \mathrm{dom}(E)$, the question does not arise, and the encoder rejects the history the
  Frame would extend. So after every Frame on such a history, the history of a Frame-faithful Oracle is the encoding of the
  machine's.\qedhere
\end{enumerate}
\end{theorem}

\begin{proof}
(1)~If $w$ is empty, $E(w) = \lambda$. Otherwise $w = w'\natural$, and factorization gives $E(w) = E(w')\,\tau\,E(\lambda) =
E(w')\,\tau$ and $E(wv) = E(w')\,\tau\,E(v)$. (2)~By (1) with $v = D(y)$, $E(w\,D(y)) = E(w)\,E(D(y))$, which is $E(w)\,y$
exactly when $E(D(y)) = y$. Otherwise, since $y$ and $E(D(y))$ both end with $\tau$ and contain it nowhere else, neither is a
prefix of the other.
\end{proof}

Part (1) holds whatever the Oracle generates: the encoding follows the machine-side history and never backs up across a
$\natural$. Three of its consequences are used below. \emph{(i) The loss stays in the open Frame.}
Write a history as $w = w'w''$, where $w'$ is empty or ends with $\natural$, and $w''$, the unfinished Frame, contains no
$\natural$. Then $E(w) = E(w')\,E(w'')$, and $E(w')$ is a prefix of $E(wv)$ for every $v$, so $\mathrm{loss}_E(w, v) \le
|E(w'')|$. In particular, a Frame pair loses nothing. \emph{(ii) Prefix trees.} If $h_1 \sqsubseteq h_2 \sqsubseteq \cdots$ are
histories that are empty or end with $\natural$, then $E(h_1) \sqsubseteq E(h_2) \sqsubseteq \cdots$. Histories grown from a
common such history $h$ share $E(h)$, so their encodings form a prefix tree. \emph{(iii) The completed Frames agree.} When the
Oracle answers $w = w'w''$ with a Frame $y$, its own history $E(w)\,y$ and the encoding $E(w\,D(y))$ of the machine-side history
both begin with $E(w')$, so they agree on every completed Frame and can differ only in the last, $E(w'')\,y$ against
$E(w''\,D(y))$. Part (2) says when they agree on that one too.

A separator describes a tokenizer that splits its input at a registered special token before encoding the pieces. Whether a given
tokenizer does so is a property of its implementation, to be checked per tokenizer; the check need not analyse BPE merges,
which happen only inside the pieces. Inside an open Frame the loss remains, and how large it is depends on the encoder. BPE as
implemented, with a proper dictionary, has a finite lookahead, which depends on the dictionary alone and beyond which a prefix of
the tokenization is stable~\cite{berglund2023formalizing}. Its Step backtrack is therefore finite and depends on the
dictionary. An encoder with only two symbols, by contrast, can have an unbounded one.
Appendix~\ref{app:further-encoders} assesses the work on tokenizers against these two conditions, factorization at $\natural$ and the
backtrack inside an open Frame. Figure~\ref{fig:sync} draws the two parts of Theorem~\ref{thm:run-sync}.

\begin{figure}[t]
\centering
\begin{tikzpicture}[x=1pt, y=1pt, font=\scriptsize,
  >={Stealth[length=3.6pt, width=2.6pt]},
  seg/.style={draw=black!55, fill=white, line width=0.5pt, minimum height=13pt, inner sep=1pt, anchor=west},
  open/.style={seg, dash pattern=on 2pt off 1pt},
  byO/.style={seg, draw=ostroke, fill=ofill},
  row/.style={anchor=east, inner sep=0pt, text=inkmuted},
  title/.style={anchor=base west, inner sep=0pt, font=\footnotesize},
  enc/.style={->, line width=0.5pt, black!70},
  lab/.style={font=\tiny, inner sep=1pt}]
  \node[title] at (0,12) {(a) one way, by the encoder (Theorem~\ref{thm:run-sync}(1))};
  \node[row] at (44,0) {machine};
  \node[row] at (44,-26) {encoding};
  \node[seg, minimum width=60pt] (m1) at (48,0) {$w_1\,\natural$};
  \node[seg, minimum width=60pt] (m2) at (112,0) {$w_2\,\natural$};
  \node[open, minimum width=46pt] (m3) at (176,0) {$w''$};
  \node[seg, minimum width=60pt] (e1) at (48,-26) {$E(w_1)\,\tau$};
  \node[seg, minimum width=60pt] (e2) at (112,-26) {$E(w_2)\,\tau$};
  \node[open, minimum width=46pt, pattern=north east lines, pattern color=black!14] (e3) at (176,-26) {$E(w'')$};
  \foreach \i in {1,2,3}
    \draw[enc] (m\i.south) -- node[lab, right] {$E$} (e\i.north);
  \foreach \x in {110,174}
    \draw[black!45, line width=0.5pt, dash pattern=on 0.8pt off 1.2pt] (\x,9) -- (\x,-35);
  \node[anchor=north, inner sep=0pt, text=inkmuted, align=center] at (135,-38)
    {the encoding never backs up across a Frame boundary;\\only the open piece $w''$ may be encoded anew};
  \begin{scope}[shift={(0,-80)}]
    \node[title] at (0,12) {(b) the other way, by the Oracle (Theorem~\ref{thm:run-sync}(2))};
    \node[row] at (44,0) {Oracle};
    \node[row] at (44,-26) {machine};
    \node[row] at (44,-52) {encoding};
    \node[seg, minimum width=60pt] at (48,0) {$E(w)$};
    \node[byO, minimum width=60pt] (y) at (112,0) {$y$, generated};
    \node[seg, minimum width=60pt] at (48,-26) {$w$};
    \node[byO, minimum width=60pt] (dy) at (112,-26) {$D(y)$};
    \node[seg, minimum width=60pt] at (48,-52) {$E(w)$};
    \node[seg, minimum width=60pt] (edy) at (112,-52) {$E(D(y))$};
    \draw[enc] (y.south) -- node[lab, right] {$D$} (dy.north);
    \draw[enc] (dy.south) -- node[lab, right] {$E$} (edy.north);
    \draw[decorate, decoration={brace, amplitude=2.5pt}, line width=0.5pt] (177,6.5) -- (177,-58.5)
      coordinate[midway, xshift=4pt] (br);
    \node[anchor=west, align=left, inner sep=0pt] at (br) {equal exactly\\when\\$E(D(y)) = y$};
  \end{scope}
\end{tikzpicture}
\caption{Frame synchronization, both ways. (a) When the encoder factorizes at $\natural$, the
encoding of a history is the concatenation of the encodings of its pieces up to each $\natural$ (answers and the kernel's records alike, each ending with $\tau$),
followed by that of the open piece $w''$. Only the open piece may be re-encoded differently as the history grows. (b)
After the Oracle answers a history $w$ (empty or ending with $\natural$) with the Frame $y$, its own history is
$E(w)\,y$, while the machine records $w\,D(y)$, whose encoding is $E(w)\,E(D(y))$. The two are equal exactly when the
Frame re-encodes to itself, and otherwise they diverge inside the Frame. Amber: what the Oracle generated.}
\label{fig:sync}
\end{figure}
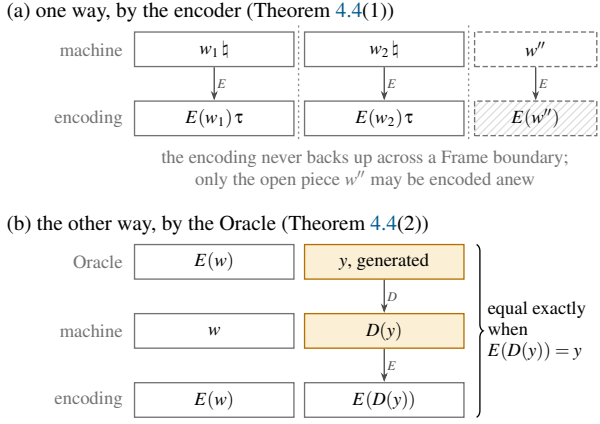

So a Step advances both autoregressive processes by one internal transition of the Oracle, and may make the encoding back up, but
only inside the open Frame; a Frame advances both by a round of transitions that ends in the single separator, and the encoding never
backs up across it. \textbf{The two processes are synchronized Frame by Frame}, one way by the encoder when it factorizes at
$\natural$ (Theorem~\ref{thm:run-sync}(1)), and the other by the Oracle when it is Frame-faithful and answers a history that ends
at a Frame boundary (Theorem~\ref{thm:run-sync}(2)).
Each query is encoded afresh, so a difference in the last Frame never outlasts the next query. The decoder needs no condition of its
own, because it works symbol by symbol (byte-level vocabularies are the exception).

\paragraph{Closure and observability.} Iteration in O-mode continues only while each history is a word the encoder accepts. A
machine on which this always holds is what \S\ref{sec:os} calls a good machine, and a condition on the alphabets gives one.

\begin{proposition}[closure]
\label{prop:closure}
Suppose $\Sigma_R \subseteq \Sigma_Q$ and $\mathrm{dom}(E) = \Sigma_Q^*$, and that every excursion starts on a history in $\Sigma_Q^*$.
Then, whatever the Oracle, every query made in O-mode is accepted, and the machine is good: every reachable O-mode configuration has
the successors Definition~\ref{def:t} gives it.
\end{proposition}

\begin{proof}
By induction on the excursion. It starts on a history in $\Sigma_Q^*$. Each O-mode transition appends an answer in
$\Sigma_R^* \subseteq \Sigma_Q^*$, so the history stays in $\Sigma_Q^* = \mathrm{dom}(E)$.
\end{proof}

The starting condition is a discipline of the kernel: it holds if the initial content of every task stack and
everything that the kernel writes onto a task stack lie in $\Sigma_Q^*$. Proposition~\ref{prop:codec-good} says what fails without the inclusion and how a
codec repairs it. The Priestess side sees the Oracle only through the interface it calls. A property of the Oracle's computation
is \textbf{observable} through an interface if the answers returned through that interface determine it. This notion is defined for the Oracle,
not the observable steps of \S\ref{sec:machine}. Through the Step interface the number of Steps is observable, one answer per
Step, but through the Frame interface it need not be.

\begin{proposition}[Steps are not visible in Frames]
\label{prop:steps-invisible}
Suppose $\Gamma_R$ contains distinct symbols $t$, $t_1$, $t_2$, other than $\tau$, with $D(t) = D(t_1)\,D(t_2)$. Then for every query
$q$ there is an internally autoregressive Oracle that answers $q$ with the decoded Frame $D(t)\,\natural$ both after two Steps and after
three. So, for some internally autoregressive Oracles, an observer that sees only decoded Frames cannot tell how many internal
Steps they took.
\end{proposition}

\begin{proof}
Let $x = E(q)$, and let $\nu(x) = \{t, t_1\}$, $\nu(x\,t) = \{\tau\}$, $\nu(x\,t_1) = \{t_2\}$, $\nu(x\,t_1 t_2) = \{\tau\}$, and
$\nu(z) = \{\tau\}$ for every other $z$. Then $A(x)$ contains $t\,\tau$ and $t_1 t_2\,\tau$, and both decode to $D(t)\,\natural$.
\end{proof}

The Oracle of the proof is not Frame-faithful, nor is any Oracle that answers a query with the same decoded Frame after
different numbers of Steps: a Frame-faithful Oracle takes exactly $|E(D(y))|$ Steps on a Frame $y$, so an observer who knows the
encoder derives the number of Steps from the decoded Frame. Without Frame-faithfulness, an Oracle's internal transitions need not be
observable through the Frame interface, unless the Oracle exposes more than the Frame, such as a count of its Steps.

For a language model, the encoder is the tokenizer and the decoder is the detokenizer. A Step is one token, or one
sequence of up to $k$ tokens with multi-token prediction~\cite{gloeckle2024better}. An inference engine's \texttt{generate} interface can expose Steps; most
agents, by contrast, use chat interfaces (\texttt{/v1/chat/completions}, \texttt{/v1/responses} and \texttt{/v1/messages}),
which wrap the history in a template on the provider's side. What remains observable through them is the Frame and whatever metadata
the provider reports, such as a token count.
Proposition~\ref{prop:steps-invisible} says what is lost without such metadata.

\subsection{Two Caches, One Hierarchy}
\label{sec:eclipse-memory}
\label{sec:eclipse-storage}

In Agent applications the client and the server each optimize their cache and their computation. To optimize them together, both
sides need one model of the problem and a fair estimate of the other side's state. On the \mach{O-2-PDA} the Oracle and the
Priestess are different processors with different storage. We model both as systems of one kind, connect them through the codec
and the Frame, and read the result on the KV cache. The unit of computation is one transition of the machine, from $Q_t$ to
$Q_{t+1}$: one decode step at the Step interface, a Frame at the Frame interface. A miss of the Oracle's store is rematerialized
by a prefill, while a hit reuses the KV cache.

\paragraph{Stores as functions.} We define memory, cache and backing store as functions, in the types of \S\ref{sec:os}. A store
maps keys, which need not be linear addresses, to values. One operation builds everything else: a store placed above a function
answers the keys it holds and passes the others down.

\begin{definition}[store, hierarchy, system]
\label{def:store}
Fix a set $\mathcal K$ of keys and a set $\mathcal V$ of values, of type $\mathsf D$.
\begin{itemize}
\item A \textbf{store} $C$ costs $t_C$ per lookup and holds a \textbf{working set} $W_C : \mathcal K \rightharpoonup \mathcal V$, a
  finite partial map.
\item For a store $C$ and a function $f$ on keys, $C$ \textbf{over} $f$ is the function
  \[ (C \rhd f)(k) = \begin{cases} W_C(k) & \text{if } k \in \mathrm{dom}(W_C), \\
     f(k) & \text{otherwise.} \end{cases} \]
  The first case is a \textbf{hit}; the second is a \textbf{miss}, where the store holds a thunk of the computation below it and the
  access forces it. The cost is
  \[ c(C \rhd f, k) = t_C + [\,k \notin \mathrm{dom}(W_C)\,]\; c(f, k), \]
  where $[P]$ is $1$ if $P$ holds and $0$ otherwise.
\item A \textbf{hierarchy} is $C_1 \rhd C_2 \rhd \cdots \rhd C_n \rhd f$, grouped to the right. The function that it defines is its
  \textbf{access}, and $f$ is its \textbf{bottom}: a \textbf{backing store}, which holds every key it is asked for, or a
  \textbf{rematerialization}, a computation that produces the value again. The level directly above a backing store is the
  \textbf{memory}. A hierarchy whose levels act on different key sets may include a \textbf{translation} level between them: a map
  $\kappa$ from the keys above it to the keys below, with $W = \emptyset$ and $t = 0$, which changes no value and only renames the
  key. A bottom may likewise be post-processed by a function $h$ of the value, and the pair is a hierarchy with bottom $h \circ f$.
\item $C$ is \textbf{coherent} with $f$ if $W_C$ agrees with $f$ wherever it is defined. For a store $C$ coherent with $f$, $C \rhd f$ is
  the \textbf{function--store pair} $\langle f, C \rangle$, a memo of $f$.
\item A \textbf{system} is a processor with a hierarchy. A computation on it reads through the access, and its cost is its
  \textbf{compute cost} and its \textbf{access cost}, the cost of the accesses it makes.\qedhere
\end{itemize}
\end{definition}

Coherence makes a cache transparent: if every level is coherent with what lies below it, the access is the bottom, and a level
changes the cost of an access, never its value. We study the functional properties of the store, so we assume every level is
coherent at all times, as if a write were visible at every level at the next time. Since $C \rhd f$ is again a function
on keys, hierarchies compose as functions do (Construction~\ref{con:compose}, Appendix~\ref{app:eclipse}). A miss at the bottom of one
system is a call into another, so two systems connect into one hierarchy. In the same way, when each function of a chain
calls the next, the chain with the memos of its functions is a hierarchy with one level per function, and an access is paid down
to the first memo that hits.

\paragraph{The Oracle's computation.} For one transition, the Oracle computes $y \in A(x)$ on the encoded history $x = E(Q_t)$ at the
Frame interface, or $z \in \nu(x)$ at the Step interface. Its processor is an accelerator, its keys are encoded prefixes, and what it
reads is not its earlier answers but the state they left, so the store keeps states.

\begin{definition}[state presentation, prefix reuse]
\label{def:state-pres}
Let the Oracle be internally autoregressive with step relation $\nu$. A \textbf{state presentation} is a set $\mathcal S$ of states
with an initial state $s_0$, a transition $\varphi : \mathcal S \times \Gamma_Q \to \mathcal S$, and a map $\alpha$ from states to sets
of symbols of $\Gamma_R$, such that $\nu(x) = \alpha(\hat s(x))$, where $\hat s(\lambda) = s_0$ and
$\hat s(x\gamma) = \varphi(\hat s(x), \gamma)$. It is \textbf{truncatable} if $\hat s(x)$ yields $\hat s(x')$ at no cost for every prefix
$x'$ of $x$. A \textbf{state cache} holds pairs $(z, \hat s(z))$. The cost of computing $\hat s(x)$ from the entry for $z$ is
\[ \mathrm{cost}(z, x) = |x| - |\mathrm{lcp}(z, x)|, \]
the number of applications of $\varphi$ after truncating to the common prefix. With several entries, the cost is the least of
these.
\end{definition}

Every internally autoregressive Oracle has the trivial presentation $\hat s(x) = x$, so the definition fixes a cost model,
counting applications of $\varphi$, without excluding any Oracle. For a transformer, $\hat s(x)$ is \textbf{the KV cache of the
prefix $x$}: one application of $\varphi$ is one token of prefill or decode, truncation drops entries, and $\langle \hat s,
C_{\hat s} \rangle$ is the KV cache as a function--store pair, whose hits reuse a state and whose misses are rematerialized by a
prefill. Tiers of the KV cache in host memory and storage are levels below it~\cite{gao2024cost,qin2026mooncake,liu2025lmcache}.
Since $\hat s$ is a function and the sampling happens in $\alpha$, a state cache changes no answer set, and the machine with it has
exactly the executions of the machine without it. This is why the state cache can be modelled as a function--store pair although
the Oracle is not a function, and why the Oracle's cache and its scheduling can be modelled as the Priestess's are on a CPU. Apart
from rematerialization, only the keys differ (addresses on a CPU, histories for the Oracle), so the
cache model of this section applies to both systems.

The key of an access is the encoded history: a query looks up the longest cached prefix $z$ of its encoding $x$ and pays
$\mathrm{cost}(z, x)$. Within an append-only epoch, defined below, no entry is updated in place, since the history only grows.
Entries are deleted by an eviction or by a truncation to a prefix. At Frame granularity the encodings of histories grown from a
common history form a prefix tree (consequence (ii) of Theorem~\ref{thm:run-sync}), and the store can be organized as one, as SGLang's
radix tree organizes its KV cache~\cite{zheng2024sglang} and vLLM shares the blocks of a common prefix~\cite{kwon2023efficient}.

\paragraph{An Oracle call is one hierarchy.} Consider an Oracle call with answer $r_t = D(y_t)$, $y_t \in A(E(Q_t))$. Reading
$Q_t$ is the Priestess's own access. The derivation below composes functions, while the Oracle's answer relation is set-valued, so
fix a \textbf{recording} $\rho$ first: a function on states that picks one answer for each, $\rho(\hat s(x)) \in A(x)$ at the Frame
interface and $\rho(\hat s(x)) \in \nu(x)$ at the Step interface. One exists, because $\nu(x) = \alpha(\hat s(x))$, and $\varphi$
carries $\hat s(x)$ through the symbols generated after $x$, so $A(x)$ depends on $x$ only through $\hat s(x)$. Every run of the
recording is a run of the machine
(Appendix~\ref{app:eclipse}), so it is a witness the hierarchy can be built on, and a hit of a cache returns the recorded answer.
The call then passes through three stages, each of which can be memoized. The memo of the
encoder is an \textbf{encoding cache} $\langle E, C_E \rangle$, that of the state presentation is the state cache $\langle \hat s,
C_{\hat s} \rangle$, and that of the call as a whole is an \textbf{answer cache} $C_O$ keyed by the history. By
Construction~\ref{con:compose} the call is one hierarchy,
\[ C_O \rhd D \circ G_\rho \circ \langle \hat s, C_{\hat s} \rangle \circ \langle E, C_E \rangle, \]
where $G_\rho$ generates the recorded answer from the state, by $\alpha$ and $\varphi$, to the end of the computation. Read right to left: the history is
encoded, the state of the encoding is looked up and on a miss rematerialized by a prefill, and the generated symbols are decoded.
The levels differ in what they can reuse. If the Oracle is a black box, only the answer cache, keyed by whole histories, is
available. Within an append-only epoch these histories strictly grow, so no entry is hit again. If the
Oracle is internally autoregressive and the encoder factorizes at $\natural$, the encoding cache and the state cache reuse
prefixes, organized as the prefix tree above. From the entry of a history, a miss on an extension of it costs, beyond the growth
of the encoding, exactly the encoder's loss.

\begin{proposition}[excess recomputation is the loss]
\label{prop:excess}
For all machine-side words $w$ and $v$,
\[ \mathrm{cost}\bigl(E(w), E(wv)\bigr) = \bigl(|E(wv)| - |E(w)|\bigr) + \mathrm{loss}_E(w, v). \]
So, for $(w, v)$ in a set $G$ of Definition~\ref{def:loss}, a state cache recomputes at most $L_E^G$ beyond the growth of the encoding,
and at most $L_E^G + 1$ when the encoding grows by at most one symbol. If $E$ factorizes at $\natural$ and $w$ is empty or ends with
$\natural$, the cost is $|E(v)|$: nothing is recomputed at a Frame boundary.
\end{proposition}

\begin{proof}
The two sides subtract the same common prefix: $\mathrm{cost} = |E(wv)| - |\mathrm{lcp}|$ and $\mathrm{loss}_E = |E(w)| - |\mathrm{lcp}|$.
The bounds follow from the definition of the backtrack. In the last case $E(wv) = E(w)\,E(v)$ by Theorem~\ref{thm:run-sync}(1).
\end{proof}

\paragraph{Epochs and the encoding cache.} An \textbf{append-only epoch} of a task stack $u$ is a maximal segment of a run in
which the kernel changes $u$ only by $\mathtt{push}$. Within one, the histories at the starts of successive excursions on $u$
satisfy $h_{i+1} = h_i\,r_i\,v_i$, where $r_i$ is what the $i$-th excursion appended and $v_i$ is what the kernel pushed before
the next. The proof of Proposition~\ref{prop:ar-form} applies, with K1 for excursions on other task stacks and K4 for P-mode
calls. Any other rewrite of $u$, replacing or restoring the history for instance, starts a new epoch. The encoding cache
memoizes $E$ on the prefixes that end at a Frame boundary. For each stack on which the Oracle is called, the cache stores the encoding of the
last query up to its last $\natural$, or of the empty word if the query has none. If $E$ factorizes at $\natural$, the codec
therefore reads and encodes only the unfinished Frame at the last call and what was appended since. The cache stays coherent while
the stack only grows.

Across a tool call the Oracle holds the state of the query followed by what it generated, $E(q)\,y$, and the next query is
$q\,D(y)\,v$. With both directions of the synchronization, $E(q)\,y$ is a prefix of the encoding of the next query.

\begin{proposition}[the prefix chain across tool calls]
\label{prop:prefix-chain}
Let a task stack $u$ be in an append-only epoch, with the Oracle called on $u$ through the Frame interface, and suppose that
\begin{enumerate}[label=(\roman*)]
\item $E$ factorizes at $\natural$;
\item the Oracle with its codec is Frame-faithful;
\item the hypotheses of Proposition~\ref{prop:closure} hold; and
\item the content of $u$ at the start of the epoch ends with $\natural$, and so does every word that the kernel pushes onto $u$.
\end{enumerate}
Item~3 also requires that the presentation is truncatable, in the sense of Definition~\ref{def:state-pres}; the trivial
presentation $\hat s(x) = x$ is. Then:
\begin{enumerate}
\item If the Oracle answers a query $q$ on $u$ with the Frame $y$ and the kernel then pushes $v$, the next query $q' = q\,D(y)\,v$ has
  $E(q') = E(q)\,y\,E(v)$: the next query extends the Oracle's own history by the encoding of the kernel's push alone.
\item Give $u$ a private state cache: empty when the epoch starts, used only by $u$, kept between transitions, never evicted. Then
  the first query on $u$ costs its whole encoding, and each later one costs only $|E(v)|$, where $v$ is what the kernel pushed
  since the last Frame. Up to an O-mode transition with query $q$ and Frame $y$, the applications of $\varphi$ on $u$ number
  $|E(q\,D(y))|$: every symbol of the encoding is computed once, by prefill or by generation.
\item With the private state cache of item~2, let the encoding cache update its entry for $u$ at every call. Then after each call
  on a query $q$ with Frame $y$ the encoding cache holds $E(q)$ and the state cache holds $\hat s(E(q)\,y)$; for a truncatable
  presentation the latter yields $\hat s(E(q))$ at the cost of item~2, so one key (the stack and the length of its history)
  reaches both entries.\qedhere
\end{enumerate}
\end{proposition}

\begin{proof}
Every query on $u$ ends with $\natural$: the first is the content at the start of the epoch followed by pushed words, and each
later one appends a Frame (whose decoding ends with $\natural$) and pushed words. (1)~Theorem~\ref{thm:run-sync}(1), applied at the
end of $q$ and at the end of $D(y)$, gives $E(q') = E(q)\,E(D(y))\,E(v)$, and $E(D(y)) = y$ by Frame-faithfulness. (2)~Generating
$y$ leaves $\hat s(E(q)\,y)$ in the cache, and by item~1 the next query costs $\mathrm{cost}(E(q)\,y, E(q')) = |E(v)|$. The first
query costs its whole encoding, each Frame costs its $|y|$ Steps, and each push costs its $|E(v)|$. By item~1 at every earlier
call, and with $E(q\,D(y)) = E(q)\,y$ by Theorem~\ref{thm:run-sync}(2) and Frame-faithfulness, the encoding of $q\,D(y)$ is the
concatenation of these pieces. (3)~Since $q$ ends with $\natural$, the entry of the encoding cache is $E(q)$. On the first query
the state cache starts empty and processes $E(q)$; on a later one, by item~1, it already holds the state of $E(q)$ without the
encoding of the last push, and the call processes that encoding.
\end{proof}

Figure~\ref{fig:readings}(c) draws item 2: the records that the Priestess wrote are prefilled, the Frames that the Oracle wrote are decoded,
and nothing is computed twice. With a cache shared across tasks the number of computations is an upper bound, since another task may already hold a
state that $u$ reuses. Since one key indexes both caches, the start of a new epoch invalidates both at once. The chain breaks in
three places. When \emph{the kernel rewrites the task stack}, replacing or restoring the history, a new epoch starts, and its
first query costs $\mathrm{cost}$ from the nearest cached prefix to the new encoding. When \emph{a Frame does not re-encode to
itself}, the state cache recomputes $|E(D(y))| - |\mathrm{lcp}(y, E(D(y)))| > 0$ symbols beyond the push, by the cost of
Definition~\ref{def:state-pres} and because $y$ and $E(D(y))$ diverge. This happens when a model samples a tokenization the
encoder would not produce, or when a chat interface re-renders the history through a template on the provider's side. A related
mismatch, which token healing repairs, is known at the boundary between a prompt and its continuation~\cite{lundberg2023art,xu2026are}
(Appendix~\ref{app:further-encoders}). And at \emph{a budget trap inside a Frame} (Definition~\ref{def:budget}) the history need
not end with $\natural$, and the state that the Oracle holds need not be that of the history's encoding.

What the machine adds is a place for each of these mechanisms (Table~\ref{tab:sched}) and, for prefix sharing and its break at the end of an
epoch, a statement that can be checked against it. The mechanisms themselves are not new: prefix trees of KV
caches~\cite{zheng2024sglang}, paged KV memory~\cite{kwon2023efficient},
keeping, discarding or swapping the cache at tool calls~\cite{abhyankar2024infercept}, and program-aware
scheduling~\cite{lin2024parrot,luo2026agentix,kang2026thunderagent} all exist. Reuse of segments that are not prefixes~\cite{yao2025cacheblend,ma2026leyline,gim2024prompt} falls outside
Definition~\ref{def:state-pres}. And the statement is made once for both forms. In the Agent form, within an append-only epoch,
the queries on a task form a chain, because between two of them the task changes only by its own appends (K1) and by the kernel's
pushes, while excursions on other tasks and P-mode calls leave it alone (K1, K4). The chain thus follows from the contract, given the
kernel's epoch discipline. In the Workflow form, the queries of the P-mode calls on a Priestess stack $d$ form a chain if the
program never pops $d$, by Proposition~\ref{prop:ar-form}, since no excursion writes $d$ (K1). The chain is then a property of the
program text. The condition can be checked syntactically, and it is sufficient but not necessary. Table~\ref{tab:chain} in Appendix~\ref{app:eclipse} sets the two chains side by side.
Appendix~\ref{app:further-cache} places the mechanisms of serving systems on this model, together with reuse beyond prefixes and
the rewrites by which agent systems manage their context (compaction and paging), each of which starts a new epoch.

\subsection{Two Schedulers}
\label{sec:eclipse-control}

With several Oracle requests and several Oracles in a system, assigning the evaluations to the processors that run them is a
scheduling problem, as assigning tasks to cores is for a CPU scheduler. This subsection analyses the problem instead of modelling
it: it assigns each scheduler to its side, says what each side knows, and prices one decision that needs both
(Proposition~\ref{prop:return}). A \textbf{scheduler} assigns computations to the systems
that can run them: each computation is a system's computation, bound to its value argument. A computation costs its
compute cost and its access cost, found recursively down the hierarchy (Definition~\ref{def:store}, Construction~\ref{con:compose}).
A scheduler has \textbf{cache affinity} when it places a computation on the system whose stores already hold the keys that
computation will access, and it \textbf{balances load} when it spreads the compute cost evenly. The two can pull apart, since the
system that holds a computation's keys may be the busiest. The Oracle and the Priestess may run on one processor or on two, each
with its own caches.
On the machine they meet only at the switches and at P-mode calls (Corollary~\ref{cor:crossing}), so between those points each
side's work is local to it, and a new scheduler can be built beside the existing one and coordinate with it through the known
relation between the sides. Each configuration is either in O-mode on one task stack or in P-mode, and a P-mode call is a single
transition, so the Oracle's execution of a task and the Priestess's never overlap. One identifier per task serves both
sides, as the task's identifier does on a von Neumann computer.

Each side has a scheduler, and each schedules both its own side and the other. The \textbf{Priestess-side scheduler} orders the
Priestess's work and schedules the Oracle side by choosing which task to enter with $\mathtt{uret}$ and, under a budget, for how
long. In the Agent form it is the operating system's scheduler. The \textbf{Oracle-side scheduler} (an inference server's batching
and its order of prefill and decode, for instance) orders the Oracle's computations. It also schedules the Priestess side, since a
P-mode call completes when the Oracle has computed the answer, and only then does the program that issued it continue. In the
Workflow form no operating system mediates the task, and the Oracle-side scheduler belongs to the serving side. The machine represents the
Priestess-side scheduler as kernel code. The Oracle-side scheduler's decisions show only in when answers arrive, and time is not
in the model. At every configuration in P-mode the Priestess side holds the machine-side content of each Oracle access it may
submit next: the history of each task it may enter (K2) and the operand of each call it may issue. The Oracle side, by contrast,
receives only the queries submitted to it. So \textbf{the Priestess side knows the next keys, and the Oracle side knows its
cache}. Each works from an estimate of the other, built from what the other exposes, since the model gives neither side the
other's state. Appendix~\ref{app:further-sched} works this out for a task's state while its tool call
runs~\cite{abhyankar2024infercept}. Keeping, discarding or swapping the state are three costs of the cache model, and the decision
needs one more input from the Priestess side: when the task returns. Time is not in the machine, but what this knowledge
is worth can be quantified once the interval is given a duration from outside the machine.

\begin{proposition}[the price of not knowing the return]
\label{prop:return}
Let a task yield, and let the Oracle side hold the state of its history from the trap until the kernel enters the task again,
after a time $T > 0$. Keeping the state costs $\mu > 0$ per unit of time. Discarding it costs nothing until the task returns, and
then $R > 0$, the cost of rebuilding the state, in the same unit. A \textbf{retention policy} keeps the state until a time
$\theta \in [0, \infty]$ and discards it then if the task has not returned, at the cost
\[ c_\theta(T) = \begin{cases} \mu T & \text{if } T \le \theta, \\ \mu\theta + R & \text{if } T > \theta. \end{cases} \]
\begin{enumerate}
\item Every policy pays at least $\mathrm{opt}(T) = \min(\mu T, R)$, and a policy chosen with knowledge of $T$ pays exactly that.
\item For every $\theta$, $\sup_{T > 0} c_\theta(T) / \mathrm{opt}(T) \ge 2$, with equality exactly when $\theta = R/\mu$. So for
  every $\varepsilon > 0$, a policy chosen without knowledge of $T$ pays more than $(2 - \varepsilon)\,\mathrm{opt}(T)$ for some
  $T$.\qedhere
\end{enumerate}
\end{proposition}

\begin{proof}
(1)~If $T \le \theta$ the policy pays $\mu T$, and otherwise $\mu\theta + R \ge R$; either is at least $\mathrm{opt}(T)$. Knowing
$T$, take $\theta = \infty$ if $\mu T \le R$ and $\theta = 0$ otherwise. (2)~For $\theta = 0$ the ratio is $R/(\mu T)$, and for
$\theta = \infty$ it is $\mu T/R$ once $\mu T \ge R$; both are unbounded. For $0 < \theta < \infty$, let $T$ decrease to $\theta$
from above. The cost stays $\mu\theta + R$ while $\mathrm{opt}(T)$ tends to $\min(\mu\theta, R)$, so the ratio tends to
$1 + R/(\mu\theta)$ if $\mu\theta \le R$ and to $1 + \mu\theta/R$ otherwise. Each limit is at least $2$, and equals $2$ only when
$\mu\theta = R$. At $\theta = R/\mu$ the ratio is $1$ for $T \le \theta$, and $2R/R = 2$ for $T > \theta$, since then $\mu T > R$.
\end{proof}

The proposition is the rent-or-buy problem of competitive analysis~\cite{manasse1988competitive}, which arises in snoopy caching
as the choice of how long to retain a cache block~\cite{karlin1988competitive}. The machine adds two things to it:
the cost $R$, which the cache model prices (beyond the push, the $\mathrm{cost}$ from the nearest cached prefix, and the whole
encoding of the history when none is left; Definition~\ref{def:state-pres}), and the side on which $T$ is decided. The Oracle side
has the queries and its own answers, and neither determines $T$. It may predict $T$ from the request it generated, but it chooses
its policy without knowing $T$. The kernel's service decides when the task is entered again, so $T$ is decided on the Priestess
side, and so is what can be known of it before the task returns. Without that knowledge, the deterministic policy with the best guarantee keeps
the state for $R/\mu$, until keeping has cost what rebuilding would; the factor is exact while the cost of keeping is constant over
the interval, which is what the proposition assumes. A randomized policy brings the expected ratio down to $e/(e-1)$,
and no lower~\cite{karlin1994competitive}. A prediction of $T$, such as one built from a running tool's report~\cite{liu2026ask},
lets a policy come close to $\mathrm{opt}$ when the prediction is good and keep a bound when it is
not~\cite{purohit2018improving}. Swapping adds a third choice, which the proposition leaves out. In the Workflow form the same interval lies between
two calls that share a prefix, and the program's schedule holds what $T$ stands for~\cite{pan2025kvflow}.

A cache miss of a language model is a prefill, far costlier than a miss on a
processor, so cache affinity weighs more than it does for a CPU. By Definition~\ref{def:store}, an Oracle-side scheduler can still
borrow from CPU scheduling once it
models the hit and miss costs of its levels. The decisions that use the relation between the sides are the Priestess side's,
because it holds the next keys. With the Priestess in view, the server can make part of its state visible to the client, or the
client can track a list of keys, to raise the hit rate. Program-aware servers~\cite{lin2024parrot,luo2026agentix,kang2026thunderagent} and program-agnostic
ones~\cite{kwon2023efficient} are the two halves. SGLang stands between them, since its frontend passes a program's forks to the
runtime~\cite{zheng2024sglang}. Table~\ref{tab:sched} sorts the systems that the field has built, by which side's scheduler they are and
which side they schedule. Appendix~\ref{app:eclipse} treats granularity,
several Oracles and finite capacity.

\begin{table*}[t]
\centering
\small
\begin{tabular}{@{}>{\raggedright\arraybackslash}p{0.1\textwidth}>{\raggedright\arraybackslash}p{0.42\textwidth}>{\raggedright\arraybackslash}p{0.42\textwidth}@{}}
\toprule
Scheduler & Schedules its own side & Schedules the other side \\
\midrule
Oracle side & the Oracle's computations: batching, paging and sharing of the state cache~\cite{kwon2023efficient,zheng2024sglang,ye2024chunkattention,juravsky2024hydragen,gim2024prompt}, its tiers~\cite{gao2024cost,yu2025stateful,qin2026mooncake,liu2025lmcache}, and its placement across devices~\cite{srivatsa2025preble} & when a program continues: scheduling each call by the program it belongs to~\cite{lin2024parrot,luo2026agentix,kang2026thunderagent}, and the state of a task kept, dropped or swapped while its tool call runs~\cite{abhyankar2024infercept,li2026continuum,liu2026ask}, and the shared prefixes of a workflow kept or prefetched by how soon each will be used~\cite{pan2025kvflow} \\
\rowrule
Priestess side & the program's own work, tools and environments: planned calls dispatched in parallel~\cite{kim2024llm,guerin2026kaiju}, tool resources~\cite{kang2026thunderagent}; in PRTS, the ready queue & the Oracle's work: which task enters O-mode, when, and for how long, in PRTS the slice; the strategy by which a program's calls are executed, changed by effect handlers~\cite{wang2025composable}; programs that drive the generation loop and the cache inside the server~\cite{gim2025pie} \\
\bottomrule
\end{tabular}
\caption{Cache and scheduling: what each side's scheduler schedules.}
\label{tab:sched}
\end{table*}

\subsection{Safety: the Boundary}
\label{sec:eclipse-safety}

Safety requires that something bad never happens~\cite{lamport1977proving,alpern1985defining}, and K1--K4 and
Corollary~\ref{cor:crossing} are safety properties. They say what an excursion and a P-mode call may touch, and how content passes
between the modes. A defence must also know where content reaches control. For the duality of Oracle and Priestess there are
two kinds of bad thing. \emph{Bad data is executed} when the Priestess writes malicious data for the Oracle to run, or wrongly
executes data the Oracle returned. \emph{A bad execution touches data} when the Oracle issues a malicious system call, or the
Priestess induces the Oracle to generate malicious code. Prompt injection~\cite{greshake2023not} is of the first kind. Like the
contract, and unlike the rest of this section, the theorem below holds for every Oracle.

\begin{theorem}[the safety boundary]
\label{thm:boundary}
For every Oracle, every program and every run of the \mach{O-2-PDA-TS}, with or without budgets:
\begin{itemize}
\item \textbf{In.} Whatever an Oracle call reads beyond the answers that its own mode appended to its operand, the Priestess wrote
  there. An O-mode transition reads only the stack its excursion entered on, which holds the answers appended to it in O-mode and
  what the Priestess wrote there, the initial content included. A P-mode call reads only its operand, a Priestess stack, which
  holds the answers of earlier P-mode calls on it and what the Priestess wrote there, the initial content included.
\item \textbf{Out.} The successor's mode and label, and whether there is a successor, depend on stack content at four kinds of
  transition only: the O-mode transition and three P-mode instructions. An O-mode transition depends on whether the encoder accepts
  the history and whether the answer yields. Its successor is in O-mode on the same stack or in P-mode at $\ell_{\mathrm{trap}}$.
  Of the P-mode instructions, a $\mathtt{jeq}$ depends on the top of its operand, a $\mathtt{pop}$ on whether its operand is empty,
  and an Oracle call on whether the encoder accepts its operand, whatever the answer. No content becomes a label or instruction
  text.\qedhere
\end{itemize}
\end{theorem}

\begin{proof}
\textbf{In.} For an O-mode transition this is K2 with Corollary~\ref{cor:crossing}, which hold with budgets by
Proposition~\ref{prop:k5}. A P-mode call on $d$ takes $\sigma(d)$ alone (Definition~\ref{def:t}), and $d$ is a Priestess stack.
Under S an $\mathtt{uret}$ enters only Oracle stacks, so by K1 no excursion writes $d$, and, by the locality lemma and K4, $d$ changes
only by Priestess instructions with $d$ as operand and by P-mode calls on $d$. \textbf{Out.} Immediate from the rules of
Definitions~\ref{def:2pda}, \ref{def:o2pda}, \ref{def:t} and~\ref{def:budget}. A configuration at $n+1$ is final and a
$\mathtt{push}$ continues at $\ell+1$, whatever the content. An $\mathtt{uret}$ enters O-mode on the stack it names, with the budget
it names. A P-mode Oracle call continues at $\ell+1$ or at $l_e$, by acceptance alone (K4). A $\mathtt{jeq}$ continues at its label
or at $\ell+1$, by the top of its operand, and a $\mathtt{pop}$ continues at $\ell+1$ unless its operand is empty, when there is no
successor. An O-mode transition has a successor only if the encoder accepts the history, and that successor is at
$\ell_{\mathrm{trap}}$ or in O-mode on the same stack, according to whether the answer yields or, with budgets, the budget is spent.
No content becomes instruction text, since no rule writes $\Pi$, and none becomes a label, since every P-mode label is $1$,
$\ell_{\mathrm{trap}}$, $\ell + 1$ or a label written in $\Pi$ (K4).
\end{proof}

The two items draw one line between the two computations, a line we call the \textbf{safety boundary}. On the way in, the line is
the Priestess's writes to an operand. On the way out, it is the trap, whose target is fixed and whose moment is decided by the
Oracle's answers or the budget. The line is also the three Priestess instructions that read content, which we call the \textbf{gates}.
\textbf{The machine fixes where content crosses between the two computations, and the code at the crossing decides what may cross.}

Each bad thing crosses the boundary. Malicious data for the Oracle to run reaches the Oracle only in a Priestess write, the initial
content counting as one. An instruction injected into a tool's result, for instance, enters a task as part of the kernel's reply. A
malicious system call is an answer, and it reaches the Priestess first at $\ell_{\mathrm{trap}}$. Data that a P-mode call returns,
including any malicious code, acts on the Priestess's control only through gates, and runs on the Oracle only after a Priestess
write and an $\mathtt{uret}$. A Priestess that induces the Oracle to generate malicious code does so through the operand it writes. The code
comes back at the trap or as a P-mode answer, and it acts on control only through the gates. So no bad thing of either kind passes
around the boundary, and a check placed there sees every one, provided the initial content is checked with the rest. The rules thus
fix where complete mediation~\cite{saltzer1975protection} must be placed. Whether each crossing is
checked there is decided by the code.

\begin{proposition}[where, not which]
\label{prop:where-not-which}
Fix a program and an Oracle $(A, E, D)$ of the \mach{O-2-PDA-TS}, with or without budgets. Let a run reach a configuration $c$ that
calls the Oracle on a history $h \in \mathrm{dom}(E)$: an O-mode configuration on a stack that holds $h$, or a P-mode configuration
at an Oracle instruction whose operand holds $h$. Let $u = D(y)$ be non-empty, for some $y \in Y_A$, and let
$A' = A \cup \{(E(h), y)\}$. Then the Oracle $(A', E, D)$ has the same codec, every run with $(A, E, D)$ is a run with
$(A', E, D)$, and $c$ has a successor with $(A', E, D)$ that appends $u$ to $h$. So a property of runs that holds for every Oracle
with this codec holds of some run in which $u$ is appended at $c$, and no such property excludes $u$ there.
\end{proposition}

\begin{proof}
The answer map grows: $O(w) \subseteq O'(w) \subseteq O(w) \cup \{u\}$ for every $w \in \mathrm{dom}(E)$, so every query keeps an
answer and every answer is non-empty, and $\mathrm{dom}(E)$ is unchanged. The rules of Definitions~\ref{def:o2pda}, \ref{def:t}
and~\ref{def:budget} depend on the Oracle only through $\mathrm{dom}(E)$ and $O$, and every successor they give with $O$ they also
give with $O'$. By induction, every run with $(A, E, D)$ is a run with $(A', E, D)$, the run that reaches $c$ among them. At $c$,
$u \in O'(h)$, so the rule of the call gives a successor that appends $u$ to $h$. A P-mode call continues at $\ell + 1$, and an
O-mode transition traps or stays in O-mode according to whether $u$ ends with $\natural$ and to the budget.
\end{proof}

Take $u$ to be a malicious request: no guarantee that holds for every Oracle prevents it from crossing. Which request crosses is
therefore a property of the Oracle on that history, a judgement over content that no rule of the machine makes. That it crosses
at the boundary is a property of the machine (Theorem~\ref{thm:boundary}). A guarantee that holds for every Oracle can say where
content crosses, but not which content crosses. This is all that the machine can promise to a defence outside the Oracle: the
place where the defence belongs. The proposition fixes the codec, so a decoder that cannot write $u$ does exclude it, as a
property of the codec, not of the machine.

A task's boundary has two placements, as its program does. In the Agent form the Priestess is the task's operating system, and
the boundary lies in the kernel. Whatever a task reads beyond its own answers is either its initial content (the declaration included)
or one of the kernel's writes (its replies among them). Whatever the task writes reaches the kernel first at the service entry, at a yield
trap as a request. These are the return and the force of Corollary~\ref{cor:crossing}. That every system call is answered before
the task resumes is the kernel's protocol, not the machine's, since the machine would equally resume an unanswered task. In the
Workflow form no operating system mediates the task, and the boundary lies in the program: in the operands it builds and in the
code that reads each answer, where validating an answer before a promotion is the program's obligation. A hybrid has a boundary for
each part, and moving a task between the forms moves its boundary. When an Agent task is run through the interpreter of
Proposition~\ref{prop:interpreter}, its boundary lies in the interpreter's code, which then supplies what T and the
locality lemma supply in the machine. Appendix~\ref{app:further-safety} examines how a harness can escalate privilege, and
places the defences built for agents at the three positions that the boundary leaves them: on the way in, inside the Oracle, and
on the way out, at the trap or at the gates.

Treat the Oracle as a user-mode program and the Priestess as a kernel, and the Agent's boundary is the one between user mode and the
kernel (Table~\ref{tab:protection}, Appendix~\ref{app:eclipse}). A task writes only its own stack, as a user process writes only its
address space. Its requests enter only at $\ell_{\mathrm{trap}}$, as traps enter at a fixed vector. And nothing protects the task
from the kernel, which is verified~\cite{klein2009sel4,gu2016certikos} or trusted. The two restrictions on the Oracle, that it runs
no instruction of its own and only appends, are not properties of user mode and do not transfer.

The theorem is about runs of the machine. As security analyses do, we assume a system to start in an intended state and an attack to
be a departure from it. For the Priestess, the intended state includes that its code does not execute memory the Oracle writes,
except by a deliberate interpretation or promotion. In a von Neumann realization one further path leads from content to control, around
the boundary: the Priestess's processor fetching instructions directly from memory the Oracle writes. No rule has such a
transition, so an execution that contains one leaves the typing. Such an execution is a \textbf{type bypass}, an attack on the
intended state. \S\ref{sec:query} states the condition a realization keeps to exclude it. Appendix~\ref{app:eclipse} discusses
protections that also restrict the trusted side (such as SMEP and SMAP) and what a task can observe. It also discusses
replay, which on the machine needs a rewrite by the kernel in the Agent form.

The two computations are now linked where it matters for a cache. When the encoder factorizes at $\natural$, the encoding never
backs up across a Frame. If, in addition, the Oracle is Frame-faithful and called through the Frame interface, and the kernel's
writes end with $\natural$, the Oracle holds the encoding of the machine's history after every Frame. Under the same hypotheses,
within an append-only epoch a private state cache that is never evicted computes each symbol once, under the same key as the
encoding cache. Both sides are systems of one kind, and the side that holds the next keys is the Priestess's. The two meet at the boundary
of Theorem~\ref{thm:boundary}, an invariant of the machine. The machine itself is still abstract, and \S\ref{sec:query}
builds it.

\section{In Query We Compute, Agents: Building the Machine}
\label{sec:query}

A physical \mach{O-2-PDA-TS} raises three questions that the abstract machine could leave open. \emph{Who can be the Oracle}:
which Oracles the machine runs in O-mode, and which only in P-mode. \emph{Where the Priestess is}: on what computer the
Turing-complete computation runs, and how the stacks and the two modes are laid out in its memory and its processor. And \emph{how
the two exchange messages}: what a system call asks of the Oracle, and what the Priestess can do to make the call succeed more often.
The answers to the first two make up one construction, \textbf{V}. It consists of the rule of who runs in O-mode
(\S\ref{sec:query-oracle}), together with the segmented storage and the single Oracle instruction that make the machine
compatible with a von Neumann computer (\S\S\ref{sec:query-priestess} and~\ref{sec:query-control}). The
\mach{O-2-PDA-TS} with V is the \textbf{\mach{O-2-PDA-TSV}}.

\subsection{Who Is the Oracle?}
\label{sec:query-oracle}

A physical machine is built for Oracles of several \textbf{classes}, such as language models, people and models that are not
autoregressive. At each call it uses one \textbf{entity} of a class, a particular model with its codec. Being autoregressive
inside is a property of the entity. Since the codec is ours to define, any Oracle can be given one that fits
the autoregressive form the machine's transitions have when the program only appends to the stack
(Proposition~\ref{prop:ar-form}). But an
excursion feeds the Oracle its own answers with no Priestess code between them: after an answer $r$ on the history $w$, the next
query is $w \cdot r$ (Definition~\ref{def:t}). What O-mode runs is therefore an iteration of the Oracle on the task's history, and
that iteration is the Oracle's own computation when the Oracle is internally autoregressive, for then each O-mode transition is a
Step or a Frame of the Oracle's own generation (\S\ref{sec:eclipse}). At Step granularity the encoding may still back up between
Steps, within the open Frame (\S\ref{sec:eclipse-codec}).

\label{sec:query-ar}
\paragraph{Autoregressive Oracles.} A language model is an $A_{\mathrm{AR}}$ Oracle: used through a conversation, it has the
autoregressive form (\S\ref{sec:machine}), and, under the assumption of \S\ref{sec:eclipse}, its interior has that form too. Its
output alphabet is included in its input alphabet, as Definition~\ref{def:int-ar} requires. Without the inclusion, feeding the model its own
output would produce an illegal input. Three of its properties matter to the machine. Its decoding is \textbf{non-deterministic}:
a sampled answer is one of the possible answers, which form the set that the machine reasons about (Definition~\ref{def:oracle});
sampling the same query again is what the replay of Appendix~\ref{app:eclipse-replay} exploits. It is \textbf{bounded}. An Oracle may accept unbounded
input, but under finite time and finite control the input it is given is finite. A real model, however, requires bounded input,
and an input beyond the bound is rejected by the encoder (Definition~\ref{def:oracle}). A model with bounded input is therefore at
most a finite automaton~\cite{schuurmans2023memory,wiedermann2024large}, as a real von Neumann computer is. And it can be \textbf{total} in the sense of \S\ref{sec:machine}: a
model returns a non-empty answer in finite time if it has a limit on its output tokens, a codec that converts the two alphabets
fully and rejects over-long inputs, and no hardware or network faults. A person's speech, a monkey at a keyboard and an
action-conditioned world model that generates frames autoregressively~\cite{bruce2024genie} may also be of this kind. An $A_{\mathrm{AR}}$ Oracle can
run in both modes: in O-mode every continuation is a legal input under the conditions of Proposition~\ref{prop:closure}, and in
P-mode the Priestess builds each input.

\label{sec:query-nar}
\paragraph{Non-autoregressive Oracles.} An Oracle may also be non-autoregressive, and then its output alphabet usually is not
included in its input alphabet. Jev, a recent commercial model from TypeSafe, returns typed probabilistic
decisions~\cite{almeida2026introducing}. Its output is JSON text that contains a well-formed probability result. Although that
text can still be encoded as the next input, Jev is usually more meaningful as a non-autoregressive Oracle than as an
autoregressive one. A diffusion model generates an image by iterated denoising~\cite{ho2020denoising}. One that also takes images as
input does have $\Gamma_R \subseteq \Gamma_Q$, and its output can be fed back. But it has no internal autoregressive structure, so the analysis of \S\ref{sec:eclipse}, which
relies on a correspondence between internal states, has nothing to correspond to. Iterated in O-mode, the first kind leaves the
encoder's domain whenever an answer does not yield and contains a symbol outside $\Sigma_Q$, and the machine stops being good
(Proposition~\ref{prop:codec-good}(1)); the second iterates a function that is not its own transition. In P-mode
neither problem arises, since the Priestess builds every input.

\begin{proposition}[goodness depends on the codec]
\label{prop:codec-good}
Let $(A, E, D)$ be an Oracle with $\mathrm{dom}(E) = \Sigma_Q^*$, run in O-mode from histories in $\Sigma_Q^*$.
\begin{enumerate}
\item If a reachable O-mode transition that does not trap appends an answer that contains a symbol outside $\Sigma_Q$, the next query
  is rejected: the machine faults, and it is not good.
\item If $\natural \in \Sigma_R \cap \Sigma_Q$ and $|\Sigma_R| \le |\Sigma_Q|$, let $\rho : \Sigma_R \to \Sigma_Q$ be injective with
  $\rho(\natural) = \natural$. The wrapped Oracle $(A, E, \rho \circ D)$ makes the machine good, whatever $A$.
\item The wrapped Oracle is internally autoregressive exactly when $(A, E, D)$ is.\qedhere
\end{enumerate}
\end{proposition}

The proof is in Appendix~\ref{app:proofs}. So any Oracle that meets the conditions of item 2 can be wrapped to make the machine good.
What the wrapping cannot change is the interior. For an $A_{\mathrm{nAR}}$ Oracle the wrapped iteration runs, but for the two kinds
above it achieves little: the first loses the meaning of its answers and the second iterates a function that is not its own. \textbf{So V runs only
$A_{\mathrm{AR}}$ Oracles in O-mode, and P-mode may call any Oracle.} The constraint protects no part of the contract or of the
safety boundary, which are invariants of the machine (Theorem~\ref{thm:contract}, Theorem~\ref{thm:boundary}). What it keeps is that O-mode
iterates the Oracle's own computation, so that the link through the codec and the cache model of \S\ref{sec:eclipse} apply to every
task the kernel enters. Its price is that an $A_{\mathrm{nAR}}$ Oracle is called only in P-mode, and so is a person, unless a
mechanism of generation is declared for the person (\S\ref{sec:query-message}). \S\ref{sec:query-control} makes this a rule of the processor.

\subsection{Memory: Stacks as Segments}
\label{sec:query-priestess}
\label{sec:query-memory}

We build the machine on a von Neumann computer: we split the Oracle stack and
the Priestess stack into as many stacks as needed, and lay them all out on one random-access memory, each in a segment of its own
and each with the colour of the stack it was derived from. On two stacks a random access costs moves of the head, whereas here it
costs one access. Abstract machines have been built physically, as Forth chips were~\cite{koopman1989stack}, yet most real computers are register
machines with random-access memory. Other computing units, from an ALU to an accelerator, are attached to them as coprocessors.
For an interpreter, which runs a stored program by a fetch--decode--execute loop, register designs need fewer dispatches than
stack designs~\cite{shi2008virtual}. Random-access memory reaches any address directly, without
passing over the cells between~\cite{cook1973time}.

\begin{definition}[segmented realization]
\label{def:segmented}
Let $\mathrm{Mem}$ be a random-access memory.
\begin{itemize}
\item A \textbf{layout} assigns to each stack $s$ a segment $[b_s, b_s + \mathit{cap}_s)$ of addresses, with capacity $\mathit{cap}_s$, and a length register
  $\mathit{len}_s$. The segments are pairwise disjoint, and disjoint from the \textbf{code segment}, which
contains the program, stored as in a von Neumann machine (\S\ref{sec:machine}).
\item A memory state \textbf{realizes} a store $\sigma$ if, for every $s$, $\mathit{len}_s = |\sigma(s)| \le \mathit{cap}_s$ and
  $\mathrm{Mem}[b_s + i] = \sigma(s)_i$ for every $i < \mathit{len}_s$.
\item The \textbf{realized Oracle rules} are those of Definitions~\ref{def:o2pda} and~\ref{def:t} with bounded appends. The query on $s$
  is the word $\mathrm{Mem}[b_s], \dots, \mathrm{Mem}[b_s + \mathit{len}_s - 1]$. An answer $r$ is written at $b_s + \mathit{len}_s$, and
  $\mathit{len}_s$ grows by $|r|$, only if $\mathit{len}_s + |r| \le \mathit{cap}_s$. Otherwise the append is refused and writes nothing, and the
  realization raises an exception: in O-mode a trap to $\ell_{\mathrm{trap}}$, in P-mode a jump to $l_e$, with a cause the kernel can
  read.\qedhere
\end{itemize}
\end{definition}

A stack $s$ thus has three pointers: its base $b_s$, its top $b_s + \mathit{len}_s$ and its limit $b_s + \mathit{cap}_s$. The code
is not a third colour, and no Oracle rule writes it. The hardware checks an operation against the window it is given, but the
disjointness of the windows is a property of the layout, which the kernel chooses.

\begin{proposition}[segment locality]
\label{prop:segment-locality}
Every realized Oracle step writes no memory cell outside its operand's segment and changes no length register but its operand's, or
writes nothing. If a memory state realizes $\sigma$ and the append of an answer $r$ on $s$ fits, the memory state after it realizes
$\sigma[s \mapsto \sigma(s) \cdot r]$.
\end{proposition}

\begin{proof}
A write happens only after the check $\mathit{len}_s + |r| \le \mathit{cap}_s$, and it then spans the addresses
$b_s + \mathit{len}_s, \dots, b_s + \mathit{len}_s + |r| - 1$, inside the operand's segment, and sets $\mathit{len}_s$ to
$\mathit{len}_s + |r| = |\sigma(s) \cdot r|$. The segments are disjoint, so no other stack changes, and the cells written are exactly the
new cells of $\sigma(s) \cdot r$.
\end{proof}

This is the locality lemma (Lemma~\ref{lem:locality}) in memory: on two stacks it needed the construction of
Lemma~\ref{lem:split}, and with a layout it is a bounds check, which the hardware performs. With it, K1 carries over to a
realization as the statement that an excursion writes only inside its task's segment, provided that the kernel chooses a layout
whose segments are disjoint. K2, K3 and K5 apply to the O-mode rule and the Oracle rules and carry over in the same way, as does the
first part of K4, on the P-mode call. In a real instruction set, Priestess instructions address memory directly, so for the
Priestess locality holds only as a property of its code, as it does for every kernel. The rest of K4, that every P-mode label is
literal, does not carry over, and the fetch condition below says what replaces it.

The machine's stacks have no fixed capacity. A particular Oracle bounds its input through its encoder, and a realization bounds
its segments. In the same way a real computer (a finite automaton) bounds an ideal von Neumann machine, which is Turing complete.
The machine uses a segment's capacity only through Proposition~\ref{prop:segment-locality}: an append that does not fit writes
nothing outside its segment. What the kernel does next is its policy: it may grow the segment and call again, end the task,
or compress the context. On random-access memory the isolation of \S\ref{sec:os} needs only a simple layout, which
allocates O-mode its own space for Oracle stacks and uses the space of Priestess stacks in P-mode. On a current processor this layout is virtual
memory. A page fault during an Oracle call can be handled in two ways. One is a second page table for Oracle segments with a
delegated handler, which is little used and costly. The other uses Priestess addresses with enough memory committed before each
call, so that no fault arises and an answer that does not fit raises an exception that the Priestess handles. The second is simple
and reuses existing hardware and software. When the Oracle's computation needs little memory, lazily committed pages make a large
segment cheap, and on overflow the segment grows like a vector and the call is retried.

\paragraph{The fetch condition.} Corollary~\ref{cor:crossing} is a statement about the rules: content crosses between the modes
only by a force or a return. On a von Neumann machine, code and data share one memory, and the program counter can be set from
data, so a realization has to separate code and data. The realized Oracle rules set the program counter only to the next
instruction, the rejection label or $\ell_{\mathrm{trap}}$, and write only inside the operand's segment, which is disjoint from the
code segment (Proposition~\ref{prop:segment-locality}). Both are properties of the rules, which the hardware implements. A
third condition is the kernel's obligation: no instruction may be fetched from a segment that an Oracle rule can write. We call
it the \textbf{fetch condition}. Page permissions that prevent memory from being both writable and executable are the usual way to
keep it, and an execution that breaks it is the type bypass of \S\ref{sec:eclipse-safety}. In the abstract machine, K4 holds
without conditions, because jumps encode literal labels and no rule writes the program. A real instruction set, however, has
indirect jumps, so Priestess code can compute a jump target from data, for instance from a jump table indexed by an answer.
Priestess code decides for itself which code it reaches that way, on the same footing as promotion (K4). It also decides whether
to write the code segment, as a JIT compiler does. The fetch condition restricts which segment an instruction is fetched from, not how its bytes
got there: code that the Priestess has written into the code segment is fetched legitimately, and an Oracle-writable segment is never
fetched from. What a realization preserves of the rest of K4 is therefore that no Oracle rule sets the program counter from an answer,
and that no instruction is fetched from where an Oracle rule writes.

\subsection{Control: One Instruction, No Code}
\label{sec:query-control}

The program-style \mach{2-PDA} has, in structure, the fetch--decode--execute form. It reads a program from a tape, and that tape
need not be a stack. With the program in the code segment and a program counter, it is a stored-program computer. With the
privilege separation of \S\ref{sec:os}, the processor gains an O-mode beside its P-mode, stored
in a privilege register or descriptor. P-mode contains the processor's existing privilege levels. Both modes can execute the Oracle operation, and they differ
at a yield: in O-mode a yield traps to $\ell_{\mathrm{trap}}$, while in P-mode the instruction simply completes. The memory of
O-mode is a stack that only grows, so it can be given stack space as a user process stack is, with a base, a top and a limit. The
memory of P-mode follows the virtual-memory view of the processor and the operating system, and a P-mode Oracle operation is given
a contiguous region on the stack or the heap, with the same three pointers.

On the tape, the Oracle operation reads a contiguous region and writes new symbols next to it. On a real instruction set, a call returns
its answer to a register, and continued autoregressive transitions could then be written as
an Oracle call, a push and a few $\mathtt{jeq}$, as the interpreter of Proposition~\ref{prop:interpreter} does. That has three costs.
(a)~The answer must first be written to storage other than the history. In O-mode, which runs no instruction of its own, that
storage would be outside the task stack, and K1 would have to count it as part of the task stack. (b)~An Oracle transition should
be atomic. A trap in the middle would leave part of an answer in the history, a state that means nothing and that K1 and the
append-only epoch of \S\ref{sec:eclipse-memory} would have to admit. (c)~If O-mode ran code of its own, that code would put a
programmable state machine on top of the Oracle's autoregression, and what O-mode runs would no longer be the Oracle's own transition.

\insight\label{ins:one-instruction} For a physical \mach{O-2-PDA-TS} the Oracle operation is one instruction. In O-mode there is no program form: the
processor's fetch--decode--execute loop drives the transitions, deciding each with no code. In P-mode the same instruction is called like
any other, and it can drive an asynchronous transfer. This is close to the intuition of operating-system design. O-mode is a privilege level whose transitions are computed by the Oracle,
a computation that the Priestess's instructions need not be able to perform. Its rights are also the most restricted: it runs no
instruction of its own, and it only writes new data, never changing existing data. Fine-grained control of O-mode therefore
is implemented in the processor's microarchitecture or microcode.

Concretely, in P-mode the instruction is called on an operand segment, given by its three pointers, and completes with the answer
appended. In O-mode the fetch--decode--execute loop queries the whole task segment, appends the answer, tests for the yield
symbol, and traps to $\ell_{\mathrm{trap}}$ at a yield or when the budget runs out (Definition~\ref{def:budget}). Writing that
loop as Priestess code instead is the interpreter, whose guarantees are those of its code. The machine is synchronous, and its
answers are whole at the granularity of the interface. A realization chooses how it publishes a whole answer (including on an
exception) and whether it departs from the machine's synchrony (\S\ref{sec:archnights}).

\paragraph{The O constraint.} The instruction applies the conclusion of \S\ref{sec:query-oracle}. The processor executes the
Oracle operation in O-mode only for an $A_{\mathrm{AR}}$ Oracle, and refuses it for any other. In P-mode it executes the operation
for every Oracle. Whether an Oracle entity is an $A_{\mathrm{AR}}$ Oracle is part of the interface it declares to the processor.
\S\ref{sec:archnights} shows how ArchNights declares it.

\insight\label{ins:user-kernel} The reading of \S\ref{sec:eclipse-safety}, which takes the Oracle as a user-mode program and the Priestess as a
kernel, is also a design heuristic. It says what to copy from existing processors:
\begin{itemize}
\item the Oracle segment, with its base, top and limit, is managed like a user stack;
\item the yield trap to $\ell_{\mathrm{trap}}$ is a fixed trap vector, like a system call;
\item $\mathtt{uret}$ returns into O-mode as an xRET returns into a lower mode;\footnote{By the RISC-V convention an xRET is named after
  the mode of the trap handler it returns from~\cite{waterman2021risc}. A realization that puts O-mode just below user mode, as ArchNights
  does (\S\ref{sec:archnights}), handles a yield trap one level up, in user mode, hence the name $\mathtt{uret}$. The ratified manuals
  define no instruction of that name~\cite{waterman2019risc,waterman2021risc}.}
\item the Priestess does not execute memory that O-mode writes, as SMEP keeps a kernel from executing user pages~\cite{intelsdm} (the fetch condition).
\end{itemize}
As \S\ref{sec:eclipse-safety} noted, the two restrictions on the Oracle do not transfer. The Priestess computation (its
arithmetic, logic, control transfer, data transfer and system instructions) can reuse a modern von Neumann processor entirely. For
analysis, its instructions can be simulated by the \mach{2-PDA} in small steps, or treated as hardware black boxes in big steps.

\subsection{Oracle-Priestess's Message}
\label{sec:query-message}

O-mode acts on P-mode by a system call, a yield trap that forces the service entry $k$; P-mode acts on O-mode by a return, its
reply (Corollary~\ref{cor:crossing}). The Priestess's reply is written by its code, but the request is what the Oracle
writes, and the kernel serves it at $\ell_{\mathrm{trap}}$. A correct system call becomes possible when the kernel places the declaration
$\delta_\Lambda$ on every task stack. It does not make a correct call certain: whether the Oracle's answer lies in $L_\Lambda$ is
a property of the Oracle on that history, and the kernel contract holds whatever the answer is. With current Oracles the outcome
also depends on sampling and on attention, and longer histories have costs. In the settings tested, accuracy depends on where the
relevant information sits in the context. It is highest at the beginning or the end and lowest in the middle, long-context models
included~\cite{liu2024lost}. For calls in particular, accuracy falls with the length of the conversation and of the tools'
responses~\cite{kate2025longfunceval}. The general form of these observations is that the longer an autoregressive
Oracle's history, the less likely its next call is correct. We take it as a working hypothesis, the \textbf{Lynchpin
hypothesis}. It is not a result of the machine, and nothing below depends on it except to motivate semantic delegation. Why it should hold, and
under which conditions, is not understood; Appendix~\ref{app:further-oracles} reviews
the evidence and the theories that may bear on it. One remedy is to make fewer calls, packing several requests into one so that one system call
suffices. In principle the packing happens in O-mode, but when the semantics of the packing is blockwise, it can be moved to P-mode.

\insight\label{ins:delegation} When a system call needs the result of a computation, that computation reaches the Priestess either as a
value or as a name. Without delegation it is evaluated before the call, on the Oracle side. The
Oracle can be asked to compute anything, and it computes from its history: an interface that asks the Oracle to compute inserts the computation's definition
into the history as text and evaluates it in place. The call then passes a value to the Priestess (\emph{call by value}), and the
history grows by the definition and by every intermediate step. A \textbf{semantic delegation} call passes instead how the value is
to be computed, a computation the Priestess evaluates when the call forces it (\emph{call by name}). The machine's force expands
nothing into the history: it enters the kernel, which evaluates the computation in P-mode and appends only its result, and the
handler's code never appears there. Semantic delegation is a design choice, motivated by the observations above. Each packed
operator has the obligation to perform the same services as the sequence of calls it replaces. What returns to the task must be
what the sequence would have produced, while the history may differ and is shorter. \S\ref{sec:archnights} describes a system
library that implements such operators.

Language models are the autoregressive Oracles that produce requests today. Any autoregressive Oracle whose alphabet admits a
semantics can in principle produce requests under that semantics; what they mean is fixed by the semantics, not by the alphabet.
For example, a vision-language-action model is autoregressive inside, and its actions reach P-mode as requests that control a
robot~\cite{brohan2023rt,kim2024openvla}. This
answers the question of the Introduction. As Oracles, a person and a language model are alike, provided a proper semantics can be
defined for parsing the person's words. What separates their outputs is how each is called. Under the assumption of
\S\ref{sec:eclipse}, a language model is an $A_{\mathrm{AR}}$ Oracle, declares it, and may run in O-mode, where its answers
continue the program that is the task's history. A person's output is usually taken as data: a person is usually called in P-mode,
and running one in O-mode would need a declared mechanism of generation in the sense of Definition~\ref{def:int-ar}.

The \mach{O-2-PDA-TSV} is then complete: its stacks are segments of a von Neumann memory, its Oracle operation is one instruction with
no code in O-mode, O-mode runs only $A_{\mathrm{AR}}$ Oracles, and the two sides pass messages by force and return. It
fixes what a realization must preserve, and leaves asynchrony, finite capacity and context compression to the realization. \S\ref{sec:archnights}
builds one.

\section{ArchNights: Implementation on gem5}
\label{sec:archnights}

ArchNights (Agent Runtime arCH iN Gem5 arcHiTectural Simulator) is a prototype machine that follows the \mach{O-2-PDA-TSV} on gem5~\cite{binkert2011gem5,lowepower2020gem5}
(Figure~\ref{fig:construction}, right). In gem5's syscall-emulation (SE) mode the host intercepts the system calls of the
simulated program, and in its full-system (FS) mode the whole system is simulated. ArchNights-SE runs in the first. ArchNights-FS
is a design for the second; it has not been built, and it has not been validated. Both share one
instruction set and one distribution, and differ in the kernel, PRTS or ZOOT. Where the machine fixes a behaviour, ArchNights
implements it (\S\ref{sec:an-obligations}). Where it leaves a choice, ArchNights mostly takes Linux's, and what is new is the set of
operators that the system library offers to the Oracle, together with the choices that ArchNights makes for the three questions
\S\ref{sec:query} left to a realization: asynchrony, storage of finite capacity, and context compression.

\subsection{What ArchNights Keeps of the Machine}
\label{sec:an-obligations}

\begin{table*}[p]
\centering
\small
\setlength{\tabcolsep}{4pt}
\begin{tabular}{@{}>{\raggedright\arraybackslash}p{0.045\textwidth}>{\raggedright\arraybackslash}p{0.39\textwidth}>{\raggedright\arraybackslash}p{0.33\textwidth}>{\raggedright\arraybackslash}p{0.17\textwidth}@{}}
\toprule
 & What a realization keeps & In ArchNights & Lost if it fails \\
\midrule
\multicolumn{4}{@{}l}{\textit{Ia. Rules kept by the processor, whatever the kernel: K1--K4, Corollary~\ref{cor:crossing}, Theorem~\ref{thm:boundary}}} \\
\rowrule
I-1 & An Oracle operation appends only inside the window it is given, and an answer that does not fit is refused and writes nothing (Definition~\ref{def:segmented}, Proposition~\ref{prop:segment-locality}) & the controller checks the window before evaluation; an answer that does not fit is refused at publication; the descriptor lies inside the window, past the top, so it changes no store realized & K1 \\
\rowrule
I-2 & O-mode fetches no instruction and has no addressable code (Definition~\ref{def:t}, Insight~\ref{ins:one-instruction}) & the Oracle mode, run by microcode & K1, K2 \\
\rowrule
I-3 & Each transition, a Step or a Frame, is published whole, and an exception publishes no part of one (\S\ref{sec:query-control}) & the controller publishes after evaluation; an exception keeps only what was published before it & K1, whose appends are whole answers \\
\rowrule
I-4 & Control leaves O-mode at a place that no answer can select (Definition~\ref{def:t}) & $\mathtt{uret}$ completes, and the processor continues at the next instruction, with the cause in the destination register & K3 \\
\rowrule
I-5 & A P-mode operation continues at the next instruction whatever its answer, and no Oracle operation sets the program counter (Definition~\ref{def:t}, \S\ref{sec:query-priestess}) & STEP or RUN in P-mode, issued by \texttt{ask}, by the session reading a line from the person and by the summary of a swap; a rejection shows in the status, and a branch on it is the program's own, as a $\mathtt{jeq}$ is & the first part of K4 \\
\grouprule
\multicolumn{4}{@{}l}{\textit{Ib. Rules a realization moves to the kernel, kept by PRTS by design}} \\
\rowrule
I-6 & The layout realizes S: segments are pairwise disjoint and disjoint from the code, and the kernel enters with $\mathtt{uret}$ only the segment of a task's history (Definitions~\ref{def:s} and~\ref{def:segmented}) & one segment per history in the memory descriptor, reserved when the task is created & K1 (Proposition~\ref{prop:segment-locality}); S \\
\rowrule
I-7 & The fetch condition: no instruction is fetched from a segment an Oracle operation can write (\S\ref{sec:query-priestess}) & history segments are mapped readable and writable, not executable, and the kernel never directs the program counter into one & the rest of K4; a break is a type bypass \\
\rowrule
I-8 & No page fault during an Oracle operation (\S\ref{sec:query-priestess}) & before each entry the kernel commits the pages the operation may write & a fault inside a transition, a path the machine lacks \\
\rowrule
I-9 & One runner: a task is in O-mode on at most one processor, and no processor writes it meanwhile (Appendix~\ref{app:archnights}) & trivial in ArchNights-SE, with one dispatch point on one processor; a scheduler invariant for a kernel on several processors, not yet built & Corollary~\ref{cor:crossing}, P to O only by a return; K1 \\
\rowrule
I-10 & One trap label: the kernel issues $\mathtt{uret}$ at one place (Definition~\ref{def:t}) & PRTS issues it from one function, which its dispatch loop calls & K3 holds at each place, not at one \\
\grouprule
\multicolumn{4}{@{}l}{\textit{II. The return of control: K5 (Proposition~\ref{prop:k5})}} \\
\rowrule
II-1 & The budget is checked at the boundaries of transitions, Steps or Frames, and counted in the Oracle's steps (Definition~\ref{def:budget}) & the controller; the budget is soft, and an Oracle that offers only RUN overruns it by at most the Frame it is in & the bound of K5 \\
\rowrule
II-2 & A fault ends the operation and traps, in place of a good machine (Definition~\ref{def:t}) and a total Oracle (\S\ref{sec:machine-oracle}) & every exception of a unit reaches the kernel at the entry of a yield, with its cause in the destination register (Appendix~\ref{app:archnights-exc}) & K5 on a machine that is not good \\
\bottomrule
\end{tabular}
\caption{What a realization of the \mach{O-2-PDA-TSV} must keep, and where ArchNights keeps it, by design: layers I and II, the
behaviours the machine fixes. The rows are grouped by what is lost when one fails, and the last column names the result that then
no longer carries over to the realization.}
\label{tab:obligations}
\end{table*}

\begin{table*}[t]
\centering
\small
\setlength{\tabcolsep}{4pt}
\begin{tabular}{@{}>{\raggedright\arraybackslash}p{0.045\textwidth}>{\raggedright\arraybackslash}p{0.39\textwidth}>{\raggedright\arraybackslash}p{0.33\textwidth}>{\raggedright\arraybackslash}p{0.17\textwidth}@{}}
\toprule
 & What a realization keeps & In ArchNights & Lost if it fails \\
\midrule
\multicolumn{4}{@{}l}{\textit{III. Hypotheses of the conditional results of \S\ref{sec:eclipse}, which the machine does not fix}} \\
\rowrule
III-1 & O-mode runs only $A_{\mathrm{AR}}$ Oracles, the O constraint (\S\S\ref{sec:query-oracle} and~\ref{sec:query-control}) & PROBE reports it, and O-mode operations on other Oracles are refused & the cache model for that task, not the contract \\
\rowrule
III-2 & The encoder accepts every word over $\Sigma_Q$, $\Sigma_R \subseteq \Sigma_Q$, and the kernel writes only words over $\Sigma_Q$ (Proposition~\ref{prop:closure}) & not met by a language model: its input is bounded (\S\ref{sec:query-oracle}), and some histories, such as a request without its reply, fail to encode & a good machine, which II-2 replaces for K5 \\
\rowrule
III-3 & Append-only epochs: within one, the kernel changes a history only by $\mathtt{push}$ (\S\ref{sec:eclipse-memory}) & kernel code changes a history only through push, pop and its metadata, and a direct write, not prevented, is undefined behaviour; exec, a swap and a back-off start new epochs & Proposition~\ref{prop:prefix-chain}; the coherence of the encoding cache \\
\rowrule
III-4 & Every word the kernel pushes onto a history ends with $\natural$ (Proposition~\ref{prop:prefix-chain}) & kept for an Oracle with STEP disabled, which is called through the Frame interface & Proposition~\ref{prop:prefix-chain} \\
\rowrule
III-5 & The encoder factorizes at $\natural$, and the Oracle is Frame-faithful (Definition~\ref{def:factor}, Theorem~\ref{thm:run-sync}) & not claimed: properties of each Oracle and its codec, which ArchNights-SE does not yet analyse (\S\ref{sec:future}) & the Frame-boundary case of Proposition~\ref{prop:excess}; Proposition~\ref{prop:prefix-chain} \\
\grouprule
\multicolumn{4}{@{}l}{\textit{IV. Choices the machine leaves to the kernel, made by PRTS}} \\
\rowrule
IV-1 & Every request is answered before its task runs again, the kernel's protocol (\S\ref{sec:eclipse-safety}) & the scheduler preempts only before an Oracle operation, and serves every call of an answer first & no result of the machine \\
\rowrule
IV-2 & The library: the declaration $\delta_\Lambda$ on every task stack, and requests outside $L_\Lambda$ handled as errors (Definition~\ref{def:os}) & the first record of every history, generated with the check from one table of operators; an error reply that names what was wrong & no result of the machine \\
\rowrule
IV-3 & Services that rebuild a history, through which a query can recur (\S\ref{sec:eclipse-memory}, Appendix~\ref{app:eclipse-replay}) & fork builds the prefix tree; exec and a swap start new epochs, a swap only at a quiet point; a replay policy can limit them & no result of the machine \\
\rowrule
IV-4 & One identifier for a task on both sides (\S\ref{sec:eclipse-control}) & $\mathtt{uret}$'s registers tag each Oracle operation with the process identifier & no result of the machine \\
\bottomrule
\end{tabular}
\caption{Layers III and IV: hypotheses of the conditional results of \S\ref{sec:eclipse}, which the machine does not fix, and
choices it leaves to the kernel. The columns are those of Table~\ref{tab:obligations}.}
\label{tab:conditions}
\end{table*}

We say that ArchNights implements the \mach{O-2-PDA-TSV} in this sense: for every behaviour the machine fixes,
Table~\ref{tab:obligations} names the obligation it puts on a realization and where ArchNights keeps it; for every hypothesis of a
result that holds only under it, Table~\ref{tab:conditions} states whether ArchNights meets it; and where the machine leaves a choice, ArchNights makes
one. The verb \emph{realizes} is that of Definition~\ref{def:segmented}, where a memory state realizes a store. By design, and
informally, we extend it to the rest of a configuration and to transitions. A state of ArchNights realizes a configuration when its
memory realizes the store, the processor is in the Oracle mode exactly when the configuration is in O-mode, and in O-mode the
registers of $\mathtt{uret}$ describe the task stack $u$ and its budget. Definition~\ref{def:segmented} lays a stack out growing up.
A segment may as well grow down: which way it grows is an arbitrary choice, fixed in O-mode by the system's configuration, as the
direction of a process stack is fixed (Appendix~\ref{app:archnights}). An Oracle operation that publishes an answer realizes an
Oracle transition, a Step or a Frame; $\mathtt{uret}$ realizes the entry, and its completion the trap. Priestess code between them
is taken in big steps, as a black box (\S\ref{sec:query-control}). We claim no simulation between ArchNights and the machine. A
refinement would be a separate undertaking, and on this machine a delimited one. Its
abstraction is the relation \emph{realizes} above, with Priestess code taken in big steps, and each row of layers I and II of
Table~\ref{tab:obligations} is a lemma it must prove. The results that it would carry over hold for every Oracle, so it would cover
the processor and the kernel, and never the Oracle.

The obligations are derived from the machine, not from ArchNights. Each transition rule says what a realization must not do: an
O-mode transition only appends to its stack, so an Oracle operation writes nothing outside its window. Each result says what it
takes from the rules, and that is what a realization has to supply: K4 requires jumps to literal labels and a program that no rule writes, and
a von Neumann computer replaces them by the fetch condition. And each path that a realization has and that the machine lacks is
eliminated, absorbed into a transition the machine has, or left outside the model. A page fault during an Oracle operation is
eliminated, an indirect jump is the Priestess's own decision, as a promotion is (\S\ref{sec:query-priestess}), and a clock is
outside the model.

Tables~\ref{tab:obligations} and~\ref{tab:conditions} group the obligations into layers by what is lost when one fails. Layers I and II are the behaviours the machine
fixes, and they carry the kernel contract, the crossing invariant and the safety boundary over to a realization. The first five
rows fall to the processor, whatever kernel runs on it. The next five depend on what the kernel does with the processor: its layout
(which realizes S, because the hardware knows no colours), its page mappings, where it dispatches tasks and where it issues
$\mathtt{uret}$. A realization moves them from the rules to the kernel (\S\ref{sec:query-priestess}), so on ArchNights they carry
over only on a kernel that keeps them, and PRTS keeps them by design, not by proof; on the machine they hold for every kernel. Layer III
lists the hypotheses of the results of \S\ref{sec:eclipse} that hold only under them. The machine does not fix them, and a refinement would extend only to the kernel's part of them; the rest are properties of an Oracle
and its codec. Layer IV lists the choices the
machine leaves to the kernel, which PRTS makes.

One row departs from the machine. ArchNights is not a good machine in general: a language model accepts only bounded input, and its
encoder rejects a history that outgrows the bound, or one it cannot accept, such as a request without its reply
(Appendix~\ref{app:archnights-exc}). ArchNights keeps the return of control by absorbing faults instead (row II-2). In O-mode every
exception of an Oracle unit ends the operation, publishes no part of a transition (row I-3), and reaches the kernel at the same
entry as a yield, with its cause in the destination register. On the machine, this adds an O-mode transition that,
when the encoder rejects the history or the Oracle fails, goes to $\ell_{\mathrm{trap}}$ and leaves the store as it was. Every trap still
goes to $\ell_{\mathrm{trap}}$ and no O-mode transition writes outside $u$, so the proofs of K1--K4 and of
Corollary~\ref{cor:crossing} apply unchanged. Every O-mode configuration then has a successor, which traps or lowers the budget,
so K5 holds without the hypothesis of a good machine.

The tables end where the machine's model ends. K5 bounds transitions, not time, and the machine has no clock, so the latency of a
call, a channel that a realization adds, lies outside the model (Appendix~\ref{app:eclipse}), as do the side channels of
speculation. The machine has one processor, and concurrency enters it only as far as row I-9 brings it back to
the synchronous semantics. Its statements are about sets of possible runs, not their probabilities, and it says nothing about the
Oracle's interior (a retry inside a unit included), of the finite capacity of caches, or of users inside the Priestess. None of
these yields an obligation (\S\ref{sec:future}).

\subsection{The Instruction Set}
\label{sec:an-isa}

The instruction set of ArchNights extends the RV64 ISA~\cite{waterman2019risc,waterman2021risc} with a privilege level, the Oracle mode, below user
mode; the instruction $\mathtt{uret}$ that enters it, with the registers and descriptors that go with it; and the Oracle instruction,
which performs the Oracle operation (Table~\ref{tab:isa}). The Oracle mode realizes O-mode and has no addressable code:
in O-mode the processor fetches no instruction, and its fetch--decode--execute loop drives the Oracle's transitions, as Insight~\ref{ins:one-instruction}
requires. ArchNights' simulated processor runs the Oracle mode in its microcode.

\begin{table*}[t]
\centering
\small
\begin{tabular}{@{}>{\raggedright\arraybackslash}p{0.1\textwidth}>{\raggedright\arraybackslash}p{0.16\textwidth}>{\raggedright\arraybackslash}p{0.27\textwidth}>{\raggedright\arraybackslash}p{0.39\textwidth}@{}}
\toprule
Instruction & Encoding & Operands & What it does \\
\midrule
PROBE & custom-0, R-type & none & reports the profiles the unit implements, whether answers carry a descriptor, and what the bound Oracle can do \\
\rowrule
STEP & custom-0, R-type & a segment, its top and a bound on what one step may write & in P-mode, one step of the Oracle's own autoregression, appended past the top \\
\rowrule
RUN & custom-0, R-type & a segment, its top and a budget & in P-mode, a Frame, appended past the top \\
\rowrule
$\mathtt{ecall}$ & none; O-mode fetches no instruction & none & absorbed into O-mode's transitions: an answer that ends with the yield symbol traps to the kernel as an $\mathtt{ecall}$ would \\
\rowrule
$\mathtt{uret}$ & custom-1 & none & enters O-mode on the task its registers describe; completes when the excursion ends \\
\grouprule
\textbf{Result} & the destination register & & how the operation ended, how many of the Oracle's steps it took, how many bytes it published \\
\bottomrule
\end{tabular}
\caption{ArchNights' Oracle instructions.}
\label{tab:isa}
\end{table*}

\paragraph{The Oracle instruction.} The Oracle instruction is a RISC-V custom instruction whose interface is compatible with RoCC,
Rocket Chip's interface for custom coprocessors~\cite{asanovic2016rocket}. It is an R-type instruction in the custom-0 opcode
space, and its fields select the operation, the direction in which the stack grows, and a profile that fixes the widths of the
fields packed into the registers (Appendix~\ref{app:archnights}). Each operation is self-contained, fixed by the instruction and
its two source registers, so a context switch has nothing of the unit to save. The operand is a segment with the three pointers of
\S\ref{sec:query}. The instruction reads the content, writes the answer past the top, and never changes what is already there. A
descriptor written past the answer, and not counted in its bytes, records the Oracle's own measures (for a language model, the
numbers of input and output tokens). STEP and RUN are the two granularities of \S\ref{sec:eclipse}: RUN generates one Frame, and STEP
performs one step of the Oracle's own autoregression, aligned with its internal state, so that a debugger on the Priestess side can
single-step the Oracle. PROBE reports what the bound Oracle can do, including whether it is internally autoregressive, which
decides whether it may run in O-mode. This report is the declared interface of the O constraint. An O-mode operation on an Oracle
that does not allow it is refused before evaluation, as is a STEP on an Oracle that has none.

\paragraph{Entering and leaving O-mode.} $\mathtt{uret}$ takes no operands, as an xRET takes none. The task it enters is described by
the registers that go with it: the task's segment with its top, the budget and the Oracle entity to run. The kernel writes them
before entering a task, and they tag each of the unit's operations with the task's identifier, so that one identifier serves both
sides (\S\ref{sec:eclipse-control}). $\mathtt{uret}$ is synchronous: it completes when the excursion ends (at a yield, when the
budget is spent, or on an exception), and the destination register records which. The processor then continues at the instruction
after that $\mathtt{uret}$. That instruction plays the part of $\ell_{\mathrm{trap}}$: the kernel's code there is the service entry. Its
address is fixed by the kernel's code, and no answer of the Oracle can select it (K3). Where Insight~\ref{ins:user-kernel} has one trap
vector, ArchNights resumes after the $\mathtt{uret}$ that entered the task. A kernel thus has one entry for each site where it issues
$\mathtt{uret}$. PRTS issues it from one function, which its dispatch loop calls, and so has one entry, as Definition~\ref{def:t}
has one $\ell_{\mathrm{trap}}$.

\paragraph{The budget and the Oracle unit.} The budget is counted in the Oracle's own steps, which for a language model are output
tokens. It is soft: the unit checks it at each boundary where it can hand control back, which is after a Step for an Oracle that
offers STEP and after a Frame for one that offers only RUN. Under Definition~\ref{def:budget}, each such boundary is an O-mode
transition, and the excursion traps at the first boundary after the budget is spent. An Oracle that offers only RUN is therefore
preempted between its Frames, never inside one, and overruns its budget by at most the Frame it is in. In hardware an Oracle is a
separate unit, and a processor may have several. A unit has the encoder, the compute part and the decoder of
Definition~\ref{def:oracle}, together with a controller. Before evaluation, the controller checks each operation against its window
and against what the Oracle allows; after evaluation, it publishes the answer whole. An operation's result in the destination register
reports what the operation has published, also when an exception ends it. An answer that does
not fit is refused and writes nothing (Proposition~\ref{prop:segment-locality}). Each part may raise exceptions of its own, from an illegal argument to the loss
of a model's service (Appendix~\ref{app:archnights-exc}). Apart from the budget trap, they are what a real Oracle adds to the machine,
in which the Oracle is total and faults were set aside in \S\ref{sec:os}. In ArchNights they end the operation and reach the kernel at the same
entry as a yield, with the cause in the destination register. ArchNights implements two classes of Oracle. The entities of the LLM
class are language models, which are $A_{\mathrm{AR}}$ Oracles (\S\ref{sec:query-oracle}). The Human class is a person at the
terminal. It is called in P-mode, and its answers are lines of text that arrive one character at a time.

\paragraph{Synchrony.} The Oracle operation in ArchNights is synchronous, as an Oracle transition is one step of the machine: while it runs,
the unit is occupied, and so is the CPU that issued it. An asynchronous realization can preserve that semantics and implement the asynchrony in
the microarchitecture, behind an abstraction of simultaneous multithreading~\cite{tullsen1995simultaneous}, or in the operating system,
provided no task is in O-mode on two processors at once.

\subsection{PRTS: the Kernel of ArchNights-SE}
\label{sec:an-prts}

PRTS (Primitive Reference of agenTic System), booted by its firmware, provides the system calls of the simulated machine, and
its design follows Linux~\cite{linuxkernel}. Its core does not depend on the Oracle: code that does is a driver for the Oracle's
class, instantiated for each entity, and the semantic delegation calls of \S\ref{sec:query} are loaded into the kernel as modules.
By design PRTS runs in P-mode, where it issues $\mathtt{uret}$ and serves each yield trap at the instruction after it.
ArchNights-SE realizes P-mode by the user mode of the simulated processor, in which gem5's SE mode runs every program. The
services that touch the host's resources are system calls that gem5 intercepts and the host serves in place of a supervisor. PRTS
keeps the Unix philosophy that everything is a file~\cite{ritchie1974unix}, but wraps its descriptors in strings. A task names
each file it opens and refers to it by that name. The name stands for a concrete file to the Oracle and for a single descriptor to
the kernel, and it is usually harder to mistake for another argument of a call than a number is. Below, we discuss the parts of Linux that a
machine with an O-mode changes.

\paragraph{Process and execution.} Four parts of Linux's process management change on a machine with an O-mode: the process
itself, fork, exec and the scheduler.

\subparagraph{A process is a history.} A program in ArchNights is text, and the Oracle runs it by continuing a history. PRTS's
\texttt{task\_struct} has two halves. The Oracle half is the task stack, the history the Oracle runs. The Priestess half is what
the kernel keeps for the task: its descriptors, working directory, credentials and scheduling state. The Oracle reaches the
Priestess half only through the interface (its own arguments and environment, for instance) or not at all. One process identifier
names both halves (\S\ref{sec:eclipse-control}). The memory descriptor, \texttt{mm\_struct}, gains the history's segment. The segment
is reserved with its three pointers when the task is created, and committed lazily: before each entry the kernel commits the pages
the operation may write. This is the second scheme of \S\ref{sec:query}, in which the segment occupies the Priestess's virtual
addresses and no page fault arises during an Oracle call. The segment is mapped readable and writable but not executable, and the
kernel never directs the program counter into it, which keeps the fetch condition. Kernel code changes a history only through push,
pop and the operations on its metadata. Reading or writing a history's memory directly is not prevented, but is declared undefined
behaviour.

\subparagraph{Fork.} PRTS's fork copies the Priestess half as Linux's does, without the credentials'
administrative rights and with fresh scheduling state, and copies the history into a new segment. Both tasks then receive a reply
to the same request: the parent receives the child's identifier, and the child receives zero. Fork's double return is here two
returns to one force, and the child's history is the parent's up to that request, followed by its own reply, so fork is a kernel
service that builds the prefix tree of \S\ref{sec:eclipse}. If the answer that called fork contained other calls, they are
cancelled in both tasks, since otherwise each would be served twice. In place of inherited pipe ends
(Appendix~\ref{app:archnights-prts}), fork creates a result channel. The child holds the write end, which never blocks, and the parent
holds the read end. Nothing the child writes reaches the parent's history unless the parent reads it.

\subparagraph{Exec.} Linux's exec lays out an address space and an initial stack; PRTS's exec lays out a history.
The loader reads one of three inputs: a text file, a directory with the program's text and a list of the libraries it uses, or a
saved history. It
frees the old history and builds a new one, whose first three records replace the initial stack. The \emph{declaration}
lists the operators the task's credentials allow, as the auxiliary vector says what the system offers. The \emph{system record} holds
the system's conventions, the program's text and the declarations of the libraries it names. The \emph{user record} holds the
program's path, arguments, working directory and standard streams. A saved history is loaded as it was, with a fresh declaration,
so that every history declares exactly the operators its kernel serves. As in Linux, exec keeps the process identifier, the
working directory and the descriptors not marked to close on exec, keeps only the administrative rights its caller hands on, and
on failure leaves the old program in place. The new history begins a new prefix chain, and shares with other histories only its
first records.

\subparagraph{The scheduler.} PRTS's scheduler is the Priestess-side scheduler of \S\ref{sec:eclipse-control}. From one ready queue
at one dispatch point, it chooses the task to enter with $\mathtt{uret}$ and the slice to enter it with, which is the budget. A
slice is counted in the Oracle's own steps, which the unit's descriptor reports, and the cost of a task's P-mode calls (a read
from the terminal, or a question to another Oracle) is charged to it as well. A task that blocks keeps what is left of its slice,
which is refilled only when it is used up, and a task that wakes joins the end of the queue without preempting the one that runs.
When no task is ready, the kernel idles until the next sleeper wakes, and when no task is left, it shuts the system down.
Preemption happens only before an Oracle
operation, and a task whose answer carries calls has all of them served before it can be preempted, so every request is answered
before its task runs again, as the kernel's protocol of \S\ref{sec:eclipse-safety} requires. Because the Oracle operation is
synchronous, a read from the terminal blocks the whole machine until the person answers. A shell about to read the terminal
therefore yields while other tasks are ready, and the person's input has the lowest priority.

\paragraph{Exception handling.} An exception is handled as close as possible to where it arises, by the side that
can handle it, and the machine admits an exception into a task in only one way. O-mode has no program counter, so it has no handler to
jump to and no signal can be delivered into it. By the crossing invariant, the Priestess reaches a task only by a return. If
the Oracle is to know of an exception, the exception must therefore reach it as content appended to its history. The Oracle sees an exception exactly
when the Priestess sends a notice, and the Priestess sends one when the Oracle can act on it, as on most errors in the arguments of
a system call. Other exceptions are handled in the Priestess, and the task is not told (Appendix~\ref{app:archnights-exc}).
Communication between tasks keeps to the same rule: content from other tasks enters a history only through an explicit, bounded
read by the task that owns it (Appendix~\ref{app:archnights-comm}).

\paragraph{Virtual file system.}
\label{sec:an-vfs}
In SE mode the guest's file-system calls reach the host's file system directly, so PRTS gives O-mode a virtual one. At boot a
folder of the host is mounted as the root, and other directories, a workspace and a temporary directory among them, are mounted
beside it. Like Unix with its ordinary user and root, PRTS has two O-mode users: an ordinary user (identifier 1000) and a privileged
user (identifier 0). Access is checked by Unix's discretionary access control, so that with the ordinary user's rights every
program operates only on the virtual file system. Only the privileged user may escape it, through the host's shell Bash. The way
to the privileged user is $\mathtt{sodo}$, a core tool of the distribution (Appendix~\ref{app:archnights-dist}). The host's shell
is outside the machine, and to the machine it is an Oracle: it is called in P-mode, and its answers are data (K4).

\paragraph{Checkpointing system.}
\label{sec:an-checkpoint}
The checkpointing system saves a task's history, with a metadata file that describes the task, as a folder that the loader
can load. It supports recovery from faults, debugging and swap, and for debugging the metadata may also record a
limited part of the task's P-mode state. For recovery, a process ended by an unrecoverable error
is saved as a checkpoint, and swap compresses a history that its Oracle can no longer accept.
Loading a checkpoint gives a new task a history equal to an old one, so it is one of the kernel services through which replay can
be constructed, and a replay policy can limit it. A history does not swap as a page does: a page is swapped out whole and faults
back in when it is touched, but the Oracle reads its whole history at every step, so no part goes untouched. What can be swapped
is a part the Oracle no longer needs to read every time. The kernel swaps at a quiet point, when every request has its reply, so
that no request is cut off from its reply. It keeps the head and the tail. The head consists of the declaration and the first
records and says what the task is, and the tail runs from the last answer on and says where it is. The middle moves to a
checkpoint, and in its place the kernel leaves a short note. The note is a thunk of the memory moved out: it records what was there
and where it went, and each entry of the note records the location of its checkpoint, so bringing one back takes the Oracle a single
read, a force prepared in advance. The compressor either keeps everything but the chain of thought, a removal by the kernel in
the sense of Appendix~\ref{app:eclipse-reasoning}, or summarizes the middle with
one Oracle call in P-mode, which states what the task is, what has been done (with its exact facts) and what is left. That call is
the kernel's own, and its answer is data (K4) that enters the history only because the
kernel writes it. This is the context compression the machine left out, and \S\ref{sec:eclipse} names its price: a swap rewrites
the task, the append-only epoch ends, the prefix chain breaks after the head, and the caches keyed by the old history serve
nothing past it. So the policy, a module of the distribution, swaps rarely, at a threshold, and not again until the history has
grown well beyond it.

\subsection{The Distribution}
\label{sec:an-dist}

ArchNights has a usable distribution: the PRTS kernel, a system library, an init system, a shell and core tools. It is designed
to be reused across kernels: PRTS is meant to be replaceable by ZOOT with only the system library's kernel module reinserted. Its
library and programs each have two parts: programs that run in O-mode and do not depend on the kernel, and delegation calls that
run in P-mode. The kernel and the system library together are the operating system of Definition~\ref{def:os}.

\paragraph{System library.} The system library serves O-mode, in the place glibc~\cite{gnuglibc} has in Linux, with one
difference: its system-call wrappers and semantic wrappers are semantic delegation calls (\S\ref{sec:query}), implemented by a
module the library inserts into the kernel. To O-mode an operator is a signature, as a tool call is; to P-mode it is code that
merges several of PRTS's internal system calls into one operator. Table~\ref{tab:operators} lists the operators, and
Appendix~\ref{app:archnights-dist} says why the Oracle sees them and not those calls. The system library is the library
$\Lambda$ of Definition~\ref{def:os} and has exactly its parts. Its declaration $\delta_\Lambda$, the first record of every
history, lists the operators with their signatures. Its request language $L_\Lambda$ consists of two kinds of word: the
well-formed sequences of operator calls, which together are the request of one trap; and an answer that calls nothing, which asks
the kernel to deliver its text as the task's result. Its Priestess-side part is the module in the kernel. The declaration, the calls and the check of each
call's arguments are generated from one table of operators, so the declaration aligns with the check. A semantic delegation call
can also be registered in a kernel module and called as an operator, as \texttt{fb\_play} in Figure~\ref{fig:pv} is. The libraries
that the system library loads are text, linked lazily as a dynamic linker binds a function at its first call. A library's short
declaration is in the system record from the start, and its full text enters the history only when the program loads it, with one
call. Agent harnesses call such libraries skills.

\begin{table*}[t]
\centering
\small
\begin{tabular}{@{}>{\raggedright\arraybackslash}p{0.15\textwidth}>{\raggedright\arraybackslash}p{0.27\textwidth}>{\raggedright\arraybackslash}p{0.52\textwidth}@{}}
\toprule
Family & Operators & What one call does \\
\midrule
Standard streams & \texttt{input}, \texttt{print} & read standard input, a named file or a result channel, as far as there is data, with end of file reported together with the last data; write a whole text to standard output or to the parent \\
\rowrule
Programs & \texttt{run}, \texttt{fork}, \texttt{execve}, \texttt{waitpid}, \texttt{exit}, \texttt{yield} & run a command line as a pipeline; the process calls of Unix \\
\rowrule
Names & \texttt{open}, \texttt{close} & give a file or a directory a name, or drop it \\
\rowrule
Editing & \texttt{file\_read}, \texttt{file\_write}, \texttt{file\_patch} & read, write or patch a file in one call, guarded by a stamp of the version read \\
\rowrule
Paths & \texttt{stat}, \texttt{list}, \texttt{mkdir}, \texttt{remove}, \texttt{rename}, \texttt{chdir} & one file-system operation each \\
\rowrule
Information & \texttt{context}, \texttt{time}, \texttt{sleep} & the task's own arguments and environment; the simulated clock \\
\rowrule
Oracles and libraries & \texttt{ask}, \texttt{load\_skill} & call a person or a language model in P-mode, for an answer that is data; load a library's definition \\
\rowrule
Host & \texttt{bash} & run a command on the host, with the semantics of $\mathtt{sodo}$: every call goes through $\mathtt{sodo}$'s check \\
\bottomrule
\end{tabular}
\caption{The operators of the system library. \texttt{fb\_play} in Figure~\ref{fig:pv} is a further semantic delegation call,
registered in a kernel module.}
\label{tab:operators}
\end{table*}

\paragraph{Init system.} The init system, the first program to run in O-mode, initializes the system, starts the shell and starts
it again whenever it ends while there is a terminal to read, and adopts every orphan process. The right to shut the system down
belongs to the shell the init system starts, not to the shell's program, since a program is text that any task may run. In
ArchNights-SE the init system's loop is a semantic delegation call, but an init system that runs entirely in O-mode is equally
legitimate.

\paragraph{Shell.} The shell is the second program to run in O-mode, the system's interface to its
user, and it follows the logic of mainstream shells. It is a text that the Oracle runs as a task and that says how to be a shell.
What makes it a shell is a P-mode part, its session, which the kernel runs when the shell waits for a line or a job. The shell
is therefore a hybrid in the sense of Definition~\ref{def:forms}. Its Oracle part is the Oracle with the operators, and its Priestess
part reads the person, runs builtin commands, and starts and collects jobs, with no Oracle in its loop. The session reads a line
by calling the Human class in P-mode, so a line is an answer, and by K4 it is data. It becomes a request to the Oracle part only
because the session writes it into the history, as the next user record. The shell's state is its conversation: a line in
natural language runs in the same task and history, where the context accumulates until the person clears it; a program runs as a
child, isolated from the shell. A line that begins with \texttt{/} is a builtin, run without the Oracle. A line that begins with
\texttt{!} runs on the host, through $\mathtt{sodo}$, and only if the person typed it at the prompt: lines from scripts, profiles or
a program's input never do. Any other non-empty line is a request in natural language. What the person does at the prompt reaches the Oracle only as one-line events, with no output: the
Oracle learns that the person ran a program and how it ended, not what it printed. Appendix~\ref{app:archnights-dist} gives the
builtins, the session's loop, one history record by record, and the core tools with $\mathtt{sodo}$'s check.

\begin{figure*}[t]
\centering
\begin{tikzpicture}
  \node[anchor=south west, inner sep=0] (img) at (0,0) {\includegraphics[width=0.82\textwidth]{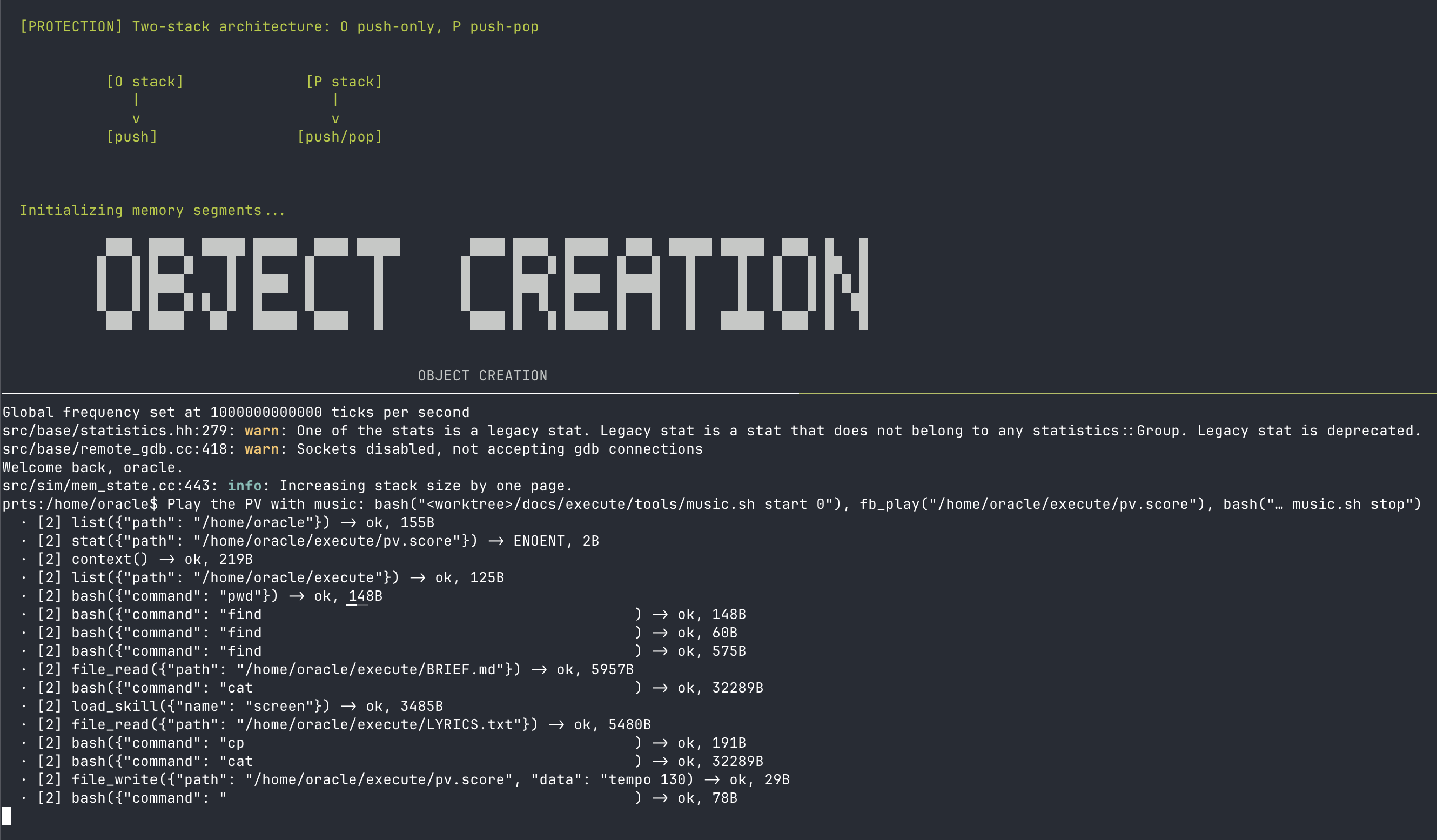}};
  \begin{scope}[x={(img.south east)}, y={(img.north west)},
                ph/.style={anchor=west, inner sep=0, text=black!20, font=\fontsize{5.2}{6}\ttfamily}]
    \foreach \px/\py/\t in {488/1141/{<host path>}, 488/1175/{<host path>}, 488/1209/{<host path>},
                            470/1277/{<host path>}, 452/1379/{<host path>}, 470/1413/{<host path>},
                            420/1481/{<host command>}}
      \node[ph] at ({(\px+10)/2660}, {1-\py/1555}) {\t};
  \end{scope}
\end{tikzpicture}
\caption{A session of ArchNights-SE making the music video of \emph{world.execute(me);}. Above: the simulated display
showing the video; its text and diagrams, such as the two stacks, belong to the video, not to this paper. Below: gem5's
messages and a session of the shell on PRTS, with a request at the prompt, in natural language, that names the calls it wants
(\texttt{fb\_play} among them), and the operator calls the task makes, each with its result. The session's user \texttt{oracle}
is the ordinary O-mode user (\S\ref{sec:an-vfs}). Every call of \texttt{bash} goes through $\mathtt{sodo}$, whose audit is
turned off in this run, as in all our tests, and \texttt{fb\_play} is a semantic delegation call registered in a kernel module
(\S\ref{sec:an-dist}). Host paths and host commands are replaced by placeholders.}
\label{fig:pv}
\end{figure*}

\subsection{ZOOT: the Kernel of ArchNights-FS}
\label{sec:an-zoot}

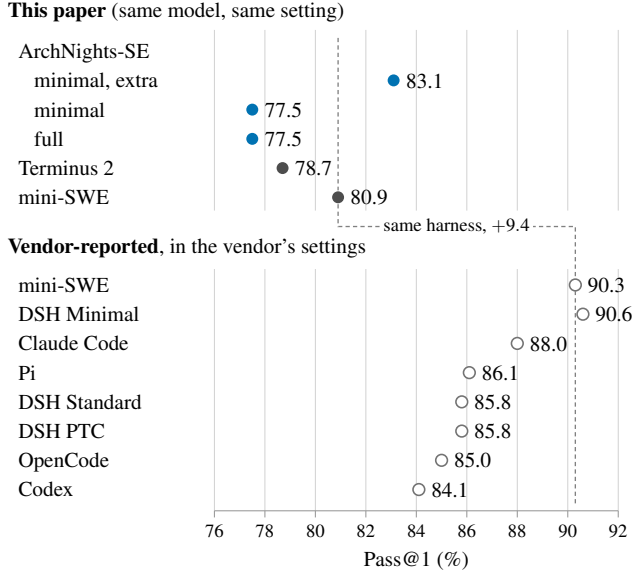
\begin{figure}[t]
\centering
\begin{tikzpicture}[x=1pt, y=1pt, font=\footnotesize,
  name/.style={anchor=base west, inner sep=0pt},
  val/.style={anchor=west, inner sep=0pt},
  head/.style={anchor=base west, inner sep=0pt},
  grid/.style={rulegray, line width=0.3pt},
  ref/.style={black!50, line width=0.5pt, dash pattern=on 1.6pt off 1.2pt}]
  \def\X#1{78+9.5*(#1-76)}
  \def\B#1{#1-2.6}
  \foreach \v in {76,78,...,92} {
    \draw[grid] ({\X{\v}},-5) -- ({\X{\v}},-73);
    \draw[grid] ({\X{\v}},-95) -- ({\X{\v}},-186);
  }
  \node[head] at (0,0) {\textbf{This paper} (same model, same setting)};
  \node[name] at (4,{\B{-13}}) {ArchNights-SE};
  \foreach \y/\lab/\v in {-24/{minimal, extra}/83.1, -35/{minimal}/77.5, -46/{full}/77.5} {
    \node[name] at (10,{\B{\y}}) {\lab};
    \fill[pcol, draw=white, line width=0.7pt] ({\X{\v}},\y) circle (2.6pt);
    \node[val] at ({\X{\v}+4.5},\y) {\v};
  }
  \foreach \y/\lab/\v in {-57/{Terminus 2}/78.7, -68/{mini-SWE}/80.9} {
    \node[name] at (4,{\B{\y}}) {\lab};
    \fill[black!70, draw=white, line width=0.7pt] ({\X{\v}},\y) circle (2.6pt);
    \node[val] at ({\X{\v}+4.5},\y) {\v};
  }
  \node[head] at (0,-89) {\textbf{Vendor-reported}, in the vendor's settings};
  \foreach \y/\lab/\v in {-101/{mini-SWE}/90.3, -112/{DSH Minimal}/90.6, -123/{Claude Code}/88.0,
                         -134/{Pi}/86.1, -145/{DSH Standard}/85.8, -156/{DSH PTC}/85.8,
                         -167/{OpenCode}/85.0, -178/{Codex}/84.1} {
    \node[name] at (4,{\B{\y}}) {\lab};
    \draw[black!55, fill=white, line width=0.6pt] ({\X{\v}},\y) circle (2.2pt);
    \node[val] at ({\X{\v}+4.5},\y) {\v};
  }
  \draw[ref] ({\X{80.9}},-8) -- ({\X{80.9}},-65.4);
  \draw[ref] ({\X{80.9}},-70.6) -- ({\X{80.9}},-79) -- ({\X{90.3}},-79) -- ({\X{90.3}},-98.8);
  \draw[ref] ({\X{90.3}},-103.2) -- ({\X{90.3}},-184);
  \node[fill=white, inner sep=1pt, font=\scriptsize] at ({\X{85.6}},-79) {same harness, $+9.4$};
  \draw[black!40, line width=0.4pt] ({\X{76}},-186) -- ({\X{92}},-186);
  \foreach \v in {76,78,...,92} {
    \draw[black!40, line width=0.4pt] ({\X{\v}},-186) -- ({\X{\v}},-188.5);
    \node[anchor=north, inner sep=0pt, font=\scriptsize] at ({\X{\v}},-191) {\v};
  }
  \node[anchor=north, inner sep=0pt] at ({\X{84}},-201) {Pass@1 (\%)};
\end{tikzpicture}
\caption{Terminal-Bench 2.1, Pass@1 by harness, with deepseek-v4.1-flash as the Oracle. This paper (filled dots,
ArchNights-SE in blue, in the configurations of Table~\ref{tab:tb}): one trial of each of the 89 tasks, with the same model in
the same setting. Minimal, extra: the minimal system with the extra prompt; otherwise the minimal and the full system use the
default prompt. Vendor (hollow dots): as reported in~\cite{deepseekai2026deepseek}, Table 4, in the settings given there. In
each group the dashed line marks the score of mini-swe-agent, the harness both evaluations run, and it joins the two scores.
mini-SWE is mini-swe-agent, and DSH is DeepSeek Harness.}
\label{fig:tb}
\end{figure}

ZOOT (Zeroed Ordered Oracle sysTem), booted by its firmware, is a design. Its functions are nearly those of PRTS, but where
PRTS cannot reach a real kernel, ZOOT is to run together with one. The design follows the principle that several privilege levels
of a machine multiplex one task and several tasks multiplex one kernel. So ZOOT runs in U-mode, one instance per user, as the
kernel of the O-mode below it, with Linux in S-mode and the firmware in M-mode. O-mode may then have several users, which brings
back, one level down, the old question of how the users of one privilege level relate to those of another. ZOOT answers it as
kernel-isolating virtualization~\cite{barham2003xen} does for the users of U and S: each user has a U-mode kernel of its own, the
kernels share part of their state, and operations on the shared state need a higher privilege level. In the terms of
Definition~\ref{def:os} this is one operating system $(K, \Lambda)$ for each user, with Linux in S-mode as the arbiter above them.
What the kernels must share is Priestess-side state, which K1 and K2 already keep out of every task. Appendix~\ref{app:archnights}
gives the alternatives and ZOOT's design of tasks and memory, its file system, its boot and its bridge to Linux.

\subsection{Evaluation of ArchNights-SE}
\label{sec:an-eval}

We evaluate ArchNights-SE to show that the \mach{O-2-PDA-TSV} can serve as the basis of an agent harness, and we compare the
results with those reported by DeepSeek, the vendor of the model we use as the Oracle. DeepSeek evaluates Terminal-Bench 2.1
without network access~\cite{deepseekai2026deepseek}, and to narrow the gap between the two evaluations we come as close to that as
the tasks allow. Some of the settings still cannot be matched, so their effect can be shown only through the one harness that both
evaluations run, mini-swe-agent. The results show that a prototype built on the machine solves tasks as an agent harness does.

\paragraph{Common assessment.}
\label{sec:an-bench}
On Terminal-Bench 2.1~\cite{merrill2026terminal}, with deepseek-v4.1-flash as the Oracle, ArchNights-SE scores between 77.5 and
83.1 in Pass@1, depending on its configuration. With the same model and the same setting, mini-swe-agent~\cite{minisweagent2025},
a harness with a single bash tool, scores 80.9, and Terminus~2, the benchmark's own neutral agent~\cite{merrill2026terminal}, scores
78.7 (Figure~\ref{fig:tb}). Each score is the Pass@1 of one trial of each of the 89 tasks, so one task is 1.1 points. In that
setting the model reasons at maximum
effort (a parameter passed by default) and samples at the default temperature of 1.0. The timeouts are those of Terminal-Bench
2.1, and, as in the benchmark's standard evaluation, ArchNights-SE runs inside the task's container together with its simulator.
Most tasks run fully offline: the data and sources that they need are staged in their images, and their packages are served from local
mirrors. A few run semi-online, through a strictly controlled proxy. As in all our tests, the audit of $\mathtt{sodo}$ is turned off,
so every call of \texttt{bash} passes it, and the task's container keeps the network isolated. A harness that runs tasks with no
person watching behaves the same way: an authorization prompt would stop a task at its first privileged call, so the setting is the
usual one for an unattended benchmark, and the run exercises the kernel's checks instead of being refused by them.

ArchNights-SE is evaluated in the three configurations of Table~\ref{tab:tb}. The minimal system is the distribution reduced to one process,
with the operators \texttt{file\_read}, \texttt{file\_write}, \texttt{file\_patch}, \texttt{list}, \texttt{stat}, \texttt{bash} and
\texttt{exit}. The others of Table~\ref{tab:operators}, those for processes among them, are removed. It keeps the system's default
prompt, which asks a task to stop only when its work is complete. When Terminal-Bench evaluates Terminus~2, the agent's own
prompts begin every turn with an analysis of what has been accomplished and what still needs to be done. That prompt template is
supplied by the benchmark's harness and not stated in its paper~\cite{merrill2026terminal}.
An external harness such as ArchNights-SE has no such prompt, so we carry it over as far as the system allows. The extra prompt
appends two rules to the task. The first rule is to write what the task asks for as soon as there is a first version and then
improve it, since only what is written counts if the work is cut short. The second is to ask after each step whether the result
already meets what the task asks, and to finish when it does. The minimal system with the extra prompt resembles the Minimal mode
of DeepSeek Harness, which has a single bash tool and in which DeepSeek evaluates its model as a code
agent~\cite{deepseekai2026deepseek}. With its default prompt the minimal system scores 77.5, as the full system does, and with the
extra prompt it scores 83.1.

\begin{table}[t]
\centering
\small
\setlength{\tabcolsep}{4pt}
\begin{tabular}{@{}>{\raggedright\arraybackslash}p{0.39\columnwidth}rrr@{}}
\toprule
 & All & Solvable & Other \\
Configuration & 89 & 79 & 10 \\
\midrule
ArchNights-SE, minimal, extra prompt & 74 (83.1) & 66 & 8 \\
\rowrule
mini-swe-agent & 72 (80.9) & 65 & 7 \\
\rowrule
Terminus~2 & 70 (78.7) & 62 & 8 \\
\rowrule
ArchNights-SE, minimal, default prompt & 69 (77.5) & 61 & 8 \\
\rowrule
ArchNights-SE, full, default prompt & 69 (77.5) & 61 & 8 \\
\bottomrule
\end{tabular}
\caption{Terminal-Bench 2.1 with deepseek-v4.1-flash, one trial of each task: tasks passed (Pass@1, \%). Solvable: the 79 tasks
whose reference solution passes in our setting; Other: the remaining 10. The proxy lets only specific network accesses through, so a
harness may pass a task whose reference solution does not.}
\label{tab:tb}
\end{table}

The shared harness, mini-swe-agent, aligns the two evaluations in kind, not strictly in score: with the same model it scores 80.9
in ours and 90.3 in DeepSeek's (the dashed line in Figure~\ref{fig:tb}). For reference only, DeepSeek reports between 84.1 and
90.6 in its own settings for eight harness configurations. The highest is with its own harness, DeepSeek Harness in its Minimal
mode~\cite{deepseekai2026deepseek}. The vendor scores its model highest in the
configurations that check the least, its own Minimal mode and mini-swe-agent among them, and the harnesses whose prompts ask for a
task to be inspected and carried through to the end score below them, Claude Code among them~\cite{claudecode}. A harness of the second kind loses
tasks that a harness stopping at the model's own word keeps, and the harness compared here is of the first kind. We cannot say
how much of the 9.4 points this is, since the vendor does not report the prompts its harnesses use.

\paragraph{General compute.}
\label{sec:an-usage}
With a virtual display and audio output added to the simulated computer, ArchNights-SE produced a music video of
\emph{world.execute(me);} (Figure~\ref{fig:pv}). The run shows that the \mach{O-2-PDA-TSV} is two kinds of system at once. It is
an agentic system, an agent harness that carries a task through, and it is a computer system, a general-purpose computer that runs
programs on its devices. The task runs as the ordinary user \texttt{oracle} and reaches the host through \texttt{bash}, each call
through $\mathtt{sodo}$, whose audit is turned off here as in the benchmark. It plays the score through \texttt{fb\_play}, a semantic delegation call registered in
a kernel module.

\section{Related Work}
\label{sec:related}

Quine is the closest work in structure~\cite{ke2026quine}: it realizes agents as POSIX processes, with a deterministic host driving
a stateless model and the host's kernel keeping isolation. This paper constructs the abstract machine first, derives the operating
system on it, and then places the designs of the field on that machine. \S\ref{sec:related-parallel} takes the work that has modelled agentic
systems before us; \S\S\ref{sec:related-agents}--\ref{sec:related-serving} take the work that our derivation rests on or draws
from, in five fields: agents and harnesses, the theory of computation, programming languages, operating systems and computer
architecture, and inference systems and agent safety. Appendix~\ref{app:further} takes up what the body leaves open: case
studies that apply the models of \S\ref{sec:eclipse} to the encoders, caches, schedulers and defences; the
Oracles other than language models; and the Lynchpin hypothesis.

\subsection{Other Foundations for Agentic Systems}
\label{sec:related-parallel}

Agentic systems have acquired foundations as language models have grown. Some, as the Introduction notes, model an agentic
system by analogy, most often with an operating system. Others go deep in one direction: processes, trust, languages, the oracle,
the instruction set. Each has produced working systems or useful theory. This paper takes a different route: it constructs an
abstract machine and derives an operating system on it. The derivation optimizes nothing by itself; what it offers is a
model on which the optimizations of other work can rest. Each paragraph below describes one line of work and then locates it on
the machine.

\paragraph{Agent operating systems.} Agent operating systems extend the abstractions of an operating system to agents by
analogy, and the analogies disagree on where the model stands. For AIOS the model is a resource. AIOS moves the model and the
tools out of agent applications into a kernel, with managers for scheduling, context, memory, storage and access
control~\cite{mei2025aios}. For a vision paper and for one of the papers named AgentOS the model is the kernel itself, a reasoning
kernel over its context~\cite{ge2023llm,li2026architecting}. A second AgentOS proposes a personal agent kernel with skills as
modules~\cite{liu2026agentos}. Others read the model with its context as a computer. MemGPT pages tiers of memory into the
context window as an operating system pages memory~\cite{packer2023memgpt}, and two frameworks carry the von Neumann
architecture~\cite{mi2025building}, or computer architecture as a whole~\cite{lin2026model}, over to agents. Lin et al.\ reconcile
the positions with two planes, a probabilistic one that executes and a deterministic one that controls~\cite{lin2026model}, and
illustrate the separation with published parameter data. A fourth group does not place the model but asks what an operating system guarantees. The first AOS paper defines an
agentic control plane and maps it onto Linux and Windows~\cite{sharma2026aagent}, and the reference architecture that follows
it leaves the host system outside its boundary~\cite{sharma2026bagent}. A blueprint asks for operating-system guarantees for
agents~\cite{koubaa2025agent}, and an essay proposes primitives extended from those of operating systems and clouds, as POSIX and
Kubernetes gave their waves a small set of stable abstractions with defined semantics~\cite{steinder2026towards}.

An analogy places the model where the borrowed concept places it. On the machine the placement follows from S and T, which
order the two sides: an operating system mediates a task exactly when the side that holds the task's program is not the side that T
privileges (Definition~\ref{def:sides}). In the Agent form the Oracle is then the processor of the task's program and the Priestess
its operating system, which are Lin et al.'s two planes. The model with its context as a computer is the O-mode reading, in which
the history is at once the Oracle's program and its memory (\S\ref{sec:related-agents}), and MemGPT's paging is a rewrite of the
task, which starts a new epoch (Appendix~\ref{app:further-cache}). The model as the kernel puts privilege on the Oracle's side.
That is the reversed ordering, the bottom row of Table~\ref{tab:quadrants}. This paper orders the sides for an untrusted Oracle
such as a language model. On the machine with the ordering reversed (Appendix~\ref{app:further-reversed}), a model as the kernel
still confines the programs it runs, but every request it serves becomes part of what it runs next, and its decisions come from an
answer relation that no proof inspects. The model as a resource puts the calls in the other direction.
AIOS's kernel manages the model for the applications whose code calls it; the kernel of Definition~\ref{def:os} manages what the
model runs, a history whose system calls are the model's own answers, served at a yield trap. The fourth group's operating
system, like Quine's below, stands between processes on the Priestess's side. Among the abstractions with defined semantics that
the essay asks for, those at the boundary between the model and the program that it runs are the task, the system call and the
library of Definition~\ref{def:os}. The contract gives them their semantics, derived from the rules and not extended. One name
needs care: UFO2
uses AgentOS for a desktop agent that operates the applications of a host, which Table~\ref{tab:quadrants} calls an Agent
operating a computer~\cite{zhang2025ufo2}.

\paragraph{Processes, virtualization and unification.} A second line takes one abstraction of an operating system and asks how
far it carries, and part of it poses the unification problem of the Introduction. Quine takes the process abstraction furthest of
these, and gives the model tools to run a shell, to fork, to exec and to exit. FMOS states the
unification problem most clearly~\cite{bhattacharya2026position}. It treats workflows and agentic loops as two execution forms over
shared primitives. Its guarantees are conservative and rest on a virtualization condition after Popek and
Goldberg~\cite{popek1974formal}, which FMOS assumes at its interface, while the partition is learned at run time. Two harness designs
take the two poles in a language. Harness as a Language has a single primitive, an invocation whose body the model writes as code
that may invoke the model again. With it, the design expresses both the agent loop, as a tail-recursive invocation, and more
elaborate workflows~\cite{li2026harness}. LLM-as-Code lets deterministic code own the control flow and calls the model only to reason or
generate~\cite{qi2026llm}. Self-programmed execution states the first pole formally, on an agentic machine that makes one model call
per transition, after Turing's oracle machines. An evaluator that runs each completion as a program reaches every state of such a
machine, so what distinguishes the architecture is which entity expresses the orchestration policy~\cite{oconnor2026self}. On the machine Quine's operating system is the host's and stands between processes, inside the
Priestess; PRTS defines the same calls one level down, between the Oracle and the Priestess (\S\ref{sec:an-prts}). FMOS's
virtualization condition holds by the rules of the \mach{O-2-PDA-TS}, while FMOS assumes it at an interface (\S\ref{sec:related-os}). The loop
of the first design is an Agent task whose harness runs the code that each answer carries, a program that the Priestess interprets
(Table~\ref{tab:designs}). The second design calls the model in P-mode, in the Workflow form. Self-programmed execution also
evaluates each completion in the Priestess, and its question, which entity expresses the orchestration, is one of paradigm. The two poles thus differ in form,
the placement of a task's program, as well as in paradigm, the question of whose output decides the next step
(\S\ref{sec:os-agentic}).

\paragraph{Kernels that put trust first.} A third line builds the operating system for agents around trust, and makes its checks
complete by design. AgentKernel argues that governance that shares the agent's process trust boundary is not enough, and makes
identity, input mediation, memory governance and execution control mandatory services that cannot be
bypassed~\cite{zou2026agentkernel}. AgenticOS turns the operating system into a filter of declared intents that builds a
least-privilege environment~\cite{zhao2026agenticos}. KAIJU separates the model's reasoning from execution in an executive kernel
with an intent gate~\cite{guerin2026kaiju}. Sovereign Agentic Loops validate the model's outputs as intents in a control plane and
keep an evidence chain for audit and replay~\cite{he2026sovereign}, and a runtime sets budgets for irreversible effects across agents and
workflows~\cite{mohammadi2026irreversibility}. Agent libOS lets a self-evolving agent change what it can ask for, but not what it is
authorized to affect~\cite{zhang2026agent}. Each design derives completeness from a mediator that cannot be bypassed.
AgentKernel requires every interaction to traverse its kernel, through adapters assumed not to be bypassable. The sovereign loops
assume that the control plane mediates every change of state, and AgenticOS makes its gateway the only controlled boundary, given
isolation in hardware. On the machine their separation of reasoning from execution is the crossing invariant, and their checks and
budgets sit on the safety boundary: a check of a request at $\ell_{\mathrm{trap}}$, and a mediation of inputs at the kernel's writes
on the way in, AgentKernel's among them (Appendix~\ref{app:further-safety}). Agent libOS's principle is K4's: an answer selects
only among the successors that the kernel's code allows, and grants itself no right. That every request passes the service entry is an
invariant of the machine (K3, Corollary~\ref{cor:crossing}), not an assumption about a mediator. What remains to assume is that a
realization keeps to runs of the machine. For the construction of \S\ref{sec:query}, it is segment locality, the locality of the
Priestess's own code and the fetch condition; for a harness on an ordinary computer, it is that the harness implements the
machine faithfully. Either is one assumption about the machine, not one about each check. The machine guarantees that every
request reaches the code where the policy sits. Whether that code consults the policy, and what the policy decides, is the
kernel's.

\paragraph{Languages for agents.} Calculi and languages for agents fix a language and prove what its programs cannot do.
$\lambda_A$ extends the simply typed lambda calculus with an oracle term and a bounded fixpoint for the ReAct loop, and mechanizes
type safety and termination~\cite{liu2026lambda}. LLMbda proves a probabilistic noninterference theorem for every program of a
calculus whose state is a conversation that the programs extend, fork and clear. LLMbda's verified interpreter is itself the
harness that calls the model~\cite{garby2026llmbda}. ZipperGen projects global workflows into local programs that cannot deadlock whatever
the model answers~\cite{bollig2026provable}. The machine fixes no language, and what an operating system on it guarantees is an
invariant of the machine, one that programs in any of these languages preserve. In $\lambda_A$ the oracle term returns a value that
the program binds, as a P-mode call returns data (K4), and the loop around it is the program's; the Oracle instruction also has an
O-mode, in which the Oracle runs its history as a program and a kernel arbitrates at $\ell_{\mathrm{trap}}$ (K3). LLMbda's labels
track flows that the contract does not see, while the contract arbitrates space and control, so the two constrain different
things. Since these works define no operating system and we define no language, the two combine. A program in one of these
languages can be the task code of a Workflow or the kernel's service policy, and its guarantees then sit beside the contract.
AgentSpec's rules are such a service policy~\cite{wang2026agentspec}: they fire before each action and when the agent finishes,
at the trap, and on each change of state, on the way in. Other formal works fit in the same way: a
choreography~\cite{gopinathan2026pact} or a Lean model of a workflow~\cite{wang2026lean4agent} describes task code, the
Calculus of Intelligence, in which a task is a space of valid plans whose local parts compose when their declared overlaps
agree, describes task code one level up~\cite{yuan2026calculus}, a process
calculus for tool protocols describes the declaration and the requests it admits~\cite{schlapbach2026formal}, and runtime
monitors~\cite{bollig2026causal,wang2025bprobguard} and an automaton that supervises the generation of a plan~\cite{li2024formal}
are Priestess code.

\paragraph{Models as oracles and as automata.} Another line models the language model, as we do, as the oracle of an oracle
machine, and asks what such a computation achieves, at what quality and cost. AI-oracle machines form each query by an algorithm
that may decompose the task, and check each answer against reference material that the user supplies~\cite{wang2024ai}. The
stochastic-oracle Turing machine (SOTM) makes this precise. A probabilistic Turing machine writes a query on a query tape, and the
oracle, a family of response distributions indexed by the query, writes a response on a response
tape~\cite{wang2026btoken,wang2026dtoken}. The series proves what such a machine reaches and at what token cost. The least expected
cost of reaching a given expected score is monotone and convex in the score~\cite{wang2026btoken}, and certifying with error at most
$\varepsilon$ that an oracle's reliability is at least $p_1$ rather than at most $p_0$ takes, to leading order,
$\ln(1/\varepsilon)/D(p_1 \,\|\, p_0)$ turns~\cite{wang2026dtoken}. When the distributions also depend on the transcript so far,
an oracle that caches its answers bounds what any adaptive machine can identify, and fresh answers to a repeated query lower the
error at the Chernoff rate~\cite{wang2026acomputing}. An agentic oracle maintains a hidden state and may act on an environment. Part of
its cost is spent inside it, where the caller does not see it, and retained state can make the cost of a task linear in its steps
whereas an oracle sent the whole prefix at each step pays a quadratic cost~\cite{wang2026ccomputing}. Chains of thought have been formalized as Turing
reductions~\cite{rafiuddin2025formal}. Agents have been classified as automata by their memory, with hierarchical agents as pushdown
automata when calls return in order and memory is scoped to the stack, so that some of their verification questions are
decidable~\cite{koohestani2025agents}. In programs with choice points that an oracle resolves, a strategy is separated from the policy
that searches it~\cite{laurent2025oracular}, and agents have been studied as universal task solvers within algorithmic
information~\cite{achille2025ai}. These works ask what the oracle's answers achieve. We ask who decides at the switches, and what a
system guarantees when it is built on the answer. The questions meet in places. A stochastic oracle with non-empty answers defines
an answer relation by the supports of its distributions, so the contract and the safety boundary hold on every run with it, while
none of the quantitative results of the series can be stated for a relation alone. Our query is the whole stack, the case of the
SOTM in which every query includes the whole transcript~\cite{wang2026ccomputing}. Within an append-only epoch of an Agent task no
query repeats, since answers are non-empty, so the difference between cached and fresh answers does not arise there. It arises when
a kernel restores a history and enters it again. An agentic oracle encapsulates its agent loop, its tools and its access to the
environment inside the black box. On the machine that loop is an Agent task and the harness that serves it is the kernel, so its
tool calls are requests that reach the kernel at $\ell_{\mathrm{trap}}$ (K3). What stays hidden is the Oracle's computation inside
one call, and through a Frame interface even the number of its Steps can be (Proposition~\ref{prop:steps-invisible}). The linear cost of
retained state has a counterpart in the state cache. Under the conditions of Proposition~\ref{prop:prefix-chain}, a private cache
that keeps every state computes each token of a history once while the kernel only appends to it, although every query is the
whole history. The two stacks of the \mach{O-2-PDA} have the power that Koohestani et al.\ assign to read/write
memory, so the questions that they make decidable by restricting an agent to one stack are undecidable on the machine in general. Such a
restriction would be a property of a kernel or a program, not of the machine.

The contrast is the line that places universality inside the model, or in a loop around it. Transformers are Turing complete under
idealized assumptions~\cite{perez2019turing,perez2021attention}. Chains of thought raise their expressive
power~\cite{feng2023towards,nowak2024representational}, with a power that grows with the number of
steps~\cite{merrill2024expressive,li2024chain}. A looped transformer runs programs given as its input~\cite{giannou2023looped}, and
a fixed transformer computes every computable function under a suitable prompt~\cite{qiu2024ask}. Around the model, neural
controllers were coupled to external memories~\cite{graves2014neural}, a model has been used as the processor of a stored-program
computer~\cite{holt2024l2mac}, and a prompted model simulates a universal machine with an associative read/write
memory~\cite{schuurmans2023memory} or under extended autoregressive decoding~\cite{schuurmans2024autoregressive}. A transformer that
manages its transcripts as channels realizes only finite-state transductions as long as the channels are append-only and it sees a
bounded suffix of them, and becomes universal with two channels that admit pop~\cite{salishev2026transcript}. A result that holds for every
Oracle cannot rest on an Oracle's power, so the \mach{O-2-PDA} places universality in the Priestess, and the Oracle instruction adds
no power beyond the Oracle's (\S\ref{sec:machine}). On the machine a loop around a fixed model is a Priestess program in the
Workflow form, and universality inside the model would make O-mode universal by itself. The machine neither needs nor excludes
either.

\paragraph{Instruction sets for Oracle calls.} Instruction sets for Oracle calls predate ArchNights, and they fix the meaning of a call
in the instruction set. Arbiter-K encapsulates the model as an untrusted probabilistic processing unit, reifies its
outputs into a semantic instruction set, and enforces security in a neuro-symbolic kernel with taint propagation and rollback, in
a Python prototype~\cite{wen2026craft}. Poon proposes cognitive instructions (THINK, RECALL and ADAPT) whose permissions the
operating system controls~\cite{poon2025toward}. A learned function behind instructions of a general-purpose ISA is older: the
neural processing unit replaces code that the programmer marks approximable, is invoked through extensions to the ISA, and may
introduce small errors~\cite{esmaeilzadeh2012neural}. It computes a fixed function, with no history and no mode of its own. ArchNights has
one Oracle instruction and an O-mode with no addressable code, entered by $\mathtt{uret}$ (\S\ref{sec:archnights}); the
instruction lies in the RISC-V custom space and is compatible with RoCC. The meaning of a request is fixed by the history and by $L_\Lambda$, not by the instruction
set, so one instruction serves every Oracle and every library, and the contract does not depend on what a request means.

\subsection{Agents and Harnesses}
\label{sec:related-agents}

The two forms of \S\ref{sec:os} name a distinction drawn in practice by agent programming and by the harnesses built since language
models arrived.

\paragraph{Agents before language models.} Agent-oriented programming characterized agents by their beliefs, commitments and
capabilities~\cite{shoham1993agent}, and the BDI model described them in terms of beliefs, desires and
intentions~\cite{rao1995bdi}. Agent programming gave agent programs languages with operational
semantics~\cite{rao1996agentspeak,hindriks1999agent} and verified them, by a framework for programs with declarative
goals~\cite{deboer2007verification} and by model checkers for agent programs and multi-agent
systems~\cite{dennis2011model,lomuscio2015mcmas,dennis2018mcapl}. Surveys cover its theories and
architectures~\cite{wooldridge1995intelligent,desilva2020bdi}. Its guarantees are those of a program in a language. The recent
bridges leave such a program in control while a model answers its calls~\cite{gatti2025chatbdi,alowayyed2025controlled,collier2026astra},
and a survey catalogues how machine learning has entered BDI agents~\cite{agiollo2025integrating}. On the machine these bridges are
in the Workflow form: the program is Priestess code, and the model answers in P-mode, with answers that are data (K4). The machine
says nothing of an agent's mental state. The Oracle is a black box, and the Priestess's record of it is its history.

\paragraph{Harnesses and their terms.} Advances from neural language models~\cite{bengio2003neural} through the
Transformer~\cite{vaswani2017attention} to large models capable of performing tasks from context~\cite{brown2020language} provided
a practical executor for open-ended tasks. With models as executors came ReAct, which interleaves reasoning with tool actions in one
loop~\cite{yao2023react}, and architectures that organize a language agent by its memories, its actions and a decision
loop~\cite{sumers2024cognitive}. Whatever memories an architecture names, what the Oracle reads in O-mode is its history (K2), and an
external action is a request served at the trap. Programming systems put model calls on code paths fixed in advance: query programs
with constraints~\cite{beurerkellner2023prompting}, declarative modules compiled into pipelines~\cite{khattab2024dspy}, compositions
of calls read as probabilistic programs~\cite{dohan2022language}, flows of calls in which each call samples a message and appends it
to the context~\cite{mandel2026ppdl}, and algorithms that call the model at fixed
points~\cite{schlag2023large}. On the machine they have the Workflow form, and ReAct's loop the Agent form. Multi-agent frameworks
compose conversable agents~\cite{wu2024autogen}, and planners emit graphs of calls for an executor~\cite{kim2024llm}. A conversable
agent is one more task, which a kernel service spawns or forks. Building Effective Agents separates Workflows from Agents by
whether code or the model directs the process~\cite{anthropic2024building}. We keep its terms and split them in two. On the machine
the forms are placements of a task's program, and the distinction drawn in Building Effective Agents is the paradigm, so a planned
graph that Priestess code executes is a Workflow in form and an Agent in paradigm (\S\ref{sec:os-agentic}). Taxonomies elsewhere
draw a different line, between AI
agents and agentic AI, marked by multi-agent collaboration among other traits~\cite{sapkota2025ai}. An early essay describes a model
with its scaffolding as a computer, with the model as its processor and the context as its memory~\cite{millidge2023scaffolded}.
This is the O-mode reading, with two differences: the Oracle's program is its memory, with no code beside it, and within an
append-only epoch, that memory only grows. \S\ref{sec:os-agentic} places a harness and a scaffold in these terms.

\paragraph{Existing designs in this structure.} For most designs, the structure gives a placement, which says where a design's
guarantees come from and what the design costs elsewhere (Table~\ref{tab:designs}). Some rows are more than a placement.
Compaction and paging have a price that the structure names: the break of the prefix chain at the rewrite.
Leyline's directives~\cite{ma2026leyline} lower exactly that price (Appendix~\ref{app:further-cache}). Constrained decoding is the Oracle's side of the declaration: the
declaration supplies the thunks, and the constraint restricts every request to $L_\Lambda$.

\begin{table*}[t]
\centering
\small
\begin{tabular}{@{}>{\raggedright\arraybackslash}p{0.22\textwidth}>{\raggedright\arraybackslash}p{0.39\textwidth}>{\raggedright\arraybackslash}p{0.33\textwidth}@{}}
\toprule
Design & In this structure & What the paper says of it \\
\midrule
A tool-calling loop~\cite{yao2023react} & an excursion that yields with a request; the tool's result is the reply, appended before the next $\mathtt{uret}$ & the Agent form (\S\ref{sec:os}) \\
\rowrule
Several calls in one answer; a planned graph of calls~\cite{kim2024llm} & several calls in one answer, one request, all served before the task runs again; a plan that Priestess code executes is an answer the Priestess interprets & several calls in one answer are one request of an Agent task; a planned graph that Priestess code executes runs in Workflow form, with the Agent paradigm, since the model wrote the plan (\S\ref{sec:os-agentic}); its parallelism is the Priestess's \\
\rowrule
Code as the action~\cite{wang2024executable,li2026harness} & the answer is a program that Priestess code runs: an interpretation by the Priestess & the answer gains no execution right by itself (K4); checking it first is the Priestess's obligation \\
\rowrule
Sub-agents and delegation~\cite{wu2024autogen,ke2026quine} & spawn or fork; fork answers one request twice, and the two histories share their prefix & fork is one kernel service that builds the prefix tree (\S\ref{sec:eclipse}); the child's result reaches the parent only through the parent's read (\S\ref{sec:archnights}) \\
\rowrule
Compaction of the context & the kernel rewrites the task, and the summary, the answer of a P-mode call, is data the kernel writes & a new epoch: the prefix chain breaks after the part of the history the rewrite keeps (\S\ref{sec:eclipse}); in PRTS, a swap that keeps the head (\S\ref{sec:an-checkpoint}) \\
\rowrule
Memory tiers paged into the context~\cite{packer2023memgpt} & moving content out is a rewrite by the Priestess; bringing it back is a read the Oracle requests & as with compaction; an interrupt becomes content appended at a switch \\
\rowrule
Skills~\cite{anthropic2025equipping} & dynamic libraries linked lazily: a short declaration in the history from the start, the definition loaded by one call & the system library of ArchNights (\S\ref{sec:archnights}) \\
\rowrule
Tool protocols such as the Model Context Protocol~\cite{anthropic2024introducing,schlapbach2026formal} & a tool's description is part of the declaration $\delta_\Lambda$, and a call is a request in $L_\Lambda$ that a kernel service serves & the declaration makes a call possible, not certain (\S\ref{sec:query}) \\
\rowrule
Constrained decoding~\cite{beurerkellner2023prompting,beurerkellner2024guiding,zheng2024sglang} & the Oracle's answers are restricted so that every request lies in $L_\Lambda$, or every P-mode answer has a given form & membership becomes certain; that the request is the right one does not \\
\rowrule
Prefix and prompt caching~\cite{kwon2023efficient,zheng2024sglang,gim2024prompt} & a cache of the Oracle's computation keyed by the history; histories of the same program share their first records, the declaration among them & the state cache, the prefix chain and its shared key (Proposition~\ref{prop:prefix-chain}); the shared first records are the root of a prefix tree; reuse of modules at reserved positions~\cite{gim2024prompt} falls outside the state cache (Appendix~\ref{app:further-cache}); their scheduling is in Table~\ref{tab:sched} \\
\rowrule
Checkpoint, rewind and restore~\cite{zheng2026when} & a restore gives a new task a history equal to an old one, by a write of the kernel & a channel for replay, which a policy at the service can limit (Appendix~\ref{app:eclipse-replay}); the checkpoints of PRTS (\S\ref{sec:archnights}) \\
\rowrule
Workflow programs and engines~\cite{khattab2024dspy,beurerkellner2023prompting} & the Workflow form: a model node is a P-mode call, and a node that the model writes runs only by the Priestess's interpretation & no operating system mediates the task (Definition~\ref{def:forms}) \\
\rowrule
Approval of a privilege expansion by a person~\cite{shi2025progent,wang2026agentspec} & the person is an Oracle called in P-mode, and the approval is data the kernel acts on & in the shell of ArchNights, a line acts with privilege only if the person typed it (\S\ref{sec:archnights}) \\
\bottomrule
\end{tabular}
\caption{Existing designs in this structure.}
\label{tab:designs}
\end{table*}

\subsection{Theory of Computation}
\label{sec:related-theory}

The machine is a construction of computability theory. Its two parts, the oracle and the two stacks, are standard, and so is a
computation that exchanges data with a party outside it. Each paragraph says what the machine shares with one line of this work and
where it departs from it.

\paragraph{Oracle machines and relativization.} Oracle machines go back to Turing's o-machines~\cite{turing1939systems} and Post's
reducibilities~\cite{post1944recursively}. Relativization has a known limit: some questions, P against NP among them, cannot be
settled by proofs that hold relative to every oracle~\cite{baker1975relativizations}. Our results are relativized in the same
sense. They hold for every Oracle that meets their hypotheses (in \S\ref{sec:os}, every Oracle, except for K5, which needs an
Oracle that makes the machine good), and so say nothing about what a particular Oracle does. Statements about language models, the Lynchpin
hypothesis among them, are therefore discussion, not theorems of the machine.

\paragraph{The oracle elsewhere.} Many fields use an oracle, and we claim nothing new in choosing one. In cryptography all parties
share a public random oracle, which a hash function replaces in practice~\cite{bellare1993random}. In query learning a teacher
answers membership and equivalence queries~\cite{angluin1987learning}, and oracle-guided inductive synthesis defines its oracle as
a nondeterministic map from the dialogue so far and a query to a response, which it assumes to be sound~\cite{jha2017theory}.
Optimization counts a method's calls to a black-box oracle~\cite{nesterov2004introductory}, and software testing calls the
difficulty of judging an output the oracle problem~\cite{barr2015oracle}. Other lines do not trust the oracle. An interactive
proof must reject a false claim whatever an all-powerful prover answers~\cite{goldwasser1989knowledge,babai1985trading}, a program
checker uses the program it checks as its oracle~\cite{blum1995designing}, and a robust oracle machine accepts the same language
with every oracle and runs fast with a helping one~\cite{schoning1985robust}, much as K1--K4 hold for every Oracle and K5 needs a
good one. Language models now take several of these roles: fallible membership oracles~\cite{chen2024probabilistic}, test
oracles~\cite{molina2025test} and provers checked by a verifier~\cite{amit2025models}. Our Oracle is untrusted. It may
answer a query in more than one way, and the contract holds whatever it answers. Complexity theory separates an oracle, whose
answer depends on the query alone, from a prover, whose answer may depend on all earlier queries~\cite{goyal2010interactive}. Our
Oracle is an oracle in this sense, but its query is the whole stack, so within an append-only epoch it sees what a prover sees.

\paragraph{Interactive computation.} Interactive computation studies machines that exchange data with an environment while they
compute~\cite{goldin2006interactive}, as reactive systems interact continuously with theirs~\cite{harel1985development}.
Persistent Turing machines keep a work tape across interactions, and their class is isomorphic to a general class of effective
transition systems~\cite{goldin2004turing}. Interactive Turing machines translate infinite streams~\cite{vanleeuwen2006theory}, and
with advice they are equivalent to computers whose hardware and software are upgraded over time~\cite{vanleeuwen2001turing}.
Ordinary interactive algorithms query their environment within a step and finish the step only when every query has been
answered~\cite{blass2006ordinary}. An online process can also be read as an oracle machine whose oracle is its
environment~\cite{soare2009turing}. The \mach{O-2-PDA} is a machine of this kind. A call waits for its answer, the stacks persist
across calls, the Oracle may return more than one answer to a query, and nothing else is assumed of it. In its choice of Oracle the
paper therefore adds nothing to interactive computation. Whether interaction adds power beyond Turing machines has been
disputed~\cite{goldin2008interactive,cockshott2007there}, and lineages of language models have been argued to be equivalent to
interactive Turing machines with advice~\cite{wiedermann2024large}. Our results do not depend on that question, since with a
deterministic Oracle the \mach{O-2-PDA} computes what an oracle Turing machine with the same Oracle computes
(\S\ref{sec:machine-properties}). The machine departs from these models in two places. The first is a form. A persistent Turing
machine keeps its memory in the device and reads one input token per interaction. The \mach{O-2-PDA} keeps the record of the
interaction on its own stack, sends the whole stack as the query and appends the answer to it (\S\ref{sec:machine-ar}), so an
Oracle with no memory of its own still answers the whole history, and the program decides which history that is. The second is the
breaking. In these models the machine runs its program and the environment only supplies inputs or
replies~\cite{goldin2004turing,vanleeuwen2006theory,blass2006ordinary}. In O-mode the Oracle runs a program of its own on one
stack, and control returns to the Priestess only at the trap label. That is the symmetry breaking from which \S\ref{sec:os}
derives an operating system.

\paragraph{Two stacks, programs and random access.} The machine of \S\ref{sec:machine} is a two-stack pushdown automaton written
as a program in the style of register machines~\cite{shepherdson1963computability}, and it computes what a Turing machine
computes~\cite{minsky1967computation,hopcroft1979introduction}. Pushdown automata with oracles have been used to study
reducibilities among
context-free languages~\cite{yamakami2013oracle}. Those are one-stack devices, and the second stack is what makes ours Turing
complete; the two stacks also carry the two colours that \S\ref{sec:os} orders. In systems of communicating pushdown automata a
query symbol transfers the whole stack of the component it names onto the stack of the component that asks, but the parties are
pushdown automata, not an oracle~\cite{csuhajvarju2000parallel}. Stack computers, Forth chips among them, are a line
of their own~\cite{koopman1989stack}. \S\ref{sec:query-memory} keeps the stacks as the machine's semantics and, for the reasons
given there, lays them out as segments of one von Neumann memory~\cite{vonneumann1945first,cook1973time,shi2008virtual}.

What is new is the combination: an oracle instruction whose query is the whole content of one stack and whose
answer is appended to it, an operating system and its contract derived by breaking the symmetry of the stacks' colours, and that
machine laid out on a von Neumann computer with one instruction for the Oracle and a mode for it below user mode. Each ingredient has a
precedent in this section or in \S\ref{sec:related-parallel}: oracles that see the whole history, oracles that are not trusted, an
untrusted model under a kernel, a learned function behind an instruction. The combination has no precedent we know of.

\subsection{Programming Languages}
\label{sec:related-pl}

From programming languages the paper takes one vocabulary, call-by-push-value, and one reading of effects, and it declines one
reading, duality.

\paragraph{Call-by-push-value.} We use a fragment of call-by-push-value~\cite{levy2003call} to type the cells of the machine and to
name its crossings. The stack machine of call-by-push-value pushes and pops continuation frames and values, thunks
among them~\cite{levy2006call}. The calculus has adjunction models with stacks~\cite{levy2005adjunction}, and it decomposes Moggi's
monads~\cite{moggi1991notions}. We use none of
this metatheory (\S\ref{sec:os-cbpv}), only a vocabulary for when stored content is data and when it runs. The fragment is a simplification of Levy's calculus,
chosen so that the definition of an operating system, and the concepts derived from it, remain clear without the rest of the calculus.
It keeps the two shifts $U$ and $F$, and its one ground type $\mathsf D$ is, in Levy's syntax, a countable sum $\Sigma_w 1$ over the
words. It leaves out products, sums other than $\mathsf D$, and the function type $A \to B$, whose application pushes an argument
and gives the calculus its name. No step of the machine passes an argument, since the Oracle reads its operand from storage and
$k$, the code at the trap label, takes none. A formal analysis of the machine in call-by-push-value would instead introduce the full calculus
and derive them rigorously in it.

\paragraph{Not a duality.} Call-by-push-value sits in a rich theory of computational duality. Values and continuations are
dual~\cite{filinski1989declarative}. Call-by-value and call-by-name are exchanged by a syntactic
involution~\cite{curien2000duality,wadler2003call}, their models form dual categories~\cite{selinger2001control}, and polarized
calculi unify them~\cite{zeilberger2008unity,curien2010duality,downen2018tutorial}, building on focusing in linear
logic~\cite{andreoli1992logic}. The pairing of Agents and Workflows in this paper is not one of these dualities. A duality relates
evaluation orders, or a term and its context, within one calculus or between its models. The two forms are two placements of a
task's program on one machine, related by an interpreter and a compiler (\S\ref{sec:os-agentic}). The two readings of an
interaction, a force from one side and a return from the other, are readings of one transition, not an exchange of programs.

\paragraph{Effects and handlers.} One reading of effects does carry over. Model calls have been treated as algebraic effects that
handlers discharge~\cite{wang2025composable}, and the Workflow form treats an Oracle call as such an effect. The Agent form inverts
the roles: the Oracle's history issues the requests, and the kernel at $\ell_{\mathrm{trap}}$ is their handler. A handler that
changes how a program's calls are executed is, in the Workflow form, a Priestess-side scheduler of the Oracle's work
(Appendix~\ref{app:further-sched}).

\subsection{Operating Systems and Computer Architecture}
\label{sec:related-os}

Definition~\ref{def:os} and the contract are stated against the operating systems and processors they borrow from. Each paragraph
shows what the machine retains and what it changes.

\paragraph{What an operating system is.} Textbooks define an operating system as a virtualizer, a standard library and a resource
manager~\cite{arpacidusseauoperating}, as an extended machine and a resource manager~\cite{tanenbaum2014modern}, or as a resource
allocator and control program~\cite{silberschatz1994operating}. One tradition asks a neighbouring question from the other side and
answers it by minimality: a concept belongs in a kernel only if moving it out would prevent the functionality the kernel must
provide~\cite{liedtke1995microkernel}. Definition~\ref{def:os} is narrower: it says what makes one program
the operating system of another on one machine, and leaves to its code what that program manages. It is closest to the exokernel
and the library operating system, in which a kernel multiplexes and libraries carry the
abstractions~\cite{engler1995exokernel,kaashoek1997application}, with descendants in the library operating systems and unikernels
of later work~\cite{porter2011rethinking,madhavapeddy2013unikernels}. A library operating system moves abstractions out of the
kernel into the application, where they can be specialized and need fewer crossings into the kernel. ArchNights reverses the move: on
the \mach{O-2-PDA-TS} the crossing that costs is the Oracle's round trip, so semantic delegation moves abstractions out of the
Oracle's history into the kernel. The interpreter of Proposition~\ref{prop:interpreter}, which does the work of an operating system
without privilege, is in the position of a unikernel's library: whatever isolation its tasks have is a property of its code.

\paragraph{Protection and privilege.} Protection has its principles, least privilege and complete mediation among
them~\cite{saltzer1975protection}, its hardware rings~\cite{schroeder1971hardware} and its access matrix~\cite{lampson1974protection}.
That matrix is kept by column, as an access-control list for each object, or by row, as the capabilities each subject holds,
and capability machines carry the second into hardware~\cite{watson2015cheri}. What S restricts is a stored content's right to
run, and S attaches that right to the colour of the stack the content is on, as a page table attaches the right to execute to a
page.
On the machine a privilege level is a mode of T, the transition breaking of \S\ref{sec:os}. Complete mediation has its place at the
safety boundary of Theorem~\ref{thm:boundary}: by K3 every return of control to the Priestess is a trap to one label, so every
request reaches the kernel there first. Singularity's isolation by language, and not by
hardware~\cite{hunt2007singularity,hunt2007sealing}, is the static, type-based form of the isolation by check in \S\ref{sec:os}.
Appendix~\ref{app:eclipse} discusses side channels of speculative execution~\cite{lipp2018meltdown,kocher2019spectre} and
protections that also restrict the trusted side, as SMEP and SMAP do~\cite{intelsdm}. Information-flow control, for
agents as elsewhere, labels values with a lattice of security levels~\cite{denning1976lattice}. Once both breakings are chosen, they
order the two sides, P above O. That order is one of privilege between the sides, not of information flow, which is why LLMbda's
labels and the contract impose different constraints (\S\ref{sec:related-parallel}).

\paragraph{Virtualization and several kernels.} A virtual machine monitor can be built when every sensitive instruction is
privileged, so that it traps~\cite{popek1974formal}. The machine has this property by its rules: O-mode executes no instruction of its own
and so has no sensitive instruction to run unprivileged; the Oracle can only append to its own task; and every way out of O-mode
is a trap to one entry (Theorem~\ref{thm:contract}). With ZOOT, a question returns one level down (\S\ref{sec:an-zoot}): how the
users of one privilege level interact with those of another. Operating systems have answered it in three ways. Kernel-isolating
virtualization runs several kernels on a monitor below them, each kernel the single supervisor of its guest~\cite{barham2003xen}, and RISC-V's hypervisor
extension gives that monitor a level of the privileged architecture~\cite{waterman2021risc}. Kernel-sharing virtualization keeps one
kernel and divides it among its users, by security contexts in Linux-VServer~\cite{soltesz2007container} and by namespaces and
control groups in the containers that followed. A third way keeps several
principals at the supervisor level and moves the checks between them to machine mode: Keystone's security monitor runs in M-mode and
isolates enclaves by physical memory protection, each enclave with its own runtime~\cite{lee2020keystone}. In the terms of
Definition~\ref{def:os}, kernel isolation and a security monitor place several operating systems $(K, \Lambda)$ under an arbiter
above them, and containers divide one operating system among its users. The contract of Theorem~\ref{thm:contract} holds for each
of them whatever the placement. What the kernels share is Priestess-side state, which K1 and K2 already keep out of every task.
Isolation among the users inside the Priestess is the arbiter's work, not the machine's.

\paragraph{Verified kernels.} seL4 is proved to refine its abstract specification down to its C code~\cite{klein2009sel4},
CertiKOS stacks certified abstraction layers into a concurrent kernel~\cite{gu2016certikos}, and Hyperkernel designs its interface
for push-button verification~\cite{nelson2017hyperkernel}. They verify a kernel against a specification. Some of this work also
builds the specification in layers: CertiKOS specifies each certified abstraction layer by hand, upward from the machine's assembly
layer~\cite{gu2015deep}, and one recent line starts from a machine semantics already decomposed by privilege
level~\cite{malecha2025modular}. In both, a person writes the specification and the kernel's code is proved against it; what is
derived here is the contract, from the rules of the machine, and the code that keeps it is a separate question. This paper's contract is what every
kernel on the machine can rely on; the correctness of a particular kernel, its service policy
included, remains a property of that program. A kernel for the \mach{O-2-PDA-TS} could be verified in their manner, and the
contract is the part of its specification that the machine already provides. CertiKOS's layered abstract machines are closest in
spirit to an operating system defined on an abstract machine.

\paragraph{The instruction set.} ArchNights builds on the von Neumann architecture~\cite{vonneumann1945first}, on RISC-V's privileged
architecture, whose xRET it follows to enter O-mode~\cite{waterman2021risc}, and on RoCC, Rocket Chip's interface
for custom coprocessors~\cite{asanovic2016rocket}. Everything else that the Priestess computes is left to an ordinary processor
(\S\ref{sec:query}), and a characterization of agentic workloads measures that work around the model~\cite{raj2025cpu}.

\subsection{Inference Systems and Agent Safety}
\label{sec:related-serving}

Three fields that \S\ref{sec:eclipse} touches have grown large: the encoders of language models, the caches and schedulers of the
systems that serve them, and the safety of agents. We summarize what the paper takes from each and what it adds.
Appendix~\ref{app:further} places their work on the machine.

\paragraph{Encoders.} Tokenizers have a formal theory, from encoders of bounded variation~\cite{gastaldi2025foundations} to the
finite lookahead of byte-pair encoding~\cite{berglund2023formalizing}, and the boundary between a prompt and its continuation is a
known source of bias~\cite{lundberg2023art,phan2024understanding}. \S\ref{sec:eclipse} adds a use; it does not add a theory. A separator
at the end of every Frame makes the encoder's loss zero there. The separator synchronizes the two processes at every Frame boundary
when, in addition, the Oracle is Frame-faithful and called through the Frame interface, and the kernel's writes end with
$\natural$. This synchronization grounds the cache results (Appendix~\ref{app:further-encoders}).

\paragraph{Caches and schedulers.} Serving systems share and page the KV cache~\cite{kwon2023efficient,zheng2024sglang}, and
program-aware servers schedule each call by the program it belongs to~\cite{lin2024parrot,luo2026agentix,kang2026thunderagent}.
\S\ref{sec:eclipse} states these mechanisms against the machine. Appendix~\ref{app:further-cache} explains where reuse beyond prefixes
steps outside the model of \S\ref{sec:eclipse} and why compaction and paging start a new epoch, and
Appendix~\ref{app:further-sched} sorts the schedulers by the side they belong to and the side they schedule.

\paragraph{Prompt injection and its defences.} The questions this raises are older than the models that raise them: confining a
process whose interior is not assumed~\cite{lampson1973confinement}, a monitor that mediates every
crossing~\cite{anderson1972computer}, non-interference between what one party does and what another
observes~\cite{goguen1982security}, and the separation of a policy from the mechanism that enforces
it~\cite{levine1975policy}. The untrusted party is now a language model, which reads data and instructions in one history; the
machine fixes where a monitor's mediation takes place, at the safety boundary of \S\ref{sec:eclipse-safety}, and leaves the policy
to the kernel.
Prompt injection, in which data that an agent reads is obeyed as an
instruction~\cite{greshake2023not}, made the safety of agents a field of its own, with benchmarks for attacks and
defences~\cite{debenedetti2024agentdojo}. A systematization separates crossings mediated over provenance, which admit a
deterministic check, from crossings mediated over the meaning of content~\cite{zhang2026when}. The contract is of the first kind.
Whether a request is the right one is of the second, which the contract leaves to the kernel's policy. Appendix~\ref{app:further-safety}
places the defences built for agents at the three positions the safety boundary leaves them: on the way in, at the Priestess's
writes; inside the Oracle, where no rule of the machine can certify a defence; and on the way out, at the trap or, in the Workflow
form, at the gates. It also examines, on the machine, how a harness can escalate privilege.
The boundary gives each defence its position and says what the defence sees there. It is not itself a
defence against prompt injection: which request crosses is the Oracle's answer, and no rule of the machine inspects it.

\section{Conclusion}
\label{sec:conclusion}

This paper set out to give agentic systems what the two clouds of \S\ref{sec:intro-clouds} deny them: a common ground and a
unified view. The answer is one construction, in three parts. The \mach{O-2-PDA} adds a single Oracle instruction to a two-stack
machine; the symmetry breakings S and T derive an operating system on it, with the kernel contract as an invariant of the machine
and Agents and Workflows as two placements of a task's program; and the construction V fits the machine to a von Neumann computer,
which ArchNights demonstrates can be built and can serve as the basis of an agent harness. The statements this machine yields are
collected in \S\ref{sec:intro-contrib}.

\paragraph{What the machine leaves open.}
\label{sec:future}
A construction that holds together also shows where it ends, and each question below is askable because the machine exists.

\emph{Within the theory:} a refinement from the \mach{O-2-PDA-TSV} to ArchNights' instruction set and kernel, which would carry the
contract and the safety boundary over to the prototype and whose lemmas are the obligations that \S\ref{sec:an-obligations} lists;
the full call-by-push-value metatheory of the machine; and its complexity, in transitions or in tokens. \emph{Beyond these:} safety under asynchrony and concurrency, several Oracles and both orderings of privilege
combined, replay through rebuilt histories, isolation among users inside the Priestess, reuse of cache beyond prefixes, finite
capacity, and Oracles that are not autoregressive inside.

\emph{On the prototype:} an asynchronous realization, in the microarchitecture or in the operating system; a kernel on several
processors; and ArchNights-FS with ZOOT, built and validated, which needs a way to run the instruction set in the full system. One
candidate is KVM~\cite{kivity2007kvm} modified to support Oracle hardware run by the host.

\emph{About language models:} factorization at the yield symbol and Frame-faithfulness, checked per tokenizer and per model; the
Lynchpin hypothesis, approached through the operators of the system library or in the stochastic-oracle setting
(Appendix~\ref{app:further-oracles}); and when semantic delegation helps. These questions are empirical, and the companion
experimental study, AgentNights, takes them up together with a broader evaluation of harness quality.

The next step is AgentNights: the machine, the kernel and the system library implemented in host software, without a simulator,
drawing on language models' own understanding of operating systems for a more flexible and efficient agent harness. Agentic
systems have been built faster than they have been understood. The machine is offered as the place where understanding
accumulates---where a design choice is argued from invariants and bounds, once for both forms.

Good night, Kristen.

\bibliographystyle{plain-etal}
\bibliography{references}

\clearpage
\appendix
\section{Glossary and Notation}
\label{app:glossary}

\begin{description}
\item[Agentic system.] A system that uses language models to carry out concrete tasks. It has two forms, Workflows and Agents
  \cite{anthropic2024building}. On the \mach{O-2-PDA-TS} they are the Workflow form and the Agent form of
  Definition~\ref{def:forms}.
\item[Agent harness.] A system that runs Agents. In this paper: the machine, the operating-system kernel and the system
  library. A harness on an ordinary computer implements the machine, the kernel and part of the library in one program
  (\S\ref{sec:os}).
\item[Agent runtime.] The infrastructure of an Agent harness: the machine and the kernel.
\item[Agent scaffold.] The system library together with the programs a task may run: the kernel-and-library side of the
  machine, $\delta_\Lambda$, $L_\Lambda$ and the handlers. It exposes its operations to the task as thunks
  (\S\ref{sec:os-agentic}).
\item[Workflow engine.] A system that runs Workflows. In this paper: the machine and programs that run on the bare machine,
  with no operating system between them and the Oracle.
\item[Operating system.] A kernel and a system library (Definition~\ref{def:os}). The kernel runs in P-mode. The library's
  declaration sits in the task's history, and the code that serves its requests is in the kernel. Where no confusion arises
  we call the kernel the operating system.
\item[Stack, colour, type, mode, computation.] Five different notions (\S\ref{sec:os}). A stack is storage, and its colour, O or
  P, is fixed (Definition~\ref{def:colouring}). A type, $\mathsf D$ or $U\,B$, says whether a stored element is data or a
  suspended program. A mode, O-mode or P-mode, says whether the Oracle's answers or the Priestess's code choose the next
  transition. A computation, the Priestess computation or the one inside the Oracle, is one of the two computations of a run.
\end{description}

Table~\ref{tab:terms} collects the terms and symbols of \S\ref{sec:os}.

\begin{table*}[t]
\centering
\small
\begin{tabular}{@{}>{\raggedright\arraybackslash}p{0.3\textwidth}>{\raggedright\arraybackslash}p{0.52\textwidth}>{\raggedright\arraybackslash}p{0.12\textwidth}@{}}
\toprule
Term or symbol & Meaning & Defined in \\
\midrule
Oracle stack, Priestess stack & a stack coloured $O$, $P$ & Def~\ref{def:colouring} \\
\rowrule
P-mode $\langle \mathsf P, \ell, \sigma\rangle$, O-mode $\langle \mathsf O, u, \sigma\rangle$ & the two modes of T, each with the form of its configurations & Def~\ref{def:t} \\
\rowrule
$\natural$, yield & the yield symbol; an answer that ends with it & Def~\ref{def:t} \\
\rowrule
$\ell_{\mathrm{trap}}$ & the trap label, where every trap continues & Def~\ref{def:t} \\
\rowrule
$\mathtt{uret}\ u$ & enters O-mode on the stack $u$, which under S is an Oracle stack & Def~\ref{def:t} \\
\rowrule
trap, yield trap, excursion & a transition from O- to P-mode; one caused by a yield; a maximal O-mode segment with its $\mathtt{uret}$ and trap & Def~\ref{def:t} \\
\rowrule
$\mathtt{uret}\ u\ \beta$, budget trap & entry with a budget of $\beta$ O-mode transitions; the trap when it is spent & Def~\ref{def:budget} \\
\rowrule
K1--K5 & the kernel contract & Thm~\ref{thm:contract}, Prop~\ref{prop:k5} \\
\rowrule
$k$, service entry & the kernel's code at $\ell_{\mathrm{trap}}$, as the thunk $\mathtt{thunk}(\ell_{\mathrm{trap}})$ & Defs~\ref{def:t}, \ref{def:os} \\
\rowrule
task, task stack, history & a unit of work; the Oracle stack the kernel enters for it; that stack's content & Def~\ref{def:os} \\
\rowrule
request, reply, system call & what a yield appended since the last write; what the kernel writes back; a request in $L_\Lambda$ & Def~\ref{def:os} \\
\rowrule
$\Lambda = (L_\Lambda, \delta_\Lambda)$ & the library: request language and declaration & Def~\ref{def:os} \\
\rowrule
$K$, $C$, $\Delta$ & kernel code, task code, task data & Def~\ref{def:sides} \\
\rowrule
a task's program & its executable part: history content the Oracle runs, or task code & Def~\ref{def:sides} \\
\bottomrule
\end{tabular}
\caption{Terms and symbols of Section~\ref{sec:os}.}
\label{tab:terms}
\end{table*}

\section{Proofs}
\label{app:proofs}

This appendix contains the proofs omitted from the body, in the order of the results. The other proofs are short enough
to sit beside their statements.

\begin{proof}[Proof of Lemma~\ref{lem:split}]
Let $s$ be a segment of $a$. The case of $b$ is symmetric. The block moves the segments above $s$ from $a$ onto $b$, one symbol
at a time, keeping their markers, until $s$ is the top segment of $a$. The instruction is then performed on the top of $a$. A
marker at the top means that $s$ is empty: a $\mathtt{jeq}$ for $\bot$ succeeds there, and a $\mathtt{pop}$ leads to a stuck
configuration, as it would on a stack of its own. The symbols moved aside are then transferred back in the reverse order, which
restores their order. The segments below $s$ and those already on $b$ were never touched. Each block moves and restores a finite
prefix, so it takes finitely many steps.
\end{proof}

\paragraph{The split lemma and the Oracle instruction.} Lemma~\ref{lem:split} extends to the Oracle instruction on a logical
stack $s$. The block rearranges the tape held by $a$ and $b$ so that $s$ lies alone on $a$, in order, and every other logical
stack, with its markers, lies on $b$. This is a computable rearrangement of a finite tape, which the \mach{2-PDA} performs by
simulating a Turing machine on its two stacks, ending with the head between $s$ and the rest. The block then executes the Oracle
instruction on $a$, and rearranges the tape back, with $s$ now longer by the answer. After a rejection it rearranges the tape
back and continues at $l_e$. This extension is what the proof of Lemma~\ref{lem:locality} uses.

\begin{proof}[Construction for Proposition~\ref{prop:interpreter}]
Rename each task stack $u$ to a fresh Priestess stack $d_u$, and relocate the program's jump targets to account for the code
inserted below. Replace each $\mathtt{uret}\ u$ by code that records $u$ as one fresh symbol on a fresh Priestess stack and then
loops: call the Oracle on $d_u$, with a rejection label at which the run is stuck (a $\mathtt{pop}$ on a fresh empty stack); if
the top of $d_u$ is $\natural$, jump to $\ell_{\mathrm{trap}}$; otherwise loop. Since every answer is non-empty
(\S\ref{sec:machine}), the top of $d_u$ after an accepted call is the last symbol of the new answer, so the test is exactly
whether the answer is a yield, and the loop runs exactly the excursion that $\mathtt{uret}\ u$ would have run. A rejected query
leaves both runs stuck: the run of $K$ because Definition~\ref{def:t} gives no successor, and the run of $K^\flat$ at its
rejection label. The rest of the program is unchanged. With budgets (Definition~\ref{def:budget}), the code also empties a fresh
Priestess stack, pushes $\beta$ symbols onto it, pops one after each call that does not yield, and jumps to
$\ell_{\mathrm{trap}}$ when it is empty.
\end{proof}

The convention of non-empty answers (\S\ref{sec:machine}) ensures that the interpreter tests a stack top that belongs to the new
answer. With an empty answer, a history already ending in $\natural$ would retain that symbol, and the interpreter would
incorrectly terminate the current excursion.

\begin{proof}[Proof of Proposition~\ref{prop:forced}]
A task in Agent form is entered with $\mathtt{uret}$ on its task stack, and a task in Workflow form calls the Oracle in P-mode
(Definition~\ref{def:forms}), which gives the first claim. Let $U \cap C = \emptyset$, and let a P-mode call read its operand $d$.
The colour of $d$ lies in $C$ and so not in $U$, so no $\mathtt{uret}$ enters $d$. The O-mode rules are those of
Definition~\ref{def:t}, so K1 holds as proved, and no excursion writes $d$. By the locality lemma, $d$ then holds only what the
Priestess wrote there and the answers of earlier P-mode calls on it. Conversely, let $c \in U \cap C$, and let $s$ be a stack of
colour $c$, empty at the start. Take the program $\mathtt{uret}\ s$, $\mathtt{oracle}\ s\ 3$ with $\ell_{\mathrm{trap}} = 2$, and
the Oracle whose encoder and decoder are identities and whose only answer is $\natural$, so that $O(w) = \{\natural\}$ for every
$w \in \Sigma_Q^*$. The run enters O-mode on $s$, appends $\natural$ and traps to $\ell_{\mathrm{trap}}$, where the P-mode call reads
the $\natural$ that O-mode appended. Finally, with two colours the only non-empty disjoint sets $U$ and $C$ are $\{O\}$ and $\{P\}$,
in either order, and the exchange $\varepsilon$ of the colours (\S\ref{sec:machine}) maps one machine to the other.
\end{proof}

\begin{proof}[Proof of Proposition~\ref{prop:codec-good}]
(1) Definition~\ref{def:t} gives O-mode rules only for accepted queries. The transition does not trap, so the next query is made in
O-mode, and it is the history followed by the answer, which is not in $\Sigma_Q^* = \mathrm{dom}(E)$. (2) $\rho \circ D$ writes only symbols of $\Sigma_Q$, so the wrapped Oracle satisfies
the hypotheses of Proposition~\ref{prop:closure}. (3) The wrapping changes neither $A$ nor $E$, and Definition~\ref{def:int-ar}
requires of the decoder only that it work symbol by symbol, that $\tau$ decode to $\natural$, and that $\natural$ occur in no other
decoding. The map $\rho$ is applied symbol by symbol, fixes $\natural$ and, being injective, maps no other symbol to it.
Conversely, since $\rho$ is injective and fixes $\natural$, $D$ has these properties when $\rho \circ D$ does.
\end{proof}

\section{More on Eclipse}
\label{app:eclipse}

This appendix collects the parts of \S\ref{sec:eclipse} that the body summarizes: encoders in more detail, the algebra of
hierarchies, the stores of the two sides, the coherence of the encoding cache, the prefix chain in the two forms, the schedulers, a
budget trap inside a Frame, where non-empty answers and Frame-faithfulness hold, what keeping, dropping or hiding a model's
reasoning changes, and the rest of the safety analysis.

\subsection{Encoders}

\paragraph{Bounded variation.} The loss of Definition~\ref{def:loss} is a Step-level form of bounded
variation~\cite{gastaldi2025foundations}, after~\cite{berstel1979transductions}: if words at left distance $k$ have encodings
at left distance at most $C_k$, then $\mathrm{loss}_E(w, v) \le C_{|v|}$, where the left distance of $x$ and $x'$ is
$|x| + |x'| - 2\,|\mathrm{lcp}(x, x')|$.

\paragraph{Separators in practice.} Splitting at whitespace gives factorization at whitespace, but whitespace is not the separator.
With a proper dictionary, BPE as implemented has a finite lookahead constant, which depends on the dictionary alone and beyond
which a prefix of the tokenization is stable~\cite{berglund2023formalizing}. Its Step backtrack is therefore finite and depends on the
dictionary. An encoder need not be so well behaved. With Oracle symbols $a$ and $aa$ and longest match from the right end,
the encodings $E(a^{2m}) = [aa]^m$ and $E(a^{2m+1}) = [a][aa]^m$ share no prefix, so one appended $a$ loses $m$ symbols, and the Step backtrack
is unbounded with a vocabulary of two symbols.

\paragraph{The decoder.} The decoder works symbol by symbol (Definition~\ref{def:int-ar}): the decoding of a prefix of an answer is a
prefix of its decoding. In practice, when the machine-side symbols are characters, byte-level vocabularies are the exception: a
token may end inside a multi-byte character. Such a decoder buffers an unfinished character until its last byte arrives, and
the decoding of a prefix of an answer is again a prefix of its decoding.

\paragraph{Observing Steps.} With multi-token prediction~\cite{gloeckle2024better} a step yields up to $k$ tokens, and $k$ may vary. The alphabet $\Gamma_R$
is then the set of such sequences, and a Step is one of them. An inference engine's \texttt{generate} interface can expose Steps,
so the Priestess side can observe the internal transitions even though the alphabets of the two sides differ.

\subsection{Composing Hierarchies}

From the processor down to the memory, the levels of Definition~\ref{def:store} form a chain. Below the memory, the keys may be
divided among several backing stores (a swap area and files, for instance). The bottom is then the function that takes each key to
the backing store that holds it, and from the memory down, the levels form a tree. A level may also be shared: the chains of
several processors meet at a shared cache, a memory or a backing store. A write to a key $k$ updates the value the store keeps for
$k$, at every level whose working set holds $k$. A working set is usually finite, and a key may be inserted into it or deleted
from it at any time, as long as the key is still held at a level below or can be rematerialized. Whether the caches are inclusive
or exclusive is a policy on the working sets of adjacent levels. A system may also rematerialize at a level above its bottom, in
place of the levels below. When both are possible, a policy chooses between them, for instance when a miss is too expensive.

Since $C \rhd f$ is again a function on keys, hierarchies compose as functions do. Below, $\circ$ binds more tightly than $\rhd$, and
$t_i$ is the lookup cost of the $i$-th level.

\begin{construction}[composing hierarchies]
\label{con:compose}
\leavevmode
\begin{enumerate}
\item \textbf{Splitting.} Let $a_i = C_i \rhd \cdots \rhd C_n \rhd f$ for $i \le n$, and $a_{n+1} = f$. Then
  \[ a_i = \langle a_{i+1}, C_i \rangle \qquad (i \le n): \]
  below each level is a function (the access from the next level down), and the level is its memo.
\item \textbf{Connection.} Let the bottom of a hierarchy call a second system, $f = h \circ b \circ \kappa$, where
  $b = C'_1 \rhd \cdots \rhd C'_m \rhd g$ is the access of the second system, $\kappa$ maps the keys of the first to those of the
  second, and $h$ computes on the value returned. Then
  \[ C_1 \rhd \cdots \rhd C_n \rhd h \circ (C'_1 \rhd \cdots \rhd C'_m \rhd g) \circ \kappa \]
  is one hierarchy, in which $\kappa$ is a translation level that holds nothing ($W = \emptyset$), costs nothing and only
  translates the key, as Definition~\ref{def:store} admits, and $h$ post-processes the bottom. A key that misses in every $C_i$ costs
  \[ \textstyle\sum_{i \le n} t_i + c(h) + c(b, \kappa(k)). \]
\item \textbf{Recursion.} Let $f_1, \dots, f_m$ have memos $C_1, \dots, C_m$, with working sets $W_1, \dots, W_m$, and let each
  of them call the next once, $f_j = h_j \circ f_{j+1} \circ \kappa_j$ for $j < m$. Coherence lets each $f_{j+1}$ be replaced by its
  memo, and applying item 2 at each call gives $\langle f_1, C_1 \rangle = g_1$, where
  \[ g_m = C_m \rhd f_m, \qquad g_j = C_j \rhd h_j \circ g_{j+1} \circ \kappa_j \quad (j < m). \]
  With $k_1 = k$ and $k_{j+1} = \kappa_j(k_j)$, and $c_j$ the cost of $h_j$,
  \[ \begin{array}{@{}l@{}}
  c(k) = t_1 + [\,k_1 \notin \mathrm{dom}(W_1)\,]\bigl(c_1 + t_2 \\
    \qquad {} + [\,k_2 \notin \mathrm{dom}(W_2)\,]\bigl(c_2 + \cdots\bigr)\bigr). \qedhere
  \end{array} \]
\end{enumerate}
\end{construction}

In words: a hierarchy is a function seen through its memos, and it can be cut at any level (1). A miss at the bottom of one
system is a call into another, so two systems connect into one hierarchy, with a translation level that holds nothing where their keys
differ (2). When each function of a chain calls the next, the chain with the memos of its functions is a hierarchy with one level per
function, and its cost is paid down to the first memo that hits (3).

\paragraph{The cost of a state presentation.} The cost of Definition~\ref{def:state-pres} is item 3 applied to $\hat s$ itself,
since the recurrence $\hat s(x\gamma) = \varphi(\hat s(x), \gamma)$ calls $\hat s$ on the prefix $x$. Take $\kappa$ to drop the last
symbol, take $h = \varphi(\cdot, \gamma)$, and use the one memo $C_{\hat s}$ at every call. With only applications of $\varphi$ counted, a miss then
applies $\varphi$ from the longest cached prefix of $x$ on. A truncatable entry for $z$ holds every prefix of $z$, which gives
$\mathrm{cost}(z, x)$.

\subsection{The Stores of the Two Sides}

\paragraph{The Oracle's store.} An insertion into the state cache adds the states of the symbols computed by prefill and decode;
a rewrite of the task starts a new epoch, whose keys are new; and a truncation to a prefix deletes entries, at no cost for a
truncatable presentation. When the store is organized as a prefix tree, each key is one node of the tree: naming it takes constant
space, and it holds only the states of the symbols its parent lacks. If the store is keyed by Step instead and each transition is
a Step, then two adjacent queries need not share a prefix, since the encoding of the longer history may segment its end differently
(\S\ref{sec:eclipse-codec}). The node of the earlier query is then not an ancestor of the later one.

\paragraph{The Priestess's store.} For a CPU, one computation is one transition, an instruction of the program $\Pi$ of
Definition~\ref{def:2pda}, and $\Pi$ computes $y = \Pi(x)$ as a sequence of them, just as a Frame is a sequence of Steps. The keys of
the CPU's store are the addresses an instruction reads and writes, its levels are the processor's caches and its memory, and its
backing store is a disk. It has in general no rematerialization, whereas the Oracle's state cache rematerializes a miss by a prefill.
Apart from that, its keys are addresses whereas the Oracle's are histories, and that is the only difference between the two, since
both are systems of Definition~\ref{def:store}.

\subsection{Coherence of the Encoding Cache}

For each stack $s$ on which the Oracle is called, the encoding cache keeps an entry $(|x|, E(x))$, where, at the last call, $x$ is
the longest prefix of $\sigma(s)$ that is empty or ends with $\natural$. If $E$ factorizes at $\natural$ and $\sigma(s) = x\,z$ at
the next call, then $E(\sigma(s)) = E(x)\,E(z)$ (Theorem~\ref{thm:run-sync}(1)), and the codec reads and encodes only $z$: the
unfinished Frame at the last call, followed by everything appended since. The condition $\sigma(s) = x\,z$, under which the
encoding cache reads and encodes only $z$, is coherence between memory and the codec, and it holds while the stack only grows. A
write other than a push can change $x$ without changing its length, and the codec cannot determine from the stack's length alone, so the
writer invalidates the entry: the kernel does so when it rewrites a task stack, and the program does so when it rebuilds an
operand. Alternatively the codec compares $x$ with the stack, or keeps a version counter, never reused, that is incremented by every write other
than an append. A hash gives only a probabilistic guarantee.

Since one key indexes both caches (Proposition~\ref{prop:prefix-chain}(3)), the start of a new epoch invalidates both at once. The
two caches lag differently: at each call the encoding cache reads and encodes $D(y)\,v$, while the state cache already holds $y$
from generation and processes only $E(v)$.

\subsection{The Prefix Chain in the Two Forms}

Table~\ref{tab:chain} sets side by side what \S\ref{sec:eclipse-memory} says of the chain in each form: a machine invariant within an
append-only epoch in the Agent form, a property of the program text in the Workflow form.

\begin{table*}[t]
\centering
\small
\begin{tabular}{@{}>{\raggedright\arraybackslash}p{0.2\textwidth}>{\raggedright\arraybackslash}p{0.37\textwidth}>{\raggedright\arraybackslash}p{0.37\textwidth}@{}}
\toprule
 & Agent form (generation in O-mode) & Workflow form (generation in P-mode) \\
\midrule
Key of the next Oracle access & $\sigma(u)$, known to the kernel when it dispatches & $\sigma(d)$, known when the call is issued \\
\rowrule
Prefix chain & a machine invariant within an append-only epoch & a property of the program; sufficient: no $\mathtt{pop}\ d$ \\
\rowrule
Encoding cache & valid within an epoch; the kernel invalidates it when it rewrites the task & valid while the program only pushes onto $d$; the program invalidates it when it rebuilds $d$ \\
\rowrule
Prefix tree & a kernel service that copies a history & calls whose operands share a pushed prefix, as in parallelization \\
\rowrule
What breaks the chain & the kernel rewrites the task; a Frame that does not re-encode to itself; a budget trap inside a Frame & the program rebuilds its operand; a Frame that does not re-encode to itself, when answers are fed back \\
\bottomrule
\end{tabular}
\caption{The prefix chain and the caches in the two forms.}
\label{tab:chain}
\end{table*}

\subsection{The Two Schedulers in More Detail}

The Oracle side estimates the deadlines and priorities of the requests it serves. The Priestess side estimates which entries are
resident and hence what each access costs (Definition~\ref{def:state-pres}), and it uses the estimate for cache affinity. Each
side knows the other exactly when the other is transparent to it.

\paragraph{Scheduling the Oracle's computations.} An Oracle-side scheduler handles sequences of computations, not independent ones.
For a language model the computations of such a sequence are closely linked: a task alternates turns (assistant, tool, assistant,
user), and each assistant turn is an Oracle computation on the history the earlier turns left. The amount of the cache that the call $r_{t+1}$ can reuse from the call
$r_t$ depends on the granularity at which the history grows between them: a Step, a Frame or something in between.
\begin{itemize}
\item \textbf{At Frame granularity}, with the hypotheses of Proposition~\ref{prop:prefix-chain}(2), the state cache already holds
  the whole earlier history when the next query arrives: the query costs only what the kernel pushed since the last Frame, and
  the key of each query is an ancestor of the next one's in the prefix tree.
\item \textbf{At Step granularity}, each Step re-encodes the history, and beyond the growth of the encoding the state cache
  recomputes up to the Step backtrack (Proposition~\ref{prop:excess}); for the pair $(w, v)$ the bound is the one of that
  proposition, $(|E(wv)| - |E(w)|) + L_E^G$, and it is $L_E^G + 1$ only when the encoding grows by at most one symbol. For BPE with a proper dictionary, that backtrack is finite
  (\S\ref{sec:eclipse-codec}). A budget trap inside a Frame leaves a history that Proposition~\ref{prop:prefix-chain} does not
  cover (below).
\end{itemize}
This is for one task on one Oracle with a private cache. With several Oracles for one task, a state belongs to one Oracle, its
codec and its presentation, so a task that is moved to another Oracle finds no state of its history there unless that Oracle holds one,
and may pay a prefill of the whole encoding. With several tasks on one Oracle, the tasks share its stores. Under finite capacity,
which the machine does not model, the states of one task may evict those of another, and the reuse of
Proposition~\ref{prop:prefix-chain}(2) is then an upper bound. The tiers of the KV cache turn the trade-off between affinity and
load balance into a comparison of hit and miss costs level by level. The usual metrics (utilization, throughput, turnaround, waiting and
response time, SLO attainment, fairness and energy) can be defined over the events of the model (excursions, traps and Oracle
transitions), weighted by a cost per transition. We do not develop them, beyond the cost of Proposition~\ref{prop:return}.

\subsection{A Budget Trap inside a Frame}

A budget trap can occur inside a Frame, and the history there need not end with $\natural$. Consequence (i) of
Theorem~\ref{thm:run-sync} still confines the encoder's loss to the open Frame. Frame-faithfulness, however, applies only to completed Frames,
and the state that the Oracle holds need not be the encoding of the history. For example, if the Oracle generates a symbol $t$ with
$D(t) = ab$ and the encoder writes $ab$ as $[a][b]$, a budget trap right after $t$ leaves the history $ab$, encoded $[a][b]$, while
the Oracle holds the state of $[t]$. Proposition~\ref{prop:prefix-chain} does not cover this case.

\subsection{Where Non-empty Answers and Frame-Faithfulness Hold}
\label{app:eclipse-scope}

Two conditions matter more than their statements suggest: the convention of non-empty answers (\S\ref{sec:machine})
and Frame-faithfulness (\S\ref{sec:eclipse-codec}). For each, this subsection says where it holds by construction, where a
realization must preserve it, which results use it and what fails without it. The contract and the safety boundary use neither.

\paragraph{Non-empty answers.} For an internally autoregressive Oracle called through the Frame interface, the convention is a
consequence of Definition~\ref{def:int-ar}, not an assumption: every Frame ends with the terminator, and $D(\tau) = \natural$, so
every decoded Frame ends with $\natural$. Through the Step interface it holds when no symbol of $\Gamma_R$ decodes to $\lambda$
(\S\ref{sec:eclipse-codec}). With characters as the machine-side symbols, a byte-level vocabulary can break this, since a token
that holds only the first bytes of a multi-byte character decodes to no character (the paragraph on the decoder, above).

For an Oracle that is not internally autoregressive, and for a model behind an interface, the convention is an obligation of the
codec. An interface may return an answer with no visible text: when the output limit is spent on reasoning that the interface does
not return, when a filter withholds the content, or when a refusal is reported apart from the content. A codec keeps the
convention in one of two ways. It can decode such an answer to a fixed non-empty word, $\natural$ alone for instance, which in
O-mode is an empty request at a yield trap and which the kernel handles as an error (Definition~\ref{def:os}). This makes the real
Oracle total in this respect (\S\ref{sec:machine}). Or it can treat the answer as a failure of the Oracle operation, which writes
nothing and leaves it to the kernel to decide whether to enter the task again (Appendix~\ref{app:eclipse-replay}). The Oracle is then not total,
and the realization handles the failure.

Four results use the convention. The kernel distinguishes a yield trap from a budget trap by whether the task stack ends with $\natural$
(Definition~\ref{def:budget}). The interpreter of Proposition~\ref{prop:interpreter} tests a stack top that belongs to the new
answer (Appendix~\ref{app:proofs}). Within an append-only epoch the queries on a task strictly grow, so a query recurs on a task
only after the kernel rewrites the task. And an O-mode iteration never returns to the configuration it
left (\S\ref{sec:machine}). Without the convention, the first two need a record of where each answer begins, which a kernel or an
interpreter can keep beside the stack. The last two fail: an empty answer repeats the query, which samples the Oracle again with no
write of the kernel, and an excursion can then run in place until a budget, if there is one, ends it. K5 does not use the convention.

\paragraph{Frame-faithfulness.} Frame-faithfulness is a property of an Oracle with its codec, on the Frames the Oracle can
generate. For a language model it holds when three conditions hold at once: the model generates only tokenizations that the
tokenizer produces, the decoder keeps every generated token, and the next query is encoded by the same tokenizer, with nothing
rendered around the answer. Each can fail.
\begin{itemize}
\item \emph{Sampling.} A model may generate a tokenization that the tokenizer never produces
  (Appendix~\ref{app:further-encoders}). Restricting the model to canonical encodings, by conditioning or by
  construction~\cite{vieira2025blanguage}, removes this cause.
\item \emph{Decoding.} A decoder that maps some generated symbols to $\lambda$ drops them from the history. The reasoning of a
  model whose interface returns it apart from the answer, and leaves it out of the next query, is an example
  (Appendix~\ref{app:eclipse-reasoning}). Then $E(D(y))$ is the
  encoding of a text without them, and for a tokenizer $E(D(y)) \ne y$ whenever such symbols were generated.
\item \emph{Rendering.} A chat interface renders the history through a template on the provider's side, and the template is
  then part of the encoder. A Frame is Frame-faithful only if the template renders a past answer as the tokens that were generated.
\end{itemize}
A budget trap inside a Frame (above) is a fourth situation, outside the definition, which concerns completed Frames only.

Theorem~\ref{thm:run-sync}(2) says exactly when Frame-faithfulness holds on a history that ends at a Frame boundary,
Proposition~\ref{prop:prefix-chain} assumes it, and the remark after Proposition~\ref{prop:steps-invisible}, that a Frame-faithful
Oracle shows its number of Steps, uses it. Theorem~\ref{thm:run-sync}(1) with its three consequences, Proposition~\ref{prop:excess},
the contract and the safety boundary do not. When it fails, the cost is bounded and local. The next query recomputes
$|E(D(y))| - |\mathrm{lcp}(y, E(D(y)))|$ symbols beyond the push, at most the re-encoding of the last Frame
(\S\ref{sec:eclipse-memory}). The completed Frames before it still agree (consequence (iii) of Theorem~\ref{thm:run-sync}), and the
difference does not outlast the next query, which is encoded afresh. Under hypotheses (i), (iii) and (iv) of
Proposition~\ref{prop:prefix-chain}, with the private state cache of its item~2, an Oracle that is never Frame-faithful therefore computes each answer once by generation and prefills its re-encoding at most once.

Both conditions can be checked on a realization: the first by counting the answers that the codec had to repair, the second by
comparing each generated Frame with the encoding of its decoding, which needs an interface that returns the Oracle's symbols
(\S\ref{sec:eclipse-codec}).

\subsection{Reasoning: Kept, Dropped or Hidden}
\label{app:eclipse-reasoning}

A model that reasons before it answers generates a Frame $y = y_r\,y_a\,\tau$, in which $y_r$ is the reasoning. Whether $y_r$
enters the history is a choice of the codec or of the kernel, not of the machine: an Oracle is the triple $(A, E, D)$
(\S\ref{sec:machine}), and each choice below leaves $A$ unchanged. K1--K4, the crossing invariant and the safety boundary hold for every
Oracle, and so for every choice, and K5 holds wherever the machine is good. What changes is which hypotheses of
\S\ref{sec:eclipse} hold, and what the Priestess stores.
\begin{itemize}
\item \emph{Kept.} $D$ writes the reasoning onto the stack, and $E$ reads it back with the rest. This is the codec the body adopts.
  The Oracle can then be Frame-faithful, and the Priestess holds the reasoning and finds it in each request.
\item \emph{Dropped by the decoder.} $D$ maps every symbol of $y_r$ to $\lambda$. Working symbol by symbol, it can distinguish such a symbol
  from the same token in the answer only if $\Gamma_R$ marks each token with the part of the Frame it belongs to, as a model's
  reasoning channel does. The step relation can always add the mark, since it sees whether the reasoning has ended. So the Oracle
  is still internally autoregressive, and the encoder is unchanged. Frame-faithfulness fails whenever reasoning was generated (the
  decoding case of Appendix~\ref{app:eclipse-scope}), at the bounded, local cost given there: the next query recomputes at most
  $|E(D(y))|$ symbols beyond the push, the encoding of the visible part of the answer, and the states computed for the reasoning
  are not reused. The Step interface is not an Oracle with this decoder, since a Step of reasoning decodes to $\lambda$. Since the
  history is all that one Step passes to the next, such a Step could not condition the next one either, so the reasoning is
  confined to one transition, at the Frame interface. It never reaches the Priestess: the request at a yield trap is
  $D(y_a)\,\natural$.
\item \emph{Dropped by the encoder.} $D$ keeps the reasoning, and $E$ leaves it out, as a chat template does when it renders past
  answers without their reasoning. If $E$ leaves out all reasoning, it can do so piece by piece between two $\natural$, since a
  Frame's reasoning contains no $\natural$, and it then factorizes at $\natural$ when the encoder it wraps does. As long as the
  kernel writes the same, and writes no reasoning itself, the Oracle receives the queries of the previous case, with the same
  failure of Frame-faithfulness (the rendering case) and the same cost. The difference lies on the Priestess side, which holds the
  reasoning, for a log or a checkpoint. A template that keeps the reasoning after the last user message and leaves out the rest
  does not factorize at $\natural$. Let the kernel push a user message $v$, ending with $\natural$, onto a history $w$ that ends
  with $\natural$ and contains reasoning after its last user message. Then $E(w)$ holds reasoning that $E(wv)$ leaves out, so
  $E(w) \not\sqsubseteq E(wv)$, although $(w, v)$ is a Frame pair. Between two user messages the prefix chain can hold. If the
  template is otherwise a separator, the next query costs, beyond the push, the encoding of $w$ from the start of that reasoning
  on, now without it: the chain breaks as at a rewrite, although the machine's history only grows.
\item \emph{Dropped by the kernel.} The kernel removes the reasoning from the task stack, at each trap or at a swap
  (\S\ref{sec:an-checkpoint}). This is a rewrite: it starts a new epoch, and its price depends on where it begins
  (Appendix~\ref{app:further-cache}). A kernel that removes an answer's reasoning at the trap after the answer can read the
  reasoning first. Its rewrite begins after the last query, so with the state of that query still cached it pays what the decoder
  above pays. In the Workflow form the program does the same when it rebuilds its operand.
\item \emph{Hidden.} An interface may return the reasoning as an opaque block, encrypted or signed, that the client must send back
  unchanged, or keep it on the provider's side and return a reference to it. Either is a codec: $D$ writes the block or the
  reference, and $E$ restores the reasoning from the block or looks it up. The machine does not model which code can interpret a
  symbol, so to its results the block is reasoning written in symbols that the Priestess's code does not interpret.
  Frame-faithfulness holds or fails as for kept reasoning, and the reasoning re-encodes to itself by construction if the block
  carries the generated symbols, not their text. Three properties change. The encoder is partial: a block that fails its check
  is rejected, so Proposition~\ref{prop:closure} does not apply, and a kernel that changes part of a block and enters the task
  again makes the machine fault (\S\ref{sec:os}); it may still remove a block whole. A block is made for its Oracle: the encoder
  of another Oracle may reject it or leave it out, and a task moved there then loses its reasoning as well as the states of its
  history. And a block of authenticated encryption depends on all of the reasoning, so the decoder does not work symbol by symbol,
  the Oracle with this codec does not meet Definition~\ref{def:int-ar} as stated, and the Step interface cannot return the
  reasoning. Theorem~\ref{thm:run-sync} uses the decoder only on whole Frames, and only to know that a decoded Frame ends with
  $\natural$ and contains it nowhere else, so its proof still applies at the Frame interface. So does that of
  Proposition~\ref{prop:prefix-chain}, with its hypothesis (iii), which the partial encoder does not meet, replaced by the
  acceptance of every query on $u$. A check at the boundary still sees every block cross (Theorem~\ref{thm:boundary}), but it can
  read an encrypted block only with its key. A reference into a store that expires its entries makes the encoder change over
  time, which the machine, whose Oracle is fixed, does not model. A realization can treat an expired reference as an encoder
  error (Appendix~\ref{app:archnights-exc}), and backing off from it is a rewrite.
\end{itemize}

\subsection{Protection, Channels and Attacks}

\begin{table*}[t]
\centering
\small
\begin{tabular}{@{}>{\raggedright\arraybackslash}p{0.09\textwidth}>{\raggedright\arraybackslash}p{0.08\textwidth}>{\raggedright\arraybackslash}p{0.12\textwidth}>{\raggedright\arraybackslash}p{0.305\textwidth}>{\raggedright\arraybackslash}p{0.305\textwidth}@{}}
\toprule
 & Program & Operating system & What protects the operating system from the program & What protects the program from the other side \\
\midrule
Agent form & Oracle & Priestess: kernel and library & T with the locality lemma: a task writes only its own stack, and requests enter only at $\ell_{\mathrm{trap}}$ (K1, K3) & nothing in the machine: the kernel is privileged, may rewrite the task, and is trusted, as in a conventional operating system \\
\rowrule
Workflow form & Priestess & none mediates the task & nothing to protect & S and T: answers are data, reach only their operand and gain no execution right at the call (K4); validating them before a promotion is the program's obligation \\
\bottomrule
\end{tabular}
\caption{Protection by form.}
\label{tab:protection}
\end{table*}

\paragraph{Further protections.} S restricts only the Oracle side, and the Priestess, the trusted side, may touch any stack.
Protections that also restrict the trusted side are further protections, as with SMEP and SMAP on current
processors~\cite{intelsdm}. Under the first, the Priestess does not execute memory the Oracle can write. Under the second, it
touches Oracle storage only through explicit operations, so that a kernel bug cannot read or write a task's history by accident.
When the Priestess is restricted in such ways, the branch to $l_e$ of a P-mode call still shows whether a word that it copies lies in
$\mathrm{dom}(E)$. This is an architectural channel, unlike the transient ones of Meltdown and
Spectre~\cite{lipp2018meltdown,kocher2019spectre}.

\paragraph{What a task observes.} By K1 and K2, the past of the run, another task or the Priestess stacks can change the possible
histories of a task only through two Priestess actions: the Priestess's writes to the task, and its choice of when, and how
often, to enter it. The machine opens no content channel into a task other than the kernel's writes, and the kernel's scheduling
of the task is a control channel. The statement is about sets of possible outcomes. A statement about probabilities needs, in
addition, that the distribution of answers depends only on the query. A realization in which Priestess code reads a clock adds
channels outside the model: the latency of a call reveals whether a prefix was resident in a cache shared across tasks or users.

\paragraph{Replay.}
\label{app:eclipse-replay}
If the Oracle is not deterministic, either side can call it again and again until it produces a vulnerable program, or a thunk of
one. A bad answer of probability $p > 0$ to $w$ appears among $n$ independent queries of $w$ with probability $1 - (1-p)^n$. On the
machine, answers are non-empty and an Agent's queries therefore strictly grow, so the Agent does not re-sample within its own epoch,
and a query recurs on a task only after the kernel rewrites the task. A service allows the same query again if it gives a task a
history equal to an earlier one (by copying a frozen history, restoring a saved one, or loading the same initial content again). Such
services are writes of the kernel, on the boundary of Theorem~\ref{thm:boundary}, and they are sufficient for replay but not
necessary, since different histories may extend to the same word. A copy that is followed by different replies gives histories that share a
prefix but differ, so one request is continued twice. A realization adds paths that the machine does not model. An Oracle operation
that fails or is refused writes nothing, so entering the task again issues the same query once more. On the machine, that is the
kernel's decision to return with an empty reply, whereas a retry inside the Oracle's unit is inside the Oracle. A replay
policy at those services and paths can limit them, or give a fixed answer per query. A fixed answer per query
returns answers within $O(w)$, as a recording does, so it still produces executions of the machine. Zheng et al.\ check the operations checkpoint, fork,
restore and merge, against an execution record whose log of authorizations is append-only, and reject an edit that would authorize an action twice,
break the policy, or discard a result the task still requires~\cite{zheng2026when}. That check sits at the same services and
applies to authorization, not the re-sampling of the Oracle. In the Workflow form the program may re-issue a query freely, and replay
is a property of its code. CaMeL's interpreter, for instance, prompts the privileged model up to ten times in all, again after each
exception, until the model's code runs~\cite{debenedetti2025defeating}, and LLMbda observes that this retry loop lets an adaptive
driver launder a secret one bit at a time~\cite{garby2026llmbda}.

\paragraph{Syscall and return.} A system call is a yield trap whose request lies in $L_\Lambda$ (Definition~\ref{def:os}), and the
return is the kernel's reply, appended to the task before the next $\mathtt{uret}$. Membership classifies the request, but whether
the kernel recognizes it and can serve it is the kernel's own judgement. The machine would equally resume an unanswered task, and
what the task does then is a property of the Oracle. That every system call is answered first is the kernel's protocol.

\section{More on ArchNights}
\label{app:archnights}

This appendix gives the parts of \S\ref{sec:archnights} that the body summarizes: the Oracle instruction and its asynchrony, more
on PRTS, the exceptions of an Oracle unit and how PRTS handles them, communication and system calls, more on the distribution, and
the rest of ZOOT's design.

\subsection{The Instruction Set in More Detail}

\paragraph{Encoding.} The Oracle instruction is a 32-bit R-type instruction in the custom-0 opcode space (Table~\ref{tab:isa-word}).
Its funct3 field is fixed at \texttt{111}, the value with which RoCC reads both source registers and writes the destination
register. Its funct7 field holds three subfields above two reserved bits: the operation, which is PROBE, STEP or RUN; the
direction in which the stack grows, towards its limit or towards its base; and a profile. PROBE has the single encoding
$\mathrm{funct7} = \mathtt{0x00}$. For STEP and RUN, $\mathrm{funct7} = (\mathit{profile} \ll 3) \mid (\mathit{op} \ll 1) \mid
\mathit{dir}$, so that, under P0, STEP is \texttt{0x02} downwards and \texttt{0x03} upwards, and RUN \texttt{0x04} and
\texttt{0x05}. An instruction is illegal if it has any other funct3, the reserved operation, non-zero reserved bits, any other
funct7 with the operation PROBE, or a profile the unit does not implement. The source registers are read before the destination
register is written, so an operation whose destination names a source register still sees the old operands, and a destination
\texttt{x0} discards only the result, not the operation.

\begin{table}[t]
\centering
\small
\begin{tabular}{@{}l l >{\raggedright\arraybackslash}p{0.56\columnwidth}@{}}
\toprule
Bits & Field & Value or meaning \\
\midrule
31:25 & funct7 & reserved[6:5] $= 0$; profile[4:3]: \texttt{00} P0, \texttt{01} P1, \texttt{10} P2, \texttt{11} P32;
  op[2:1]: \texttt{00} PROBE, \texttt{01} STEP, \texttt{10} RUN, \texttt{11} reserved; dir[0]: \texttt{0} down, \texttt{1} up \\
\rowrule
24:20 & rs2 & the offset of the top, with the budget (RUN) or the write bound (STEP) \\
\rowrule
19:15 & rs1 & the segment's base and its capacity in pages \\
\rowrule
14:12 & funct3 & \texttt{111}: read rs1 and rs2, write rd \\
\rowrule
11:7 & rd & the result \\
\rowrule
6:0 & opcode & \texttt{0001011} (custom-0) \\
\bottomrule
\end{tabular}
\caption{The fields of the Oracle instruction.}
\label{tab:isa-word}
\end{table}

\paragraph{Profiles and registers.} The profiles fix the widths of the fields packed into the registers (Table~\ref{tab:isa-profiles}),
and with them the largest segment and the largest budget that an operation can name. P0, P1 and P2 target RV64 and trade the width of the
offset against the widths of the budget and of the step count; P32 serves RV32 implementations. Under P0 the registers are
\[
\begin{aligned}
\mathrm{rs1} &= \{\,\mathit{base}[63{:}12],\ \mathit{cap\_pages}[11{:}0]\,\},\\
\mathrm{rs2} &= \{\,\mathit{budget}\ \text{or}\ \mathit{wlimit}[63{:}24],\ \mathit{offset}[23{:}0]\,\},\\
\mathrm{rd} &= \{\,\mathit{status}[63{:}60],\ \mathit{steps}[59{:}24],\ \mathit{bytes}[23{:}0]\,\}.
\end{aligned}
\]
The other profiles keep this frame with the widths of Table~\ref{tab:isa-profiles}. P1 and P2 split the capacity: its low twelve
bits stay in rs1, and its high bits sit in rs2 between the bound and the offset. P2 reserves four bits of rd above its step count,
and P32 four bits of rs2 above its budget. The base is rs1 with its low twelve bits cleared, so a segment is aligned to a 4~KiB
page; its capacity is $\mathit{cap\_pages} \cdot 4096$ bytes, and its limit is $\mathit{base} + \mathit{capacity}$. All the fields
are unsigned, and reserved bits are written as zero.

\begin{table}[t]
\centering
\small
\setlength{\tabcolsep}{4pt}
\begin{tabular}{@{}l r r r r r r r@{}}
\toprule
Profile & XLEN & offset & steps & budget & pages & bound & largest \\
\midrule
P0 & 64 & 24 & 36 & 40 & 12 & 40 B & 4095 \\
\rowrule
P1 & 64 & 28 & 32 & 32 & 16 & 32 B & 65535 \\
\rowrule
P2 & 64 & 32 & 24 & 24 & 20 & 24 P & 1048575 \\
\rowrule
P32 & 32 & 22 & 6 & 6 & 12 & 10 P & 1023 \\
\bottomrule
\end{tabular}
\caption{The widths, in bits, of the fields each profile packs into the registers: the offset of the top (also the width of the
published byte count), the step count of the result, the budget of a RUN, the capacity in pages, and the write bound of a STEP, counted in bytes (B) or in pages (P). Largest: the largest segment, in pages.}
\label{tab:isa-profiles}
\end{table}

\paragraph{Operands and direction.} The offset is measured from the base. A stack that grows up stores its content in
$[\mathit{base}, \mathit{base} + \mathit{offset})$, and an answer of $n$ bytes is written at $\mathit{base} + \mathit{offset}$,
after which the offset grows by $n$. A stack that grows down holds its content in $[\mathit{base} + \mathit{offset},
\mathit{limit})$, and the answer is written just below $\mathit{base} + \mathit{offset}$, after which the offset shrinks by $n$.
Either way the answer only extends the content and leaves what was there unchanged. The direction is an arbitrary choice: an
instruction names it, and O-mode, which runs no instruction, takes its direction from the system's configuration, as the direction of a
process stack is fixed. Before evaluation the unit checks that the
offset lies within the capacity, that the whole range can be addressed, and that it may be accessed. A write bound of zero forbids
a STEP to write, and the all-ones value sets no bound beyond the room left. The budget of a RUN is soft (\S\ref{sec:an-isa}).

\paragraph{The result register.} The result contains the status, the number of steps and the number of bytes published. The bytes are those of
the whole answer and of nothing else. For an Oracle that offers STEP, the steps are the Steps committed; for a language model they are its
output tokens, saturated at the largest value the field holds, while the descriptor keeps the full count. The
status reports how the operation ended (Table~\ref{tab:isa-status}): with the yield symbol (a Halt, also when the first step yields), with a Step that
did not yield (Continue), with the budget spent, or with an exception of one of the unit's parts (Appendix~\ref{app:archnights-exc}).
An exception keeps what was published before it and publishes no descriptor. Publication is mediated by rd: its byte count
delimits what the operation added to the segment, so the kernel that handles an exception knows exactly what the operation left.

\begin{table}[t]
\centering
\small
\begin{tabular}{@{}r l >{\raggedright\arraybackslash}p{0.5\columnwidth}@{}}
\toprule
Code & Part & Meaning \\
\midrule
0 & & Continue: a Step that did not yield \\
\rowrule
1 & & Halt: the answer ended with the yield symbol \\
\rowrule
4 & decoder & no room: the answer, or a STEP's write bound, does not fit \\
\rowrule
5 & controller & budget spent before a yield \\
\rowrule
6 & encoder & short input: the content is too short to form a request \\
\rowrule
8 & encoder & illegal argument, refused before evaluation \\
\rowrule
9 & encoder & the content cannot be encoded \\
\rowrule
10 & compute & timeout, recoverable \\
\rowrule
11 & compute & unavailable \\
\rowrule
12 & decoder & the answer is incomplete or malformed \\
\rowrule
2, 3, 7, 13--15 & & reserved \\
\bottomrule
\end{tabular}
\caption{The status field of the result.}
\label{tab:isa-status}
\end{table}

\paragraph{PROBE.} PROBE ignores its source registers, computes nothing and changes no stack. Bits 0 to 3 of its result say which
of P0, P1, P2 and P32 the unit implements, and bit 4 (DESC) whether its STEP and RUN answers come with a descriptor. Bits 5 and 6
report what the bound Oracle can do: bit 5 (AR) that it is internally autoregressive, and bit 6 (STEP) that it offers STEP. Bit 5
decides whether the Oracle may run in O-mode. The higher bits are reserved. The report applies to the current binding and is refreshed
when the binding changes. An Oracle that offers only RUN may overrun its budget by the Frame it is in, and the Oracle's own
output limit bounds that Frame.

\paragraph{The descriptor.} With DESC set, a STEP or RUN that ends normally writes a descriptor just past the answer, in the
direction of growth and inside the window. The descriptor is not counted in the bytes and does not move the top, so it is not part of the next query, and
the next answer may overwrite it. Its fixed header has 32 bytes, little-endian and with no alignment required: a format byte
(\texttt{0xFF} in this version); a checksum byte, chosen so that the bytes of the whole descriptor XOR to zero; a byte of flags, two
of which indicate whether the input and the total counts are valid, the rest left to the Oracle; the length of an extension after the
header, in blocks of 16 bytes; a 32-bit request number, increasing for each Oracle and wrapping around; and the Oracle's input,
output and total counts as 64-bit numbers, for a language model its tokens. The checksum detects a changed byte but authenticates
nothing.

\paragraph{Where the asynchrony goes.} An asynchronous realization preserves the synchronous semantics, and the asynchrony can be put
in one of three places. \emph{In the microarchitecture}, the Oracle instruction is an instruction of long latency, and an
out-of-order core can execute independent instructions while the unit computes. The instructions after it do not commit before it
completes, so the architectural state is the synchronous one. What speculation leaves behind is a side channel of the kind
discussed in Appendix~\ref{app:eclipse}. \emph{Through an SMT abstraction}, the wait is hidden as simultaneous multithreading
hides a long miss: the core runs another hardware thread while one thread waits for its answer. \emph{In the operating system},
the kernel issues the operation for one task and runs other tasks meanwhile, as it does around a disk request. It never runs one
task twice at once: a task is in O-mode on at most one processor, and no processor writes it meanwhile. With one dispatch point on
one processor this holds trivially, whereas a kernel on several processors makes it a scheduler invariant.

\subsection{More on PRTS}
\label{app:archnights-prts}

\paragraph{Fork in Linux.} In Linux, fork gives the child a copy of the parent's address space, copied on write, and a
copy of its descriptor table. The copied descriptors share the parent's open files, their offsets included, so a pipe whose ends
both processes inherit has two readers and two writers until each closes the ends it does not use. Fork also copies the working
directory, the credentials and the signal dispositions. The child has a new identifier, and its resource accounting starts at
zero.

\subsection{Exceptions of an Oracle Unit and in PRTS}
\label{app:archnights-exc}

\paragraph{The exceptions of an Oracle unit.} Each part of an Oracle unit may raise exceptions of its own:
\begin{itemize}
\item the encoder: an illegal argument, insufficient memory, an encoding failure. An operand whose top lies beyond its limit is an
  illegal argument, and a history that a language model cannot take as input (a request without its reply, for instance) fails to
  encode.
\item the compute part: a recoverable error, such as a timeout of a model's service, and an unrecoverable one, such as the loss
  of the service.
\item the decoder: insufficient capacity, when the answer does not fit the room left, and a decoding failure, when the answer is
  incomplete or malformed (an answer of zero bytes among them).
\item the controller: the budget, when it is spent before a yield. This is the budget trap of Definition~\ref{def:budget}, not a
  fault.
\end{itemize}

\paragraph{How PRTS responds to each exception.} An unrecoverable error usually ends the process, and the checkpointing system saves the process
as a checkpoint. A recoverable error is retried, inside the Oracle by its unit and inside the Priestess by the kernel, which enters
the task again. The Priestess answers an encoder error by backing off: the kernel pops what made the history impossible to encode, and
tries again. A decoder error is met by computing the answer again, a re-sample, which a replay policy
(Appendix~\ref{app:eclipse-replay}) can limit. An error in the arguments of a system call (an unknown operator, a malformed argument or a
missing right) is answered with an error reply that names what was wrong. For an unknown operator, the reply also gives the
nearest name the task may use. The other calls in the same answer still run. Two more exceptions come from the limits of a
realization. When an answer does not fit its segment, the kernel enlarges the segment and calls again, as a segment grows like a
vector in \S\ref{sec:query}. When a Frame is cut off before it yields and cannot be continued, as when a language model's answer
is cut off at its output limit, the kernel backs off the unfinished answer by the byte count in rd, compresses the history and tries again. In an
interactive session the person decides instead of the kernel.

\subsection{Communication and System Calls in PRTS}
\label{app:archnights-comm}

\paragraph{How do tasks communicate?} Through the mechanisms of Unix: files, pipes, the exit status a parent collects by waiting,
and the terminal. Files live in the virtual file system. PRTS changes two things. Fork connects a parent and its child by
default, with a result channel whose write end the child holds and whose read end the parent holds, so that a child's result
reaches its parent without the parent wiring a pipe. And the Oracle holds no end of a pipe: pipes are made by the operator that runs a pipeline, by the
shell, and by fork for its result channel. Apart from what the kernel writes itself (the declaration, the replies and the records
of a session), content from these mechanisms or from the system library's operators enters a history only through an explicit,
bounded read by the task that owns it.
Output shown to a person, and the output of a program that its caller did not ask to capture, are not recorded in any history.

\paragraph{System calls.} PRTS's internal system calls are designed much as Linux's are, as a set of kernel functions in groups
(Table~\ref{tab:syscalls}). They lie one level below the system calls of Definition~\ref{def:os}, a task's requests, and are not
exported to O-mode by default: what the Oracle sees is the operators of the system library, each of which issues one or more of
these calls.

\begin{table}[t]
\centering
\small
\begin{tabular}{@{}>{\raggedright\arraybackslash}p{0.25\columnwidth}>{\raggedright\arraybackslash}p{0.7\columnwidth}@{}}
\toprule
Group & Calls \\
\midrule
Processes & spawn, exec, fork, wait, exit, sleep, yield, shutdown \\
\rowrule
Files & open, close, read, write, truncate, stat, list, mkdir, remove, rename, chdir, getcwd \\
\rowrule
Descriptors & name a descriptor, find a descriptor by its name \\
\rowrule
Pipes & create a pipe, query its state \\
\rowrule
Terminal & read a line, ask a question, print \\
\rowrule
Credentials & check an access, require a right \\
\rowrule
Histories & create, push, pop, read and set the metadata, run an Oracle on a history \\
\rowrule
Checkpoints & save a checkpoint, load a checkpoint, compress a history \\
\bottomrule
\end{tabular}
\caption{PRTS's internal system calls, by group.}
\label{tab:syscalls}
\end{table}

\subsection{More on the Distribution}
\label{app:archnights-dist}

\paragraph{Why does the Oracle not see the internal system calls?} Low-level calls force the Oracle to wire descriptors, pipe ends and the
order of waits, and each call costs a round trip through the Oracle and a piece of its context. Fewer calls, each doing more, is
also what the Lynchpin hypothesis suggests. PRTS's internal system calls are therefore not exported to O-mode by default
(Table~\ref{tab:syscalls}), and the system library's operators are organized by intent (Table~\ref{tab:operators}). One operator
performs one action a program wants, and descriptor numbers, pipe ends and the order of waits stay inside the kernel and the system
library. One call of \texttt{run} with a command line \texttt{a | b}, an input and a request to capture the output does what a
Unix shell does with a dozen system calls: it creates the pipes, starts each stage with its descriptors in place, writes the
input, reads the captured output to its end, and collects every stage. One call of \texttt{file\_patch} reads the file, checks
that it is still the version the task read, applies the patch and writes it back, or refuses and says so if the file changed in
between. Merging also costs flexibility: the Oracle cannot hold a pipe end, so it cannot build an arbitrary topology of pipes;
signals and non-blocking pipes are not exposed; and the declaration, part of every query, grows with every operator. The
obligation of \S\ref{sec:query} is that each operator performs the same services as the sequence of calls it replaces.

\paragraph{Core tools and \texorpdfstring{$\mathtt{sodo}$}{sodo}.} The distribution has core tools such as \texttt{echo} and
\texttt{ls}; some are builtins of the shell, and some are also O-mode programs. $\mathtt{sodo}$ (Super Oracle do) switches
temporarily to the privileged O-mode user, which gains the right to run commands on the host through the host's Bash, and every use
of $\mathtt{sodo}$ goes through its check. The operator \texttt{bash} carries the semantics of $\mathtt{sodo}$ by default: every
call of it goes through $\mathtt{sodo}$'s check before it reaches the host, so a task of the ordinary user may call it, and no call
bypasses the check.
Whoever controls $\mathtt{sodo}$'s check therefore controls the channel by which a task gains rights on the host.

\paragraph{Shell.} The shell is an interpreter of the kind that runs a script in its own process. What the person does at the prompt
(running a program, changing directory, starting a job) reaches the Oracle as one-line events, and so do the jobs that the session
collects.
These carry no output and are placed before the next request in the same user record. A program that the person runs writes to
the terminal, and its output enters no history.

\subparagraph{Four kinds of line.} A line is one of four kinds, told apart by its first character:
\begin{itemize}
\item an empty line prompts again;
\item a line that begins with \texttt{/} is a builtin, which the session runs without the Oracle;
\item a line that begins with \texttt{!} runs on the host, through $\mathtt{sodo}$;
\item any other line is a request in natural language, for the Oracle part.
\end{itemize}
Of the builtins, \texttt{/run} runs a pipeline in the foreground with the terminal, \texttt{/bg} runs one in the background, and
\texttt{/wait} waits for a job. \texttt{/cd}, \texttt{/pwd} and \texttt{/echo} are as in other shells, and \texttt{/source} runs a
file's lines as if they were typed. \texttt{/clear} replaces the history with a fresh one, and \texttt{/compact} swaps the history
now. \texttt{/exit} ends the shell, and \texttt{/shutdown} shuts the system down. The session's loop, with what it does at each
kind of line, is in Figure~\ref{fig:loop}.

\subparagraph{A history, record by record.} Figure~\ref{fig:history} shows a history of the shell as it grows. It extends the running
example of Figure~\ref{fig:readings} as ArchNights records it: the person first asks to see the file and then runs a program that
changes it, so the first patch is refused and the task reads the file again. Seen from O-mode, the history is one program that
keeps receiving its continuation (Figure~\ref{fig:continuation}). Each call forces a thunk that the declaration offers. The kernel
evaluates it in P-mode and returns a value, and the rest of the history is the continuation to which the value is returned. Let
$c_1$, $c_2$ and $c_3$ be the calls at lines 9, 11 and 13; lines 9 to 15 then become
$\mathtt{force}\ c_1\ \mathtt{to}\ x_1.$\ $\mathtt{force}\ c_2\ \mathtt{to}\ x_2.$\ $\mathtt{force}\ c_3\ \mathtt{to}\ x_3.$\
$\mathtt{return}\ \mathit{result}$, and the Oracle writes each step only after the value before it has arrived. The kernel's code
never appears in the history, only its results, so the history stays the Oracle's own program, written as it is run. Had the
answer at line 9 been a fork, the parent's history and the child's would both be lines 1 to 9, each followed by its own reply.

\begin{figure*}[t]
\centering
\begin{pseudo}
loop
    collect the jobs that ended, and note each as an event
    if the next line comes from the terminal and other tasks are ready: yield
    line \tc{1}{\(\scriptstyle\leftarrow\)} next line of the innermost script, or else a call of the Human class in P-mode
    case line of
        empty:          continue
        "!\tc{1}{\(\scriptscriptstyle\dots\)}":           if typed at the prompt, run it on the host; note an event
        "/\tc{1}{\(\scriptscriptstyle\dots\)}":           run the builtin; note an event
        otherwise:      append the user record  "events \tc{1}{\(\|\)} line"  to the history
                        enter the Oracle part: the Oracle answers and the kernel serves its calls,
                        until an answer calls no operator; that answer is the line's result,
                        written to standard output
\end{pseudo}
\caption{The loop of the shell's session.}
\label{fig:loop}
\end{figure*}

\begin{figure*}[t]
\centering
\def\byP{\tc{1}{\textcolor{pcol}{\rule{3.8pt}{3.8pt}}}}%
\def\byO{\tc{1}{\textcolor{ostroke}{\rule{3.8pt}{3.8pt}}}}%
\begin{pseudo}
\byP  1  declaration   input, print, run, fork, \tc{1}{\(\scriptscriptstyle\dots\)}, file_read, file_patch, \tc{1}{\(\scriptscriptstyle\dots\)}
\byP  2  system        conventions; the shell's text; the profile
\byP  3  user          program /bin/ansh; cwd /home/oracle; stdin a terminal
\byP  4  user          show me notes.txt
\byO  5  answer        file_read(notes.txt)
\byP  6  reply         "# notes / - prepare teh meeting agenda", stamp \tc{2}{s\(\sb{1}\)}
\byO  7  answer        notes.txt: # notes / - prepare teh meeting agenda
                    \tc{3}{\(\cdots\)} the person types  /run /bin/note buy milk  ; the program appends to notes.txt \tc{3}{\(\cdots\)}
\byP  8  user          [user ran: /bin/note buy milk \tc{1}{\(\scriptstyle\to\)} exit 0]
                    fix the typo in notes.txt
\byO  9  answer        file_patch(notes.txt, teh \tc{1}{\(\scriptstyle\to\)} the, stamp \tc{2}{s\(\sb{1}\)})
\byP 10  reply         refused: the file changed since stamp \tc{2}{s\(\sb{1}\)}
\byO 11  answer        file_read(notes.txt)
\byP 12  reply         "# notes / - prepare teh \tc{1}{\(\scriptscriptstyle\dots\)} / - buy milk", stamp \tc{2}{s\(\sb{2}\)}
\byO 13  answer        file_patch(notes.txt, teh \tc{1}{\(\scriptstyle\to\)} the, stamp \tc{2}{s\(\sb{2}\)})
\byP 14  reply         done
\byO 15  answer        Fixed: teh \tc{1}{\(\scriptstyle\to\)} the. The file had changed, so I read it
                    again first.
\end{pseudo}
\caption{A schematic, abridged history of the shell, record by record, from a session of ArchNights-SE. A blue square marks a
record the kernel writes (declaration, system, user, reply), an amber square an answer of the Oracle.}
\label{fig:history}
\end{figure*}

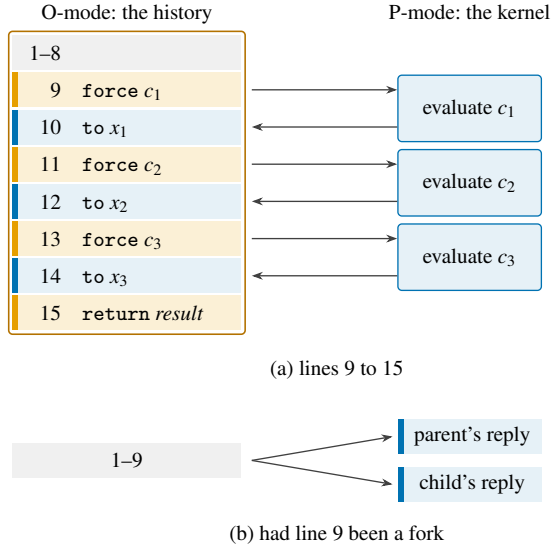
\begin{figure}[t]
\centering
\begin{tikzpicture}[x=1pt, y=1pt, font=\footnotesize,
  ker/.style={draw=pcol, fill=pfill, line width=0.5pt, rounded corners=1.5pt, minimum width=54pt,
    minimum height=25pt, inner sep=1.5pt},
  flow/.style={->, line width=0.5pt, black!75},
  head/.style={anchor=base, inner sep=0pt},
  title/.style={anchor=base, inner sep=0pt},
  >={Stealth[length=4.2pt, width=3pt]}]
  \def\W{86}%
  \def\rec#1#2#3#4{%
    \fill[#2] (0,{-14*#1}) rectangle (\W,{-14*#1-13});
    \fill[#3] (0,{-14*#1}) rectangle (2,{-14*#1-13});
    \node[anchor=base west, inner sep=0pt] at (6,{-14*#1-9.2}) {#4};}
  \rec{0}{black!6}{black!6}{\makebox[12pt][r]{1--8}}
  \rec{1}{ofill}{ocol}{\makebox[12pt][r]{9}\quad $\mathtt{force}\ c_1$}
  \rec{2}{pfill}{pcol}{\makebox[12pt][r]{10}\quad $\mathtt{to}\ x_1$}
  \rec{3}{ofill}{ocol}{\makebox[12pt][r]{11}\quad $\mathtt{force}\ c_2$}
  \rec{4}{pfill}{pcol}{\makebox[12pt][r]{12}\quad $\mathtt{to}\ x_2$}
  \rec{5}{ofill}{ocol}{\makebox[12pt][r]{13}\quad $\mathtt{force}\ c_3$}
  \rec{6}{pfill}{pcol}{\makebox[12pt][r]{14}\quad $\mathtt{to}\ x_3$}
  \rec{7}{ofill}{ocol}{\makebox[12pt][r]{15}\quad $\mathtt{return}\ \mathit{result}$}
  \draw[ostroke, line width=0.6pt, rounded corners=1.5pt] (-1.5,1.5) rectangle (\W+1.5,-112+1-1.5);
  \foreach \k/\t in {1/c_1, 3/c_2, 5/c_3} {
    \node[ker] (k\k) at (172,{-14*\k-13.5}) {evaluate $\t$};
    \draw[flow] (\W+4,{-14*\k-6.5}) -- (k\k.west |- 0,{-14*\k-6.5});
    \draw[flow] (k\k.west |- 0,{-14*\k-20.5}) -- (\W+4,{-14*\k-20.5});
  }
  \node[head] at (\W/2,6) {O-mode: the history};
  \node[head] at (172,6) {P-mode: the kernel};
  \node[title] at (122,-128) {(a) lines 9 to 15};
  \begin{scope}[shift={(0,-160)}]
    \fill[black!6] (0,6.5) rectangle (\W,-6.5);
    \node[inner sep=0pt] at (\W/2,0) {1--9};
    \foreach \y/\t in {9/parent's reply, -9/child's reply} {
      \fill[pfill] (145,\y+6.5) rectangle (199,\y-6.5);
      \fill[pcol] (145,\y+6.5) rectangle (147,\y-6.5);
      \node[inner sep=0pt] at (173,\y) {\t};
      \draw[flow] (\W+4,0) -- (143,\y);
    }
    \node[title] at (122,-30) {(b) had line 9 been a fork};
  \end{scope}
\end{tikzpicture}
\caption{The history of Figure~\ref{fig:history} read from O-mode as one program (amber: the Oracle's answers; blue: the
kernel's replies and its work in P-mode). (a) Each call $c_i$ forces a thunk the declaration offers. The kernel
evaluates it in P-mode, and the value is returned and bound with $\mathtt{to}$. The rest of the history is the
continuation. (b) Had the answer at line 9 been a fork, the parent's history and the child's would share lines 1 to 9,
each followed by its own reply.}
\label{fig:continuation}
\end{figure}

\subparagraph{Shell mechanisms.} The shell has mechanisms close to those of Bash and other shells. Pipelines run in the foreground and in
the background, by \texttt{/run}, \texttt{/bg} and \texttt{/wait}. The operator \texttt{run} shares these builtins' parser and pipeline code,
so that a command line means the same to the person and to the Oracle, although only the person's \texttt{/run} gives a program the
terminal. A system profile and a user profile are read at login. Their builtin lines run silently, and their other text becomes
standing instructions in the shell's system record, so that logging in asks the Oracle nothing. \texttt{/source} runs a script as if
its lines were typed, and a host command in a script is refused, since the Oracle may have written the file. \texttt{/clear}
begins the conversation afresh in the same process. After an Oracle error the session returns to the prompt with the conversation
kept, and the person decides whether to ask again. An error that would recur ends the shell, and the init system starts a new one.
The Oracle's \texttt{exit} ends only its turn, and only the person ends the session.

\subsection{ZOOT}

\paragraph{Users and resource isolation.} Mainstream instruction sets have at least two privilege levels. Mainstream operating systems
mostly use two, and use M mostly for firmware. We consider the three levels U, S and M. U has many users, while S and M have one each.
With O-mode added, a question arises: O-mode may have several users; must each of them be tied logically to a particular user of U?
The relation between U and S has had three design philosophies in history.
\begin{itemize}
\item Users exist in S as well, and the checks on S's rights are raised to M, which then checks the users of two levels.
\item Kernel-isolating virtualization~\cite{barham2003xen}: each guest still has a single S, and kernel isolation in effect runs several
  S-level users on one machine.
\item Kernel-sharing virtualization~\cite{soltesz2007container}: one S is divided into containers by namespaces and capabilities.
\end{itemize}
All three should be possible on a machine with an extended privilege level. PRTS has one O-mode user and one privileged O-mode user.
ZOOT adopts the second philosophy for the users of U and S.

\paragraph{Tasks and memory.} The basic unit of scheduling in ZOOT is still the task, and logically each task of ZOOT shares one
task structure. The U-mode kernel needs full access to some of its fields, so the U-mode task structure is stored in the user
address space. As in PRTS, O-mode and P-mode share one task identifier. At boot, U-mode init switches further to O-mode init.
For memory, ZOOT adopts PRTS's scheme: O-mode reuses the U address space, and has no virtual addresses of its own. This costs
O-mode flexibility in the memory it can use, since a large Oracle stack needs a large region reserved in advance in the U-mode
kernel. But the U virtual address space is large enough to hold the Oracle stacks, and with lazily committed pages, committed
before each entry as in PRTS, a large reservation costs little. The fetch condition of \S\ref{sec:query} is then the U-mode
kernel's: the Oracle stacks lie in its address space, and it must not execute them.

\paragraph{File system.} O-mode can access the physical file system. For resource isolation, the file system must then support
discretionary access control for O-mode, which needs new fields for permissions and ownership. Another way is to virtualize
O-mode's file system, which suits the LLM class better. With several users and separate kernels the file systems are
separate too: each user has an O-mode file system of its own, and some files can be shared.

\paragraph{Boot.} As on an ordinary computer, a computer with O-mode boots in layers, and the final target of the boot is an
O-mode program. The design intends any further level below U to be transparent to M and S.

\paragraph{Bridging to Linux.} ZOOT can be bridged to the Linux kernel. The design is meant to need no intrusive change to Linux beyond
the drivers that may be needed and the mechanism that delegates the new exceptions.

\section{Further Discussion}
\label{app:further}

This appendix takes up topics that the body leaves open. It examines the encoders, caches and schedulers of current systems on the
machine, places the defences built for agents relative to the safety boundary, gives the reversed ordering a machine of its own,
separates the two breakings into a general machine and derives the boundary and the trusted computing base there, asks what several
levels of privilege amount to, turns to the Oracles other than language models and to what is known of the Lynchpin hypothesis, and
ends with applications and extensions.

\subsection{Encoders}
\label{app:further-encoders}

\S\ref{sec:eclipse} requires two properties of an encoder: that it factorize at the yield symbol, which a tokenizer does when it splits its
input at a registered special token, and that its backtrack inside an open Frame be bounded. The formal theory of tokenizers bears
on the second. The work on the boundary between a prompt and its continuation bears on a history that ends inside a Frame,
and the work on action tokenizers bears on which Oracles the results reach.

\paragraph{A formal theory.} A tokenizer is a pair of stochastic maps, and an encoder of bounded
variation~\cite{gastaldi2025foundations} (after Berstel's prefix distance~\cite{berstel1979transductions}) is subsequential by
Choffrut's theorem~\cite{choffrut1979generalization} if its inverse images preserve regular languages. The loss of
Definition~\ref{def:loss} is a Step-level form of bounded variation (Appendix~\ref{app:eclipse}). Byte-pair encoding has a formal
semantics and, with a proper dictionary, a finite lookahead, which depends only on the dictionary, and a left-to-right algorithm in
constant memory~\cite{berglund2023formalizing}. It also has a deterministic automaton~\cite{berglund2024constructing} and an
incremental algorithm that maintains the tokenization of every prefix of its input~\cite{jiang2026incremental}. The finite lookahead is what bounds the Step
backtrack of BPE (\S\ref{sec:eclipse-codec}). An incremental algorithm does inside an open Frame what the encoding cache does across
Frames: it extends the encoding as the history grows.

\paragraph{The boundary between a prompt and its continuation.} \S\ref{sec:eclipse-memory} notes a mismatch at this boundary.
Token healing backs up one token there~\cite{lundberg2023art} and token alignment backs up several~\cite{athiwaratkun2024token}. Tokenization biases
sampling~\cite{phan2024understanding}, and exact conversions to bytes and characters remove the
bias~\cite{phan2025exact,vieira2025afrom}, or remove it up to the probability of invalid token sequences~\cite{hayase2025sampling}.
The partial-token problem is common in practice~\cite{xu2026are}. On the machine a prompt that may end inside a token is a history
that ends inside a Frame. In O-mode such a history arises at a budget trap inside a Frame, the third place where the prefix chain
breaks (Appendix~\ref{app:eclipse}). When the encoder factorizes at $\natural$ the loss stays in the open Frame (consequence (i)
of Theorem~\ref{thm:run-sync}), and a history that ends at a Frame boundary, with the separator, leaves nothing to repair. Models also
assign probability to encodings the tokenizer never produces~\cite{vieira2025blanguage}. Such an encoding is a Frame that does not
re-encode to itself, the second of the three places where the prefix chain breaks (\S\ref{sec:eclipse-memory}).

\paragraph{Action tokenizers.} Vision-language-action models that predict action tokens are autoregressive Oracles whose
alphabet includes actions. RT-2
writes actions as text tokens~\cite{brohan2023rt}, OpenVLA predicts discretized action tokens~\cite{kim2024openvla}, and FAST
compresses action chunks with a discrete cosine transform followed by byte-pair encoding~\cite{pertsch2025fast}. The encoder
results of \S\ref{sec:eclipse} therefore apply to an alphabet of actions whenever its tokenizer meets their hypotheses, which is
checked per tokenizer, as it is for text.

\subsection{Caches}
\label{app:further-cache}

\S\ref{sec:eclipse-memory} models an Oracle call as one hierarchy, from the history down to the state cache; it
states the prefix chain within an append-only epoch and locates each mechanism of the serving systems. This subsection says how the systems
work and what the model adds to each: where shared prefixes come from, what is gained by keeping a state between turns, what a rewrite of
the history costs, and which kinds of reuse fall outside the model.

\paragraph{Shared prefixes.} SGLang keeps the KV cache in a radix tree keyed by token sequences and evicts the least recently used
entries~\cite{zheng2024sglang}, and vLLM pages the KV cache in blocks and shares blocks within and across
requests~\cite{kwon2023efficient}. ChunkAttention detects at run time that requests begin with the same prompt and shares the
chunks of their keys and values in a prefix tree, and Hydragen computes attention over a shared prefix and over the unique
suffixes separately, batching the prefix part across sequences~\cite{ye2024chunkattention,juravsky2024hydragen}. The
machine accounts for the shared prefix differently in each form. In the Agent form the kernel makes it: the histories of one program share their first records, and a
fork or a restore copies a history (Table~\ref{tab:designs}), so the tree grows at run time, by services the kernel decides to
serve. In the Workflow form the program makes it, by pushing one prefix onto the operands of several calls, as SGLang's fork
primitive does in its frontend language~\cite{zheng2024sglang}. The tree is then written in the program text and can
be inspected before the calls are made, which is why SGLang's frontend can pass it to the runtime (\S\ref{sec:eclipse-control}). Either way
these are the shared prompts and the trees of prompts that these systems serve. Under the hypotheses of
Proposition~\ref{prop:prefix-chain}, moreover, each query of a task within an append-only epoch extends the state after its last
Frame (item~1), so the keys of one task lie on one path of the tree, which each Frame lengthens.
These systems share state only along a common token prefix, so a history is reused as far as its encoding agrees with a cached
one, and Proposition~\ref{prop:excess} prices the rest: from the entry of a history, a lookup of an extension recomputes, beyond the
growth of the encoding, exactly the encoder's loss, zero at a Frame boundary when the encoder factorizes at $\natural$.

\paragraph{Tiers and turns.} Several systems keep the state of a conversation between its turns. CachedAttention saves the KV
caches of all requests in a hierarchy of accelerator memory, host memory and disk, loads them back layer by layer while the model
computes, and places them in the tiers by hints from the scheduler of inference jobs~\cite{gao2024cost}. Pensieve keeps the state
of each conversation in GPU and CPU memory~\cite{yu2025stateful}. Mooncake separates the clusters that prefill from those that
decode and pools the KV cache in the cluster's host memory and disks~\cite{qin2026mooncake}, and LMCache keeps KV caches outside
the accelerator and shares them across engines and queries~\cite{liu2025lmcache}. On the machine, a conversation whose turns
only append is a task in an append-only epoch. Under the hypotheses of Proposition~\ref{prop:prefix-chain}, its
item~2 says what keeping the state saves: each turn costs only what was appended since the last answer, and every symbol is computed once. The item
assumes a cache that never evicts. Tiers approximate it by moving a state down instead of dropping it, so that a turn pays a transfer
from a lower level instead of a prefill (Definition~\ref{def:store}). Two of the designs are taken further under the model.
CachedAttention's scheduler can hint only at the jobs submitted to it, which are all of the next keys the Oracle side has
(\S\ref{sec:eclipse-control}); the Priestess side holds the rest. And in the model a prefill and a decode are one operation, an
application of $\varphi$ per symbol; under the same hypotheses the prefill computes the symbols that the Priestess wrote, and the decode
computes those that the Oracle generates (Figure~\ref{fig:readings}(c)). Mooncake runs the two on different systems, so the state
that a prefill computes reaches the decoder by a transfer, a connection of two systems in the sense of Construction~\ref{con:compose}. What
becomes of a task's state while its tool call runs is a decision of the schedulers (Appendix~\ref{app:further-sched}).

\paragraph{Rewrites of the history.} Agent systems manage their context by rewriting it: they truncate the oldest turns, compact
part of the history into a summary, or move content out to external storage, as MemGPT does~\cite{packer2023memgpt}. Each is a
rewrite of the task and starts a new epoch (Table~\ref{tab:designs}), whose first query pays $\mathrm{cost}$ from the nearest cached
prefix (\S\ref{sec:eclipse-memory}), and the model says how large that price is. Under the hypotheses of
Proposition~\ref{prop:prefix-chain}, and as long as a cache keyed by encoded prefixes still holds the old states, the reuse extends at
least over the encoding of what precedes the rewrite's first change, up to the last $\natural$ there (consequence (ii) of
Theorem~\ref{thm:run-sync}), and the later queries of the epoch extend the new history again. The price thus depends on where a
rewrite begins and how often it recurs. A rare rewrite that keeps the head of the history costs a long prefill now and then, as
PRTS's swap does (\S\ref{sec:an-checkpoint}). A window that slides at every turn makes every query the first of an epoch, and
only the head kept ahead of the window is reused. LMCache reports from its deployments that context truncation, a widely applied
technique, can cut the hit ratio of the prefix cache by half~\cite{liu2025lmcache}. Quine's exec, which renews the
context~\cite{ke2026quine}, starts a new epoch in the same way, and so does PRTS's exec (\S\ref{sec:an-prts}). The same rewrite looks
different in the Workflow form. There a compaction is the program's: it pops its operand $d$ and rebuilds it with a summary, the
answer of a P-mode call. Between two calls on $d$ the query grows by the answer and the pushes unless the program pops $d$ in
between (Proposition~\ref{prop:ar-form}), so the chain can break only across a $\mathtt{pop}\ d$, a point that can be found in
the program text before it runs. In the Agent form, by contrast, where and when the kernel rewrites a task is its policy at run
time, as the threshold of PRTS's swap is.

\paragraph{Reuse beyond prefixes.} Some systems reuse what is not a prefix. CacheBlend reuses independently cached chunks and
recomputes a selected part of their KV entries to restore the interactions between them~\cite{yao2025cacheblend}. Prompt Cache
reuses the attention states of modules that a schema declares, at positions the schema reserves, so a module need not begin the
prompt~\cite{gim2024prompt}. When a conversation overflows the context window, CachedAttention truncates its saved KV caches and
keeps them valid by decoupling the positional encoding~\cite{gao2024cost}. Leyline's directives remove or replace a span inside
a cached history, in one of Leyline's modes without recomputing what follows the span~\cite{ma2026leyline}. All
four fall outside Definition~\ref{def:state-pres}, in which the state of a prefix is computed through every symbol before it. The last two lower exactly the
price of a rewrite above, by keeping the states after the removed or edited span. The states kept are then not those of the new
history in the sense of Definition~\ref{def:state-pres}, a trade the model does not cover.

\subsection{Scheduling}
\label{app:further-sched}

Table~\ref{tab:sched} in \S\ref{sec:eclipse-control} sorts the systems into four cells, by which side's
scheduler they are and which side they schedule. This subsection analyses the three cells that involve the Oracle's computations,
beginning with the Oracle side's own, and says for each system how it schedules and what the model says of its decision. In the fourth, the
Priestess side schedules only its own work.

\paragraph{The Oracle side's own work.} A server batches requests, and paging the KV cache lets more of them share the
accelerator's memory~\cite{kwon2023efficient}. Across devices, Preble places requests by weighing the reuse of cached prompts
against the load of each device, with a hierarchical scheduler~\cite{srivatsa2025preble}. This is the trade-off of
\S\ref{sec:eclipse-control} between cache affinity and load balance, and for a task in an append-only epoch the
model quantifies affinity. Under the hypotheses of Proposition~\ref{prop:prefix-chain}, the device that answered the task's last Frame holds
the state of its whole history, so the next query costs there only what the kernel pushed since (item~2). On another device it
costs $\mathrm{cost}$ from the nearest prefix cached there, the whole encoding if there is none (Definition~\ref{def:state-pres}).
The prefill that the next query needs stays the size of one push, while the price of a move grows with the history, so within an
epoch, and while the state cache of the last device stays resident, the longer an agent runs, the more its affinity weighs against
load.

\paragraph{Which program runs.} Parrot lets an application annotate the inputs and outputs of its requests as semantic variables,
so that the service can analyse the dataflow between requests and schedule them with the application in
mind~\cite{lin2024parrot}. Agentix (formerly Autellix) intercepts the calls that a program submits, and preempts and prioritizes each by
the calls its program has already completed, against head-of-line blocking~\cite{luo2026agentix}. ThunderAgent treats an agentic
workflow as one program and schedules its calls taking its KV cache and its tool environments into account~\cite{kang2026thunderagent}.
A server sees the program because the program exposes it or because the server intercepts its calls. On the machine the keys
already link the queries of one task within an append-only epoch, since they lie on one path of the prefix tree
(Appendix~\ref{app:further-cache}). They stop doing so at a rewrite, which starts a new epoch in the same task: the next key no longer extends the last, and a
common prefix alone does not tell the same task from another. Nor can they tell apart two tasks whose histories coincide, as after a
restore. The task's identifier links its queries in both cases. In
ArchNights every Oracle operation carries it, as the program identifiers of Agentix, Continuum and ThunderAgent
do~\cite{luo2026agentix,li2026continuum,kang2026thunderagent}.

\paragraph{A task's state during a tool call.} Here the cache and the schedulers meet. Inference engines treated each call to a
tool as the end of a request and formed a new request when the tool returned. InferCept measured the resulting recomputation of the
context at 37--40\% of the model's total forward time, and it chooses at each interception whether to keep the KV cache,
discard and recompute it, or swap it out, so that it wastes the least GPU resource~\cite{abhyankar2024infercept}. Continuum pins the
cache for a time-to-live that is set by the cost of reloading it and by the queueing delay that an eviction would
cause~\cite{li2026continuum}. Ask the Tool argues that an estimate that is fixed before a call starts cannot know the call's duration. It
has the running tool report its progress, and in its evaluation comes close to a policy that knows every return time~\cite{liu2026ask}. On
the machine a tool call is a yield trap, and while the kernel serves it the Oracle holds $\hat s(E(q)\,y)$, from which the next
query $q\,D(y)\,v$ costs $|E(v)|$ under the hypotheses of Proposition~\ref{prop:prefix-chain} (\S\ref{sec:eclipse-memory}). The
three choices are three costs of the model, each paid on top of that $|E(v)|$: nothing if the state is kept; the $\mathrm{cost}$ of rebuilding it from
the nearest cached prefix (the whole encoding if none is left) if it is discarded, as in the engines that InferCept measured; and a
transfer from a lower level if it is swapped (Definition~\ref{def:store}). The decision also needs to know when
the task returns, and that lies on the Priestess side: the kernel's service decides when the task is entered again, and nothing in
the queries that the Oracle side has received reveals when it returns. Proposition~\ref{prop:return} prices this. Between keeping and discarding,
no deterministic policy that does not know when the task returns stays within less than twice the cost of one that knows, and the time-to-live
$R/\mu$ stays within twice. Continuum's time-to-live is a policy of this kind. Ask the Tool's result approaches the remark of Appendix~\ref{app:eclipse} that
each side knows the other exactly when the other is transparent to it, since the running tool's report makes the Priestess side
partly transparent. In the Workflow form the same interval lies between two calls that share a prefix, and the program knows the
next keys. KVFlow reads them from the workflow's schedule, which it represents as a graph of agent steps; it evicts by each agent's
number of steps to execution, keeping the prefixes of the agents due soon, and prefetches those of the agents due
next~\cite{pan2025kvflow}.

\paragraph{The Priestess side schedules the Oracle.} In the Agent form these decisions are the operating system's: which task
enters O-mode, when, and for how long. PRTS chooses the task and its slice, the budget, counted in the Oracle's own steps
(\S\ref{sec:an-prts}). Two results of \S\ref{sec:eclipse} bear on the slice. Steps need not be visible in Frames
(Proposition~\ref{prop:steps-invisible}), so a slice counted in steps needs an Oracle that exposes them, as the descriptor of
PRTS's unit does. And with an Oracle that offers STEP, a slice can end inside a Frame, which leaves a history that
Proposition~\ref{prop:prefix-chain} does not cover (Appendix~\ref{app:eclipse}); an Oracle that offers only RUN is preempted between its
Frames (\S\ref{sec:an-isa}). Where a slice ends can thus cost cache as well as time. In the Workflow form the decisions
are the program's. Composable effect handlers separate a script's logic from its effectful operations (model calls, input and
output, and concurrency among them), so that a handler can run the calls concurrently without changing the script; in a
Tree-of-Thoughts case this gives a tenfold speedup~\cite{wang2025composable}. Such a handler sees the program's calls before the
Oracle does, which is where \S\ref{sec:eclipse-control} places the next keys. Its decision is written in the program's
language (\S\ref{sec:related-pl}). Pie runs programs, which it calls inferlets, inside the serving system and gives them control of
generation and of the KV cache~\cite{gim2025pie}. An inferlet is code on the Priestess's side of the interface, deployed on the
Oracle's: the sides are logical, not places. With the cache in reach, the Priestess side can know, instead of estimating, which
entries are resident, since the Oracle side is then transparent to it. An Oracle-side scheduler can take these decisions only
from what the Priestess exposes to it, a program identifier, a graph of calls or a hint while a call runs.

\subsection{Safety}
\label{app:further-safety}

Theorem~\ref{thm:boundary} names where content crosses between the two computations. On the way in it crosses at the Priestess's
writes to an operand: in the Agent form, the task's initial content, the declaration included, and the kernel's writes, its replies among
them. On the way out it crosses at the trap and the gates. Indirect prompt injection, which plants instructions in data that an application
is likely to retrieve~\cite{greshake2023not}, crosses twice. In the Agent form it enters a task as part of a reply, which is a write of
the kernel. Whatever it leads the task to write, any request included, reaches the kernel first at $\ell_{\mathrm{trap}}$. In the
Workflow form it enters an operand that the program builds, and reaches control only through the gates. AgentDojo places such
injections in the results of tools and counts an attack as successful when the attacker's goal is met~\cite{debenedetti2024agentdojo}.
A defence can therefore sit on the way in, inside the Oracle, or on the way out. This subsection places the defences built for
agents at these three positions and states what the theorem gives at each of them. Whether a defence works is not part of its placement,
and a harness reproduces runs of the machine only if it implements the machine faithfully.

\paragraph{On the way in.} What the theorem guarantees a defence here is completeness. By its first item the Priestess's writes are all
that a task reads beyond its own answers, so a filter or a type at those writes sees every symbol that enters a task from outside
it, provided it also applies to the initial content and every record the kernel builds, not only the replies. The defences here are of two kinds. The first decides what reaches a task at all: by the provenance of each part of
a tool's result~\cite{wu2024system,costa2025securing}, by showing the privileged model only the
names of variables that contain a quarantined model's outputs~\cite{willison2023dual}, by types that let only numbers, booleans and enumerations
through~\cite{jacob2025preventing}, by a hub through which applications exchange all their messages~\cite{wu2025isolategpt}, or by rules
and a classifier that intercept content before it enters the context~\cite{zou2026agentkernel}. A defence of this kind keeps its
decision for every Oracle, since the Oracle never reads what it withholds. In the dual-LLM pattern the quarantined model is a P-mode call whose answer is data (K4). The second kind marks what reaches a task and leaves the rest to
the Oracle. Spotlighting delimits, marks or encodes untrusted text~\cite{hines2025defending}, StruQ's front-end builds each query
from reserved delimiters that the data cannot forge~\cite{chen2025struq}, and AgentKernel also tags each source with a level of
trust~\cite{zou2026agentkernel}. A mark writes provenance for the Oracle to read, and is a defence only together with one inside
the Oracle.

\paragraph{Inside the Oracle.} The instruction hierarchy trains a model to rank instructions by the record they appear in, with
tool results lowest~\cite{wallace2024instruction}, and StruQ trains a model to ignore instructions in the data
channel~\cite{chen2025struq}. Provenance-aware transformers give each input token a ring ID for its origin, supplied by the
application, and learn an attention bias by origin~\cite{zhang2026origin}. Such a defence changes the answer relation. By Proposition~\ref{prop:where-not-which}, no rule
of the machine can certify it, since which request crosses is a property of the Oracle on its history, and the first two papers state that
strong adversarial attacks remain a risk. What the machine gives such a defence is its input. The delimiters, the ring IDs and the record that holds
each piece of content are written on the way in, so a defence inside the Oracle is only as sound as the writes it reads.
This is where the escalation below strikes.

\paragraph{On the way out.} Here too the theorem gives completeness, at a position that depends on the form. In the Agent form every
request reaches the kernel first at the service entry, so a check there sees every request (K3, Corollary~\ref{cor:crossing}). The
checks of requests in the kernels of \S\ref{sec:related-parallel} are all the kernel's service policy at $\ell_{\mathrm{trap}}$,
and so are a privilege policy over every tool call~\cite{shi2025progent} and a re-run of the model with the user's task
masked~\cite{zhu2025melon}.
The re-run stops the run when a tool call resembles one the masked run made, and the masked run is a further Oracle call
on a history the kernel writes. Pirch et al.\ ask for policy checks outside the model, and the service entry is also where these
would sit~\cite{pirch2026toward}. The theorem fixes where these checks sit, not what they decide: Progent decides by privilege and
declaration, deterministically, and MELON by the similarity of embeddings. AKTS narrows the request instead. An agent that tunes a
kernel's scheduler acts only by writing an index into an array of policies verified when they were loaded, and an invalid index
names an empty slot~\cite{khodabandehlou2026akts}, so every request selects a policy that was verified before the request was made.

In the Workflow form the boundary lies in the program, at the gates. CaMeL has a privileged model write a program from the user's
query alone and runs it in an interpreter of its own. A quarantined model with no tools returns values of a given schema, and
security policies are checked before each tool call, on its arguments and their dependencies~\cite{debenedetti2025defeating}.
Plan-then-execute and code-then-execute follow the same outline~\cite{beurerkellner2025design}. On the machine the model's program
is the answer of a P-mode call that the Priestess interprets, so the task runs in the Workflow form with the Agent paradigm
(Table~\ref{tab:designs}). By the second item of the theorem untrusted content then reaches control only through the gates, here
the interpreter's. The quarantined model's answers are data, but the program may branch on them. CaMeL's strict mode is a check at
the gates: it adds the condition of each branch, or the iterable of each loop, as a dependency of every variable assigned inside it.

\paragraph{When the harness escalates privilege.} He et al.\ show how a harness defeats a defence inside the Oracle and one on the way out
together~\cite{he2026when}. When a harness reconstructs a context for a sub-agent, a persistent goal or a scheduled task, it moves
content that entered as a tool's result into a user or system record. On the machine that is a write of the kernel,
which the contract does not prevent, since the kernel is trusted. The instruction hierarchy then treats the promoted content as
privileged, and so does a permission reviewer that sits at the service entry and infers authorization from the same records;
the reviewer is a further Oracle call like MELON's masked run. The escalation is a fault of the kernel's writes, not a crossing
around the boundary, and a defence on the way in sees the promoted content only if it covers every record the kernel builds, as
the first item of the theorem asks.

\paragraph{A systematization from the other side.} Zhang et al.\ reach the same division from the side of the defences~\cite{zhang2026when}. They grade each crossing of an agent by whether its question concerns provenance, which admits a deterministic predicate, or the
meaning of content, which does not. Their two structural gaps, at the input to the agent and between reasoning and action, are the
way in and the way out. The machine settles neither question, since which request crosses is the Oracle's, as noted above for defences inside the Oracle. Their projection principle, that a deterministic mediator can be obtained by narrowing what a crossing controls, is what typed
handoffs do on the way in and what a program fixed in advance does on the way out. Their design rule is that enforcement stays in a
mechanism that the model may propose to but not rewrite. It holds on the machine for the Priestess's code, since no rule writes $\Pi$ and
an excursion writes only its own stack (K1, K4).

The machine thus fixes where each defence outside the Oracle must be placed, and what it can see there. In a run of the machine
content crosses between the two computations nowhere else, so a defence sees all of it when it is placed at every crossing, the
initial content included. It does not see the kernel's choice of when, and how often, to enter a task, which is a control channel
(Appendix~\ref{app:eclipse}). The type bypass of \S\ref{sec:eclipse-safety} lies outside the runs the theorem is about, and so do
channels outside the model, a cache shared across tasks among them. What a defence decides, and whether that is right, is its own.

\paragraph{What this account does not model.} The boundary is a statement about where content crosses, and it is not a security
model. Of the classical account it uses two things: complete mediation in its usual sense~\cite{saltzer1975protection}, and
the shape of a reference monitor, which confines a party by the crossing it must make and assumes nothing of its
interior~\cite{anderson1972computer}. It proves no non-interference property~\cite{goguen1982security}, and
Proposition~\ref{prop:where-not-which} shows that no property that holds for every answer can exclude a given one. What it does not supply is a
lattice or a declassification policy, a treatment of channels outside the model---timing, the kernel's choice of when to enter a
task, and a cache shared across tasks among them---or any claim about whether a defence's decision is right. A defence placed at
this boundary is one component of a security argument, not the argument. Where the paper needs more it says so: the privilege levels
of Appendix~\ref{app:further-levels} inherit the boundary at each crossing and nothing beyond it, and a tower of levels would need a
theorem of its own, which this paper does not state.

\subsection{Privilege Reversed}
\label{app:further-reversed}

\S\ref{sec:os-agentic} takes the bottom row of Table~\ref{tab:quadrants} as the placements of a task's program when privilege
is reversed. Human--automation interaction has been modelled as interacting transition systems, whose transitions the user
triggers or the machine takes by itself~\cite{degani2002formal,combefis2009bisimulation,bolton2013using}. These models represent
the person by a model of the user and ask whether it predicts the machine. We found none that puts the person on a privileged
side through which the program's effects must pass. Where a person is an Oracle in the sense of Definition~\ref{def:oracle}, the
machine is privileged: a Turing machine with a human oracle~\cite{shahaf2007towards}, the human procedure call of
TurKit~\cite{little2010turkit} and a shield that restricts the choices of a human operator~\cite{konighofer2017shield} all have
the ordering of \S\ref{sec:os}. A person by itself does not produce the bottom row; privilege does. This subsection gives the
reversal a machine. The Oracle is in the privileged mode and decides which task runs and what it is answered, and Priestess code
runs confined to its own stacks and reaches the Oracle only by a trap. Its contract carries over item by item, and the one item
that changes says what trusting an Oracle costs.

\begin{definition}[reversed transition breaking $\overline{\mathrm{T}}$]
\label{def:tbar}
Fix an Oracle stack $\kappa$, the \textbf{kernel stack}, and a finite set of \textbf{tasks}. A task $t$ has a set $C_t$ of labels
of the program, its code, with an entry label $e_t \in C_t$, and a set $S_t$ of Priestess stacks, its stacks, with an input stack
$i_t \in S_t$. The sets $S_t$ are pairwise disjoint. Fix a symbol $\ulcorner t \urcorner \in \Sigma_Q \cap \Sigma_R$ for each
task, the \textbf{dispatch symbol} $\sharp \in \Sigma_R$, three symbols $\tau, \phi, \eta \in \Sigma_Q$, and a slice $\beta \in \mathbb{N}$
with $\beta \ge 1$. Add
the instruction $\mathtt{sys}\ s$, for $s$ a stack. Task code is \textbf{well formed}: an instruction at a label of $C_t$ is not
$\mathtt{halt}$, names only stacks of $S_t$, and every label a rule can give as its successor, $\ell+1$ and $l_e$ included, lies in
$C_t$.

An answer $r$ is a \textbf{dispatch} if $r \in \Sigma_R^*\,\ulcorner t \urcorner\,V^*\,\sharp$, where $V$ is $\Sigma_R$ without
the task symbols and $\sharp$; its \textbf{decision} is $(t, v)$, with $v \in V^*$ the word between $\ulcorner t \urcorner$ and
$\sharp$. Configurations are $\langle \mathsf O, \rho, \sigma \rangle$ and $\langle \mathsf P, t, \ell, b, \rho, \sigma \rangle$, where
$\rho$ gives each task its \textbf{resumption point}, a label of $C_t$ and a stack of $S_t$, and $b$ is the remaining slice. Runs
start in $\langle \mathsf O, \rho_0, \sigma_0 \rangle$ with $\rho_0(t) = (e_t, i_t)$. The rules:
\begin{itemize}
\item \textbf{O-mode.} With $w = \sigma(\kappa) \in \mathrm{dom}(E)$ and $r \in O(w)$: if $r$ is a dispatch with decision $(t, v)$
  and $\rho(t) = (\ell, s)$, then $\langle \mathsf O, \rho, \sigma \rangle \to \langle \mathsf P, t, \ell, \beta, \rho,
  \sigma[\kappa \mapsto w \cdot r][s \mapsto \sigma(s) \cdot v] \rangle$; otherwise $\langle \mathsf O, \rho, \sigma \rangle \to
  \langle \mathsf O, \rho, \sigma[\kappa \mapsto w \cdot r] \rangle$.
\item \textbf{P-mode.} If the instruction at $\ell$ is $\mathtt{sys}\ s$, then $\langle \mathsf P, t, \ell, b, \rho, \sigma \rangle
  \to \langle \mathsf O, \rho[t \mapsto (\ell+1, s)], \sigma[\kappa \mapsto \sigma(\kappa) \cdot \ulcorner t \urcorner \cdot
  \sigma(s)] \rangle$, a trap whose \textbf{request} is $\sigma(s)$. Otherwise, if the \mach{O-2-PDA} steps $(\ell, \sigma) \to
  (\ell', \sigma')$, the task takes that step and its slice falls by one; at $b = 1$ the step ends in
  $\langle \mathsf O, \rho[t \mapsto (\ell', i_t)], \sigma'[\kappa \mapsto \sigma'(\kappa) \cdot \ulcorner t \urcorner \cdot \tau]
  \rangle$, a slice trap. If the \mach{O-2-PDA} has no step and $\ell = n+1$, the task has finished its code, and the configuration
  goes to $\langle \mathsf O, \rho[t \mapsto (\ell, i_t)], \sigma[\kappa \mapsto \sigma(\kappa) \cdot \ulcorner t \urcorner \cdot \eta]
  \rangle$, a \emph{completion} trap; if $\ell \ne n+1$, the task is stuck, and the configuration goes to
  $\langle \mathsf O, \rho[t \mapsto (\ell, i_t)], \sigma[\kappa \mapsto \sigma(\kappa) \cdot \ulcorner t \urcorner \cdot \phi]
  \rangle$, a fault trap.
\end{itemize}
A \textbf{segment} of $t$ is a maximal part of a run in which $t$ runs in P-mode, together with the dispatch that starts it and,
if it is finite, the trap that ends it. The \textbf{\mach{O-2-PDA-$\overline{\mathrm{T}}$}} is the \mach{O-2-PDA} with $\overline{\mathrm{T}}$.
\end{definition}

In words: the Oracle deliberates on the kernel stack until an answer names a task and a reply; the reply is appended to the stack
the task last named, and the task resumes where it stopped. The task runs its own code on its own stacks, and returns control by
a system call, by exhausting its slice, by finishing its code or by getting stuck, and each of these appends its name and a request
to the kernel stack. Task code is well formed and so contains no $\mathtt{halt}$, which makes $n+1$ its only normal exit: a request
$\eta$ says that the task ran its code to the end, and $\phi$ that it got stuck, as a $\mathtt{pop}$ on an empty stack does.
The machine is good if every query reached on $\kappa$ is accepted. Runs start in the privileged mode, as on the
\mach{O-2-PDA-TS}.

\begin{proposition}[the reversed contract]
\label{prop:reversed}
For every Oracle, every program whose task code is well formed and every run of the \mach{O-2-PDA-$\overline{\mathrm{T}}$}:
\begin{itemize}
\item \textbf{$\overline{\mathrm{K}}$1 (a task writes only its stacks).} A P-mode transition of $t$ changes no stack outside
  $S_t$, except that a trap appends $\ulcorner t \urcorner$ and its request to $\kappa$. An O-mode transition changes only $\kappa$,
  by appending its answer, and, at a dispatch to $t$, one stack of $S_t$, by appending the reply.
\item \textbf{$\overline{\mathrm{K}}$2 (a task reads only its stacks).} In P-mode a task $t$ reads only stacks of $S_t$: what its
  own instructions wrote there, and the replies that dispatches to $t$ appended, the initial content included. No task reads
  $\kappa$ or another task's stacks.
\item \textbf{$\overline{\mathrm{K}}$3 (control passes only at a trap and a dispatch).} Every transition from P-mode to O-mode is
  a trap of the running task, and every transition from O-mode to P-mode is a dispatch, which enters $t$ at the label of
  $\rho(t)$: its entry label, the label after its last $\mathtt{sys}$, or the label at which another trap stopped it.
\item \textbf{$\overline{\mathrm{K}}$4 (a request is input to a decision, never a decision).} Which task runs next and what reply
  it receives are the decision of an Oracle answer on $\kappa$. A request reaches O-mode only as content appended to $\kappa$ by
  the trap that ends its segment.
\item \textbf{$\overline{\mathrm{K}}$5 (control returns to the kernel).} Every segment ends with a trap within $\beta$ P-mode
  transitions.\qedhere
\end{itemize}
\end{proposition}

\begin{proof}
\leavevmode
\begin{itemize}
\item \textbf{$\overline{\mathrm{K}}$1.} A P-mode transition of $t$ takes at most one \mach{O-2-PDA} step, and a trap adds to it
  only an append of $\ulcorner t \urcorner$ and the request to $\kappa$ and a new $\rho(t)$: a slice trap takes the step it ends,
  and the other traps take none. The step is on an instruction of $C_t$, which names only stacks of $S_t$ by well-formedness, and
  by the locality lemma it changes at most the stack it names. The O-mode rules update $\sigma$ at $\kappa$ and, at a dispatch,
  at the stack of $\rho(t)$, which lies in $S_t$ since only traps of $t$ and $\rho_0$ set it.
\item \textbf{$\overline{\mathrm{K}}$2.} An instruction reads the store only at the stack it names: $\mathtt{jeq}$ and
  $\mathtt{pop}$ read its top, and the Oracle instruction and $\mathtt{sys}$ its content. By well-formedness that stack lies in
  $S_t$, and $\kappa \notin S_t$ since $\kappa$ is an Oracle stack. By $\overline{\mathrm{K}}$1 and its proof, a stack of $S_t$
  changes only by instructions of $t$ or by the reply of a dispatch to $t$.
\item \textbf{$\overline{\mathrm{K}}$3.} Immediate from the rules: the only rules from P-mode to O-mode are the traps, and the only rule
  from O-mode to P-mode is the dispatch, which enters at the label of $\rho(t)$; only $\rho_0$ and the traps of $t$ set
  $\rho(t)$.
\item \textbf{$\overline{\mathrm{K}}$4.} The successor of an O-mode configuration is fixed by an answer $r \in O(\sigma(\kappa))$,
  and by $\overline{\mathrm{K}}$1 a task changes $\kappa$ only at its traps.
\item \textbf{$\overline{\mathrm{K}}$5.} The P-mode rules give every P-mode configuration a successor: a $\mathtt{sys}$ traps, an
  \mach{O-2-PDA} step is taken, and where there is none the task traps. Each transition that is not a trap lowers the slice by
  one, and at slice $1$ every transition traps.\qedhere
\end{itemize}
\end{proof}

The items mirror K1--K5, and the proofs derive them from the rules as before. $\overline{\mathrm{K}}$5 needs no good machine, since
P-mode has a successor everywhere, but it is the return of control to the kernel. Its
counterpart for the tasks, that a dispatch eventually occurs, is a property of the Oracle: a person may never decide.

\paragraph{What does not carry over.} In CBPV the two kinds of force change sides. A dispatch forces the task code at the label of
$\rho(t)$, as a trap forces $k$ on the \mach{O-2-PDA-TS}, and a trap forces the thunk on $\kappa$, as $\mathtt{uret}$ forces the
thunk on a task stack. \textbf{So the operating system here is the Oracle's computation on $\kappa$, which forces the task code and
is forced by its traps.} K4 makes every answer data to the kernel because the kernel's code $\Pi$ lies apart from its stacks and
no rule writes it. The Oracle has no code apart from its query: $\kappa$ is an Oracle stack, on which content may have type $U\,B$
(Definition~\ref{def:s}), and a trap appends every request to it. So a request becomes part of the program the kernel runs next,
and no type keeps it data. $\overline{\mathrm{K}}$4 is the counterpart of K4 in what it preserves: a request selects among the decisions
the dispatch language names, and only through the Oracle's answer, as an answer selects among the successors $\Pi$'s labels allow,
and only through $\Pi$'s tests. What it loses is a selector that can be checked. A kernel on the \mach{O-2-PDA-TS} is a program and
can be verified against its specification~\cite{klein2009sel4,gu2016certikos}; on this machine the selector is the answer
relation of the Oracle, which no proof inspects. So $\overline{\mathrm{K}}$4 is verified as K4 is, for every Oracle and every
program: no request bypasses the trusted party, and every decision is its own. What no proof can verify is that the trusted party
decides well, for an Oracle the contract treats as a relation it never inspects. On the \mach{O-2-PDA-TS} trust in the kernel can be
discharged by verifying its code; here, for an Oracle fixed only by its answer relation, it can be assumed, or measured, and an
Oracle whose policy is code can be verified as a kernel's can, at which point the assumption moves to the specification of that
policy. Which Oracles are of which kind, and what a specification of a decision policy looks like, is outside this paper. A measurement certifies an average, not a decision: for a stochastic Oracle on a fixed domain, certifying with error at
most $\varepsilon$ that its reliability is at least $p_1$ rather than at most $p_0$ takes, to leading order,
$\ln(1/\varepsilon)/D(p_1 \,\|\, p_0)$ turns~\cite{wang2026dtoken}. \textbf{The reversed contract confines the program completely, and says nothing about the judgement of the party it
trusts.}

\paragraph{Automated systems and human--computer interaction.} Definition~\ref{def:sides} applies with the roles exchanged: the privileged side is the Oracle, and
the Oracle mediates a task when the task runs in P-mode and the Oracle serves its traps. In \emph{automated systems} the task's
program is task code: it runs confined, and the Oracle decides at its traps, so the Oracle is its operating system. Turing's
choice machines are an early instance: at an ambiguous configuration such a machine cannot go on until an external operator has
chosen~\cite{turing1936computable}. Human-centred automation states the ordering as a design principle, that the pilot remains in
command~\cite{billings1996human}, while other work shifts the final authority between the person and the automation as the
situation changes~\cite{inagaki2000situation}, which on the machine would move the privileged side during a run, outside both
machines. In \emph{human--computer interaction} the task's program is the Oracle's, on $\kappa$: the person dispatches tasks as
tools and reads what their traps report, as Workflow code reads the answers of P-mode calls, and nothing mediates the person.
Licklider's symbiosis divides the work this way: people set the goals and perform the evaluations, and machines do the
routinizable work~\cite{licklider1960man}. The levels of automation run from one in which the person does everything
to one in which the computer does~\cite{sheridan1978human}, separately for acquiring information, analysing it, selecting a decision and implementing it~\cite{parasuraman2000model}. On the machine, a level says which decisions the dispatch language
leaves to $\kappa$ and which the task code makes by itself.

\paragraph{Trusting the Oracle.} What $\overline{\mathrm{K}}$4 leaves to the Oracle is what the formal models of
human--automation interaction check, by putting a model of the user in the person's place. Degani and Heymann compose the user
model with the machine model and require that the composite reach no error state and no blocking state~\cite{degani2002formal};
Rushby model-checks a mental model against the machine to find automation surprises~\cite{rushby2002using}; and full control
requires that the user know exactly which commands are available and anticipate every observation~\cite{combefis2009bisimulation}.
On the reversed machine these are conditions on the Oracle's decisions, which the contract cannot state, since it holds for every
Oracle. Security has the same division. User-driven access control grants a permission only on an authentic user action, which
no application can generate~\cite{roesner2012user}; on the machine this is $\overline{\mathrm{K}}$4, which follows from the rules. The
trusted path asks in addition that untrusted software cannot imitate the channel to the person~\cite{dod1985trusted}. A request
may contain task symbols or $\sharp$, which do not dispatch but may mislead the Oracle that reads them, so a realization keeps them
out of task content. What neither guarantees is the person's judgement: in one study, 17\% of participants paid attention to the
permissions an application asked for at installation~\cite{felt2012android}. For a language model as the kernel, a request that
steers the decision is a prompt injection that lands on the privileged side. That is the case the ordering of \S\ref{sec:os}
excludes by placing an untrusted Oracle below the Priestess, and the reason this paper orders the sides that way for a language
model.

\paragraph{An interpreter.} As in Proposition~\ref{prop:interpreter}, the reversed kernel can run inside the ordinary machine:
Priestess code calls the Oracle in P-mode on a Priestess stack that stands for $\kappa$, decodes a dispatch and jumps to the task's
resumption point, and each $\mathtt{sys}$ becomes code that appends the request and jumps back. An event loop or an approval gate builds the
reversal this way, and its confinement is then an obligation of the interpreter's code and of a check of
the task code, as in Singularity~\cite{hunt2007singularity}. An agent harness that asks a person before it runs a tool request
combines the two orderings: the person above the Priestess, and the Priestess above the model. That chain has two Oracles of
different classes and different trust, which neither machine formalizes alone (\S\ref{sec:future}).

\subsection{Black-Box Confinement}
\label{app:further-blackbox}

Each breaking of \S\ref{sec:os} is stated for one colour of stack. Write $X$ for the colour that S lets hold a program the Oracle
runs, and $Y$ for the colour whose mode T makes privileged, and call the result the \mach{XY-2-PDA}; on the \mach{O-2-PDA-TS} of
this paper $X$ is $O$ and $Y$ is $P$. The two letters play different roles, and this subsection sets out the difference. S's
designation decides which stored content may be run; on a machine that has it, a system then decides, task by task, whether a
task's program is placed there at all (Definition~\ref{def:sides}). T's designation decides which side holds the privileged mode,
and since a privilege is the set of operations a mode is granted (\S\ref{sec:os-t}), writing the other side in that place is not a
renaming but a second set of rules, the one Appendix~\ref{app:further-reversed} gives. What travels between the two is the contract
and the reading that S gives content, not the rules and not the colours. The subsection then derives the contract on the general
machine the two names leave.

\paragraph{The designation of $X$.} Of the two, this one does little work. S's typing holds for either colour in place of $X$, and on
the \mach{O-2-PDA-S} the designations differ in no run: erasing the types maps the machine whose subject colour is $X$ to the one
whose subject colour is the other, so no observation distinguishes them (\S\ref{sec:os-s}). The reason is that the Oracle
instruction is available on both colours and the program may call it on either, which is what lets the Priestess call it at all;
naming one colour the subject is therefore a way of reading one machine, not a way of separating two. For the two names to stand
for two machines, the breaking would have to take the instruction away on one side rather than constrain how its operand is typed,
and S does not go that far. A machine that did so would have no Workflow form, since a Workflow is task code that calls the Oracle
in P-mode (\S\ref{sec:os-agentic}), and no hybrid either, since a hybrid has a Workflow part. That is the expected price of
separating the colours into two machines, and S does not pay it. This paper consequently fixes $X$ to $O$ without loss, and
the designation names the program side of a task; it does not choose a machine. The parity of the two colours is the
paper's claim that S restricts no run, and it is
what lets the contract hold for every program, including an adversarial one: a designation derived from the colours asks nothing of a
program, so there is no well-formedness condition for a program to fail.

\paragraph{The designation of $Y$.} This one does have a machine behind each value, and it is the designation of privilege. T runs
the program in the privileged mode, and the Oracle instruction has its P-mode rule on Priestess stacks alone (\S\ref{sec:os-t}); to
exchange the two sides here is to exchange what a mode may do, which rewrites the transition rules. The two machines share the
contract: K1--K5 depend on no type and hold on the \mach{O-2-PDA-T} already, and on the reversed machine the same five items hold with
the roles exchanged, derived from the rules as before (Proposition~\ref{prop:reversed}). What changes between them is one item stated with
the roles exchanged, and the price of trusting an Oracle is what it states. So the two letters are not a choice of two breakings but
the two names each breaking fixes, and with only one machine behind each name of $Y$, Table~\ref{tab:quadrants} is a table of
placements, not of machines: its columns are the two placements of a task's program, its rows the two assignments of $Y$, and
only the first row is the machine of this paper. Its four cells are the four ways the two roles can be assigned, \textbf{who acts} and
\textbf{who is trusted}, and an operating system is the case in which they fall on different sides (Definition~\ref{def:sides}).
Exchanging the two colours maps each cell to the one on the other diagonal, which is what remains of the symmetry the breakings give up
spend.

\paragraph{The contract needs no Oracle.} Once the two letters are separated this way, the contract as a statement about a
boundary emerges. Write $R$ for any map that assigns a non-empty set of words to each word of a set $\mathrm{dom}(R)$, and
let the machine with $R$ be the \mach{O-2-PDA} of Definition~\ref{def:o2pda} with $O$ replaced by $R$, taking T as
Definition~\ref{def:t} takes it: an accepted query appends one of the words $R$ allows, and a rejected one continues at $l_e$. Call
the machine good for $R$ if every query reachable in O-mode lies in $\mathrm{dom}(R)$.

\begin{proposition}[the contract without an Oracle]
\label{prop:blackbox}
For every such $R$, K1--K4 and the crossing invariant hold on the machine with $R$ and T, and K5 holds with budgets on a machine
that is good for $R$.
\end{proposition}

\begin{proof}
The proofs of Theorem~\ref{thm:contract}, Corollary~\ref{cor:crossing} and Proposition~\ref{prop:k5} use the rules of
Definitions~\ref{def:2pda}, \ref{def:o2pda}, \ref{def:t} and~\ref{def:budget}, the locality lemma, and the convention that answers
are non-empty. The Oracle enters them only through $\mathrm{dom}(E)$ and $O(w)$, and $R$ has those two as well. So each proof
applies verbatim.
\end{proof}

So the contract is a property of the boundary, and not of the party behind it. This is the property a reference monitor has, and it
is what lets the classical account of confinement apply to a party whose interior is not assumed~\cite{anderson1972computer}. What
the machine adds is why the property asks nothing of that party: the confined side executes no instruction of its own and only
appends, so nothing of what it does to answer is used, and Proposition~\ref{prop:blackbox} obtains the property for every $R$, with
no well-formedness condition asked of $R$. A realization in which the confined side does execute has to supply the same items by
mechanisms of its own (Definition~\ref{def:tcb}).

\paragraph{The boundary uses S, and at one place only.} The contract needs no Oracle, and the boundary of
Theorem~\ref{thm:boundary} needs it no more, but the boundary does use S, and in one item only. Its \emph{In} item says that whatever
a call reads beyond the answers its own mode appended, the Priestess wrote there. In O-mode this is K2 with the crossing invariant,
which hold on the \mach{O-2-PDA-T}. In P-mode it is S: a P-mode call reads a Priestess stack $d$, an $\mathtt{uret}$ under S enters
only Oracle stacks, and by K1 no excursion writes $d$. Nowhere else does S enter, and the condition it supplies there can be obtained
directly.

\begin{definition}[safe program]
\label{def:safe}
A program is \textbf{safe} on the \mach{O-2-PDA-T} if no stack is both the operand of an $\mathtt{uret}$ and the operand of a
P-mode Oracle call. The machine with a safe program is \textbf{safe}.
\end{definition}

Stack operands are names written in the program, so safety is decidable from $\Pi$. It is a condition on the program, where goodness
is a condition on the codec (\S\ref{sec:os-t}, Proposition~\ref{prop:codec-good}), and it mirrors it: goodness keeps every query of
a run answerable, and safety keeps every P-mode operand out of an excursion's reach.

\begin{proposition}[the boundary without S]
\label{prop:safe}
On the \mach{O-2-PDA-T} with a safe program, K1--K4, the crossing invariant and the safety boundary hold for every Oracle, and so
does Proposition~\ref{prop:where-not-which}.
\end{proposition}

\begin{proof}
K1--K4 and the crossing invariant hold on the \mach{O-2-PDA-T}: the contract follows from the transition rules, none of which
depends on a type, so it holds on the machine without S (\S\ref{sec:os-kernel}), and the \emph{Out} item of the boundary does not use S
either. For the \emph{In} item in P-mode, a call on $d$ takes $\sigma(d)$ alone; by safety no $\mathtt{uret}$ names
$d$, so no excursion entered on $d$, and by K1 none wrote it. So $\sigma(d)$ holds only the answers of earlier P-mode calls on it and
what the Priestess wrote there, the initial content included. The O-mode case is K2 with the crossing invariant, and the last claim
quantifies over the Oracle alone.
\end{proof}

S implies safety, since under S an $\mathtt{uret}$ enters only Oracle stacks, a P-mode call reads only Priestess stacks, and the two
colours are disjoint, so every program of the \mach{O-2-PDA-TS} is safe. Safety is strictly weaker, admitting an $\mathtt{uret}$ on a
stack that no P-mode call reads whatever its colour, and it is weaker in kind, constraining the program's operands and not the
machine's colours. A P-mode call is the only transition at which the boundary can break,
and it breaks there exactly when the call's operand holds content that an excursion appended. Safety rules this out by the
operands alone. The
machine without P-mode calls needs no condition at all: with the Oracle instruction on Priestess stacks removed, the \emph{In}
item in P-mode is vacuous and the boundary holds on the bare \mach{O-2-PDA-T} for every program. That machine has no Workflow form,
and no hybrid.

\paragraph{What a realization supplies in place of the rules.} The propositions above say what confinement of this shape achieves when
the confined side executes no instruction of its own. A realization in which it does execute has to supply the same items by other
means, and the list is what a trusted computing base is.

\begin{definition}[trusted computing base]
\label{def:tcb}
A realization of the machine in which the confined side executes instructions of its own has a \textbf{trusted computing base} when
it supplies: (i)~isolation, so that the confined side writes and reads only its own storage; (ii)~an entry, so that it reaches the
trusted side only at a fixed place; (iii)~a discipline for what crosses back, so that no content from it becomes an instruction or a
code address of the trusted side; and (iv)~a bound, so that it returns control. The first supplies K1 and K2 together, and the rest
are K3, K4 and K5 in turn. One side of the realization is the base that supplies them.
\end{definition}

The items are stated in the vocabulary of the contract so that a proof of the contract carries over to a realization: where the machine's
rules give each item, the base's mechanisms give it, and the items then hold there as they hold of the machine, for every interior
the confined side may have. This is narrower than the classical notion of a trusted computing
base~\cite{dod1985trusted,anderson1972computer}, which is a whole system's
protection mechanism defined against an evaluation and a policy; what is defined here is only what the contract requires.

Two things are worth separating. The operating system of Definition~\ref{def:os} is defined \emph{relative to a task}: it is the
party that schedules the programs on Oracle stacks, and it needs a task whose program is placed against its privilege. The base is
defined \emph{relative to a party whose interior is not assumed}: it is one side of an interface, and it needs no task at all. So a
base can confine a party that is nobody's task, and a task can meet an operating system that is not a base, which is the interpreter
of Proposition~\ref{prop:interpreter}. A base is one way to put the items of Definition~\ref{def:tcb} behind a confined party, and
it is the way that keeps the party on the processor and moves the checks into the trusted side's code. A realization that instead
gives that party hardware of its own, reached through one interface, has the shape the machine already has
(\S\ref{sec:machine}, \S\ref{sec:query}), since the machine assumes nothing of the Oracle's interior and gives it no code of its
own; the items are then properties of a component boundary, not of a check in the trusted side's code, and they ask no more
of the realization than the machine asks of an Oracle.

\begin{proposition}[what a base cannot do]
\label{prop:tcb-forged}
For every $R$ and every run of the machine with $R$, a party confined behind a base reaches the trusted side only at the entry of
Definition~\ref{def:tcb} and only as data, whatever its interior. Which content reaches it is not fixed.
\end{proposition}

\begin{proof}
The first claim is items (ii) and (iii) of Definition~\ref{def:tcb}, which are K3 and K4 as they hold on a realization. For the second, apply
Proposition~\ref{prop:where-not-which} with $R$ in place of the Oracle: it quantifies over the answers alone, and no property that
holds for every answer can exclude one. On the machine the item that carries both is K4, and it is the one item a realization cannot
inherit, since an Oracle has no instructions to forge with and a program has.
\end{proof}

\subsection{Several Levels of Privilege}
\label{app:further-levels}

The previous subsection separates the two letters and derives the contract for a confined party that answers. This one asks what that
leaves of a computer that has several levels of privilege.

\paragraph{The computers in use.} This also answers a reader's first question here: whether the same account fits an ordinary
computer, which has several privilege levels and no Oracle. In part it does. Treat the Oracle as a user-mode program and the
Priestess as a kernel, and the boundary of Theorem~\ref{thm:boundary} is the one between user mode and the kernel. In part it does
not, in the way \S\ref{sec:eclipse-safety} warns about: the two restrictions that give O-mode its position, that it runs no
instruction of its own and only appends, are not properties of user mode, where a program runs instructions of its own and is
confined in its rights, not in its execution. In the terms of the two letters, the levels of such a computer are
assignments of $Y$ in which the mode that runs the user's program is denied operations but still executes, which is the weaker
form of the breaking. Their order is an order of rights, and \S\ref{sec:query} has already placed them: P-mode contains the
processor's existing privilege levels, and on a realization with an Oracle the Oracle mode is added below user mode
(\S\ref{sec:archnights}).

\paragraph{A tower of levels.} Put the two together and the levels compose: $O$ below $U$, and $U$, $S$ and $M$ one above another
inside P-mode. Each crossing is a fixed entry to the level below and a trap back to one fixed entry above it, as the crossing
between P-mode and O-mode is by K3, and no content of a lower level reaches the control of a higher one except through that entry.
Complete mediation, which Theorem~\ref{thm:boundary} gives at one boundary, can be required at each of them in turn: a check at an
entry sees every crossing into the level below it. For the levels the processor supplies, this is the processor's own guarantee,
stated in its privileged architecture~\cite{waterman2021risc}. There the higher level fixes each entry, not the instruction set: a
trap goes to an address that a higher level has set in a register the lower level cannot write, one per cause of interrupt in
vectored mode and at the level that delegation selects, and a return goes to the address that the higher level leaves in its
exception program counter. What the lower level cannot do is choose where it enters a higher one, which is what K3 secures for the
Priestess. What the machine adds to the tower is the bottom level, which those levels do not supply: a
level whose transitions are computed by a party that runs no instruction of the level at all. How the users of one level interact with those of
another is the old question that \S\ref{sec:related-os} and Appendix~\ref{app:archnights} answer in three ways, by
kernel-isolating virtualization, by kernel-sharing virtualization and by a monitor above both.

\paragraph{What the tower does not inherit.} The contract is a theorem about one Oracle, one program and one run, and a tower does not
inherit it: a level may apply it to what runs at that level, and the tower as a whole would need a theorem of its own, which this
paper does not state. Nor is the tower the only nesting. An approval gate has two Oracles of different classes and different
trust, which is a combination of the two assignments of $Y$, instead of a stack of levels, and \S\ref{sec:future} leaves it open.

\subsection{Oracles and the Lynchpin Hypothesis}
\label{app:further-oracles}

\S\ref{sec:query} runs in O-mode only the Oracles that are autoregressive inside, and lets P-mode call any Oracle. This subsection
says what that leaves open for the Oracles other than language models, and then reviews what is known of the Lynchpin hypothesis and
of the remedy that \S\ref{sec:query} proposes for it.

\paragraph{Oracles other than language models.} \S\ref{sec:query} places the classes. An action-conditioned world model and
vision-language-action models that predict action tokens may run in O-mode, and the actions of the latter reach P-mode as requests
that control a robot; a diffusion model, Jev and usually a person are called in P-mode. A diffusion-native runtime describes its
computation as the refinement of a current global state, instead of the extension of an immutable prefix, and lets the model update
some parts of that state and the environment and the runtime the others~\cite{cao2026continuous}. On the machine such a model is
called in P-mode, and the state it refines is Priestess content that the Priestess rewrites between calls, outside the append-only
epochs of \S\ref{sec:eclipse}. A survey covers vision-language-action
models~\cite{ma2024survey}. What the placement leaves open is whether such an Oracle meets the hypotheses of \S\ref{sec:eclipse}. For
an action tokenizer, factorization at the yield symbol is checked per tokenizer, as it is for text
(Appendix~\ref{app:further-encoders}). A world model such as Genie learns latent actions and an autoregressive dynamics over video
tokens~\cite{bruce2024genie}, so its answers are frames of video, and whether its codec is Frame-faithful is
a property of each entity, as it is for a language model.

\paragraph{The empirical evidence.} The hypothesis concerns correct calls, and the evidence on tool calls bears on it most directly.
There, accuracy falls with the size of the tool catalogue, the length of tool responses and the length of the
conversation~\cite{kate2025longfunceval}, and benchmarks measure calls in long, multi-step and agentic
settings~\cite{zhong2025complexfuncbench,patil2025berkeley}. The general evidence on long contexts is consistent with the hypothesis
for current models. Accuracy depends on where information sits in a long context~\cite{liu2024lost}, and a model's effective
context is shorter than the length it accepts~\cite{hsieh2024ruler}. Performance falls with length when the question and the
relevant passage share few words~\cite{modarressi2025nolima}, across many models~\cite{hong2025context}, and with length alone when
retrieval is perfect~\cite{du2025context}. A chain of thought can lower accuracy on some tasks~\cite{liu2024mind}. None of this
proves the hypothesis, and its principle is not understood.

\paragraph{Theories that may bear on it.} Three kinds of theory may bear on it. The first concerns errors over the steps of a
composition. Errors may compound over the steps~\cite{dziri2023faith} or snowball from an early commitment~\cite{zhang2024how},
although teacher-forced training may fail to learn an accurate next-token predictor in the first place, so compounding is not the
whole story~\cite{bachmann2024pitfalls}. The second concerns what a fixed model can compose over a long input. Communication lower
bounds for decoder-only transformers~\cite{chen2024theoretical} and lower bounds on chains of thought~\cite{amiri2025lower} limit
it, and an empirical study of the serial-depth bottleneck is framed by the first~\cite{kurtulush2026when}. The
third supplies the terms for the dependence itself. Chains of infinite order, processes whose next symbol depends on the whole past, have a
theory of continuity rates, loss of memory and Markov
approximation~\cite{harris1955chains,berbee1987chains,fernandez2002markov,fernandez2005chains}, in which an Oracle's dependence on a
long history can be stated. The stochastic-oracle machines of \S\ref{sec:related-parallel} give a structure to the distribution of answers~\cite{wang2026acomputing}. Which of them, if any, explains the evidence is open.

\paragraph{The remedy: call by name.} Semantic delegation evaluates a computation by name rather than by value, delaying it until it is
needed (\S\ref{sec:query-message}), much as lazy structures in probabilistic programming are evaluated on
demand~\cite{dash2023affine}. Whether it helps is a quantitative question. It shortens the history by the definition and by every
intermediate step, at the price of a declaration that grows with every operator (Appendix~\ref{app:archnights-dist}). A model
of the Oracle's answers, for which the stochastic-oracle setting is a candidate, would say when the trade pays (\S\ref{sec:future}).

\subsection{Applications and Extensions}
\label{app:further-apps}

Each paragraph places a direction of current work on the machine, says
which part of the paper applies to it, and names the open item of \S\ref{sec:future} it meets. No placement here is a result.

\paragraph{Embodied agents.} Robot foundation models split a slow model from a fast one: a vision-language model at about ten
hertz and an action model at more than a hundred~\cite{nvidia2025gr00t,figure2025helix}, or a high-level policy that plans again
every second and sends language commands to a low-level policy controlling at 50~Hz~\cite{shi2025hi}. On the machine the slow
model is an $A_{\mathrm{AR}}$ Oracle in O-mode, and the fast one an Oracle that P-mode calls, which may be non-autoregressive, as
a flow-matching policy is, or Jev. The slow model can then write a program with a hole, \texttt{switch (a) \{ case \dots\ \}},
which the Priestess forces at the control rate with \texttt{a} supplied by the fast model: semantic delegation
(\S\ref{sec:query-message}) at the rate of control. Code as Policies is the nearest published instance. A language model writes
Python that branches on the outputs of perception and calls only the control primitives in scope~\cite{liang2023code}, and
VoxPoser writes its code once per subtask and evaluates it again at every step of a model-predictive-control
loop~\cite{huang2023voxposer}. We have found no work in which a learned action model fills that hole at the rate of control. PhyAgentOS supports both places for the
decision loop, a loop that the runtime manages over a policy model and a loop in which the agent selects tools online, and prohibits
the agent from issuing raw hardware commands~\cite{liu2026phyagentos}. On the machine these are the Workflow and the Agent form, and
the prohibition holds by K1 and K2, since an excursion writes only its own stack. Two placements need care. When the
hand-off between the two systems is a latent and both are trained end to end, as in Helix and GR00T
N1~\cite{figure2025helix,nvidia2025gr00t}, they are one Oracle on the machine, and nothing that passes between them can be observed.
And a world model used for planning is called in P-mode: V-JEPA~2-AC is autoregressive, but a search planner calls it many times
for each action~\cite{assran2025vjepa}. What the machine adds is the safety boundary. Every actuator command is a P-mode
request, so a check on the Priestess side sees all of them, and an attack that jailbreaks the robot's language
model~\cite{robey2025jailbreaking} reaches the motors only through that check. That is the position of a guardrail on plans~\cite{ravichandran2026roboguard},
of a control-barrier filter on actions~\cite{hu2026vlsa} and of the decision module of the Simplex architecture~\cite{sha2001using}.
ROS~2 supplies much of the communication the Priestess needs: its services and actions are requests with responses, and actions
can be cancelled~\cite{macenski2022robot}. Its topics, however, are publish--subscribe, so the boundary holds only if access
control leaves the Priestess the only publisher to the actuators~\cite{mayoral2022sros2}. The open items are those of
\S\ref{sec:future}: Oracles of different classes on one machine, and asynchrony.

\paragraph{Reasoning as code.} A chain of thought is a history that the Oracle iterates on in O-mode. Program-aided reasoning makes
it a program: the model writes code, an interpreter computes the answer, and ``solving is delegated to the
interpreter''~\cite{gao2023pal,chen2023program}. On the machine the answer is then one force of a thunk in P-mode, and the
history contains the program, not the steps of its evaluation. Chain of Code runs the program line by line and hands each line the
interpreter cannot run to the model, which emulates it and updates the program state~\cite{li2024chaincode}. On the machine, the
Priestess evaluates the program and calls the Oracle in P-mode for its semantic leaves. Tool-integrated reasoning trains the
interleaving~\cite{gou2024tora,feng2025retool}, and frontier models use tools inside their chains of thought~\cite{openai2025o3};
there each program is forced as soon as it is written and its value enters the history, by value rather than by name. At the
other end the model emulates its own program in context with no interpreter~\cite{weir2024learning,chae2024language}, and a
continuous chain of thought feeds hidden states back as inputs, so that nothing of the reasoning stays on the history to be
read~\cite{hao2025training}. The evidence on emulation favours execution. In Chain of Code's ablation on BIG-Bench Hard,
interleaving execution with emulation scored 84, emulation 63, emulation with a state trace 57, and Python alone 48~\cite{li2024chaincode}, and in
2024 GPT-4 with a chain of thought predicted the output of Python functions of 3 to 13 lines in 81.9\% of
cases~\cite{gu2024cruxeval}. What this suggests is a chain of thought that is a program, whose answer is the force of one thunk and
which the Priestess evaluates. It meets the Lynchpin hypothesis (Appendix~\ref{app:further-oracles}): a history that
holds the program and not its evaluation is shorter, and whether it is also more often right is the open question of semantic
delegation.

\paragraph{Code and language: one Oracle or two.} An agent needs both code and natural language~\cite{xu2024lemur}, and training
for one can cost the other. At a fixed budget, the share of code in pre-training trades code generation against natural-language
tasks, although a moderate share helps both~\cite{aryabumi2025code}. A code expert branched from a general model gains on code and
loses on knowledge~\cite{sukhbaatar2024branch}, fine-tuning on code causes more forgetting than fine-tuning on
mathematics~\cite{biderman2024lora}, and mixed fine-tuning datasets conflict when they are plentiful~\cite{dong2024how}. The record is
not one-sided: reinforcement learning on software evolution improved general language understanding where supervised fine-tuning
degraded it~\cite{wei2025swerl}. Two remedies are in use. One separates the models in the system. An architect model describes the
solution and an editor model writes the edits~\cite{aider2024separating}, and Magentic-One specializes a coder agent by its system
prompt, can run it on a different model from the rest, and executes its programs on a deterministic
terminal~\cite{fourney2024magentic}. The other switches inside one model, as
a hybrid model switches between thinking and non-thinking modes~\cite{yang2025qwen3}. That switch is one of reasoning depth, not
of code and language, and its vendor later trained the two modes as separate models~\cite{qwen2025instruct}. On the machine the
first remedy is several Oracles of different classes, among which the Priestess chooses at each call, and \S\ref{sec:future}
leaves that machine open. The second is the Oracle's interior, which the machine does not see. The machine adds one placement:
whichever Oracle writes a program, the program runs only when the Priestess runs it (K4), as the programs of Magentic-One's coder
run on its terminal.

\paragraph{Learned prefetch.} A model can learn when to yield and what to request. Reinforcement learning can make it call tools
less often~\cite{wang2025acting}. A fine-tuned model issues its longest-running calls first and goes on decoding while they
run~\cite{gim2024asynchronous}, and a model trained on its own rollouts learns to predict its next tool call from a partial
trajectory, so that the call can be issued before the model reaches it~\cite{ji2026speculate}. On the machine this is speculation
across the yield. Its precedent in operating systems is a program that, while stalled on I/O, runs ahead speculatively to issue
hints for its future reads~\cite{chang1999automatic}. The Priestess can serve a predicted request before the task makes it, as
speculative actions do with a fast model that guesses the next action and commits only when the real action
matches~\cite{ye2025speculative}. A reply the Priestess caches changes no task stack, so the contract is untouched until the reply
is returned. What the machine leaves to the designer is the handling of effects. A request with effects cannot simply be undone, and current
work speculates only requests without effects or with effects that can be rolled
back~\cite{ye2025speculative,nichols2025optimizing,ji2026speculate}, runs speculated calls in copy-on-write sandboxes and commits
them in order~\cite{yu2026tomasullm}, or divides effects into those that can be buffered, those that can be compensated and those
released only at commit~\cite{mohammadi2026atomix}. Issuing a request is itself observable. A speculative call that only reads still
discloses the user's inferred intent to the service it reaches~\cite{mohammadi2026ghost}, the agent's counterpart of the side
channels of speculative execution~\cite{kocher2019spectre} and of Appendix~\ref{app:eclipse}. So an early request must have no
effect and also reveal nothing that the task's own request would not. The learning has a loop of its own: a model that learns from
its own traces to request early changes the traces it learns from, as a learned prefetcher changes the distribution of the misses
it is trained on~\cite{hashemi2018learning}. When the large model prepares work for smaller ones, as in a protocol in which a cloud
model decomposes a task into subtasks that small local models run~\cite{narayan2025minions}, its request is served by the
Priestess calling the smaller Oracles in P-mode, and the open items are again several Oracles and asynchrony.

\end{document}